\documentclass[11pt]{article}
\usepackage[a4paper]{geometry}
\usepackage{amsthm}
\usepackage{amsmath}
\usepackage{amssymb}

\usepackage{amsfonts}
\usepackage{graphicx}
\usepackage{color}
\usepackage[bf,SL,BF]{subfigure}
\usepackage{url}
\usepackage{epstopdf}
\usepackage{pdfsync}
\usepackage[colorlinks]{hyperref}
\usepackage[all]{xypic}
\usepackage{comment}
\usepackage{mathabx}
\usepackage{upgreek}

\hypersetup{
    linkcolor=blue,
}

\newlength{\fixboxwidth}
\usepackage{epsfig,amsbsy,graphicx,multirow}

\usepackage{ algorithm, algorithmic}

\renewcommand{\algorithmiccomment}[1]{\bgroup\hfill//~#1\egroup}

\usepackage[all]{xy}
\usepackage{accents}

 \numberwithin{equation}{section}

\def\low{\mathrm{low}}

\def\P{\mathbb{P}}
\def\R{\mathbb{R}}

\def\cN{\mathcal{N}}
\def\S{\mathcal{S}}
\def\cP{\mathcal{P}}

\def\E{\mathbb{E}}

\def\B{\mathcal{B}}

\def\er{\mathcal{E}}

\def\I{\mathcal{I}}
\def\J{\mathcal{J}}

\def\T{\mathcal{T}}
\def\Td{\top}

\def\N{\mathcal{N}}

\def\restrict#1{\raise-.5ex\hbox{\ensuremath|}_{#1}}

\def\<{\big\langle}
\def\>{\big\rangle}

\def\Img{\operatorname{Im}}
\def\Ker{\operatorname{Ker}}

\def\Var{\operatorname{Var}}

\def\diag{{\operatorname{diag}}}

\def\dim{{\operatorname{dim}}}

\def\Span{\operatorname{span}}

\definecolor{red}{rgb}{0.9, 0, 0}

\newtheorem{Theorem}{Theorem}[section]
\newtheorem{Proposition}[Theorem]{Proposition}
\newtheorem{Lemma}[Theorem]{Lemma}

\newtheorem{Remark}[Theorem]{Remark}
\newtheorem{Example}[Theorem]{Example}

\newtheorem{Definition}[Theorem]{Definition}

\newtheorem{Condition}[Theorem]{Condition}
\newtheorem{Problem}{Problem}

\begin{document}
\title{Kernel Mode Decomposition\\
and programmable/interpretable  regression networks}

\date{\today}

\author{Houman Owhadi\thanks{Corresponding author. Caltech,  MC 9-94, Pasadena, CA 91125, USA, owhadi@caltech.edu} \and Clint Scovel\thanks{Caltech,  MC 9-94, Pasadena, CA 91125, USA, clintscovel@gmail.com} \and Gene Ryan Yoo\thanks{Caltech, MC 253-37, Pasadena, CA 91125, USA, gyoo@caltech.edu}}

\maketitle

\begin{abstract}
Mode decomposition is a prototypical pattern recognition problem that can be addressed from the (a priori distinct) perspectives of numerical approximation, statistical inference and deep learning.
Could its analysis through these combined perspectives be used as a Rosetta stone for
deciphering  mechanisms  at play in deep learning? Motivated by this question we introduce programmable and interpretable regression networks for pattern recognition and address mode decomposition as a prototypical problem. The
programming of these networks is achieved by assembling  elementary modules decomposing and recomposing kernels and data. These elementary steps are repeated across levels of abstraction and interpreted from the equivalent perspectives of  optimal recovery, game theory and Gaussian process regression (GPR).
The prototypical mode/kernel decomposition module produces an approximation  $(w_1,w_2,\cdots,w_m)$ of an element $(v_1,v_2,\ldots,v_m)\in V_1\times \cdots \times V_m$ of a product of Hilbert subspaces $(V_i,\|\cdot\|_{V_i})$ of a common Hilbert space  from the observation of the sum $v:=v_1+\cdots+v_m \in V_1+\cdots+V_m$. This approximation is minmax optimal with respect to the relative error in the product norm $\sum_{i=1}^m \|\cdot\|_{V_i}^2$ and obtained as $w_i=Q_i(\sum_j Q_j)^{-1}v=\E[\xi_i|\sum_j \xi_j=v]$ where  $Q_i$ and $\xi_i\sim \mathcal{N}(0,Q_i)$ are the covariance operator and the Gaussian process defined by the norm $\|\cdot\|_{V_i}$.  The prototypical mode/kernel recomposition module performs partial sums of the recovered modes $w_i$ and
 covariance
 operators $Q_i$ based on the alignment between each recovered mode $w_i$  and the data $v$ with respect to the inner product defined by
 $S^{-1}$ with $S:=\sum_i Q_i$ (which has a natural interpretation as model/data alignment $\<w_i,v\>_{S^{-1}}=\E[\<\xi_i,v\>^2_{S^{-1}}]$ and variance decomposition  in the GPR setting).
We illustrate the proposed framework by programming regression networks approximating
the modes $v_i= a_i(t)y_i\big(\theta_i(t)\big)$ of a (possibly noisy)   signal $\sum_i v_i$ when the amplitudes $a_i$, instantaneous phases $\theta_i$ and periodic waveforms $y_i$ may all be unknown and show near machine precision recovery under regularity and separation assumptions on the instantaneous  amplitudes $a_i$ and frequencies $\dot{\theta}_i$.
 The structure of some of these networks share intriguing similarities with convolutional neural networks while being interpretable, programmable and  amenable to theoretical analysis.

\end{abstract}

\section{Introduction}

The purpose of the \emph{Empirical Mode Decomposition} (EMD) algorithm \cite{huang1998empirical}
 can be loosely expressed as solving a (usually noiseless) version of the following problem, illustrated in Figure \ref{figemd}.
 \begin{Problem}\label{pb2}
For $m\in \mathbb{N}^*$, let $a_1,\ldots,a_m$ be piecewise smooth  functions on $[0,1]$ and let $\theta_1,\ldots,\theta_m$ be strictly increasing  functions on $[0,1]$.
Assume that $m$ and the $a_i, \theta_i$ are unknown. Given the (possibly noisy) observation of
$v(t)=\sum_{i=1}^m a_i(t)\cos\big(\theta_i(t)\big), t\in [0,1],$  recover the modes $v_i(t):=a_i(t)\cos\big(\theta_i(t)\big)$.
\end{Problem}

  \begin{figure}[h!]
	\begin{center}
			\includegraphics[width=\textwidth]{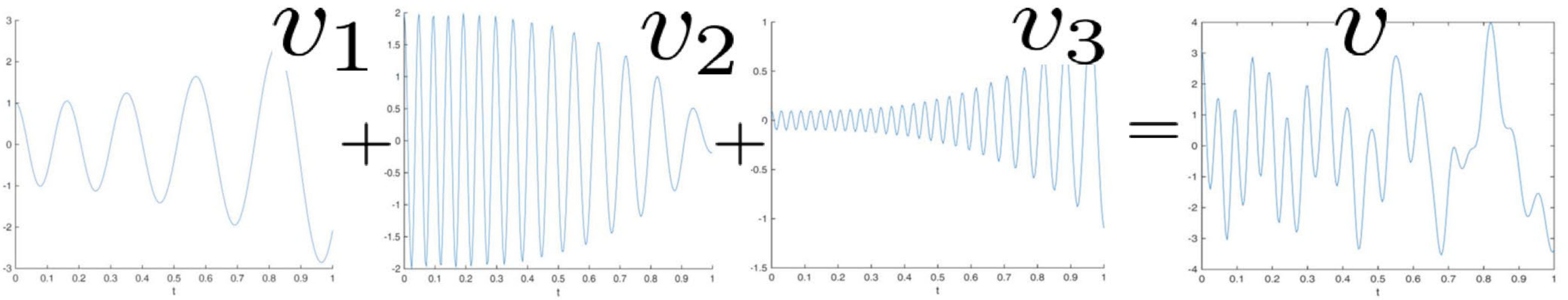}
		\caption{A prototypical mode decomposition problem: given $v= v_{1}+v_{2}+v_{3}$ recover $v_1, v_2, v_3$.}\label{figemd}
	\end{center}
\end{figure}

 In practical applications,  generally the \emph{instantaneous frequencies} $\omega_i=\frac{d \theta_i}{dt}$ are assumed to be smooth and well separated. Furthermore the $\omega_i$ and the \emph{instantaneous amplitudes} are assumed to be
varying at a slower rate than the  \emph{instantaneous phases} $\theta_i$ so that near $\tau\in [0,1]$ the \emph{intrinsic mode function} $v_i$ can be approximated by   a trigonometric function, i.e.
\begin{equation}\label{eqjkhgkjehdgd}
v_i(t)\approx a_i(\tau) \cos\big(\omega_i(\tau) (t-\tau)+\theta_i(\tau)\big)\text{ for }t\approx \tau\,.
\end{equation}
The difficulty of analyzing and generalizing the EMD approach and its popularity and success in practical applications \cite{huang2014hilbert} have stimulated the design of alternative methods aimed at solving Problem \ref{pb2}.  Methods that are amenable to a greater degree of analysis include  synchrosqueezing \cite{daubechies2011synchrosqueezed, lin2018wave}, variational mode decomposition \cite{dragomiretskiy2014variational} and non-linear $L_1$ minimization with sparse time-frequency representations \cite{hou2011adaptive, hou2015sparse}.

\paragraph{A Rosetta stone for deep learning?}
Since Problem \ref{pb2} can be seen as prototypical pattern recognition problem that can be addressed from the perspectives of numerical approximation, statistical inference and machine learning, one may wonder if its analysis, from the combined approaches of numerical approximation and statistical inference, could  be used as a Rosetta stone for deciphering deep learning.
Indeed, although successful  industrial applications \cite{lecun2015deep} have consolidated the recognition of artificial neural networks (ANNs)  as powerful pattern recognition tools, their utilization has recently been compared to ``operating on an alien technology'' \cite{hutson2018ai} due to the challenges brought by a lag  in  theoretical understanding: (1) because ANNs are not easily interpretable  the resulting models may not be interpretable (and identifying causes of success or failure may be challenging) (2) because ANNs rely on the resolution of non-convex (possibly stochastic) optimization problems, they are not easily amenable to a complete uncertainty quantification analysis (3) because the architecture design of ANNs essentially relies on trial and error, the design of  architectures  with good generalization properties may involve a significant amount  of experimentation.

Since elementary    operations performed by ANNs can be interpreted \cite{patel2016probabilistic} as stacking Gaussian process regression steps with nonlinear thresholding and pooling operations  across levels of abstractions,  it is natural to wonder whether
  interpretable  Gaussian process regression (GPR) based networks could be conceived for mode decomposition/pattern recognition. Could such networks (1) be programmable based on rational and modular (object oriented) design? (2) be amenable to analysis and convergence results? (3)   help
 our understanding of fundamental mechanisms that might be at play in pattern recognition and thereby help elaborate a rigorous theory for  Deep Learning?
This paper is an attempt to address these questions, while using mode decomposition \cite{huang1998empirical} as a prototypical pattern recognition problem.
As an application of the programmable and interpretable regression networks introduced in this paper, we will
 also address the following  generalization  of Problem \ref{pb2},  where the periodic waveforms may all be  non-trigonometric, distinct, and unknown and present an algorithm producing near machine precision
 ($10^{-7}$ to $10^{-4}$) recoveries of the modes.

\begin{Problem}\label{unk wave pb}
For $m\in \mathbb{N}^*$, let $a_1,\ldots,a_m$ be piecewise smooth  functions on $[-1,1]$, let $\theta_1,\ldots,\theta_m$ be  piecewise smooth  functions  on $[-1,1]$ such that the instantaneous frequencies $\dot{\theta}_i$ are   strictly positive and well separated,
and let $y_1, \ldots, y_m$ be square-integrable $2\pi$-periodic functions.
Assume that $m$ and the $a_i, \theta_i, y_i$ are all unknown. Given the observation
$v(t)=\sum_{i=1}^m a_i(t)y_i\big(\theta_i(t)\big)$ (for $t\in [-1,1]$) recover the modes $v_i(t):=a_i(t)y_i\big(\theta_i(t)\big)$.
\end{Problem}

One fundamental idea is that although Problems \ref{pb2} and \ref{unk wave pb} are nonlinear, they can be, to some degree, linearized by recovering the modes $v_i$ as aggregates of sufficiently fine modes living in  linear spaces (which, as suggested by the approximation \eqref{eqjkhgkjehdgd}, can be chosen as linear spans of  functions $t\rightarrow \cos(\omega (t-\tau)+\theta)$ windowed around $\tau$, i.e.~Gabor wavelets). The first part of the resulting network  recovers those finer modes through a linear optimal recovery operation. Its second part recovers the modes $v_i$ through a hierarchy of (linear) aggregation steps sandwiched between (nonlinear) ancestor/descendant identification steps. These identification steps  are obtained by composing the alignments between $v$ and the aggregates of the fine modes with simple and interpretable nonlinearities (such as thresholding, graph-cuts, etc...), as presented in Section \ref{sechgjjhjbvf}.

  \begin{figure}[h!]
	\begin{center}
			\includegraphics[width=\textwidth]{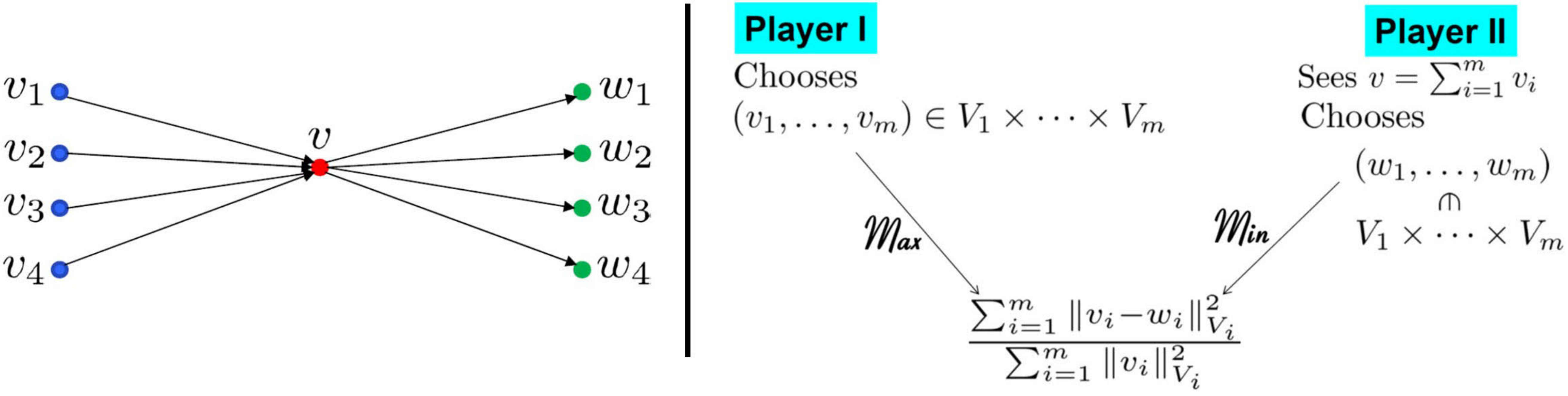}
		\caption{Left: The mode decomposition problem. Right: The game theoretic approach.}\label{figgame}
	\end{center}
\end{figure}

\section{Review of additive Gaussian process regression, empirical mode decomposition and synchrosqueezing } \label{sec_rev} 
 The kernel mode decomposition framework has relations to the fields of additive Gaussian process
 regression, empirical mode decomposition and synchrosqueezing. Consequently, here we review these subjects giving a context to our work. This section is not essential to understanding the paper and so can be skipped on first reading.

Although  simple kriging and GPR are derived differently, they can be shown to be equivalent and are often referred to as the same, 
 see e.g.~Yoo \cite[Sec.~1.1]{Yoo_thesis} for a review of kriging and its relationship with GPR.
Regarding the
 origins of kriging, paraphrasing Cressie \cite{cressie1990origins}, known for introducing kriging in
spatial statistics,
"both  Matheron   \cite{matheron1963traite} (see also \cite{matheron1963principles}) and
  Gandin  \cite{gandin1963objective} were the first to publish a definitive development of spatial kriging.
 D. G. Krige's contributions in mining engineering were considerable but he did not discover kriging, illustrating once again Stigler's Law of Eponymy (Stigler \cite{stigler1980stigler}), which states that
 "no scientific discovery is named after its original discoverer."  The eponymous title of
  Stigler's work is playfully consistent with his law,
 since in it he essentially names Merton \cite[p.~356]{merton1973sociology} as the discoverer of Stigler's law.

\subsection{Additive Gaussian processes}
\label{sec_AGP}
Following Hastie and Tibshirani \cite{hastie1990generalized,Hastie}, the generalized additive model (GAM) replaces a linear predictor
$ \sum_{j}{\beta_{j}x_{j}}$, where the $\beta_{j}$ are parameters, with
$ \sum_{j}{f_{j}(x_{j})}$ where the $f_{j}$ are unspecified functions.
  For certain types of prediction problems  such as binary target  variables, one may add a final function
$ h\bigl(\sum_{j}{f_{j}(x_{j})})\bigr)$.  To incorporate fully dependent responses we can consider models of the form  $f(x_{1},\ldots,x_{N})$. Additive models have been successfully used in regression,
 see  Stone \cite{stone1985additive} and
 Fan et al.~\cite{fan2015functional}.  Vector valued generalizations of GAMs  have been developed in
 Yee and Wild \cite{yee1996vector}
 and Yee \cite{yee2015vector}.  For vector valued additive models of large vector dimension with a  large number of dimensions in the observation data, Yee \cite{yee2015vector} develops methods for reducing the rank of the  systems used in their estimation.
 
When the underlying random variables are Gaussian and we apply to regression,   we naturally  describe the model in terms of its covariance  kernel $k(x_{1},\ldots,x_{N}, x'_{1},\ldots,x'_{N})$ or as an additive model
$\sum{k_{i}(x_{i},x'_{i})}$,  where the  kernel is an additive sum of  kernels depending on lower dimensional variables. 
It is natural to generalize  this setting to a
covariance defined by a weighted  sum over all orders $d$ of  dependency
of weighted sums of kernels depending only on $d$ $D$ dimensional  variables, where $N=Dd$.  Of course, such general kernels  are exponentially complex in the dimension $N$, so are not very useful.  
  Nearly simultaneously,
 Duvenaud et al.~\cite{duvenaud2011additive} and
 Durrande et al.~\cite{durrande2010additive, durrande2012additive},
 introducing Gaussian Additive Processes, addressed this problem.
  Duvenaud et al.~\cite{duvenaud2011additive} 
restricts  the sum at order $d$  to be symmetric in the scalar components in the vector variables, thus
 reduces this complexity in such a way that their complexity is mild and their estimation is computationally tractable. Durrande et al.~\cite{durrande2010additive, durrande2012additive} consider additive versions of vector dependent kernels  and product versions of them, and study their respective performance properties along the the performance of their sum.
 Moreover, because of the  additive nature of these methodologies,   they both  achieve strong interpretability as described by Plate  \cite{plate1999accuracy}. 
 
\subsection{Gaussian Process Regression}
\label{sec_GPR}
 Williams and Rasmussen \cite{williams1996gaussian} provide an introduction to Gaussian Process Regression (GPR). More generally, an
 excellent introduction to Gaussian processes in  machine learning, along with a description of many of its applications and its history,  can be found in
Rasmussen \cite{rasmussen2003gaussian}, and 
   Rasmussen and Williams \cite{williams2006gaussian}, see also 
Yoo \cite{Yoo_thesis}. Recent application domain developments include
source separation, which is related to subject of this book,  by  Park and Choi \cite{park2008gaussian} and  Liutkus et al.~\cite{liutkus2011gaussian} and
the  detection of periodicities by 
 Durrande  et al.~\cite{durrande2016detecting,durrande2016gaussian} and  Preo{\c{t}}iuc-Pietro and Cohn
\cite{preoctiuc2013temporal}.

When the number of dimensions of the observational data is large, computational efficiency becomes extremely important. There has been much work in this area, the so-called {\em sparse methods}, e.g.~Tresp \cite{tresp2000bayesian}, Smola and Bartlett \cite{smola2001sparse}, Williams and Seeger \cite{williams2001using}, Csat{\'o} and Opper \cite{csato2002sparse}, Csat{\'o} et al.~\cite{csato2002tap},  Csat{\'o} \cite{csato2002gaussian}, Qui{\~n}onero-Candela \cite{quinonero2004learning}, Lawrence  et al.~\cite{herbrich2003fast}, Seeger \cite{seeger2003bayesian}, Seeger et al.~\cite{seeger2003fast}, Schwaighofer and Tresp \cite{schwaighofer2003transductive}, Snelson and Ghahramani
  \cite{snelson2006sparse}.
Qui{\~n}onero-Candela and Rasmussen \cite{quinonero2005unifying} provide 
 a unifying framework for the sparse methods based on expressing them in terms of their {\em effective prior}. The majority of these methods utilize the so-called {\em inducing variable} methods, which are data points in the same  domain as the unlabeled data. Some require these to be a subset of the training data while others, such as Snelson and Ghahramani
  \cite{snelson2006sparse} allow them to inferred along with the the primary hyperparameters using optimization. However, there are notable exceptions such as 
 Hensman et al.~\cite{hensman2017variational}  who apply a  
Kullback-Liebler derived variational formulation and  utilize
 Bochner's theorem on positive definite functions to choose optimal features in Fourier space.

The majority of these methods use the Kullback-Liebler (KL) criterion to select the  induced points,
See  Rasmussen and Williams \cite[Ch.~8]{williams2006gaussian} for a review. 
In particular,  Seeger et al.~\cite{seeger2003fast}, Seeger \cite{seeger2003bayesian} among others, utilize  the KL criterion to optimize both
the model hyperparameters and the inducing variables. However, they observe that the
 approximation of the  marginal likelihood is sensitive to the choice of inducing variables and therefore convergence of the method is problematic.  
Snelson  and Ghahramani \cite{snelson2006sparse} attempt to resolve this problem by developing a KL formulation where the model hyperparameters and the inducing variables are jointly optimized. Nevertheless, since the inducing variables determine an approximate marginal likelihood, these methods can suffer from overfitting.   
Titsias' \cite{titsias2009variational} breakthrough, a development of Csat{\'o} and Opper  \cite{csato2002sparse} and Seeger \cite{seeger2003bayesian},  was the introduction of a KL variational framework where the model hyperparameters and the inducing variables are selected in such a way as to maximize a lower bound to the true marginal likelihood, and thus are selected to minimize the KL distance between the sparse model and the true one.   
When the dimensions of the observational data are very large,  
Hensman et al.~\cite{hensman2015scalable}, utilizing recent advances in {\em stochastic variational inference} of Hoffman et al.~\cite{hoffman2013stochastic}  and Hensman et al.~\cite{hensman2013gaussian},
appear to develop methods which scale well.  Adam et al.~\cite{adam2016scalable} develop these results in the context 
Additive GP applied to the source separation problem.

For vector Gaussian processes, one can proceed basically as in the scalar case, including the development of sparse methods,  however
one needs to take care that the vector covariance structure is positive definite (see the review
by  Alvarez et al.~\cite{alvarez2012kernels})
  See e.g.~Yu et al.~\cite{yu2005learning},
 Boyle and Frean
\cite{boyle2005multiple,boyle2005dependent}, Melkumyan and Ramos \cite{melkumyan2011multi},
Alvarez and Lawrence \cite{alvarez2009sparse,alvarez2011computationally}, Titsias and L{\'a}zaro-Gredilla \cite{titsias2011spike}. Raissi et al.~\cite{raissi2017machine} develop methods to learn linear differential equations using GPs.

\subsection{Empirical Mode Decomposition (EMD)}
The definition of an {\em instantaneous frequency} of a signal $x(t)$ is normally accomplished through application of the Hilbert transform $\mathcal{H}$ defined by the principle value of the singular integral
\[ \bigl(\mathcal{H}(x)\bigr)(t):=\frac{1}{\pi}PV \int_{\R}{\frac{x(\tau)}{t-\tau}d\tau},\]
which, when it is well defined,  determines the harmonic conjugate $y:=\mathcal{H}(x)$ of $x(t)$ of a function 
\[ x(t)+iy(t)=a(t)e^{i\theta(t)}\] 
which has an analytic extension to the upper complex half plane in $t$, allowing
the derivative $\omega:=\dot{\theta}$ the interpretation of an instantaneous frequency
of \[ x(t))=a(t)\cos(\theta(t))\,. \]

 However, this definition is controversial, see e.g.~Boashash \cite{boashash1992estimating} for a review,
and possesses many difficulties, and
the Empirical Mode Decomposition (EMD) algorithm was invented by Huang et al.~\cite{huang1998empirical} 
to circumvent them  by 
decomposing a signal into a sum of {\em intrinsic mode functions} (IMFs), essentially
functions whose number of local extrema and zero crossings are either equal or differ by $1$ 
and such that the mean of the envelope of the local maxima and  the local minima is $0$,
which are processed without difficulty by the  Hilbert transform.  See Huang \cite{HilbertHuang} for a more comprehensive discussion. This combination of the EMD and the Hilbert transform, called the Hilbert-Huang transform,
is used 
 decompose a signal into its fundamental AM-FM components. 
  Following 
Rilling et al.~\cite{rilling2003empirical}, the EMD appears as follows: Given a signal $x(t)$
\begin{enumerate}
\item  identify all local extrema of  $x(t)$
\item interpolate between the local minima (resp.~maxima) to obtain the envelope $e_{\min}(t)$
 (resp.~$e_{\max}(t)$)
\item compute the mean $m(t):=\frac{e_{\min}(t)+e_{\max}(t)}{2}$
\item extract the detail $d(t):=x(t)-m(t)$
\item iterate on the residual $m(t)$
\end{enumerate}
The {\em sifting} process  iterates steps (1) through (4) on the detail until it is close enough to zero mean. Then the residual is computed and step (5) is applied.

Despite its remarkable success,
see e.g.~\cite{coughlin200411,huang1998empirical,
neto2004assessment,wu2001impact,
costa2007noise,cummings2004travelling,djemili2016application} and the review on geophysical applications of
 Huang and Wu \cite{huang2008review}.  the original method is defined by an algorithm and therefore its performance is difficult to analyze.
 In particular, sifting and other iterative methods usually do not allow for backward error propagation.  Despite this, much is known about it, improvements have been made and efforts are underway to develop formulations which facilitate a performance analysis. To begin,
it appears that the EMD algorithm is sensitive to noise, so that
 Wu and Huang  \cite{wu2009ensemble} introduce and study an Ensemble EMD,
 further developed in Torres et al.~\cite{torres2011complete}, which 
 appears to resolve the noise problem while increasing the computational costs.  On the other hand, when applied to white noise
Flandrin et al.~\cite{flandrin2004empirical,flandrin2004empirical2,flandrin2005emd} and  Wu and Huang \cite{wu2004study} demonstrate that it acts as an adaptive  wavelet-like filter bank, leading to  
 Gilles' \cite{gilles2013empirical}  development
of empirical wavelets. 
 Rilling and Flandrin \cite{rilling2007one} successfully analyze
 the performance of the the algorithm on  the sum of two cosines. 
Lin et al.~\cite{lin2009iterative} consider an alternative framework for the empirical mode decomposition problem considering a moving average operator instead of the mean function of the EMD. 
This leads to a mathematically analyzable framework, and in some cases (such as the stationary case)
to the analysis of 
  Toeplitz operators, a good theory with good results. This technique has been further developed
by   Huang et al.~\cite{huang2009convergence}, with some success.  
Approaches based on variational principles, such as Feldman \cite{feldman2006time}, utilizing an iterative variational approach using the Hilbert transform, Hou and Shi \cite{hou2011adaptive},  a compressed sensing approach,    
Daubechies et al.~\cite{daubechies2011synchrosqueezed},
 the wavelet base {\em synchrosqueezing} method to be discussed in a moment,  and  Dragomiretskiy and Zosso \cite{dragomiretskiy2014variational}, a generalization of the classic Wiener filter using the
 alternate direction method of multipliers method,  see Boyd et al. \cite{boyd2011distributed},
 to solve the resulting
 bi-variate minimization problem,  appear to be good candidates for analysis.  
 However,  the variational objective function in \cite{hou2011adaptive} uses
 higher order total variational terms so appears sensitive to noise, \cite{feldman2006time} is an iterative variational approach, and the selection of the relevant modes  in \cite{dragomiretskiy2014variational}
for problems with noise is currently under investigation, see e.g.~Ma et al.~\cite{ma2017variational}
and the references therein. 
On the other hand,   Daubechies et al.~\cite{daubechies2011synchrosqueezed} 
 provide rigorous performance guarantees under certain conditions. Nevertheless,  there is still much  effort in developing their work, see e.g.~Auger
 et al.~\cite{auger2013time} for a review of synchrosqeezing and its relationship with time-frequency reassignment.

\subsection{Synchrosqueezing}
\label{sec_synchro}
 Synchrosqueezing,  introduced in Daubechies and Maes \cite{DaubechiesMaes}, was developed
  in Daubechies, Lu and Wu \cite{daubechies2011synchrosqueezed} as an alternative to the EMD algorithm  
which would allow mathematical performance analysis, and has generated much interest, see e.g.~\cite{oberlin2014fourier,thakur2013synchrosqueezing,thakur2015synchrosqueezing,auger2013time,li2012time,wang2013matching}.  Informally following \cite{daubechies2011synchrosqueezed}, for a signal $x(t)$ we let
\[W(a,b):= a^{-\frac{1}{2}}\int_{\R}{x(t)\overline{\psi\Bigl(\frac{t-b}{a}\Bigr)}dt}\]
denote the wavelet transform of the signal $x(t)$ using the wavelet $\psi$. They demonstrate that
for a wavelet such that its Fourier transform satisfies $\hat{\psi}(\xi)=0, \xi< 0$, when applied to a pure tone 
\begin{equation}\label{oekiiieii}
x(t):=A\cos(\omega t)
\end{equation}
 that 
\begin{equation}
\label{oeooijiii}
 \omega(a,b):=-i \frac{\partial \ln W(a,b)}{\partial b}
\end{equation}
satisfies
\[  \omega(a,b)=\omega\,,\]
that is, it provides a perfect estimate of the frequency of the signal \eqref{oekiiieii}.
This suggests using \eqref{oeooijiii} to define the map
\[ (a,b) \mapsto (\omega(a,b),b)\] to push the mass in the reconstruction formula 
\[ x(b)=\Re\Bigl[ C^{-1}_{\psi}\int_{0}^{\infty}{W(a,b)a^{-\frac{3}{2}}da} \Bigr]\,,\]
where $ C_{\psi}:=\int_{0}^{\infty}{\frac{\overline{\hat{\psi}(\xi)}}{\xi}d\xi}$,
to obtain  the identity
\begin{equation}
\label{eoejiiirir}
Re\Bigl[ C^{-1}_{\psi}\int_{0}^{\infty}{W(a,b)a^{-\frac{3}{2}}da} \Bigr] =
\Re\Bigl[ C^{-1}_{\psi}\int_{\R}{T(\omega,b)d\omega} \Bigr]\,,
\end{equation}
where 
\begin{equation}
\label{einjiiiriir}
 T(\omega,b)=\int_{A(b)}{W(a,b)a^{-\frac{3}{2}}\updelta\bigl(\omega(a,b)-\omega\bigr)da} 
\end{equation}
where
\[ A(b):=\{a:W(a,b)\neq 0\} \]
and $\omega(a,b)$ is defined as in \eqref{oeooijiii} for $(a,b)$ such that
$a \in A(b)$. 
We therefore obtain the reconstruction formula 
\begin{equation}
\label{eooiroior}
 x(b)=\Re\Bigl[ C^{-1}_{\psi}\int_{\R}{T(\omega,b)d\omega} \Bigr]\,
\end{equation}
 for the synchrosqueezed transform $T$.
In addition, \cite[Thm.~3.3]{daubechies2011synchrosqueezed} demonstrates that
for a signal $x$ comprised of a sum of AM-FM modes with sufficiently separated frequencies
whose amplitudes are slowly varying with respect to their phases, that
 the synchrosqueezed transform $T(\omega,b)$ is concentrated in narrow bands 
$\omega \approx \dot{\theta}_{i}(b)$ about the  instantaneous frequency of the $i$-th mode and 
restricting the integration  in \eqref{eooiroior} to these bands provides a good recovery of the modes.
\section{The mode decomposition problem}\label{secthpb}

To begin the general (abstract) formulation of the {\em mode decomposition problem},
let $V$ be a separable Hilbert space  with  inner product $\<\cdot,\cdot\>$ and corresponding norm $\|\cdot\|$. Also
let $\I$ be a finite set of indices and let
 $(V_i)_{i\in \I}$ be  linear subspaces $V_{i} \subset V$ such that
\begin{equation}
V=\sum_{i\in \I}V_i\,.
\end{equation}

The mode decomposition problem can be informally formulated as follows
\begin{Problem}\label{pb1}
Given $v\in V$ recover  $v_i \in V_i, i \in \I,$ such that
$v=\sum_{i\in \I}v_i$.
\end{Problem}

Our solution to Problem \ref{pb1} will use the interface between numerical approximation, inference and learning (as presented in \cite{OwhScobook2018,OwhScoSchNotAMS2019}), which although traditionally seen as entirely separate subjects,  are intimately connected through the common purpose of making estimations with partial information \cite{OwhScoSchNotAMS2019}. Since the study of this interface
has been shown to help automate the process of discovery in numerical analysis and the design of fast solvers \cite{owhadi2017multigrid, OwhScobook2018, schafer2017compression}, this paper is also motivated by the idea it might,  in a similar manner and to some degree,  also help the process of discovery in machine learning. Here, these interplays will be exploited to address the general formulation Problem \ref{pb1} of the mode recovery problem from the three perspectives of optimal recovery, game theory and Gaussian process regression. The corresponding minmax recovery framework (illustrated in Figure \ref{figgame} and presented below) will then be used as a building block for the proposed programmable networks.

\subsection{Optimal recovery setting}\label{secoptrec}
 Problem \ref{pb1} is ill-posed if the subspaces  $(V_i)_{i\in \I}$ are not linearly independent, in the sense that such a recovery will not be unique. Nevertheless,   optimal solutions can be defined in the optimal recovery setting of Micchelli and Rivlin \cite{micchelli1977survey}.
To this end,  let $\|\cdot\|_{\B}$ be a quadratic  norm on the product space
\begin{equation}
\B=\prod_{i\in \I} V_i,
\end{equation}
making $\B$ a Hilbert space,
and
let \[\Phi:\B \rightarrow  V\]
 be the information map defined by
\begin{equation}
\Phi u:=\sum_{i\in \I}u_i,\qquad u=(u_i)_{i\in \I}\in \B\, .
\end{equation}
An optimal recovery solution mapping
\[ \Psi:V \rightarrow \B\]
for the mode decomposition problem is defined as follows: for
given $v\in V$, we define $\Psi(v)$ to be the minimizer $w$
of
\begin{equation}\label{eqjhegdjedg}
\min_{w\in \B\mid \Phi w=v} \max_{u\in \B\mid \Phi u=v}\frac{\|u-w\|_\B}{\|u\|_\B}\, .
\end{equation}

\begin{Lemma}
\label{lem_basic}
Let $\Phi: \B \rightarrow V$ be surjective.
For $v\in V$, the solution $w$ of the convex optimization problem
\begin{equation}\label{eqmindsobfirstdeb}
\begin{cases}
\text{Minimize }\|w\|_\B\\
\text{Subject to }w\in \B\text{ and }\Phi w=v\,.
\end{cases}
\end{equation}
determines the unique optimal minmax solution $w=\Psi(v)$ to  \eqref{eqjhegdjedg}.
Moreover,
\[ \Psi(v)=\Phi^{+}v,\]
where  the  Moore-Penrose inverse $\Phi^{+}:V \rightarrow \B$ of $\Phi$ is defined by
\[ \Phi^{+}:=\Phi^{T}\bigl(\Phi \Phi^{T} \bigr)^{-1}\, . \]
\end{Lemma}

Now let us be more specific about the structure of $\B$ that we will assume.
Indeed,  let the subspaces $(V_i)_{i\in \I}$ be equipped with quadratic norms
$(\|\cdot\|_{V_i})_{i\in \I}$ making each \[(V_i,\|\cdot\|_{V_i})\]
 a Hilbert space, and equip  their product $\B=\prod_{i\in \I} V_i$ with the product norm
\begin{equation}\label{eqjehdjhdbdd}
\|u\|_{\B}^2:=\sum_{i\in \I} \|u_i\|_{V_i}^2, \qquad u=(u_i)_{i\in \I}\in \B\, .
\end{equation}
 We use the notation  $[\cdot,\cdot]$  for the duality product between $V^*$ on the left  and $V$ on the right, and also for the
duality product between $V_{i}^*$  and $V_{i}$ for all $i$.
The norm $\|\cdot\|_{V_i}$ makes $V_{i}$  into a Hilbert space if and only if
\begin{equation}\label{eqkkejddh}
\|v_{i}\|^2_{V_i} =[Q_i^{-1}v_{i}, v_{i}], \qquad v_i\in V_i,
\end{equation}
for some positive symmetric linear bijection
\[Q_i\,:\,V^*_i\rightarrow V_i,\]
   where  by positive and symmetric we mean
$[\phi,Q_i\phi]\geq 0$ and $[\phi,Q_i \varphi]=[\varphi,Q_i \phi]$ for all $\varphi,\phi \in V_i^*$.
For each $i \in \I$, the dual space
$V_{i}^*$ to $(V_{i}, \|\cdot\|_{V_{i}}) $ is also a Hilbert space with norm
\begin{equation}\label{eqkkejddhbis}
\|\phi_{i}\|^2_{V_i^{*}} :=[\phi_{i},Q_{i}\phi_{i}], \qquad \phi_i\in V^{*}_i\,,
\end{equation}
and therefore
the dual space $\B^*$ of $\B$ can be identified with the product of the  dual spaces
\begin{equation}
\B^*=\prod_{i\in \I}V_i^*
\end{equation}
with (product) duality product
\begin{equation}
[\phi,u]=\sum_{i\in \I} [\phi_i,u_i], \qquad  \phi=(\phi_i)_{i\in \I}\in \B^*,  \quad
u=(u_i)_{i\in \I}\in \B\, .
\end{equation}
Moreover the symmetric positive linear bijection
\begin{equation}
\label{Qform}
 Q:\B^{*} \rightarrow \B
\end{equation}
defining the quadratic norm
$\|\cdot\|_\B$ is the block-diagonal operator
\[Q:=\diag(Q_i)_{i\in \I}  \]
 defined by its action
$Q \phi=(Q_i\phi_i)_{i\in \I},\,  \phi\in \B^*.$

Let  \[e_i\,:\,V_i\rightarrow V\]
be the subset inclusion and let its adjoint
\[e_i^*\,:\,V^* \rightarrow V_i^*\]
be defined through $[e_i^* \phi, v_i]=[ \phi, e_i v_i]$ for $\phi\in V^*, v_i\in V_i$. These operations
naturally transform the family of operators
\[Q_i\,:\,V_i^*\rightarrow V_i, \quad  i \in \I,\]
into a family  of operators
\[e_i Q_i e_i^* \,:\,V^{*} \rightarrow V,  \quad  i \in \I,\]
all defined on the same space, so that
we can define their sum $S\,:\,V^{*} \rightarrow V$ by
\begin{equation}
\label{S_def}
S=\sum_{i\in \I}e_i Q_i e_i^*\,.
\end{equation}
The following proposition demonstrates that $S$ is invertible and  that $S^{-1}$ and $S$ naturally generate dual Hilbert space  norms on
$V$ and $V^{*}$ respectively.
\begin{Lemma}
\label{lem_S}
The operator
$S:V^{*} \rightarrow V$, defined in \eqref{S_def}, is invertible. Moreover,
\begin{equation}
\|v\|^2_{S^{-1}}:=[S^{-1}v,v], \quad v \in V,
\end{equation}
 defines a Hilbert space norm on $V$  and
 \begin{equation}
\|\phi\|_{S}^2:=[\phi, S\phi] =\sum_{i\in \I} \|e_i^* \phi\|_{V_i^*}^2, \quad  \phi\in V^*\,,
\end{equation}
defines a Hilbert space norm on $V^{*}$ which is dual to that on $V$.
\end{Lemma}

The following theorem
 determines the optimal recovery map $\Psi$.

\begin{Theorem}\label{thmkjhgjhgyuy}
 For $v \in V$,  the minimizer of \eqref{eqmindsobfirstdeb} and therefore the minmax solution of  \eqref{eqjhegdjedg} is
\begin{equation}\label{eqjhhejdkejd}
\Psi(v)= \big(Q_i e_i^* S^{-1} v\big)_{i\in \I}\, .
\end{equation}
Furthermore \[\Phi\big(\Psi(v)\big)=v, \qquad v \in V, \]
where
 \[\Psi: (V,\|\cdot\|_{S^{-1}})\rightarrow  (\B,\|\cdot\|_\B)\]
and  \[\Phi^{*}:(V^{*},\|\cdot\|_{S})\rightarrow (\B^{*},\|\cdot\|_{\B^{*}}) \]
are isometries.
In particular,
  writing
 $\Psi_i(v):=Q_i e_i^* S^{-1} v$, we have
\begin{equation}\label{eqjehdejhdbdhs}
\|v\|_{S^{-1}}^2=\|\Psi(v)\|_{\B}^2=\sum_{i\in \I} \|\Psi_i(v)\|_{V_i}^2\, \qquad v \in V\, .
\end{equation}
\end{Theorem}

Observe that the adjoint \[\Phi^*:V^{*} \rightarrow \B^{*}\] of $\Phi:\B \rightarrow V$, defined by $[\varphi,\Phi u]=[\Phi^*(\varphi),u]$ for $\varphi\in V^*$ and $u\in \B$, is computed to be
\begin{equation}
\Phi^*(\varphi)=(e_i^* \varphi)_{i\in \I}, \qquad \varphi\in V^*\,.
\end{equation}
The following theorem presents optimality results in terms of $\Phi^{*}$.
\begin{Theorem}
\label{thm_eieiei}
We have
\begin{equation}
\|u-\Psi(\Phi u)\|_{\B}^2=\inf_{\phi \in V^*} \|u-Q \Phi^*(\phi) \|_{\B}^2=\inf_{\phi \in V^*}\sum_{i\in \I} \|u_i-Q_i e_i^* \phi\|_{V_i}^2\,.
\end{equation}
\end{Theorem}

\subsection{Game/decision theoretic setting}
\label{sec_game}
Optimal solutions to Problem \ref{pb1} can also be defined in the setting of the game/decision theoretic approach to numerical approximation presented in \cite{OwhScobook2018}. In this setting the
 minmax problem  \eqref{eqjhegdjedg} is interpreted as an adversarial zero sum game (illustrated in Figure \ref{figgame}) between two players  and lifted to mixed strategies to identify a saddle point. Let $\cP_2(\B)$ be the set of Borel probability measures
 $\mu$ on $\B$ such that $\E_{u\sim \mu}\big[\|u\|_\B^2\big]<\infty$, and let $L(V,\B)$ be the set of Borel
  measurable functions $\psi:V\rightarrow  \B.$ Let $\er:\cP_2(\B)\times L(V,\B) \rightarrow \R$ be the loss
 function
  defined by
\begin{equation}\label{eqkjehdjedhd}
\er(\mu,\psi)=\frac{\E_{u\sim \mu}\big[\|u-\psi(\Phi u)\|_\B^2\big]}{\E_{u\sim \mu}\big[\|u\|_\B^2\big]},
\qquad  \mu \in \cP_2(\B), \psi \in L(V,\B)\, .
\end{equation}
Let us also recall the more general notion of a Gaussian field as described in
 \cite[Chap.~17]{OwhScobook2018}. To that end, a Gaussian space $\mathbf{H}$ is a linear subspace
 $\mathbf{H}\subset L^{2}(\Omega,\Sigma, \P)$ of the $L^{2}$  space of a probability space consisting of centered Gaussian random variables.
A centered  Gaussian field $\xi$ on $\B$ with covariance operator $Q:\B^{*} \rightarrow \B$,
 written  $\xi\sim \cN(0,Q)$, is
 an isometry
\[ \xi:\B^{*} \rightarrow \mathbf{H}\]
 from $\B^*$ to a Gaussian space $\mathbf{H}$, in that
\[[\phi,\xi]\sim \cN\big(0, [\phi,Q \phi]\big), \qquad  \phi\in \B^*,\]
where we use the notation $[\phi,\xi]$ to denote the action $\xi(\phi)$ of $\xi$ on the element $\phi\in \B^*$, thus
indicating
  that $\xi$ is a  weak $\B$-valued Gaussian random variable.
  \interfootnotelinepenalty=10000
  As discussed in \cite[Chap.~17]{OwhScobook2018}, there is a one to one correspondence between  Gaussian cylinder measures and Gaussian fields\footnote{
The {\em cylinder sets} of $\B$ consists of all sets of the form
$F^{-1}(B)$ where $B \in \R^{n}$ is a Borel set and $F:\B \rightarrow \R^{n}$ is a continuous linear map, over all integers $n$.
A {\em cylinder measure } $\mu$, see also \cite[Chap.~17]{OwhScobook2018}, on $B$,
is a  collection of measures $\mu_{F}$ indexed by $F:\B \rightarrow \R^{n}$  over all $n$ such that
each $\mu_{F}$ is a Borel measure on $\R^{n}$ and such that
for $F_{1}:\B \rightarrow \R^{n_{1}}$ and $F_{2}:\B \rightarrow \R^{n_{2}}$
and $G:\R^{n_{1}}\rightarrow \R^{n_{2}}$ linear and continuous with
$F_{2}=GF_{1}$, we have
$G_{*}\mu_{F_{1}}=\mu_{F_{2}}$, where $G_{*}$ is the pushforward operator on  measures corresponding to the map $G$, defined by
$(G_{*}\nu)(B):=\nu(G^{-1}B)$.  When each measure $\mu_{F}$ is Gaussian, the cylinder measure is said to be a Gaussian cylinder measure.  A sequence $\mu_{n}$ of cylinder measures such that the sequence
$(\mu_{n})_{F}$ converges in the weak topology for each $F$, is said to converge
 in the {\em weak cylinder measure topology}.
}.
Let $\xi$ denote the Gaussian field
\[\xi\sim \cN(0,Q)\]
 on $\B$ where  $Q:\B^{*} \rightarrow \B$ is the  block diagonal operator
$Q:=\diag(Q_i)_{i\in \I}$, and
let $\mu^\dagger$ denote  the  cylinder measure defined by the Gaussian field $\xi-\E[\xi|\Phi \xi]$,
or the corresponding Gaussian measure in finite dimensions.

We say that a tuple $(\mu', \psi') $ is a saddle point of the loss function
$\er:  \cP_2(\B) \times L(V,\B)  \rightarrow \R$  if
\[\er(\mu,\psi') \leq \er(\mu',\psi') \leq \er(\mu', \psi), \quad \mu  \in \cP_2(\B) ,\,\, \psi \in  L(V,\B)
  \, .\]
Theorem \ref{thmkjjdhkejd} shows that the optimal strategy of Player I is the Gaussian field $\xi-\E[\xi|\Phi \xi]$,  the optimal strategy of Player II is the conditional expectation
\begin{equation}\label{eqjjghgjhgy}
\Psi(v)=\E\big[\xi\big|\Phi \xi=v\big]\,,
\end{equation}
and  \eqref{eqjjghgjhgy} is equal to \eqref{eqjhhejdkejd}.

\begin{Theorem}\label{thmkjjdhkejd}
Let $\er$ be defined as in \eqref{eqkjehdjedhd}. It holds true that
\begin{equation}\label{eqkjehdsdejedhd}
\max_{\mu \in \cP_2(\B)} \min_{\psi \in L(V,\B)}\er(\mu,\psi)=\min_{\psi \in L(V,\B)} \max_{\mu \in \cP_2(\B)} \er(\mu,\psi)\,.
\end{equation}
Furthermore,
\begin{itemize}
\item If $\dim(V)<\infty$ then $\big(\mu^\dagger, \Psi\big)$
 is a saddle point for the loss \eqref{eqkjehdjedhd},
 where $\Psi$ is as in \eqref{eqjhhejdkejd} and  \eqref{eqjjghgjhgy}.
\item If $\dim(V)=\infty$, then the loss \eqref{eqkjehdjedhd} admits a sequence of saddle points $(\mu_n,\Psi) \in \cP_2(\B)\times L(V,\B)$ where $\Psi$ is  as in \eqref{eqjhhejdkejd} and \eqref{eqjjghgjhgy},
and
the $\mu_n$ are
Gaussian measures, with finite dimensional support, converging towards
$\mu^\dagger$ in the weak cylinder measure topology.
\end{itemize}
\end{Theorem}
\begin{proof}
The proof is essentially that of \cite[Thm.~18.2]{OwhScobook2018}
\end{proof}

\subsection{Gaussian process regression setting}\label{subsec64d}

    \begin{figure}[h!]
        \begin{center}
                        \includegraphics[width=0.75\textwidth]{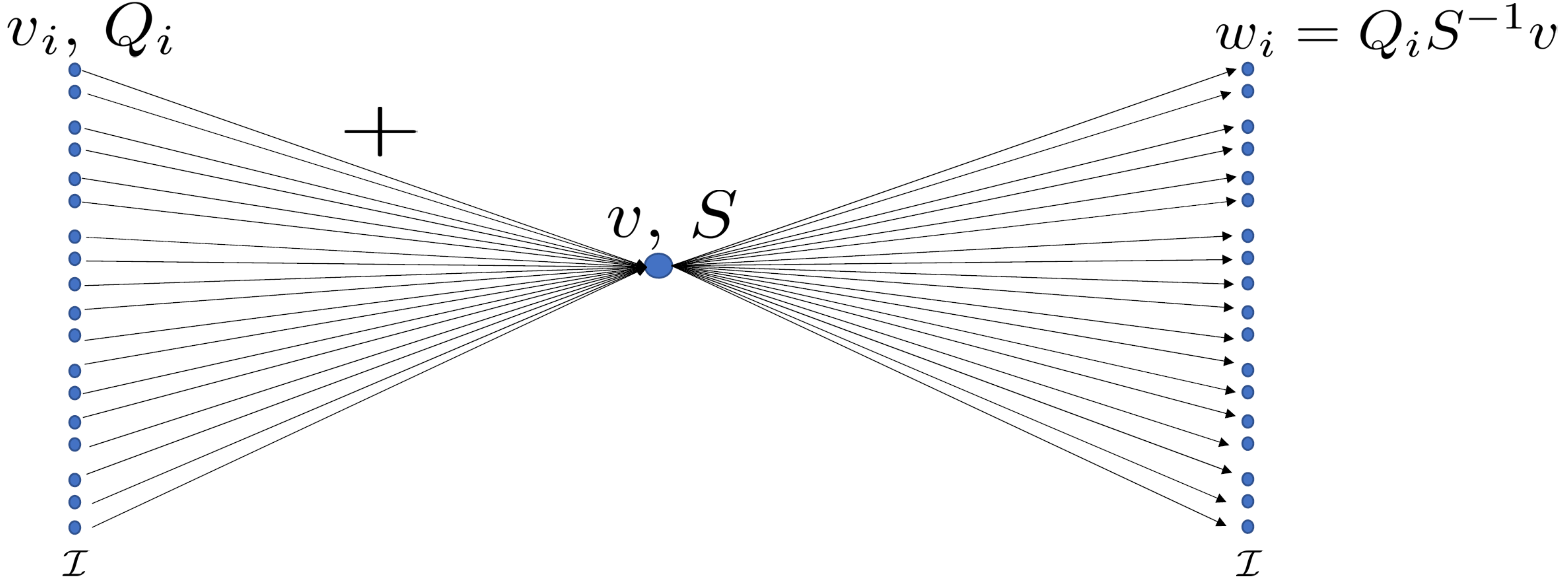}
                \caption{The minmax solution of the mode decomposition problem.}\label{figmultilevel1}
        \end{center}
\end{figure}

Let us demonstrate that Theorem \ref{thmkjjdhkejd} implies that the minmax optimal solution
to Problem \ref{pb1} with loss measured as the relative error in the norm \eqref{eqjehdjhdbdd} can be obtained via Gaussian process regression. To that end,
 let  $\xi_i \sim \cN(0,Q_i),\, i\in \I,$ be independent
$V_{i}$-valued Gaussian fields defined by the norms $\|\cdot\|_{V_i}$. Recall that $Q_i$ is defined in \eqref{eqkkejddh} and that $\xi_i$ is an isometry from $(V_i^*,\|\cdot\|_{V_i^*})$ onto a Gaussian space, mapping $\phi \in V_i^*$ to $[\phi,\xi_i]\sim \cN(0, [\phi,Q_i \phi])$.
Theorem  \ref{thmkjjdhkejd} asserts that the minmax
estimator is \eqref{eqjjghgjhgy}, which,  written componentwise,
 determines the  optimal reconstruction of each mode $v_j$ of $v=\sum_{i \in \I}v_i$ to be
\begin{equation}\label{eqkjgdjhed}
\E\big[\xi_j\big|\sum_{i\in \I}\xi_i=v\big]=Q_j (\sum_{i\in \I} Q_i)^{-1} v\,.
\end{equation}
where the right hand side of \eqref{eqkjgdjhed} is obtained from \eqref{eqjhhejdkejd}, and
$\sum_{i\in \I} Q_i$ is a shorthand notation for
$\sum_{i} e_{i}Q_{i}e_{i}^*$ obtained by dropping  the indications of the injections $e_i$ and their adjoint projections $ e_i^*$. From now on, we will use such simplified notations  whenever there is no risk of confusion.
In summary,  the minmax solution of the abstract mode decomposition problem, illustrated in Figure \ref{figmultilevel1},  is obtained  based on the specification of the operators $Q_i\,:\, V_i^* \rightarrow V_i$ and the injections $e_{i}:V_{i}\rightarrow V$, of  which the former can be interpreted as quadratic norm defining operators or as covariance operators.
Table \ref{table_slide4} illustrates the  three equivalent interpretations
 -optimal recovery/operator kernel/Gaussian process regression
of our methodology.

\begin{table*}[h]
\centering
\begin{tabular}{ |p{5.cm}|p{4.cm}|p{4.cm}|  }
\hline
  Norm  & Operator/Kernel  &  GP\\
 \hline
 \vspace{.01cm} $ \|v_{i}\|^{2}_{V_{i}}:=\langle Q_{i}^{-1}v_{i},v_{i} \rangle $
&\vspace{.01cm} $Q_{i}:V_{i}^{*}\rightarrow V_{i}$ &\vspace{.01cm}
 $\xi_{i} \sim \mathcal{N}(0,Q_{i})$  \\
 \vspace{.02cm} $\arg \min
 \begin{cases}
\text{minimize} \,\, \sum_{i}{\|w_{i}\|^{2}_{V_{i}}}\\
\sum_{i}{w_{i}}=v
\end{cases}$
 &\vspace{.24cm} $ Q_{i}\bigl(\sum_{j}{Q_{j}}\bigr)^{-1}v$ &\vspace{.24cm} $\E[\xi_{i}\mid \sum_{j}{\xi_{j}}=v]$
\vspace{.6cm} \\
 \hline
\end{tabular}
\caption{Three equivalent interpretations
 -optimal recovery/operator kernel/Gaussian process regression
of our methodology.
}
\label{table_slide4}
\end{table*}

  \begin{figure}[h!]
	\begin{center}
			\includegraphics[width=\textwidth]{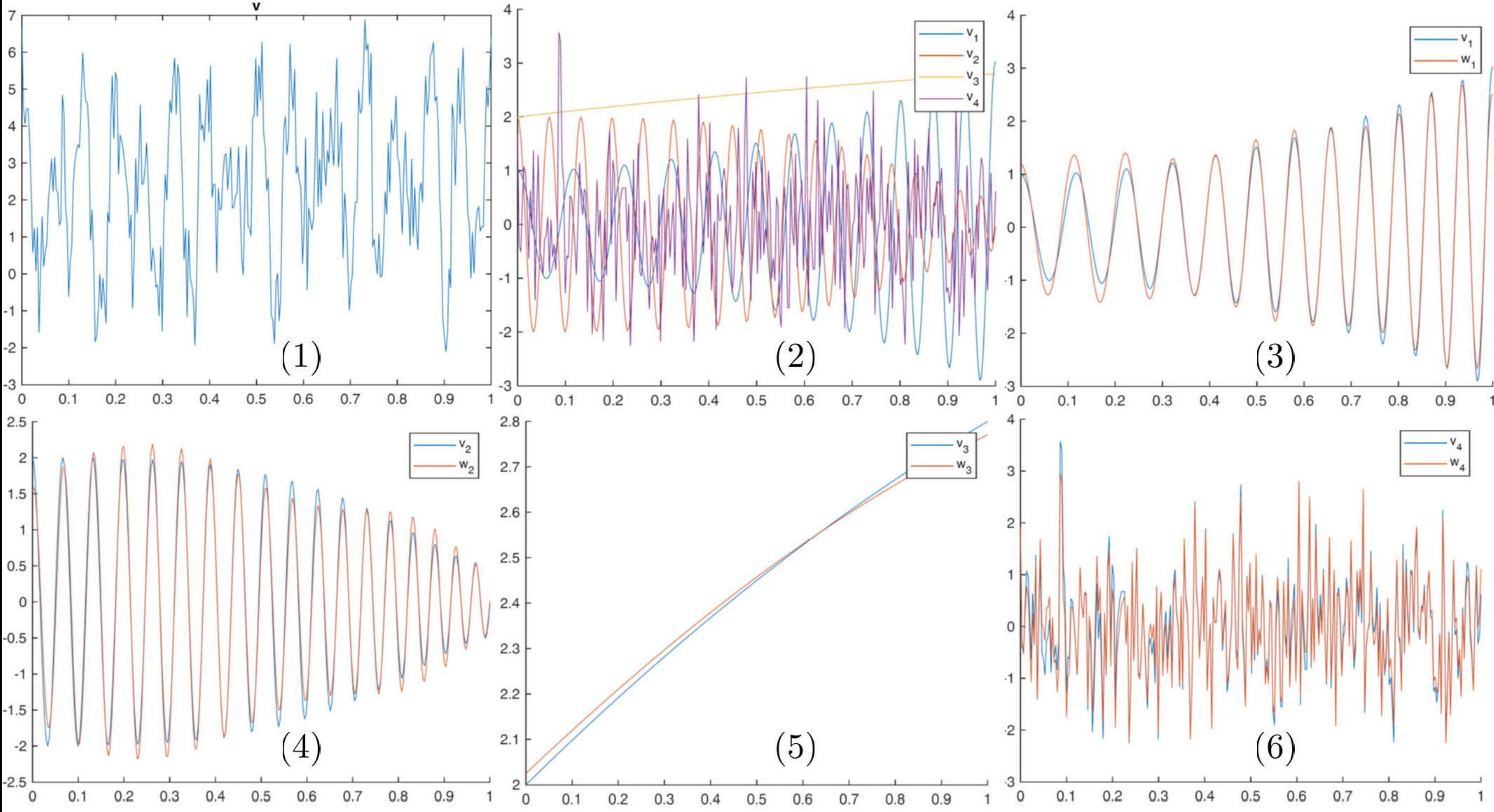}
		\caption{(1) The signal $v=v_1+v_2+v_3+v_4$ (2) The modes $v_1, v_2, v_3, v_4$ (3) $v_1$ and its approximation $w_1$ (4) $v_2$ and its approximation $w_2$ (5) $v_3$ and its approximation $w_3$ (6) $v_4$ and its approximation $w_4$.}\label{figknownfw1}
	\end{center}
\end{figure}
\begin{Example}\label{eglsdljdlekd}
Consider the problem of recovering the modes $v_1, v_2, v_3, v_4$ from the observation of the signal $v=v_1+v_2+v_3+v_4$ illustrated in
Figure \ref{figknownfw1}. In this example all modes are defined on the interval $[0,1]$, $v_1(t)=(1+2t^2)\cos(\theta_1(t))-0.5t\sin(\theta_1(t)) $, $v_2(t)=2(1-t^3)\cos(\theta_2(t))+(-t+0.5t^2)\sin(\theta_2(t)) $, $v_3(t)=2+t-0.2 t^2$,  and $v_4$ is white-noise (the instantiation of a centered GP with covariance function $\updelta(s-t)$).
 $\theta_1(t)=\int_0^t \omega_1(s)\,ds $ and $\theta_2(t)=\int_0^t \omega_2(s)\,ds $ are defined by the instantaneous frequencies
 $\omega_1(t)=16\pi (1+t)$ and $\omega_2(t)=30\pi (1+t^2/2)$.
In this recovery problem $\omega_1(t)$ and $\omega_2(t)$ are known,  $v_3$ and the amplitudes of the oscillations of $v_1$ and $v_2$ are unknown smooth functions of time, only the distribution of $v_4$ is known.
To define optimal recovery solutions one can either define the normed subspaces $(V_i,\|\cdot\|_{V_i})$ or (equivalently via \eqref{eqkkejddh})  the covariance functions/operators of the Gaussian processes $\xi_i$. In this example it is simpler to use the latter. To define the covariance function of the GP $\xi_1$ we assume that
 $\xi_1(t)=\zeta_{1,c}(t) \cos(\theta_1(t))+\zeta_{1,s}(t) \sin(\theta_1(t))$, where $\zeta_{1,c}$ and $\zeta_{1,s}$ are independent identically distributed centered Gaussian processes with covariance function
  $\E[\zeta_{1,c}(s)\zeta_{1,c}(t) ]=\E[\zeta_{1,s}(s)\zeta_{1,s}(t) ]=e^{-\frac{(s-t)^2}{\gamma^2}}$ (chosen with $\gamma=0.2$ as a prior regularity assumption). Under this choice $\xi_1$ is a centered GP with covariance function
$K_1(s,t)= e^{-\frac{(s-t)^2}{\gamma^2}}\big(\cos(\theta_1(s))\cos(\theta_1(t))+
\sin(\theta_1(s))\sin(\theta_1(t))\big)$.
Note that the cosine and sine summation formulas imply that
translating $\theta_1$ by an arbitrary phase $b$ leaves $K_{1}$ invariant  (knowing $\theta_1$ up to a phase shift is sufficient to construct that kernel).
Similarly we select the covariance function of the independent centered GP $\xi_2$ to be
$K_2(s,t)= e^{-\frac{(s-t)^2}{\gamma^2}}\big(\cos(\theta_2(s))\cos(\theta_2(t))+\sin(\theta_2(s))
\sin(\theta_2(t))\big)$.
To enforce the regularity of $\xi_3$ we select its covariance function to be
$K_3(s,t)= 1+st+e^{-\frac{(s-t)^2}{4}}$. Finally since $v_4$ is white noise we represent it with a centered GP with covariance function
$K_4(s,t)= \updelta(s-t)$. Figure \ref{figknownfw1} shows the recovered modes using \eqref{eqkjgdjhed} (or equivalently defined as \eqref{eqjhhejdkejd} and the minimizer of  \eqref{eqmindsobfirstdeb}). In this numerical implementation the interval $[0,1]$ is discretized with $302$ points (with uniform time steps between points), $\xi_4$ is a discretized  centered Gaussian vector of dimension $302$ and of identity covariance matrix and $\xi_1,\xi_2,\xi_3$ are discretized as centered Gaussian vectors with covariance matrices corresponding to
 the kernel matrices  $\bigl(K(t_{i},t_{j})\bigr)_{i,j=1}^{302}$ corresponding to
  $K_1, K_2$ and  $K_3$ determined by the sample points $t_{i}, i=1,\ldots, 302$.
\end{Example}

Table  \ref{table_slide6} provides a summary of the approach of  Example \ref{eglsdljdlekd}, illustrating
the connection between the assumed mode structure and corresponding Gaussian process structure and its corresponding reproducing kernel structure.
\begin{table*}[h]
\centering
\begin{tabular}{ |p{3.8cm}|p{3.9cm}|p{6.4cm}|  }
\hline
 \hspace{2cm}Mode  &\hspace{2cm} GP  & \hspace{2cm} Kernel\\
 \hline
\vspace{-1.5em}
\[v_{1}(t)=a_{1}(t)\cos\bigl(\theta_{1}(t)\bigr)\]
\[ \theta_{1}\, \text{known}\]
\[a_{1}\, \text{unknown smooth} \vspace{-1em}\]
\vspace{-1em}
&\vspace{-1.5em}\[\xi_{1}(t)=\zeta_{1}(t)\cos\bigl(\theta_{1}(t)\bigr)\]
\[\E[\zeta_{1}(s)\zeta_{1}(t)]=e^{-\frac{|s-t|^{2}}{\gamma^{2}}}\]
&
 \[K_{1}(s,t)=e^{-\frac{|s-t|^{2}}{\gamma^{2}}}\cos\bigl(\theta_{1}(s)\bigr)\cos\bigl(\theta_{1}(t)\bigr)\]
\\[-1ex]
\hline
\vspace{-1.5em}
\[v_{2}(t)=a_{2}(t)\cos\bigl(\theta_{2}(t)\bigr)\]
\[ \theta_{2}\, \text{known}\]
\[a_{2}\, \text{unknown smooth} \vspace{-1.5em}\]
& \vspace{-1.5em} \[
\xi_{2}(t)=\zeta_{2}(t)\cos\bigl(\theta_{2}(t)\bigr)\]
\[\E[\zeta_{2}(s)\zeta_{2}(t)]=e^{-\frac{|s-t|^{2}}{\gamma^{2}}}\]
&
 \[K_{2}(s,t)=e^{-\frac{|s-t|^{2}}{\gamma^{2}}}\cos\bigl(\theta_{2}(s)\bigr)\cos\bigl(\theta_{2}(t)\bigr)\]
 \\ [-1ex]
\hline
\hspace{.8cm}\vspace{-1em}
\[v_{3}\, \text{unknown smooth} \]  &\hspace{.8cm} \vspace{-1.3em}\[\E[\xi_{3}(s)\xi_{3}(t)]=e^{-\frac{|s-t|^{2}}{\gamma^{2}}}\]
& \vspace{-.2em} \[K_{3}(s,t)=e^{-\frac{|s-t|^{2}}{\gamma^{2}}}\] \\[-2ex]
\hline \vspace{-1em}
\hspace{.6cm} \[v_{4}\, \text{unknown white noise}\]  &   \vspace{-1em} \[\hspace{-.1cm} \E[\xi_{4}(s)\xi_{4}(t)]=\sigma^{2}\updelta(s-t)\]
&  \[K_{4}(s,t)= \sigma^{2}\updelta(s-t)\] \vspace{.6cm}\\[-5ex]
\hline
\vspace{.1cm}
\hspace{.2cm}$v=v_{1}+v_{2}+v_{3}+v_{4}$
&\vspace{.1cm}\hspace{.2cm}$\xi=\xi_{1}+\xi_{2}+\xi_{3}+\xi_{4}$
& \vspace{.1cm}$\hspace{.8cm} K=K_{1}+K_{2}+K_{3}+K_{4}$
\vspace{.3cm}\\
 \hline
\end{tabular}
\caption{A summary of the approach of  Example \ref{eglsdljdlekd}, illustrating
the connection between the assumed mode structure and corresponding Gaussian process structure and its corresponding reproducing kernel structure. Note that, for clarity of presentation,
 this summary does not exactly match that of  Example \ref{eglsdljdlekd}.     }
\label{table_slide6}
\end{table*}

\paragraph{On additive models.}
The recovery approach of Example \ref{eglsdljdlekd} is based on the design of an appropriate additive regression model.
Additive regression models are not new. They were introduced in \cite{stone1985additive}
 for approximating  multivariate functions with sums of univariate functions. Generalized additive models (GAMs) \cite{hastie1990generalized} replace a linear regression  model $\sum_i \alpha_i X_i$ with an additive regression model  $\sum_i f_i(X_i)$ where the $f_i$ are unspecified (smooth) functions estimated from the data. Since their inception GAMs have become increasingly popular because they are both easy to interpret and easy to fit \cite{plate1999accuracy}.
 This popularity has motivated the introduction of additive Gaussian processes  \cite{duvenaud2011additive, durrande2012additive}
 defined as Gaussian processes whose high dimensional  covariance kernels are obtained from
  sums of low dimensional ones. Such kernels are expected
   to overcome the curse of dimensionality by exploiting additive non-local effects when such effects are present \cite{duvenaud2011additive}. See Section \ref{sec_AGP}.
 Of course, performing regression or mode decomposition with Gaussian processes (GPs) obtained as sums of independent GPs (i.e.~performing kriging with kernels obtained as sums of simpler kernels) is much older since  Tikhonov regularization (for signal/noise separation)  has a natural interpretation as a conditional expectation $\E[\xi_{s}|\xi_s+\xi_\sigma]$ where $\xi_s$ is a GP with a smooth prior (for the signal) and $\xi_\sigma$ is a white noise GP independent from $\xi_s$.
 More recent applications include classification \cite{Maji2013EfficientCF},  source separation \cite{park2008gaussian, liutkus2011gaussian}, and the
  detection of  the periodic part of a function from partial point evaluations
 \cite{durrande2016gaussian, adam2016scalable}. For that latter application, the approach of \cite{durrande2016gaussian} is to (1) consider the RKHS $H_{K}$ defined by a Mat\'{e}rn kernel $K$ (2)
   interpolate the data with the kernel $K$ and (3)  recover the periodic part by projecting the interpolator
  (using a projection that is orthogonal with respect to  the RKHS  scalar product  onto $H_p:=\Span\{\cos(2\pi k t/\lambda), \sin(2\pi k t/\lambda)\mid 1\leq k \leq q \}$ (the parameters of the Mat\'{e}rn kernel and the period $\lambda$ are obtained via maximum likelihood estimation). Defining $K_p$ and $K_{np}$ as the kernels induced on $H_p$ and its orthogonal complement in $H_{K}$, we have $K=K_{p}+K_{np}$ and the recovery (after MLE estimation of the parameters) can also be identified as the conditional expectation of the GP induced by $K_p$ conditioned on the GP induced by $K_{p}+K_{np}$.

\section{Kernel mode decomposition networks (KMDNets)}\label{sechgjjhjbvf}
\begin{figure}[h]
        \begin{center}
                        \includegraphics[width=\textwidth]{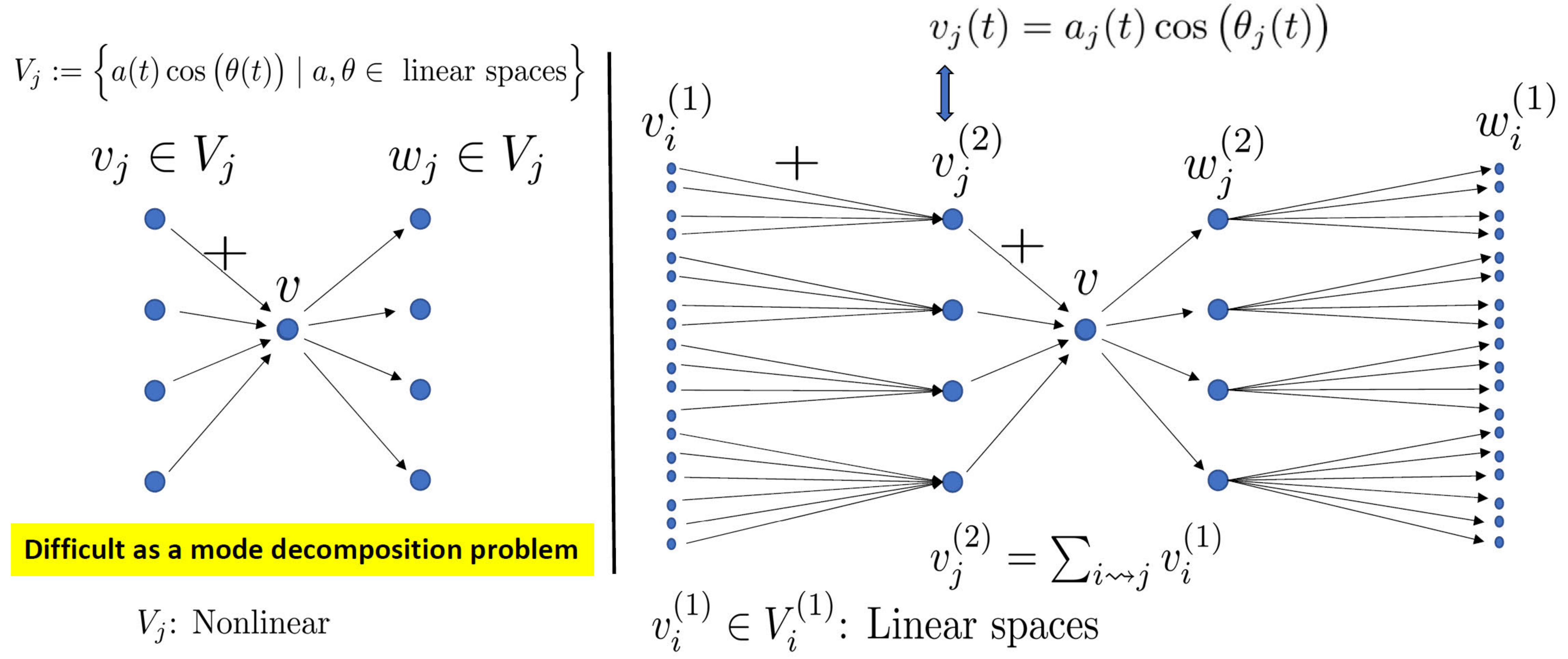}
                \caption{Left: Problem \ref{pb2} is hard as a mode decomposition problem  because
                the modes $v_j=a_j(t)\cos(\theta_j(t))$ live in non-linear functional spaces. Right: One fundamental idea is to recover those modes as aggregates of finer modes  $v_i$ living in linear spaces.
}\label{figlinearize}
        \end{center}
\end{figure}
   The recovery approach described in Example \ref{eglsdljdlekd} is based on the prior knowledge of (1) the number of quasi-periodic modes (2) their phase functions $\theta_{i}$  and (3) their base periodic waveform (which need not be a cosine function). In most applications (1) and (2) are not available and the base waveform may not be trigonometric and may not be known.
   Even when the base waveforms are known and trigonometric (as in Problem \ref{pb2}),
 when the modes' phase functions are unknown, the recovery of the modes is still significantly harder than when they are known because, as illustrated in Figure \ref{figlinearize}, the functional spaces defined by the modes $a_j(t) \cos\big(\theta_j(t)\big)$ (under regularity assumptions on the $a_j$ and $\theta_j$) are no longer linear spaces and the simple calculus of Section \ref{secthpb} requires the spaces $V_j$ to be linear.

To address the full Problem \ref{pb2},
   one fundamental idea is to recover those modes $v_j$ as aggregates of finer modes  $v_i$ living in linear spaces $V_i$ (see Figure \ref{figlinearize}). In particular, we will identify $i$ with time-frequency-phase triples $(\tau,\omega,\theta)$ and the spaces $V_i$ with one dimensional spaces spanned  by functions that are maximally localized in the time-frequency-phase domain
   (i.e.~by Gabor wavelets as suggested by the approximation \eqref{eqjkhgkjehdgd}) and recover the modes $a_j(t) \cos\big(\theta_j(t)\big)$ by aggregating the finer recovered modes.
    \begin{figure}[h]
        \begin{center}
                        \includegraphics[width=\textwidth]{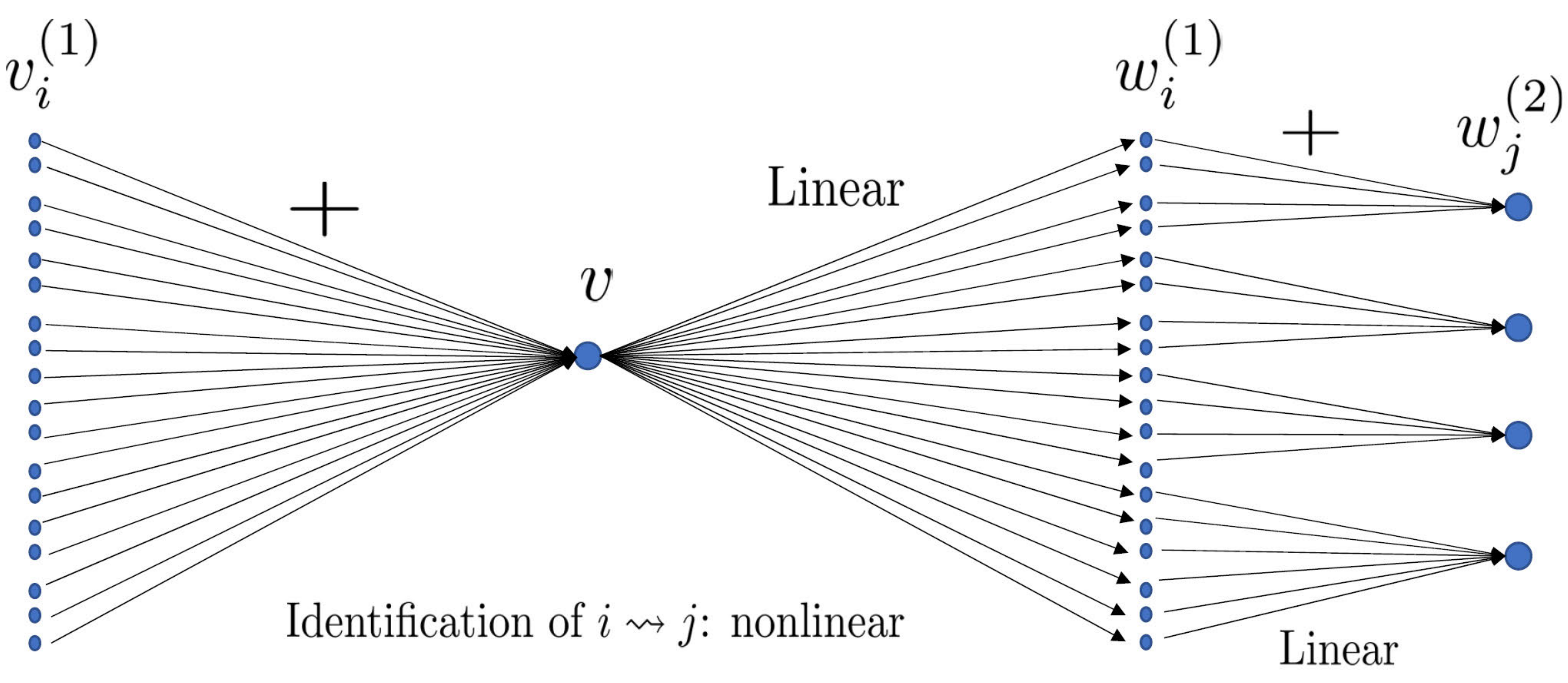}
                \caption{Mode decomposition/recomposition problem.  Note that the nonlinearity of this model is fully represented in the identification
of the relation $i\leadsto j$;  once this identification is determined all other operations are linear. 
}\label{figmdecrec}
        \end{center}
\end{figure}
The implementation of this idea will therefore transform the nonlinear mode decomposition problem illustrated on the left hand side of Figure
\ref{figlinearize} into the mode decomposition/recomposition problem illustrated in Figure \ref{figmdecrec} and transfer its nonlinearity to the identification of ancestor/descendant relationships $i\leadsto j$.

To identify these ancestor/descendant relations  we will compute the
energy $E(i):=\|w_i\|_{V_i}^2$  for each recovered mode $w_i$, which as illustrated in Figure \ref{figenergynonlinear} and discussed in Section \ref{secenvar}, can also be identified as
the alignment $\<w_i,v\>_{S^{-1}}$ between recovered mode $w_i$ and the signal $v$ or as the alignment
$\E[\Var[\<\xi_i,v\>_{S^{-1}}]$ between the model $\xi_i$ and the data $v$. Furthermore $E$ satisfy an energy preservation identity
$\sum_{i} E(i)=\|v\|_{S^{-1}}^2$ which leads to its variance decomposition interpretation.
Although  alignment calculations are linear, the calculations of the resulting child-ancestor relations may involve a nonlinearity (such as thresholding, graph-cut, computation of a maximizer) and the resulting network can be seen
 as a sequence of sandwiched linear operations and simple non-linear steps having striking similarities with artificial neural networks.

  \begin{figure}[h!]
	\begin{center}
			\includegraphics[width=\textwidth]{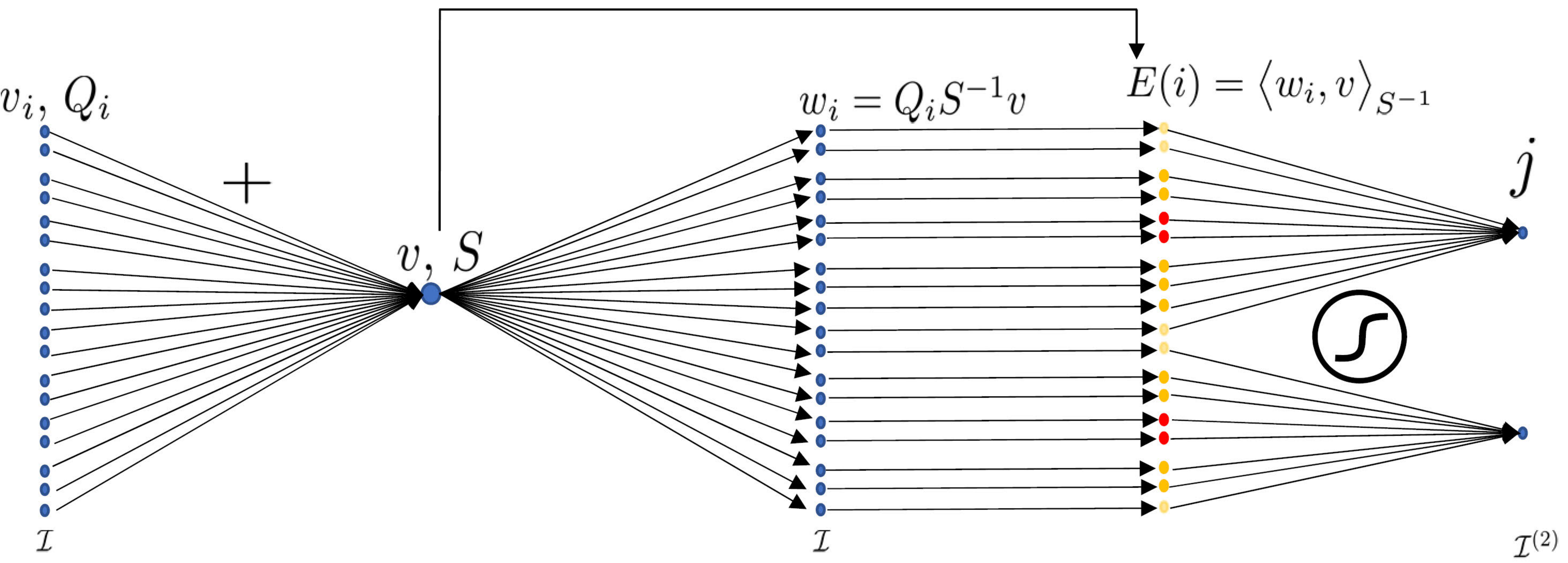}
		\caption{Derivation of ancestor/descendant relations from energy calculations.}\label{figenergynonlinear}
	\end{center}
\end{figure}

Of course this strategy can be repeated across levels of abstractions and its complete deployment
will also require the generalization of the setting of Section \ref{secthpb} (illustrated in Figure \ref{figmultilevel1}) to a hierarchical setting (illustrated in Figure \ref{figmultilevel2} and described in Section \ref{secjhgjhg7ggt}).

\subsection{Model/data alignment and energy/variance decomposition}\label{secenvar}
Using the setting and notations of Section \ref{secthpb} and
fixing the observed data $v\in V$,  let $E\,:\,\I\rightarrow \R_+$ be the function defined by
\begin{equation}
\label{ieieieihhgggg}
E(i):=\|\Psi_i(v)\|_{V_i}^2,\quad i \in \I,
\end{equation}
where $\Psi_i$ are the components of the optimal recovery map $\Psi$ evaluated in 
Theorem \ref{thmkjhgjhgyuy}. 
We will refer to $E(i)$ as the energy of the mode $i$ in reference to its numerical analysis interpretation (motivated by the "energy" representation of  $E(i)=[Q_i^{-1}\Psi_i(v),\Psi_i(v)]$ determined by \eqref{eqkkejddh}, and the interpretation of $Q_i^{-1}$ as an elliptic operator) and our general approach will be based on using its local and/or global maximizers to decompose/recompose kernels.

Writing $E_{\rm{tot}}:=\|v\|_{S^{-1}}^2$, note that \eqref{eqjehdejhdbdhs} implies that
\begin{equation}\label{eqjehdejhdbdjkjhss}
E_{\rm{tot}}=\sum_{i\in \I }E(i)\, .
\end{equation}
Let $\<\cdot,\cdot\>_{S^{-1}}$ be the scalar product on $V$ defined by the norm $\|\cdot\|_{S^{-1}}$.

\begin{Proposition}\label{propeqejdkejdde}
Let $\xi \sim \cN(0,Q)$ and $\phi:=S^{-1} v$.
It holds true that for $i\in \I$,
\begin{equation}\label{eqkejdlnedkjds}
E(i)=\<\Psi_i(v),v\>_{S^{-1}}=\Var\big([\phi,\xi_i]\big)=\Var\big(\<\xi_i,v\>_{S^{-1}}\big)\,.
\end{equation}
\end{Proposition}
Observe that $E(i)=\Var\big(\<\xi_i,v\>_{S^{-1}}\big)$ implies that $E(i)$
 is a measure of the alignment between the  Gaussian process (GP) model  $\xi_i$ and the data $v$ in $V$
 and \eqref{eqjehdejhdbdjkjhss} corresponds to the variance decomposition
 \begin{equation}\label{eqkjguyguygy}
 \Var\big( \<\sum_{i\in \I }\xi_i,v\>_{S^{-1}}\big) =\sum_{i\in \I }\Var\big(\<\xi_i,v\>_{S^{-1}}\big)\, .
 \end{equation}
 Therefore, the stronger this alignment $E(i)$ is, the better the model $\xi_i$ is at explaining/representing the data. Consequently, we refer to the energy  $E(i)$ as the {\em alignment energy}.
 Observe also that the identity $E(i)=\<w_i,v\>_{S^{-1}}$ with $w_i=\Psi_i(v)$ implies that $E(i)$ is also a measure of the alignment between the optimal approximation $w_i$ of $v_i$ and the signal $v$.
Table \ref{table_slide5} illustrates the relations between the conservation of alignment energies
  and the variance decomposition derived from Theorem \ref{thmkjhgjhgyuy} and Proposition \ref{propeqejdkejdde}.
\begin{table*}[h]
\centering
\begin{tabular}{ |p{1,3cm}|p{4.4cm}|p{3.3cm}|p{3cm}|  }
\hline
 & Norm  & Operator/Kernel  &  GP\\
 \hline
 $ E(i)$ & $\|\Psi_{i}(v)\|^{2}_{V_{i}} =\<\Psi_i(v),v\>_{S^{-1}}$
&$[S^{-1}v,Q_{i}S^{-1}v]$ &
 $\Var\bigl(\langle \xi_{i}, v\rangle_{S^{-1}}\bigr)$
\\
\hline
$ \sum_{i}{E(i)}$ & $ \|v\|^{2}_{S^{-1}} $
 & $  [S^{-1}v,v]$ & $
\Var\bigl(\langle\sum_{i}{ \xi_{i}}, v\rangle_{S^{-1}}\bigr)$
 \\
 \hline
\end{tabular}
\caption{Identities for $E(i)$ and $\sum_i E(i)$}
\label{table_slide5}
\end{table*}

\subsection{Programming modules and feedforward network}\label{seckejhdkjdh}
\begin{figure}[h]
        \begin{center}
                        \includegraphics[width=\textwidth]{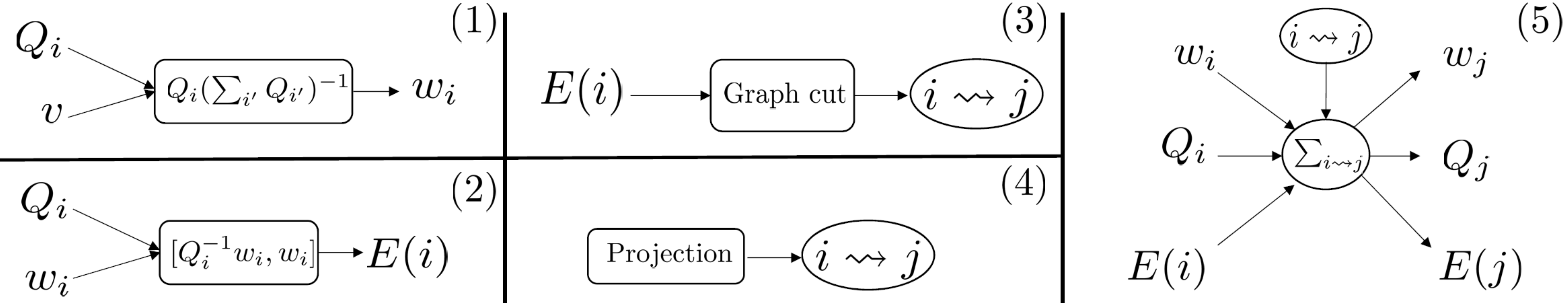}
                \caption{Elementary programming modules for Kernel Mode Decomposition.
}\label{figblock1}
        \end{center}
\end{figure}

We will now  combine the alignment energies of Section \ref{secenvar} with the mode decomposition approach of Section
\ref{secthpb} to design  elementary programming modules (illustrated in Figure  \ref{figblock1}) for kernel mode decomposition  networks (KMDNets). These  will be introduced in this section and developed in the following ones.
Per Section \ref{secthpb} and Theorem \ref{thmkjhgjhgyuy},  the optimal recoveries of the modes $(v_i)_{i\in \I}$ given the covariance operators $(Q_i)_{i\in \I}$ and
the observation of
$\sum_{i\in \I}v_i$ are the elements  $Q_i (\sum_{i'} Q_{i'})^{-1} v$ in $V_i$. This operation is illustrated in module (1) of Figure  \ref{figblock1}.
An important quantity derived from this recovery is the energy function  $E:\I\rightarrow  \R_+$,
 defined in \eqref{ieieieihhgggg} by $E(i):=[Q_i^{-1}w_i,w_i]$ with $w_{i}:=\Psi_{i}(v)$, and illustrated in module (2). Since, per \eqref{eqjehdejhdbdjkjhss},
$E_{tot}=
 \sum_{i\in \I} E(i)$, where $ E_{tot}:=\|v\|^{2}_{S^{-1}}$ is the total energy
\eqref{ieieieihhgggg},
the function $E$ can be interpreted as performing a decomposition of the total energy  over the set of labels $\I$.
When $\I$ can be identified with the set of vertices of a graph, the values of the $E(i)$ can be used to cut that graph into subgraphs indexed by labels $j\in \J$
and define a relation $i\leadsto j$ mapping $i\in \I$ to its subgraph $j$. This graph-cut operation is illustrated in module (3).
Since, per Section \ref{secenvar}, $E(i)$ is also the mean squared alignment between the model $\xi_i$ and the data $v$, and \eqref{eqkjguyguygy} is a variance decomposition, this clustering operation combines variance/model alignment information (as done with PCA) with the geometric information (as done with mixture models \cite{mclachlan2019finite}) provided by the graph to assign a class $j\in \J$ to each element $i\in \I$. However,
the  relation $i\leadsto j$ may also be obtained through a projection step, possibly ignoring the values of $E(i)$,  as illustrated in module (4) (e.g.~when $i$ is an $r$-tuple $(i_1,i_2,\ldots,i_r)$
 then the truncation/projection map $(i_1,\ldots,i_r)\leadsto (i_1,\ldots,i_{r-1})$
  naturally defines a relation $\leadsto$). As illustrated in module (5), combining the relation
 $\leadsto$  with a sum $\sum_{i\leadsto j}$ produces aggregated covariance operators $Q_j:=\sum_{i\leadsto j}Q_i$, modes $w_j:=\sum_{i\leadsto j} w_i$ and energies $E(j):=\sum_{i\leadsto j} E(i)$ such that for $V_j:=\sum_{i\leadsto j}V_i$, the modes $(w_i)_{i\leadsto j}$ are
 (which can be proven directly or as an elementary application of Theorem \ref{thmakdjhgjhgyuy} in the next section) to be  optimal recovery modes in $\prod_{i\leadsto j} V_i$ given the covariance operators $(Q_i)_{i\leadsto j}$ and the observation of
 $w_j=\sum_{i\leadsto j}w_i$ in $ V_j$. Furthermore, we have $E(j)=[Q_{j}^{-1}w_j,w_j]$. Naturally, combining these elementary modules leads to more complex secondary modules (illustrated in  Figure \ref{figblock2}) whose nesting  produces a
 network  aggregating the fine modes $w_i$ into increasingly coarse modes with the last node corresponding to $v$.

\begin{figure}[h]
        \begin{center}
                        \includegraphics[width=0.75\textwidth]{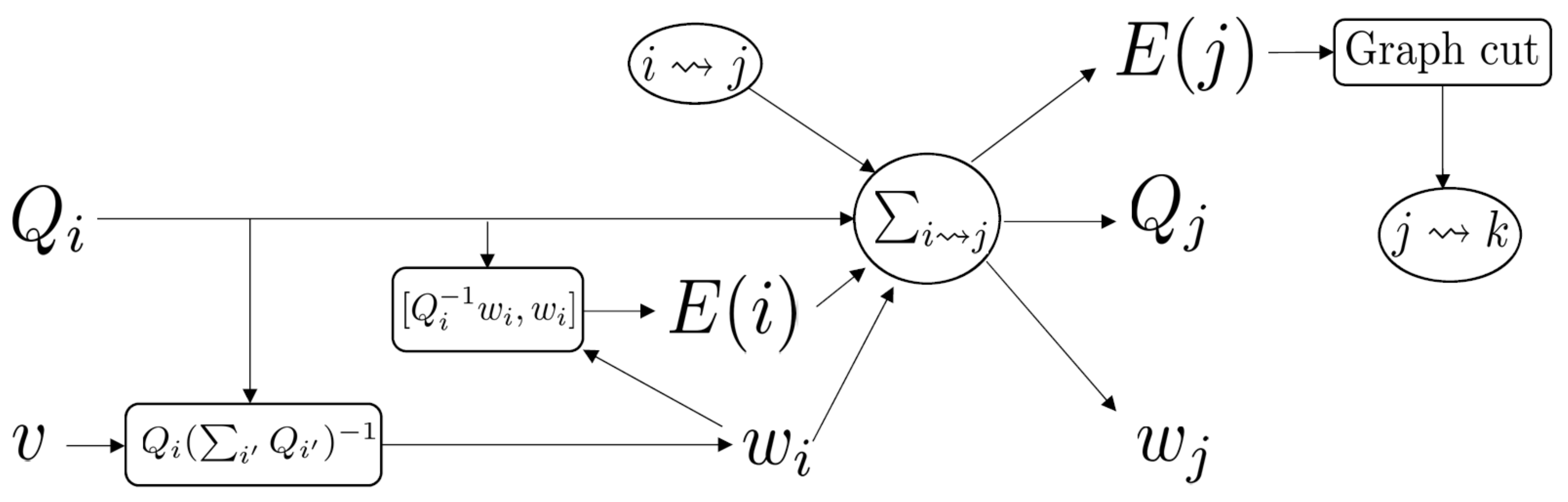}
                \caption{Programming modules derived from the elementary modules of Figure \ref{figblock1}.
}\label{figblock2}
        \end{center}
\end{figure}

\subsection{Hierarchical mode decomposition}\label{secjhgjhg7ggt}
\begin{figure}[h!]
        \begin{center}
                        \includegraphics[width=\textwidth]{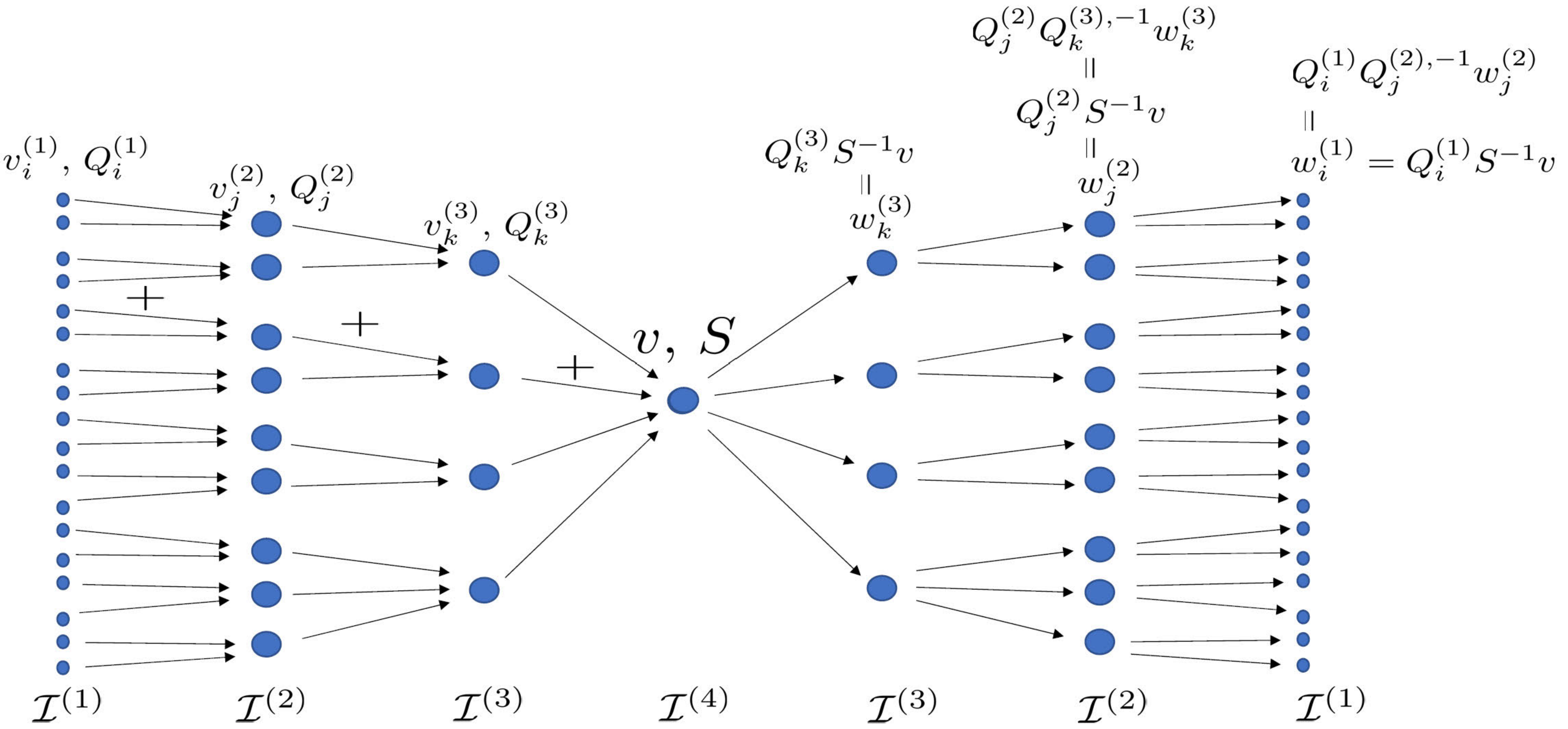}
                \caption{The  generalization of abstract mode decomposition problem of Figure \ref{figmultilevel1}
 to a hierarchy  as described in Section \ref{secjhgjhg7ggt}.}\label{figmultilevel2}
        \end{center}
\end{figure}
We now describe how 
 a hierarchy of mode decomposition/recomposition steps discussed in Section \ref{seckejhdkjdh} naturally produces a hierarchy of labels, covariance operators, subspaces and recoveries (illustrated in Figure \ref{figmultilevel2}) along with important geometries and inter-relationships.  This description will lead to the
 meta-algorithm  Algorithm \ref{alggrseg},  presented in  Section \ref{secljlkjkhuh}, aimed at the production of a KMDNet such as the one illustrated in Figure \ref{figmultilevel2}.
  Section \ref{sectfdecom} will present a practical application to Problem \ref{pb2}.

 Our first step   is to
  generalize the recovery approach of Section \ref{secthpb}
  to the case where
$V$ is the sum of a hierarchy of  linear nested subspaces labeled by a hierarchy of indices, as  defined below.

 \begin{Definition}
\label{defindextree}
 For $q\in \mathbb{N}^*$, let $\I^{(1)},\ldots,\I^{(q)}$ be  finite sets of indices such that $\I^{(q)}=\{1\}$ has only one element. Let $\cup_{l=1}^q \I^{(l)}$ be endowed with a relation $\leadsto$ that is
  (1)  transitive, i.e., $i\leadsto j$ and $j \leadsto k$ implies $i\leadsto k$ (2) directed, i.e.,  $i\in \I^{(s)}$ and $j\in \I^{(r)}$ with $s\geq r$ implies $i\not\leadsto j$ (that is, $i$ does not lead to $j$) and
  (3)  locally surjective, i.e.,  any element
   $j\in \I^{(r)}$ with  $r>1$ has at least one  $i\in \I^{(r-1)}$  such that
    $i\leadsto j$.
 For
 $ 1 \leq k< r \leq q$ and an element   $i \in \I^{(r)}$,
 write $i^{(k)}:=\{j\in \I^{(k)}\mid j\leadsto i\}$ for the level $k$ ancestors of $i$.
\end{Definition}

Let $V^{(k)}_i, \, i\in \I^{(k)}, k\in \{1,\ldots,q\}$, be a hierarchy
 of nested linear subspaces of a separable Hilbert space $V$  such that
 \[V^{(q)}_1=V\]
 and, for each level in the hierarchy $ k\in\{1,\ldots,q-1\}$,
  \begin{equation}\label{eqjkhgkhhj}
 V^{(k+1)}_i=\sum_{j\in i^{(k)}} V^{(k)}_j, \qquad i\in \I^{(k+1)}\, .
 \end{equation}
 Let $\B^{(q)}=V$ and for $k\in \{1,\ldots,q-1\}$, let $\B^{(k)}$ be the product space
  \begin{equation}
\label{eiemjiuriiur0}
\B^{(k)}:=\prod_{i\in \I^{(k)}} V_i^{(k)}\,.
 \end{equation}
For $k < r$ and $j \in \I^{(r)}$, let
\begin{equation}
\label{eiemjiuriiur}
\B_j^{(k)}:=\prod_{i\in j^{(k)}} V_i^{(k)}\,
\end{equation}
and let
\[\Phi^{(r,k)}_{j}:\B_j^{(k)} \rightarrow
 V^{(r)}_j\]
be defined by
\begin{equation}
\label{Phi_j}
\Phi^{(r,k)}_{j}(u):=\sum_{i\in j^{(k)} } u_i, \qquad  u\in \B^{(k)}_j\, .
\end{equation}
Putting these components together as  $\Phi^{(r,k)}=(\Phi^{(r,k)}_{j})_{j\in \I^{(r)}}$,  we  obtain
the multi-linear map
\[\Phi^{(r,k)}:\B^{(k)}\rightarrow  \B^{(r)},\quad 1 \leq k < r \leq q, \]
 defined by
 \begin{equation}
\label{def_Phi}
 \Phi^{(r,k)}(u):=\big(\sum_{i\in j^{(k)}} u_i \big)_{j\in \I^{(r)}}, \qquad   u=(u_i)_{i\in \I^{(k)}}\in \B^{(k)}\,.
 \end{equation}

To put hierarchical metric structure on these spaces,  for $k\in \{1,\ldots,q\}$ and $i\in \I^{(k)}$,
 let
\[Q_i^{(k)}\,:\,V^{(k),*}_i\rightarrow V^{(k)}_i\]
be positive symmetric linear bijections determining
the quadratic norms
\begin{equation}
\label{def_norms}
\|v\|_{V_i^{(k)}}^2=[Q_i^{(k),-1}v,v], \qquad v\in V_i^{(k)},
\end{equation}
 on the $V_i^{(k)}$. Then
for $k\in \{1,\ldots,q\}$,
 let  $\B^{(k)}$ be endowed with the quadratic norm  defined by
 \begin{equation}
\|u\|_{\B^{(k)}}^2=\sum_{i\in \I^{(k)}} \|u_i\|_{ V_i^{(k)}}^2, \qquad  u\in \B^{(k)}\,,
\end{equation}
and, for
 $ k < r \leq q$ and $j\in \I^{(r)}$, let $\B_j^{(k)}:=\prod_{i\in j^{(k)}} V_i^{(k)}$ be
 endowed with the quadratic norm  defined by
\[\|u\|_{\B^{(k)}_{j}}^2=\sum_{i\in j^{(k)}} \|u_i\|_{ V_i^{(k)}}^2, \quad u\in \B^{(k)}_{j}.\]

For $1\leq k < r \leq q$,  the nesting relations
\eqref{eqjkhgkhhj}  imply that
 \[V_i^{(k)} \subset V_j^{(r)},\quad i\in j^{(k)},\,\, j\in \I^{(r)},\]
so that the subset injection
 \begin{equation}
\label{erk}
e_{j,i}^{(r,k)}\,:\,V_i^{(k)}\rightarrow V_j^{(r)}
\end{equation}  is well defined for all
$i\in j^{(k)},\,\, j\in \I^{(r)}$, and since all spaces are complete, they have well-defined adjoints,
 which we  write  \begin{equation}
\label{ekr}
e_{i,j}^{(k,r)}\,:\,V_j^{(r),*} \rightarrow V_i^{(k),*} .
\end{equation}
For $1\leq k < r \leq q$,  $i \in \I^{(k)}$  and $j \in \I^{(r)}$, let
\[\Psi^{(k,r)}_{i,j}:V_{j}^{(r)} \rightarrow V_{i}^{(k)} \]
be defined by
\begin{equation}
\label{Psi_ij}
 \Psi^{(k,r)}_{i,j}(v_{j})= Q_i^{(k)} e_{i,j}^{(k,r)} Q^{(r),-1}_j v_j,\qquad  v_j\in V^{(r)}_j\, ,
\end{equation}
so that, when putting the components together as
\begin{equation}
\label{Psi_j}
\Psi^{(k,r)}_{j}:=(\Psi^{(k,r)}_{i,j})_{i\in j^{(k)}},
\end{equation}
\eqref{eiemjiuriiur} determines the multi-linear  map
\[ \Psi^{(k,r)}_{j}:V^{(r)}_j\rightarrow  \B^{(k)}_j\, .\]

Further collecting components simultaneously over the range and domain as
\[\Psi^{(k,r)}=(\Psi^{(k,r)}_{j})_{j\in \I^{(r)}}\]
we obtain from \eqref{eiemjiuriiur0} the multi-linear map
\[ \Psi^{(k,r)}:\B^{(r)}\rightarrow \prod_{j\in \I^{(r)}}\B^{(k)}_j\]
 defined by
\begin{equation}
\label{def_Psi}
\Psi^{(k,r)}(v)= \big(Q_i^{(k)} e_{i,j}^{(k,r)} Q^{(r),-1}_j v_j\big)_{i\in j^{(k)}}, \qquad   v=(v_j)_{j\in \I^{(r)}}\in \B^{(r)}\,.
\end{equation}

The following condition assumes that the relation $\leadsto$ determines a mapping
$\leadsto:\I^{(k)}\rightarrow \I^{(k+1)}$ for all $k=1,\ldots, q-1$.
\begin{Condition}\label{condpart}
For $k\in \{1,\ldots,q-1\}$, every $i\in \I^{(k)}$ has a unique descendant in $ \I^{(k+1)}$. That is,
 there exists a $j \in \I^{(k+1)}$ with
$i\leadsto j$ and there is no other $j'\in \I^{(k+1)}$ such that $i\leadsto j'$.
\end{Condition}

 Condition \ref{condpart} simplifies the previous results as follows: the subsets
 $(\{i\in j^{(k)}\})_{j\in \I^{(k+1)}}$ form a partition of $\I^{(k)}$,
so that, for $k<r$,  we obtain the simultaneous product structure
\begin{eqnarray}
\label{prod}
\B^{(k)}&=&\prod_{j\in \I^{(r)}}\B^{(k)}_j\nonumber\\
\B^{(r)}&=&\prod_{j\in \I^{(r)}}V^{(r)}_{i}
\end{eqnarray}
so that both
\[\Phi^{(k,r)}:\B^{(k)} \rightarrow \B^{(r)}\]
and
\[\Psi^{(k,r)}:\B^{(r)} \rightarrow \B^{(k)}\]
are diagonal
   multi-linear maps with components
\[\Phi^{(r,k)}_{j}:\B_j^{(k)} \rightarrow
 V^{(r)}_j\]
and \[\Psi^{(k,r)}_{j}:V^{(r)}_j\rightarrow \B_j^{(k)}
 \]
respectively.
Moreover, both maps are {\em linear}
 under the isomorphism between products and external direct sums of vector spaces. For $r>k$, we have the  following
 connections between $\B^{(k)}, \B^{(r)}, V_i^{(k)}$ and $V_j^{(r)}$.
\begin{equation}\label{eqcaseredop}
\text{\xymatrixcolsep{10pc}\xymatrix{
\B^{(k)} \ar@<1ex>[d]^{\Phi^{(r,k)}}     & {V^{(k)}_i}\ar[l]^{\prod_{i\in \I^{(k)}}}\ar[d]^{\sum_{i\in j^{(k)}}}\\
\B^{(r)} \ar@<1ex>[u]^{\Psi^{(k,r)}}          &V_j^{(r)}\ar[l]_{\prod_{j\in \I^{(r)}}}}}
\end{equation}

The following theorem is a consequence of Theorem \ref{thmkjhgjhgyuy}.
 \begin{Theorem}\label{thmakdjhgjhgyuy}
 Assume that Condition \ref{condpart} holds and that
  the $Q^{(k)}_i:V_{i}^{(k),*} \rightarrow :V_{i}^{(k)}$ satisfy the nesting relations
\begin{equation}\label{equyuyg76g6}
Q^{(k+1)}_j=\sum_{i\in j^{(k)}}  e_{j,i}^{(k+1,k)} Q_i^{(k)} e_{i,j}^{(k,k+1)},\quad
j\in \I^{(k+1)},
\end{equation}
for $k\in\{1,\ldots,q-1\}$. Then for
$1\leq k < r \leq q$,
\begin{itemize}
  \item $\Psi^{(k,r)}\circ \Phi^{(r,k)}(u)$ is the minmax recovery of $u\in \B^{(k)}$ given the observation of $\Phi^{(r,k)}(u)\in \B^{(r)}$ using the relative error in $\|\cdot\|_{\B^{(k)}}$ norm as a loss.
 \item  $\Phi^{(r,k)}\circ \Psi^{(k,r)}$ is the identity map on $\B^{(r)}$
\item $\Psi^{(k,r)}:
 (\B^{(r)},\|\cdot\|_{\B^{(r)}})\rightarrow (\B^{(k)},\|\cdot\|_{\B^{(k)}})$ is an isometry.
\item $\Phi^{(k,r), *}:
 (\B^{(r), *},\|\cdot\|_{\B^{(r), *}})\rightarrow (\B^{(k),*},\|\cdot\|_{\B^{(k),*}})$ is an isometry.
\end{itemize}
Moreover we have the following semigroup properties for $1 \leq k < r < s\leq q$:
\begin{itemize}
  \item $\Phi^{(s,k)}=\Phi^{(s,r)}\circ \Phi^{(r,k)}$
\item $\Psi^{(k,s)}=\Psi^{(k,r)}\circ \Psi^{(r,s)}$
\item   $\Psi^{(r,s)}=\Phi^{(r,k)}\circ \Psi^{(k,s)}$
\end{itemize}
\end{Theorem}
\begin{Remark}
The proof of Theorem  \ref{thmakdjhgjhgyuy} also demonstrates that, under its assumptions,  for $1 \leq k < r \leq q$ and $j\in \I^{(r)}$, $\Psi^{(k,r)}_j\circ \Phi^{(r,k)}_j(u)$ is the minmax recovery of $u\in \B^{(k)}_j$ given the observation of $\Phi^{(r,k)}_j(u)\in V^{(r)}_j$ using the relative error in $\|\cdot\|_{\B^{(k)}_j}$ norm as a loss. Furthermore, $\Phi^{(r,k)}_j\circ \Psi^{(k,r)}_j$ is the identity map on $V_j^{(r)}$ and
 $\Psi^{(k,r)}_j:(V^{(r)}_j,\|\cdot\|_{V^{(r)}_j})\rightarrow (\B^{(k)}_j,\|\cdot\|_{\B^{(k)}_j})$ and
$\Phi^{(k,r),*}_j:(V^{(r),*}_j,\|\cdot\|_{V^{(r),*}_j})\rightarrow (\B^{(k),*}_j,\|\cdot\|_{\B^{(k),*}_j})$
 are isometries.
\end{Remark}

\paragraph{Gaussian process regression interpretation}

As in the setting of Section \ref{secjhgjhg7ggt}, for $k\in \{1,\ldots,q\}$,  let
 \[ Q^{(k)}:\B^{(k),*}\rightarrow \B^{(k)}\]
be the
 block-diagonal operator
\[Q^{(k)}:=\diag{(Q_i^{(k)})}_{i\in \I^{(k)}}  \]
 defined by its action
$
Q^{(k)} \phi:=(Q_i^{(k)}\phi_i)_{i\in \I^{(k)}}, \,  \phi\in \B^{(k),*}\,,$
and, as discussed in  Section \ref{sec_game}, write
\[\xi^{(k)}\sim \cN(0,Q^{(k)})\]
 for the centered Gaussian field on $\B^{(k)}$ with covariance operator $Q^{(k)}$.
\begin{Theorem}
\label{thm_uurirn}
Under the assumptions of Theorem \ref{thmakdjhgjhgyuy},
for $1 < k  \leq q$,  the distribution of
 $\xi^{(k)}$ is that of $\Phi^{(k,1)}(\xi^{(1)})$. Furthermore $\xi^{(1)}$ conditioned on $\Phi^{(k,1)}(\xi^{(1)})$ is a time reverse  martingale\footnote{If $\mathcal{F}_{n}$ is a decreasing sequence of sub-$\sigma$ fields of a $\sigma$-field  $\mathcal{F}$ and $Y$ is a $\mathcal{F}$ measurable random variable, then $(X_{n},\mathcal{F}_{n})$, where $E_{n}:=\E[Y|\mathcal{F}_{n}]$ is a reverse martingale, in that 
$\E[X_{n}|\mathcal{F}_{n+1}]=X_{n+1}$} in $k$  and, for $1\leq k < r \leq q$, we have
 \begin{equation}
 \Psi^{(k,r)}(v)=\E\big[\xi^{(k)}\mid \Phi^{(r,k)}(\xi^{(k)})=v\big],\qquad  v\in \B^{(r)}\,.
 \end{equation}
\end{Theorem}

\subsection{Mode decomposition through partitioning and integration}\label{secljlkjkhuh}
  \begin{figure}[h!]
        \begin{center}
                        \includegraphics[width=\textwidth]{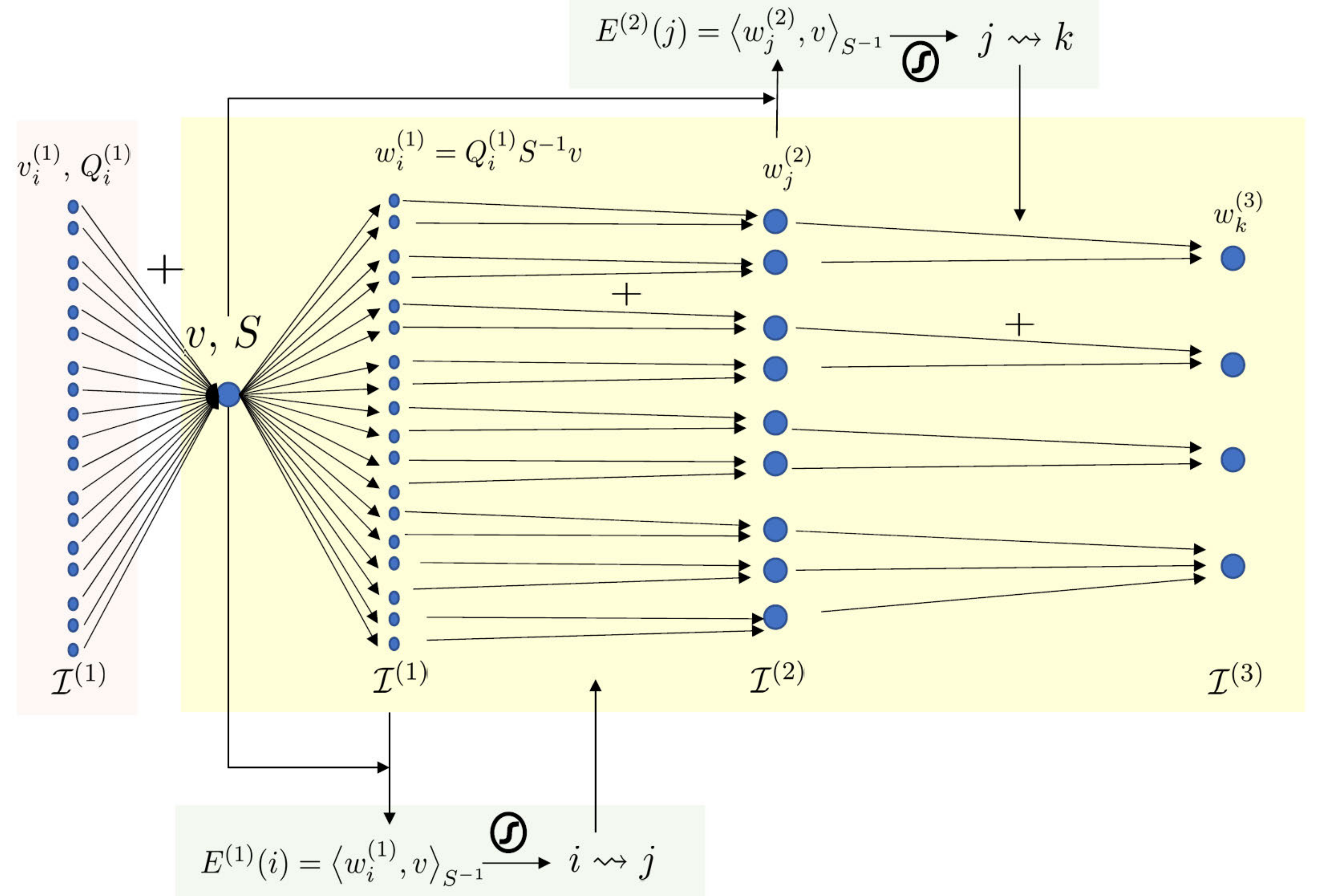}
                \caption{Derivation of the hierarchy from alignments.}\label{figactualcomputation}
        \end{center}
\end{figure}
In the setting of Section \ref{secjhgjhg7ggt},
recall that $\I^{(q)}=\{1\}$ and $V_{1}^{(q)}=V$ so that the index $j$ in
$
\Psi^{(k,q)}_{i,j}$ defined in \eqref{Psi_ij} only has one value $j=1$ and $1^{(k)}=\I^{(k)}$,  and therefore
\begin{equation}
\label{eiejjjguuuuiui}
\Psi^{(k,q)}_{i,1}(v):=Q_i^{(k)} e_{i,1}^{(k,q)} Q^{(q),-1}_1 v,
\qquad
   v\in V,\, i \in \I^{(k)}\, .
\end{equation}
Fix
  a  $v\in V$ and  for $k\in \{1,\ldots,q\}$,  let
\[ E^{(k)}:\I^{(k)} \rightarrow \R,\]
defined  by
\begin{equation}\label{eqhjghjgfjhgfy}
E^{(k)}(i):=\big\|\Psi^{(k,q)}_{i,1}(v)\big\|_{V^{(k)}_i}^2,\quad i\in \I^{(k)},
\end{equation}
 be the alignment energy of the mode $i\in \I^{(k)}$.
Under
the nesting relations
\eqref{equyuyg76g6},  the definition \eqref{def_norms} of the norms and the semigroup properties
of the subspace embeddings \eqref{erk}
  imply
 that
\begin{equation}
\label{Esum}
E^{(k+1)}(i)=\sum_{i'\in i^{(k)}}E^{(k)}(i'), \quad i\in \I^{(k+1)},\, k\in \{1,\ldots,q-1\}\,.
\end{equation}

We will now consider applications where the space $(V,\|\cdot\|_V)$ is known,  and  the spaces
  $(V_i^{(1)},\|\cdot\|_{V_i^{(1)}})$, including their index set $\I^{(1)}$,
are known, but the spaces $(V_j^{(k)},\|\cdot\|_{V_j^{(k)}})$ and  their indices $\I^{(k)}$,
are unknown  for $1<k<q$, as is any relation $\leadsto$ connecting them.
 Instead,  they  will be  constructed  by induction from model/data alignments as illustrated in Figures \ref{figenergynonlinear} and \ref{figactualcomputation} and explained below. In these applications
\[(V,\|\cdot\|_V)=(V^{(q)}_1,\|\cdot\|_{V^{(q)}_1}),\]
 $V=\sum_{i\in \I^{(1)}} V_i^{(1)}$ and the operator
$Q^{(q)}_1:V^{*} \rightarrow V$ associated with the norm  $\|\cdot\|_{V^{(q)}_1}$ is the sum
\begin{equation}\label{equyhuhg6}
Q^{(q)}_1=\sum_{i\in \I^{(1)}}  e_{1,i}^{(q,1)} Q_i^{(1)} e_{i,1}^{(1,q)}\,.
\end{equation}

In this construction we assume that the set of indices  $\I^{(1)}$
 are vertices of a graph $G^{(1)}$, whose  edges provide neighbor relations among the indices. The following
meta-algorithm,
 Algorithm \ref{alggrseg},  forms a general algorithmic framework for
the adaptive determination of the intermediate spaces $(V_j^{(k)},\|\cdot\|_{V_j^{(k)}})$,  their indices $\I^{(k)}$, and a relation $\leadsto$, in such a way that
 Theorem \ref{thmakdjhgjhgyuy} applies. Observe that this meta-algorithm is obtained by combining   the elementary programming modules illustrated in Figures \ref{figblock1} and \ref{figblock2} and discussed in Section \ref{seckejhdkjdh}.
 In the following Section \ref{sectfdecom}, it is  demonstrated
 on a problem in
time-frequency mode decomposition.

\begin{algorithm}[h]
\caption{Mode decomposition through partitioning and integration.}\label{alggrseg}
\begin{algorithmic}[1]
\FOR{$k=1$ to $q-2$}
\STATE\label{step7g} Compute the function $E^{(k)}\,:\,\I^{(k)}\rightarrow \R_+$ defined by \eqref{eiejjjguuuuiui} and
 \eqref{eqhjghjgfjhgfy}.
\STATE\label{step9g}  Use the function $E^{(k)}$ to segment/partition the graph $G^{(k)}$ into subgraphs
$(G^{(k+1)}_j)_{j\in \I^{(k+1)}}$, thereby determining the indices $\I^{(k+1)}$.
  Define the ancestors $j^{(k)}$ of $j\in \I^{(k+1)}$ as the vertices $i\in \I^{(k)}$ of the sub-graph  $G^{(k+1)}_j$.
\STATE\label{step4g} Identify the subspaces $V^{(k+1)}_j$ and the operators $Q^{(k+1)}_j$ through \eqref{eqjkhgkhhj} and \eqref{equyuyg76g6}.
\ENDFOR
\STATE\label{step3g} Recover the modes $(\Psi_i^{(q-1,q)}(v))_{i\in \I^{(q-1)}}$ of $v$.
\end{algorithmic}
\end{algorithm}

\subsection{Application to time-frequency decomposition}\label{sectfdecom}
We will now propose a solution to Problem \ref{pb2}  based on the hierarchical segmentation approach described in Section \ref{secljlkjkhuh}.
 We will employ the  GPR interpretation of Section \ref{subsec64d} and assume
 that the  noisy signal $v=u+v_\sigma$, where $v_\sigma$ is the noise, is the realization of a Gaussian process $\xi$ obtained by integrating Gabor wavelets \cite{gabor1946theory} against white noise.  To that end, for $\tau, \theta\in \R$ and $\omega, \alpha>0$, let
 \begin{equation}\label{eqkjdkjdjedhjks}
\chi_{\tau,\omega,\theta}(t):= \Bigl(\frac{2}{\pi^{3}}\Bigr)^\frac{1}{4} \sqrt{\frac{\omega}{\alpha}}
\cos\bigl(\omega (t-\tau)+\theta\bigr)e^{-\frac{ \omega^2 (t-\tau)^2}{\alpha^2}}, \qquad t \in \R\,,
\end{equation}
be the shifted/scaled Gabor wavelet,
 whose scaling is motivated by the normalization $\int_{-\pi}^{\pi}\int_{\R}\chi_{\tau,\omega,\theta}^2(t)\,dt\,d\theta=1$. See Figure \ref{figgabor} for an illustration of the Gabor wavelets.
Recall  \cite{gabor1946theory} that each $\chi$ is minimally localized in the time-frequency domain (it minimizes the product of standard deviations in the time and frequency domains) and the parameter $\alpha$ is proportional to the ratio between localization in frequency and localization in space.
\begin{figure}[h]
        \begin{center}
                        \includegraphics[width=\textwidth]{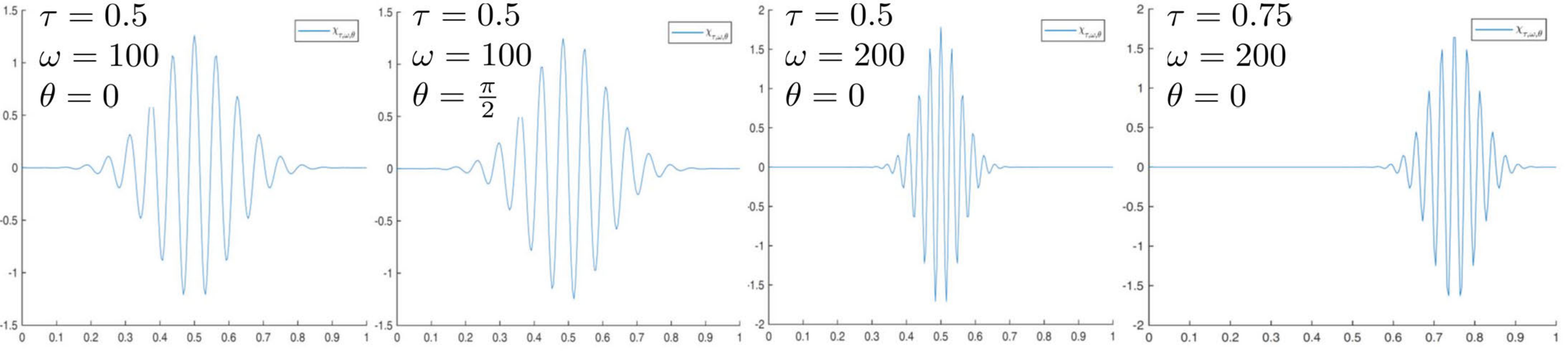}
                \caption{Gabor wavelets $\chi_{\tau,\omega,\theta}$ \eqref{eqkjdkjdjedhjks}
for various parameter values with $\alpha=16$.
}\label{figgabor}
        \end{center}
\end{figure}

Let $\zeta(\tau,\omega,\theta)$ be a
 white noise process on $\R^3$ (a centered GP with covariance function $\E\big[\zeta(\tau,\omega,\theta)\zeta(\tau', \omega',\theta')\big]=\updelta(\tau-\tau')\updelta(\omega-\omega')\updelta(\theta-\theta')$) and let
\begin{equation}
\xi_u(t):=\int_{-\pi}^{\pi} \int_{\omega_{\min}}^{\omega_{\max}} \int_0^1 \zeta(\tau, \omega,\theta)
 \chi_{\tau,\omega,\theta}(t)  d\tau\, d\omega\, d\theta, \quad t\in \R\, .
\end{equation}
Letting,
for each $\tau, \omega$ and $\theta$,
\begin{equation}
\label{eienutnnhgdghd}
K_{\tau,\omega,\theta}(s,t):=\chi_{\tau,\omega,\theta}(s) \chi_{\tau,\omega,\theta}(t),\quad s,t \in \R,
\end{equation}
be the reproducing kernel
associated with the wavelet $\chi_{\tau,\omega,\theta}$, it follows
 that $\xi_u$ is a centered GP with covariance function
\begin{equation}\label{eqjhdjbdehdds}
K_u(s,t)=\int_{-\pi}^{\pi} \int_{\omega_{\min}}^{\omega_{\max}} \int_0^1  K_{\tau,\omega,\theta}(s,t) d\tau \,d\omega\, d\theta,\quad s,t \in \R\, .
\end{equation}

Given $\sigma>0$, let $\xi_\sigma(t)$ be a white noise process on  $\R$ (independent from $\zeta$) of variance $\sigma^2$ (a centered GP with covariance function $\E\big[\xi_\sigma(s)\xi_\sigma(t)\big]=\sigma^2 \updelta(s-t)$) and
let $\xi$, the GP defined by
\begin{equation}
\xi:=\xi_u+\xi_\sigma\, ,
\end{equation}
be used to generate the observed signal $v=u+v_{\sigma}$. 
 $\xi$ is a centered GP with covariance function defined by the kernel
\begin{equation}
\label{K_def}
K:=K_u+K_\sigma
\end{equation}
 with
\begin{equation}\label{eqlkiejdheidnde}
K_\sigma(s,t)=\sigma^2 \updelta(s-t)\,.
\end{equation}

Hence, compared to the setting of  Section \ref{secthpb}, and apart from the mode corresponding to the noise $\xi_\sigma$, the finite number of modes indexed  by $\I$ has been turned into a continuum of modes indexed by
\[
\I:=\bigl \{(\tau,\omega,\theta)\in [0,1]\times [\omega_{\min},\omega_{\max}]\times (-\pi,\pi]\bigr\}\]
  with corresponding one dimensional subspaces
\[V_{(\tau,\omega,\theta)}^{(1)}=\Span\{ \chi_{\tau,\omega,\theta} \},\]  positive operators $Q_{\tau,\omega,\theta}$ defined by the kernels $K_{\tau,\omega,\theta}(s,t)$ and the integral
\[ K_u(s,t)=\int_{-\pi}^{\pi} \int_{\omega_{\min}}^{\omega_{\max}} \int_0^1  K_{\tau,\omega,\theta}(s,t)
d\tau\, d\omega\, d\theta,\quad s,t \in \R\,, \]
 of these kernels \eqref{eqjhdjbdehdds} to obtain a master kernel $K_u$ instead of a sum
\[S=\sum_{i\in \I}e_i Q_i e_i^*\,\]
as in \eqref{S_def}.  Table \ref{table_slide10}  illustrates the time-frequency version of
Table \ref{table_slide6} we have just developed and
the following remark explains the connection between kernels and operators in more detail.

\begin{table*}[h]
\centering
\makebox[1 \textwidth][c]{
\begin{tabular}{ |p{4.5cm}|p{5cm}|p{5.6cm}|  }
\hline
 \hspace{2cm}Mode  &\hspace{2cm} GP  & \hspace{2cm} Kernel\\
 \hline
\vspace{-.6cm}
\[v_{\tau,\omega,\theta}(t)=a_{\tau, \omega, \theta}(t) \chi_{\tau,\omega,\theta}(t)\]
\[a_{\tau, \omega, \theta} \,\,\text{unknown in} \, \, L^{2}\]
&  \vspace{-.6cm}
\[\xi_{\tau, \omega, \theta}(t)=\zeta(\tau, \omega, \theta)\chi_{\tau,\omega,\theta}(t)\]
\[\E[\zeta(\tau, \omega, \theta)\zeta_(\tau', \omega', \theta')]\]
\[=\updelta(\tau-\tau')
\updelta(\omega-\omega')\updelta(\theta-\theta')\]
& \vspace{.5cm}
 $K_{\tau, \omega, \theta}(s,t)=\chi_{\tau,\omega,\theta}(s)\chi_{\tau,\omega,\theta}(t)$\\[-2.6ex]
\hline
\vspace{-.1cm}\hspace{.65cm}$v_{\tau, \omega}=\int_{-\pi}^{\pi}v_{\tau, \omega, \theta} d\theta$
&\vspace{-.1cm}\hspace{.3cm}
$\xi_{\tau, \omega}(t)=\int_{-\pi}^{\pi}\xi_{\tau, \omega, \theta}(t) d\theta$
& \vspace{-.1cm}
 $K_{\tau,\omega}(s,t)=\int_{-\pi}^{\pi}K_{\tau, \omega, \theta}(s,t) d\theta$
 \\[1.4ex]
\hline
\vspace{.1cm}
\hspace{.5cm}$v_{u}=\int\int\int v_{\tau, \omega, \theta}d\tau d\omega d\theta$
&\vspace{.1cm}
$\xi_{u}(t)=\int\int\int \xi_{\tau, \omega, \theta}(t)d\tau d\omega d\theta$
& \vspace{.1cm}
 $K_{u}(s,t)=\int\int\int K_{\tau, \omega, \theta}(s,t)
d\tau d\omega
d\theta$
\vspace{.5cm}
 \\[-3.5ex]
\hline
\vspace{-.3ex} \hspace{1ex}$v_{\sigma}$\,\, \text{unknown white noise}   & \vspace{-.3ex}\hspace{.5cm}  $\E[\xi_{\sigma}(s)\xi_{\sigma}(t)]=\sigma^{2}\updelta(s-t)$
&\vspace{-.3ex} \hspace{.5cm} $K_{\sigma}(s,t)=\sigma^{2}\updelta(s-t)$ \vspace{.6cm}\\[-2ex]
\hline \vspace{-.1cm}
\hspace{1.2cm} $v=v_{u}+v_{\sigma}$   &\vspace{-.1cm} \hspace{1.2cm}  $\xi=\xi_{u}+\xi_{\sigma}$
& \vspace{-.1cm} \hspace{1.2cm} $K=K_{u}+K_{\sigma}$ \\[1ex]
\hline
 \vspace{.1ex} \hspace{3ex}
 $ v_{i}=\int_{A(i)}{v_{\tau,\omega}d\tau d\omega} $
& \vspace{.1ex} \hspace{3ex} $\xi_{i}=\int_{A(i)}{\xi_{\tau,\omega}d\tau d\omega}$
 & \vspace{.1ex} \hspace{3ex}
 $K_{i}=\int_{A(i)}{K_{\tau,\omega}d\tau d\omega}$
\\[2ex]
\hline
\end{tabular}
}
\caption{The time-frequency version of
Table \ref{table_slide6}}
\label{table_slide10}
\end{table*}
\begin{Remark}[{\bf Kernels, operators, and discretizations}]
\label{rm_kernelsoperators}
This kernel mode decomposition framework constructs reproducing kernels $K$
 through the integration of elementary reproducing kernels, but the recovery formula of Theorem  \ref{thmkjhgjhgyuy}
requires the  application of operators, and their inverses,  corresponding to these kernels.
In general, there is no canonical connection between kernels and operators, but here we consider restricting
to the unit interval $[0,1] \subset \R$ in the time variable $t$. Then,  each kernel
$K$ under consideration other than $K_{\sigma}$
 corresponds to the symmetric  positive integral operator
\[ \bar{K}:L^{2}[0,1] \rightarrow L^{2}[0,1]\]
defined by
\[ \bigl(\bar{K}f\bigr)(s):=\int_{0}^{1}{K(s,t)f(t)dt}, \quad s \in [0,1],\,  f  \in L^{2}[0,1]\,.\]
Moreover, these  kernels all have sufficient regularity that
$\bar{K}$ is {\em compact} and therefore not invertible,
 see e.g.~Steinwart and Christmann \cite[Thm.~4.27]{steinwart2008support}.
On the other hand, the operator
\[\bar{K}_{\sigma}:L^{2}[0,1] \rightarrow L^{2}[0,1]\]
 corresponding to the white noise kernel
$K_{\sigma}$ \eqref{eqlkiejdheidnde} is
\[ \bar{K}_{\sigma}=\sigma^{2}I\]
where
\[I:L^{2}[0,1] \rightarrow L^{2}[0,1]\]
is the identity map.
 Since $K=K_{u}+K_{\sigma}$ \eqref{K_def},
the operator $\bar{K}=\bar{K}_{u}+\bar{K}_{\sigma}$
 is a symmetric positive compact operator plus a positive multiple of the identity and therefore it is Fredholm and invertible. Consequently,  we can apply
Theorem \ref{thmkjhgjhgyuy} for the optimal recovery.

In addition,
in numerical applications, $\tau$ and $\omega$ are discretized (using $N+1$ discretization steps) and the integrals in \eqref{eqjehdjehddsd} are replaced by sums over $\tau_k:=k/N$  and $\omega_k:=\omega_{\min}+ \frac{k}{N}(\omega_{\max}-\omega_{\min})$ ($k\in \{0,1,\ldots,N\}$). Moreover, as in Example
\ref{eglsdljdlekd}, the time interval $[0,1]$  is discretized into $M$ points and the corresponding operators
on $\R^{M}$ are
 $\sigma^{2}I$, where $I:\R^{M} \rightarrow \R^{M}$ is the identity,  plus the kernel matrix
$(K_{u}(t_{i},t_{j})\bigr)_{i,j=1}^{M}$ corresponding to
 the  sample points $t_{i}, i=1,\ldots, M$.

For simplicity and conciseness, henceforth we will keep describing the proposed approach in the continuous setting. Moreover, except in Section \ref{sec_con}, we will overload notation and not use the $\bar{K}$ notation, but instead use the same symbol $K$ for a kernel and its corresponding operator.
\end{Remark}

\begin{figure}[h]
        \begin{center}
                        \includegraphics[width=\textwidth]{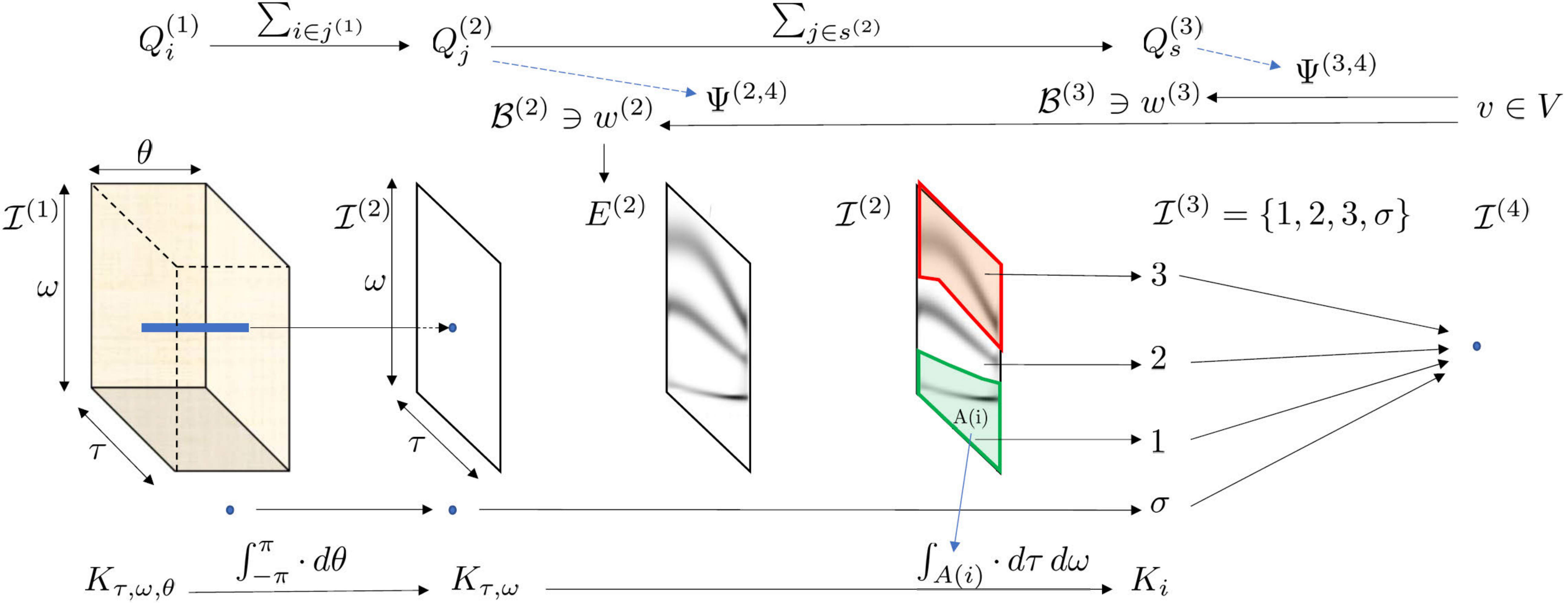}
                \caption{Mode decomposition through partitioning and integration. $q=4$,
 $w^{(3)}:=\Psi^{(3,4)}v$,  $w^{(2)}:=\Psi^{(2,4)}v$, and $\sigma$ corresponds to the noise component.
}\label{fighmd2}
        \end{center}
\end{figure}

We now describe the hierarchical  approach of Section \ref{secljlkjkhuh} to this
time-frequency setting and illustrate it in Figure \ref{fighmd2}.
To that end, we identify
$\I$ with $\I^{(1)}$ so that  \[
\I^{(1)}=\bigl\{(\tau,\omega,\theta)\in [0,1]\times [\omega_{\min},\omega_{\max}]\times (-\pi,\pi]\bigr\}\cup \{\sigma\},\]
where the noise mode has been illustrated in Figure \ref{fighmd2} by adding an isolated point with label  $\sigma$  to each set $\I^{(k)}$ with $k<q=4$.

Although Line \ref{step9g} of  Algorithm  \ref{alggrseg} uses the energy $E^{(1)}$ at level $k=1$ to partition the index set $\I^{(1)}$, the algorithm is flexible with regards to if or how  we use it. In this  particular application we first ignore  the computation of $E^{(1)}$ and straightforward partition
$\I^{(1)}$ into a family of subsets
\[
\I^{(1)}_{\tau, \omega}:=\bigl\{(\tau,\omega,\theta): \theta \in (-\pi,\pi]\bigr\}\cup \bigl\{\sigma\bigr\}, \qquad  (\tau, \omega) \in
[0,1]\times [\omega_{\min},\omega_{\max}],\]
indexed by $\tau$ and $\omega$,
so that the corresponding index set  at level $k=2$ is
\[\I^{(2)}=\bigl\{(\tau,\omega)\in [0,1]\times [\omega_{\min},\omega_{\max}]\bigr\}\cup \bigl\{\sigma\bigr\},\]
  and the ancestors  of $(\tau,\omega, \sigma)$ are
\[(\tau,\omega,\sigma)^{(2)}=\bigl\{(\tau,\omega,\theta):\theta\in (-\pi, \pi]\bigr\} \cup \bigl\{\sigma\bigr\}.\]
 The subspace corresponding to the label $(\tau,\omega)$ is then
\[V_{(\tau,\omega)}^{(2)}=\Span\bigl\{ \chi_{\tau,\omega,\theta} \mid \theta \in (-\pi, \pi]\bigr\}\]
 and, as in  \eqref{equyuyg76g6}, its associated positive operator is characterized by the kernel
\begin{equation}
\label{eqjhedehdius}
K_{\tau,\omega}:=\int_{-\pi}^{ \pi}{K_{\tau,\omega,\theta}d\theta}\, .
\end{equation}
We can evaluate $K_{\tau,\omega}$ using \eqref{eienutnnhgdghd} and \eqref{eqkjdkjdjedhjks} by
defining
\begin{eqnarray}\label{eqjhgjguyggt}
\chi_{\tau,\omega,c}(t)&:=&\Bigl(\frac{2}{\pi}\Bigr)^\frac{1}{4} \sqrt{\frac{\omega}{\alpha}}\cos(\omega (t-\tau))e^{-\frac{ \omega^2 (t-\tau)^2}{\alpha^2}}\,, \quad t \in \R, \nonumber\\
\chi_{\tau,\omega,s}(t)&:=&\Bigl(\frac{2}{\pi}\Bigr)^\frac{1}{4} \sqrt{\frac{\omega}{\alpha}}\sin(\omega (t-\tau))e^{-\frac{ \omega^2 (t-\tau)^2}{\alpha^2}}\,, \quad t \in \R,
\end{eqnarray}
and using the cosine summation formula
to obtain
\begin{equation}
K_{\tau,\omega}(s,t):=\chi_{\tau,\omega,c}(s) \chi_{\tau,\omega,c}(t)+\chi_{\tau,\omega,s}(s) \chi_{\tau,\omega,s}(t)\,.
\end{equation}
  Therefore $V_{(\tau,\omega)}^{(2)}=\Span\{ \chi_{\tau,\omega,c},\chi_{\tau,\omega,s}\}$
and
 \eqref{eqjhdjbdehdds} reduces to
\begin{equation}\label{eqjehdjehddsd}
K_u(s,t)=  \int_{\omega_{\min}}^{\omega_{\max}} \int_0^1  K_{\tau,\omega}(s,t) d\tau \, d\omega\, .
\end{equation}

Using $K:=K_u+K_\sigma$ \eqref{K_def},
let $f$ be the solution of the linear system $\int_0^1 K(s,t)f(t)\,dt =v(s)$, i.e.
\begin{equation}\label{eqbdjehdbdshd}
K f = v\,,
\end{equation}
and let $E(\tau,\omega)$ be the energy of the recovered mode indexed by $(\tau,\omega)$, i.e.
\begin{equation}\label{eqknddnkjednd}
E(\tau,\omega)=\int_0^1 \int_0^1  f(s) K_{\tau,\omega}(s,t) f(t)\,ds\,dt,\qquad (\tau,\omega) \in
[0,1]\times [\omega_{\min},\omega_{\max}]\, .
\end{equation}
Since  $Kf=v$ implies that
\[ v^T K^{-1} v =f^{T}Kf, \]
 it follows that
\begin{equation}
v^T K^{-1} v=   \int_{\omega_{\min}}^{\omega_{\max}} \int_0^1  E(\tau,\omega)\,d\tau\,d\omega +f^T K_\sigma f\,  .
\end{equation}
For  the recovery of the $m$ (which is unknown)  modes using
Algorithm \ref{alggrseg}, at the second level $k=2$ we  use $E(\tau,\omega)$ to partition the time-frequency domain of $(\tau,\omega)$ into $n$ disjoint subsets $A(1), A(2), \ldots  A(n)$.
 As illustrated in Figure \ref{fighmd2},  $n=3$ is determined   from $E(\tau,\omega)$,
 and $\I^{(3)}$ is defined as $\{1,2,\ldots, n, \sigma\}$, the subspace corresponding to the mode $i\not=\sigma$ as $V^{(3)}_i=\Span\{ \chi_{\tau,\omega,c},\chi_{\tau,\omega,s} \mid (\tau,\omega)\in A(i)\}$ and
 the kernel associated with the mode $i\not=\sigma$ as
\begin{equation}\label{eqjehdjehddsdbis}
K_i(s,t)=\int_{(\tau,\omega)\in A(i)} K_{\tau,\omega}(s,t) d\tau\, d\omega, \qquad s,t \in \R\,,
\end{equation}
as displayed in the bottom row in Table \ref{table_slide10},
so that
\[ K_{u}=\sum_{i=1}^{n}{K_{i}}\, .\]
We then apply the optimal recovery formula of Theorem \ref{thmkjhgjhgyuy}
to approximate   the modes of $v_1,\ldots,v_n$ of $u$ from the noisy observation of $v=u+v_\sigma$
(where $v_\sigma$ is a realization of $\xi_\sigma$) with the elements $w_1,\ldots,w_n$ obtained via
\[w_i=K_i K^{-1}v=K_i f,\]
that is, the integration
\begin{equation}\label{eqkeldheiudhiud}
w_i=K_i f\,.
\end{equation}

  \begin{figure}[h]
        \begin{center}
                        \includegraphics[width=\textwidth]{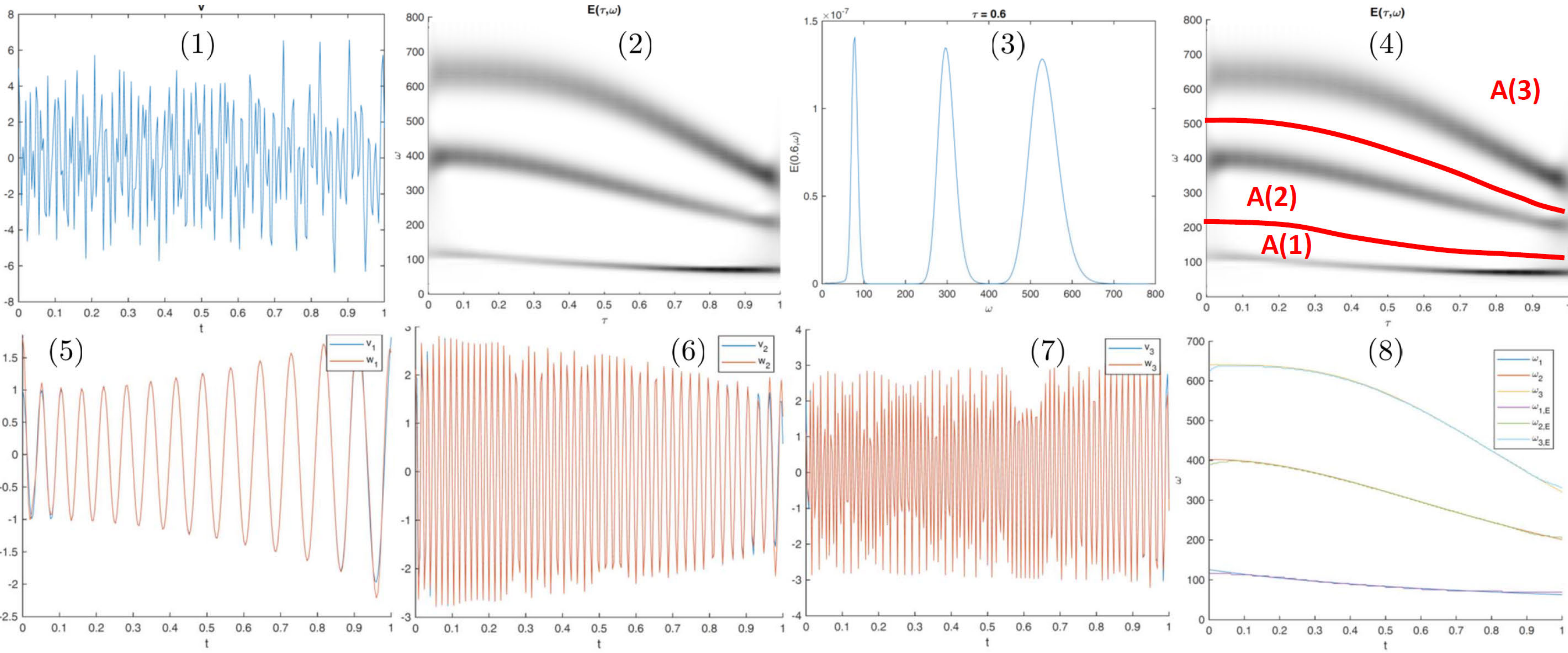}
                \caption{(1) The signal $v=u+v_\sigma$ where $u=v_1+v_2+v_3$, $v_\sigma\sim \cN(0, \sigma^2 \updelta(t-s))$  and $\sigma=0.01$ (2) $(\tau,\omega)\rightarrow E(\tau,\omega)$ defined by \eqref{eqknddnkjednd} (one can identify three stripes) (3) $\omega\rightarrow E(0.6,\omega)$ (4)
Partitioning $[0,1]\times [\omega_{\min},\omega_{\max}]=\cup_{i=1}^3 A(i)$ of the time frequency domain into three disjoint subsets identified from $E$ (5) $v_1$ and its approximation $w_1$ (6) $v_2$ and its approximation $w_2$ (7)
$v_3$ and its approximation $w_3$ (8) $\omega_1,\omega_2,\omega_3$ and their approximations $\omega_{1,E},\omega_{2,E},\omega_{3,E}$.}\label{figunknownf2}
        \end{center}
\end{figure}

Figure \ref{figunknownf2} illustrates a three mode $m=3$ noisy signal, the correct determination of
$n=m=3$,  and the recovery of its modes.
Figure \ref{figunknownf2}.1  displays the total observed signal $v=u+v_{\sigma}$ and the three modes
$v_{1},v_{2}, v_{3} $ constituting $u=v_{1}+v_{2}+v_{3}$
 are displayed in Figures \ref{figunknownf2}.5, 6  and 7,  along with their recoveries
$w_{1},w_{2}$  and $w_{3} $\footnote{The recoveries $w_{i}$ in Figure \ref{figunknownf2}.5,6 and 7,
 are indicated in red and the modes $v_{i}$ of the signal are in blue. When the recovery is accurate, the red recovery blocks the blue and appears red\label{redblue}.}.
 Figure \ref{figunknownf2}.8 also shows  approximations of the instantaneous frequencies obtained as
\begin{equation}\label{eqkjhjhedbdjhdb}
\omega_{i,E}(t):=\operatorname{argmax}_{\omega: (t,\omega)\in A(i)} E(t,\omega)\,.
\end{equation}

\subsection{Convergence of the numerical methods}
\label{sec_con}
This section, which can be skipped on the first reading,
provides  a rough overview of how the empirical approach describe in Remark \ref{rm_kernelsoperators}
 generates convergence results.
 To keep this discussion simple,  we assume that
 the reproducing kernel $K$ is continuous and  its corresponding integral operator
$\bar{K}$
 is injective (the more general case
 is handled by  quotienting with respect to its nullspace).
 Then the
RKHS $H_{K}$ can be described as the image $\bar{K}^{\frac{1}{2}}(L^{2}[0,1]) \subset L^{2}[0,1]$
of the unique positive symmetric square root of  $\bar{K}$ and the map
$\bar{K}^{\frac{1}{2}}:L^{2}[0,1] \rightarrow H_{K}$ is an isometric isomorphism,
 see e.g. \cite[Thm.~17.12]{kress1989linear}. Moreover, by the zero-one law of Luki{\'c} and Beder
 \cite[Thm.~7.2]{lukic2001stochastic}, the Gaussian stochastic process with covariance $K$ has
 its sample paths in $H_{K}$ with probability $1$. Consequently, the  Gaussian stochastic process with covariance
  $K$ will have some approximation error when the observation $v$ is not an element of $H_{K}$. This is the classical situation justifying the employment of Tikhonov regularization, motivating our introduction of the additive
 white noise component to the stochastic model.  However, before we discuss Tikhonov regularization, let us   begin
with the case when $v$ is an element of $H_{K}$.
Then, according to
 Engl, Hanke and Neubauer's \cite[Ex.~3.25]{engl1996regularization}  analysis of
 the least-squares collocation method in \cite[Ex.~3.25]{engl1996regularization}
applied to solving
the operator equation $\bar{K}^{\frac{1}{2}}f=v$,  where $\bar{K}^{\frac{1}{2}}$ is considered as
 $\bar{K}^{\frac{1}{2}}:L^{2}[0,1] \rightarrow H_{K}$,  application of the
 {\em dual least-squares} method of regularization, described in
Engl, Hanke and Neubauer \cite[Ch.~3.3]{engl1996regularization},
reveals that our collocation discretization produces the least-squares collocation approximation $f_{m}$ of the solution $f$ of
 $\bar{K}^{\frac{1}{2}}f=v$,
 i.e.~the minimal norm solution $f_{m}$ of
$Q_{m}\bar{K}^{\frac{1}{2}}f_{m}=Q_{m}v$, where
$Q_{m}: H_{K} \rightarrow  H_{K} $ denotes the $H_{K}$-orthogonal projection
onto the span $\mathcal{Y}_{m}$ of the representers $\Phi_{x_{j}} \in H_{K}$ of the point evaluations at
the collocation points
$x_{j}$ (i.e.~we have $\langle w, \Phi_{x_{j}} \rangle_{H_{K}} =w(x_{j})$,
$w \in H_{K}$,  $j=1,\ldots m$).
Moreover,  \cite[Thm.~3.24]{engl1996regularization} asserts that
the resulting solution $f_{m}$ satisfies $f_{m}=P_{m}f$ where
$P_{m}:L^{2}[0,1] \rightarrow L^{2}[0,1]$ is the orthogonal projection onto
$\bar{K}^{\frac{1}{2}, *}\mathcal{Y}_{m} $.
 Quantitative analysis of the convergence  of $f_{m}$ to $f$ is then a function of the strong convergence of $P_{m}$ to the identity operator and can be assessed in terms of
 the expressivity of the set of representers $\Phi_{x_{j}}$.
  For $v$ not an element of $H_{K}$, Tikhonov regularization is applied together with  least-squares collocation  as in  \cite[Ch.~5.2]{engl1996regularization}.

\section{Additional programming modules and squeezing}
The KMDNets described in Section \ref{sechgjjhjbvf} not only introduce hierarchical structures
to implement nonlinear estimations using linear techniques, but can also be thought of as
 a sparsification technique whose goal is to reduce the computational complexity of solving the corresponding GPR problem, much like the sparse methods have been invented for GPR discussed in Section \ref{sec_GPR}. The primary difference is that, whereas those methods generally use a set of inducing points determining a low rank approximation and then choose the location of those points to optimize its approximation,
here we utilize the  the landscape of the energy  function $E:\mathcal{I}\rightarrow \R_{+}$, defined in \eqref{ieieieihhgggg}
 and analyzed in Proposition \ref{propeqejdkejdde},  interpreted as  
{\em alignment energies}  near \eqref{eqkjguyguygy}. In this section, 
 this analogue of sparse methods will be further developed for the KMDNets using the energy alignment landscape to  further develop programming modules which improve the efficacy and accuracy of the reconstruction. For another application of the alignment energies in model construction,
 see  Hamzi and Owhadi \cite[Sec.~3.3.2]{hamzi2020learning} where it is used to estimate
the optimal  time lag of a ARMA-like time series model.

In the  approach described in  Section \ref{secljlkjkhuh}, $\I^{(k)}$ was partitioned into subsets $(j^{(k)})_{j\in \I^{(k+1)}}$ and the $Q_i^{(k)}$ were integrated
(that is, summed over or  average-pooled)
  using \eqref{eqjkhgkhhj} and \eqref{equyuyg76g6}
in Line \ref{step4g}  of Algorithm \ref{alggrseg}, over each subset to obtain the $Q_j^{(k+1)}$.
 This partitioning approach
 can naturally be generalized to a {\em domain decomposition} approach by letting the subsets be non-disjoint and such that, for some $k$,  $\cup_{j\in \I^{(k+1)}} j^{(k)}$ forms a strict subset\footnote{Although  the results of Theorem \ref{thmakdjhgjhgyuy} do not hold true under this general domain-decomposition, those of Theorem \ref{thmkjhgjhgyuy} remain true between levels $k$ and $q$ (in particular, at each level $k$ the $v_i^{(k)}$ are optimal recovered modes given the $Q_i^{(k)}$ and the observation $v$).}
 of $\I^{(k)}$ (i.e. some $i\in \I^{(k)}$ may not have descendants).
We will now generalize the relation
 $\leadsto$ so as to  (1) not satisfy Condition \ref{condpart}, that is, it does not define a map (a label $i$ may have multiple descendants) 
 (2) be non directed, that is, not satisfy Definition \ref{defindextree} (some $j\in \I^{(k+1)}$ may have descendants in $\I^{(k)}$)  and (3) enable loops.

With this generalization the proposed framework is closer (in spirit) to an object oriented programming language than to a meta-algorithm. This is consistent with
what Yann LeCun in his recent lecture at the SIAM Conference on Mathematics of Data Science (MDS20) \cite{LeCun} has stated;  paraphrasing him: ''The types of architectures people use
nowdays are not just chains of alternating  linear and  pointwise nonlinearities, they are more like
 programs now."
We will therefore describe it as such via the introduction of additional elementary programming modules and illustrate the proposed language by programming increasingly efficient networks for mode decomposition.

\subsection{Elementary programming modules}
\begin{figure}[hbt!]
        \begin{center}
                        \includegraphics[width=\textwidth]{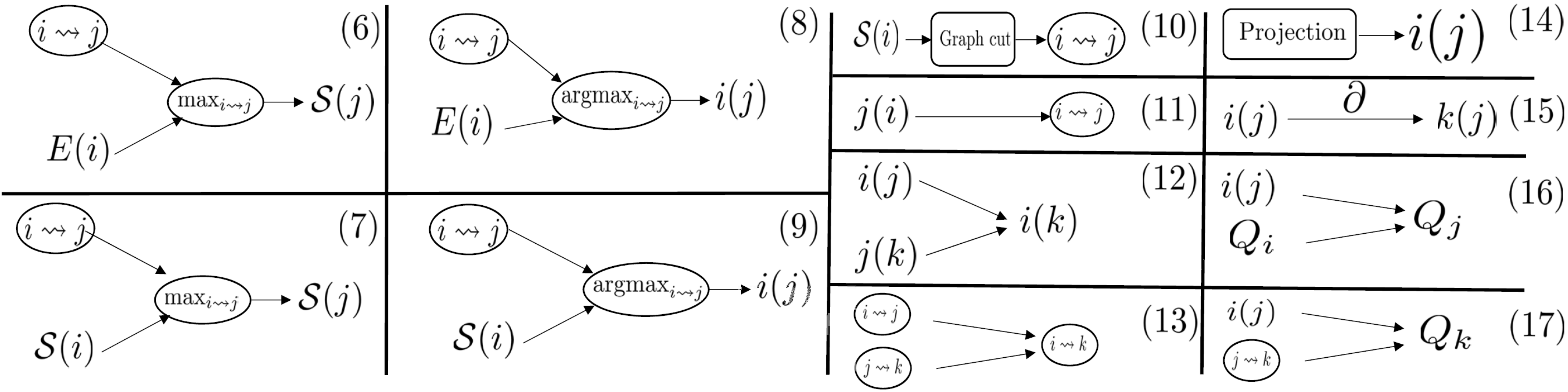}
                \caption{Elementary programming modules.
}\label{figblock3}
        \end{center}
\end{figure}
We will now introduce new  elementary programming modules in addition to the five illustrated in Figure \ref{figblock1} and discussed in Section \ref{seckejhdkjdh}. These new modules
 are illustrated in Figure \ref{figblock3},  beginning with module (6). Here they will be discussed abstractly but forward reference to specific examples..
The first module (module (6)) of Figure \ref{figblock3}  replaces the average-pooling operation to the define the energy $E$ by a max-pool operation.
More precisely module (6) combines a  relation $i \leadsto j$  with an energy $E$ to produce a  {\em max-pool energy} via
\begin{equation}
\mathcal{S}(j)=\max_{i\leadsto j} E(i)\,,
\end{equation}
where $i\leadsto j$ here is over  $i$ from the previous level to that of $j$. In what follows we will adhere to this semantic convention.
As shown in module (7), this combination can also be performed starting with a  {\em max-pool energy}, i.e.~module (7) combines a relation $i \leadsto j$  with a max-pool energy $\mathcal{S}$ at one level to produce a  {\em max-pool energy} at the next level via
\begin{equation}
\mathcal{S}(j)=\max_{i\leadsto j} \mathcal{S}(i)\,.
\end{equation}
Maximizers can naturally be derived from this max-pooling operation and modules (8) and (9)  define $i(j)$ as the maximizer (or the set of maximizers if non-unique) of the energy or the max-pool energy. More precisely module (8) combines a relation $i\leadsto j$  with an energy function $E(i)$ to produce
\begin{equation}
i(j)=\operatorname{argmax}_{i\leadsto j} E(i)\,,
\end{equation}
and module (9)\footnotemark combines a relation $i\leadsto j$  with a max-pool energy function $\mathcal{S}(i)$ to produce
\begin{equation}
i(j)=\operatorname{argmax}_{i\leadsto j} \mathcal{S}(i)\,.
\end{equation}
\footnotetext{The description of the remaining modules (10)-(17), which can be skipped on first reading, is as follows. Similarly to module (3) of Figure \ref{figblock1}, module (10) of Figure \ref{figblock3}  combines the max-pool energy $\mathcal{S}$
with a graph operation to produce the ancestor-descendant relation $i\leadsto j$. We will show that module (10) leads to a  more robust domain decomposition than module (3) due  to its insensitivity to domain discretization.
Module (11) uses the functional dependence $j(i)$ to define the relation $i \leadsto j$.
Module (12) expresses the transitivity of function dependence, i.e. it combines  $j(i)$  and $k(j)$ to produce $k(i)$.
Similarly, module (13) expresses the transitivity of the  relation $\leadsto$, i.e.~$i\leadsto j$ and $j\leadsto k$ can be combined to produce $i\leadsto k$. Module (14) (analogously to module (4)) uses an injection
 step to define a functional dependence $i(j)$ (e.g.~for the time-frequency application in
 Figure \ref{figcross1}, if  $\J$ is the set
 of $(\tau,\omega')$
    and $\I$ is that of $(\tau,\omega)$ the injection $\iota:\I\cap \J\rightarrow  \I$
defines a  functional dependence $i(j)$).
 Module (15) uses a functional dependence $i(j)$ to produce another functional dependence $k(j)$
 (e.g.~for the time-frequency-phase application in Figures \ref{figcross2} and \ref{figblock5}, we can define the functional dependence $(\tau,\omega')(\tau,\omega)$ from the functional dependence $(\tau,\omega,\theta)(\tau,\omega)$ via $\omega'(\tau,\omega)=\partial_\tau \theta(\tau,\omega)$).
Module (16) utilizes the functional dependence $i(j)$ to produce a pullback covariance operator
 $Q_j:=Q_{i(j)}$ ($:=\sum_{i\in i(j)}Q_i$ if $i(j)$ is a set-valued rather than a single-valued mapping).
Module (17) combines a functional dependence $i(j)$ with a relation $j\leadsto k$ to produce a covariance operator $Q_k$ (e.g. for the time-frequency-phase application of Figures \ref{figcross2} and \ref{figblock5},
 for $i=(\tau,\omega,\theta)\in \I^{(1)}$ and $j=(\tau,k)\in \I^{(4)}$ where the index $k$ is the mode index, the functional dependence $i(j)$ defines through \eqref{eieburimg}
estimated phases $\theta_{k,e}(\cdot)$ which can then be substituted  for $\theta(\cdot)$
  in the  kernel  $K(s,t)=e^{-|s-t|^2/\gamma^2}\big(\cos(\theta(s))\cos(\theta(t)+ \sin(\theta(s))\sin(\theta(t)\bigr)$, producing for each mode index $k$ a kernel with corresponding operator $Q_{k}$).
}
\subsection{Programming the network}\label{seckdljhdjkd}

Programming of the network is  achieved by assembling the modules of Figures \ref{figblock1} and \ref{figblock3} in a manner  that (1) $v$ is one of the inputs of the network and (if the network is used for mode decomposition/pattern recognition) (2) the  modes $v_m$ are one of the outputs of the network. As with any interpretable programming language
avoiding inefficient coding and bugs remains important. We will now use  this language to program KMDNets.

\subsection{Squeezing}

We will now present an interpretation and a variant (illustrated in Figure \ref{shmd}) of the synchrosqueezing transform due Daubechies et al.~\cite{DaubechiesMaes,daubechies2011synchrosqueezed} (see Section \ref{sec_synchro} for a description), in 
the setting of KMDNets, and thereby initiate its GP regression version.
We will demonstrate that this version generalizes to the case where the basic waveform is non-periodic and/or unknown.
 We  use the setting and  notations of Section \ref{sectfdecom}.

Let $f$ be the solution of $Kf=v$  \eqref{eqbdjehdbdshd} and let
\begin{equation}\label{eqksweddnd}
E(\tau,\omega,\theta):=\int_0^1 \int_0^1  f(s) K_{\tau,\omega,\theta}(s,t) f(t)\,ds\,dt
\end{equation}
 be the energy of the mode indexed by $(\tau,\omega,\theta)$.
For $(\tau,\omega)\in [0,1]\times [\omega_{\min},\omega_{\max}]$, write
\begin{equation}\label{eqthetate}
\theta_e(\tau,\omega):=\operatorname{argmax}_{\theta \in (-\pi, \pi] }E(\tau,\omega,\theta)  \,.
\end{equation}
Since the definitions \eqref{eqkjdkjdjedhjks} of
$\chi_{\tau,\omega,\theta} $ and \eqref{eqjhgjguyggt} of $\chi_{\tau,\omega,c}$ and $\chi_{\tau,\omega,s}$,
together with the cosine summation formula,
 imply that
\[\chi_{\tau,\omega,\theta}(t) =\frac{1}{\sqrt{\pi}}\bigl( \chi_{\tau,\omega,c}(t) \cos(\theta)-
\chi_{\tau,\omega,s}(t) \sin(\theta) \bigr), \quad t \in \R,\,  \]
it follows that, if we define
\begin{eqnarray}
\label{oeoeiiiie}
W_c(\tau,\omega)&:=&\int_0^1 \chi_{\tau,\omega,c}(t) f(t) \,dt\nonumber\\
 W_s(\tau,\omega)&:=&\int_0^1 \chi_{\tau,\omega,s}(t) f(t) \,dt\,,
\end{eqnarray}
 we obtain
\begin{equation}
\label{eoeijhurung}
\int_0^{1}
\chi_{\tau,\omega,\theta}(t) f(t)\,dt=\frac{1}{\sqrt{\pi}}
\bigl(\cos(\theta)W_c(\tau,\omega)-\sin(\theta) W_s(\tau,\omega)\bigr)\,.
\end{equation}
 Consequently, we deduce from \eqref{eqksweddnd}  and \eqref{eienutnnhgdghd} that
\begin{equation}
\label{eoeniuir}
E(\tau,\omega,\theta)= \frac{1}{\pi}\big(\cos(\theta)W_c(\tau,\omega)-\sin(\theta) W_s(\tau,\omega)\big)^2 \,.
\end{equation}
It follows that, when either $W_c(\tau,\omega)\neq 0$ or  $W_s(\tau,\omega) \neq 0$, that
\begin{equation}\label{eqjhdbejhdbs}
\theta_e(\tau,\omega)=\operatorname{phase}\big(W_c(\tau,\omega)-i W_s(\tau,\omega)\big)\,,
\end{equation}
where, for a complex number $z$, 
\begin{equation}
\label{def_phase}
\operatorname{phase}(z) :=\theta \in (-\pi, \pi] :\,  z=re^{i\theta}, \,r >0\,.
\end{equation}
Moreover, it follows from \eqref{eqjhedehdius},  \eqref{eqknddnkjednd}  and \eqref{eqksweddnd}  that
\[ E(\tau,\omega)=\int_{-\pi}^{\pi}{ E(\tau,\omega,\theta)d\theta}\, , \]
so that it follows from  \eqref{eoeniuir}  that
\begin{equation}
E(\tau,\omega)= W^{2}_c(\tau,\omega)+W^{2}_s(\tau,\omega) \,.
\end{equation}

Now  consider the mode decomposition problem with observation
$v =\sum{v_{i}}$ under the assumption that the phases vary much faster than the amplitudes. It follows that
for the determination of frequencies (not the determination of the phases) we can, without 
 loss of generality, assume each mode is of the  form
\begin{equation}
\label{eoioinriir}
v_i(t)=a_i(t)\cos(\theta_i(t))
\end{equation}
 where $a_i$  is slowly varying compared to $\theta_i$.
We will  use the symbol $\approx$ to denote an informal approximation analysis. 
Theorem \ref{thm_Iden} asserts that $K$ is approximately a multiple of the identity operator, so we conclude 
that the solution $f$ to $Kf=v$ in \eqref{eqbdjehdbdshd} is
$f\approx c v$ for some constant $c$. Because we will be performing a phase calculation the constant $c$ is irrelevant and so can be set to $1$, that is we have $f \approx v$ and therefore we can write \eqref{oeoeiiiie} as
\begin{eqnarray}
\label{oooomkiinriir}
W_c(\tau,\omega)&\approx &\int_0^1 \chi_{\tau,\omega,c}(t) v(t) \,dt\nonumber\\
 W_s(\tau,\omega)&\approx&\int_0^1 \chi_{\tau,\omega,s}(t) v(t) \,dt\,.
\end{eqnarray}
For fixed $\tau$, for $t$ near $\tau$,
  \begin{equation}
\label{eoioinriira}
v_i(t)\approx a_i(\tau)\cos((t-\tau)\dot{\theta}_i(\tau)+\theta_i(\tau))
\end{equation}
 so that, 
since the frequencies $\dot{\theta}_{i}$ are relatively large and well separated, it follows
from the nullification effect of integrating cosines of high frequencies,
that for $\omega \approx \dot{\theta}_{i}(\tau)$, \eqref{oooomkiinriir} holds true with
$v_{i}$ instead of $v$ in the right-hand side. 
Because the amplitudes of $v_{i}$ in \eqref{eoioinriir} are slowly varying compared to their frequencies,
it again follows from the nullification effect of integrating cosines of high frequencies, the approximation formula \eqref{eoioinriira}, the representation \eqref{eqjhgjguyggt} of $\chi_{\tau,\omega,c}$ and $\chi_{\tau,\omega,s}$  and the sine and cosine summation formulas,  that
\begin{eqnarray*}
W_c(\tau,\omega)&\approx &a_{i}(\tau)\cos(\theta_i(\tau)) \int_0^1 \chi_{\tau,\omega,c}(t)\cos((t-\tau)\omega)  \,dt\nonumber\\
 W_s(\tau,\omega)&\approx&  -a_{i}(\tau)\sin(\theta_i(\tau))\int_0^1 \chi_{\tau,\omega,s}(t)\sin((t-\tau)\omega)  \,dt\,.
\end{eqnarray*}
Since the representation \eqref{eqjhgjguyggt} of $\chi_{\tau,\omega,c}$ and $\chi_{\tau,\omega,s}$,  and the sine and cosine summation formulas, also imply that  
$\int_0^1 \chi_{\tau,\omega,c}(t)\cos((t-\tau)\omega)  \,dt \approx \int_0^1 \chi_{\tau,\omega,s}(t)\sin((t-\tau)\omega)  \,dt  >0\,,$
it follows that
\[W_c(\tau,\omega)-i W_s(\tau,\omega) \approx a_{i}(\tau) e^{i\theta_i(\tau)} \int_0^1 \chi_{\tau,\omega,c}(t)\cos((t-\tau)\omega)  \,dt\,, \]
so that 
 $\theta_e(\tau,\omega)$, defined in \eqref{eqjhdbejhdbs},  
is an approximation of $\theta_i(\tau)$,  and
\begin{equation}\label{eqkjhedkjhed}
\omega_e(\tau,\omega)=\frac{\partial \theta_e}{\partial \tau}(\tau,\omega)\,
\end{equation}
is an approximation of the instantaneous frequency $\dot{\theta}_i(\tau)$.
\begin{Remark}
In the discrete case, on a set  $\{\tau_{k}\}$ of points, we proceed differently
than in \eqref{eqkjhedkjhed}.
Ignoring for the moment the requirement  \eqref{def_phase} 
that the phase 
$\theta_e(\tau,\omega)$ defined in \eqref{eqjhdbejhdbs}  lies in $(-\pi, \pi]$,
an accurate finite difference approximation $\omega_e(\tau_k,\omega) $ to the frequency
   is determined by
\[\theta_e(\tau_k,\omega)+ \omega_e(\tau_k,\omega) (\tau_{k+1}-\tau_k) = \theta_e(\tau_{k+1},\omega)  .\]
To incorporating the requirement, it is natural to instead
define $\omega_e(\tau_k,\omega)$ as solving
\[e^{i\,\omega_e(\tau_k,\omega) (\tau_{k+1}-\tau_k)} e^{i\theta_e(\tau_k,\omega)} =e^{i\theta_e(\tau_{k+1},\omega)},\] which using
\eqref{eqjhdbejhdbs} becomes
  \[e^{i\,\omega_e(\tau_k,\omega) (\tau_{k+1}-\tau_k)} e^{i \,\operatorname{phase}(W_c(\tau_k,\omega)-i W_s(\tau_k,\omega))}= e^{i \,\operatorname{phase}(W_c(\tau_{k+1},\omega)-i W_s(\tau_{k+1},\omega))}\, ,\]
and has the solution
\begin{equation}\label{eqjedhdhedgjhdg}
\omega_e(\tau_k,\omega)=\frac{1}{\tau_{k+1}-\tau_k}\operatorname{atan2}
\biggl(\frac{W_c(\tau_{k+1},\omega)W_s(\tau_{k},\omega)-W_s(\tau_{k+1},\omega)W_c(\tau_{k},\omega)}{W_c(\tau_{k+1},\omega)W_c(\tau_{k},\omega)+W_s(\tau_{k+1},\omega)W_s(\tau_{k},\omega)}\biggr)\,,
\end{equation}
where $\operatorname{atan2}$  is Fortran's four-quadrant inverse tangent.
\end{Remark}

\begin{figure}[h]
        \begin{center}
                        \includegraphics[width=\textwidth]{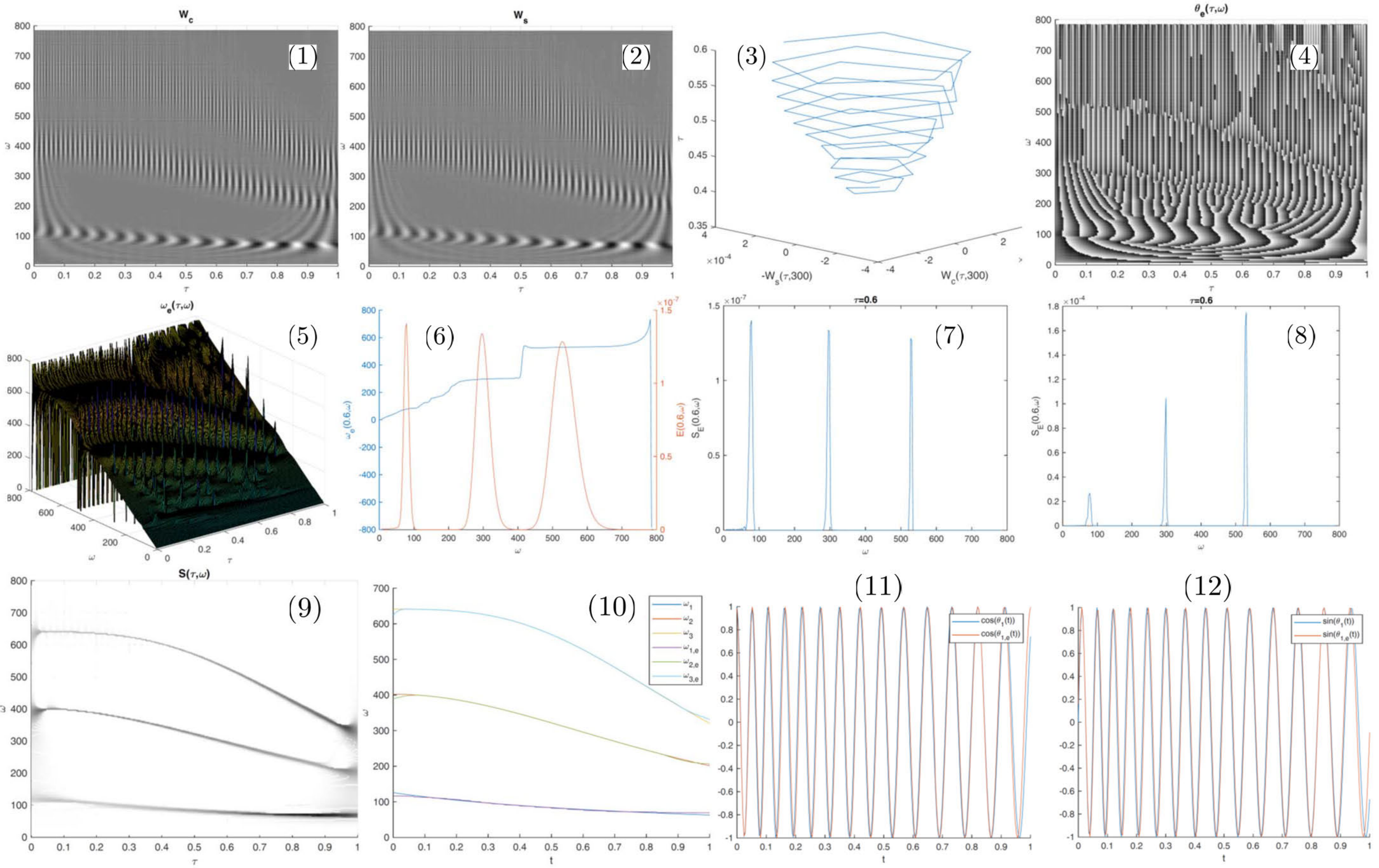}
                \caption{(1) $W_c(\tau,\omega)$ (2) $W_s(\tau,\omega)$ (3) $\tau \rightarrow (W_c(\tau,300),W_s(\tau,300),\tau)$ (4)
$(\tau,\omega)\rightarrow \theta_e(\tau,\omega)$ (5) $(\tau,\omega)\rightarrow \omega_e(\tau,\omega)$ (6) $\omega\rightarrow \omega_e(0.6,\omega)$ and $\omega\rightarrow E(0.6,\omega)$ (7)
$\omega\rightarrow \mathcal{S}(0.6,\omega)$  (8)$\omega\rightarrow \mathcal{S}_E(0.6,\omega)$ (9) $(\tau,\omega)\rightarrow \mathcal{S}(\tau,\omega)$ (10) $t\rightarrow \omega_i(t)$ and $t\rightarrow \omega_{i,e}(t)$ for $i\in \{1,2,3\}$ (11) $t\rightarrow \cos(\theta_1(t))$ and $t\rightarrow \cos(\theta_{1,e}(t))$ (12)  $t\rightarrow \sin(\theta_1(t))$ and $t\rightarrow \sin(\theta_{1,e}(t))$.}\label{figsqueeze}
        \end{center}
\end{figure}
In preparation for illustrating the application of the programming of KMDNets,
 as  a synchrosqueezing algorithm, to the decomposition
problem when $v$ and its modes are as in Figure \ref{figunknownf2},
Figure  \ref{figsqueeze} illustrates
the basic quantities we have just been developing. In particular,
\begin{itemize}
\item The functions
$W_c$ and $W_s$  are shown in Figures \ref{figsqueeze}.1 and \ref{figsqueeze}.2.

\item  The function $\tau \rightarrow (W_c(\tau,300),-W_s(\tau,300))$ is shown in
Figure \ref{figsqueeze}.3 with  $\tau$ the vertical axis. The functions  $\theta_e(\tau,300)$, $E(\tau,300)$
 and $\omega_e(\tau,300)$  are the phase, square modulus and angular velocity of this function.

\item The functions $(\tau,\omega)\rightarrow \theta_e(\tau,\omega)$, $\tau \rightarrow \theta_e(\tau,\omega_{i,E}(\tau) )$ (with $\omega_{i,E}$ defined in \eqref{eqkjhjhedbdjhdb}) and $t\rightarrow \theta_i(t)$
are shown in Figures \ref{figsqueeze}.4, 11 and 12.
Observe that $\tau \rightarrow \theta_e(\tau,\omega_{i,E}(\tau) )$  is an approximation of
 $\tau\rightarrow \theta_i(\tau)$.

\item The functions $(\tau,\omega)\rightarrow \omega_e(\tau,\omega)$, $\omega \rightarrow \omega_e(0.6,\omega)$ and $\tau \rightarrow \omega_e(\tau,\omega_{i,E}(\tau) )$
are shown in Figures  \ref{figsqueeze}.5, 6 and 10. Observe that $\tau \rightarrow \omega_e(\tau,\omega_{i,E}(\tau) )$ is an approximation of the instantaneous frequency $\tau\rightarrow \omega_{i}(\tau)=\dot{\theta}_i(\tau)$ of the mode $v_i$.

\end{itemize}

To describe the remaining components of Figure \ref{figsqueeze} and simultaneously complete
the application of the programming of KMDNets
 as  a synchrosqueezing algorithm and introduce  a {\em max-pool} version of synchrosqueezing, we now introduce
the synchrosqueezed energy $\mathcal{S}_E(\tau,\omega)$ and the  max-pool energy  $\mathcal{S}(\tau,\omega)$: 
Motivated by the synchrosqueezed transform introduced in Daubechies et al.~\cite{daubechies2011synchrosqueezed},
   the synchrosqueezed energy $\mathcal{S}_E(\tau,\omega)$ 
  is obtained by transporting the energy $E(\tau,\omega)$ via the map $(\tau,\omega)\rightarrow (\tau,\omega_{e}(\tau,\omega))$ (as discussed in Section \ref{sec_synchro},  especially near \eqref{eoejiiirir}
), and therefore satisfies
\[\int_{\omega_{\min}}^{\omega_{\max}} \varphi(\omega) \mathcal{S}_E(\tau,\omega)\,d\omega=
\int_{\omega_{\min}}^{\omega_{\max}} \varphi(\omega_e(\tau,\omega')) E(\tau,\omega')\,d\omega'\]
 for all regular test function $\varphi$, i.e.
\begin{equation}
\label{eqjehdgjhdjfh}
\mathcal{S}_E(\tau,\omega) =\lim_{\delta \rightarrow 0} \frac{1}{\delta}\int_{\omega':\omega \leq \omega_e(\tau,\omega')\leq \omega+ \delta } E(\tau,\omega')\,d\omega'\,,
\end{equation}
where numerically  approximate  \eqref{eqjehdgjhdjfh}
by taking $\delta $ small.

Returning to the  application,
the transport of the energy $E(\tau,\omega)$ via the map $(\tau,\omega)\rightarrow (\tau,\omega_{e}(\tau,\omega))$ is illustrated for $\tau=0.6$ by comparing the plots of
 the functions $\omega\rightarrow \omega_{e}(0.6,\omega)$ and  $\omega \rightarrow E(0.6,\omega)$
  in Figure \ref{figsqueeze}.6 with the function
 $\omega \rightarrow \mathcal{S}_E(0.6,\omega)$  shown in Figure  \ref{figsqueeze}.8.
 As in \cite{daubechies2011synchrosqueezed}, the value of $\mathcal{S}_E(\tau,\omega)$ (and  thereby the height of the peaks in Figure \ref{figsqueeze}.8) depends on the discretization and the measure $d\omega'$ used in the integration  \eqref{eqjehdgjhdjfh}.
For example, using a logarithmic discretization or replacing  the  Lebesgue measure $d\omega'$ by
 $\omega' d\omega'$ in \eqref{eqjehdgjhdjfh} will impact the height of those peaks.
 To avoid this dependence on the choice of  measure,
 we define the max-pool energy
\begin{equation}\label{eqklelhdekjdhud}
\mathcal{S}(\tau,\omega)=\max_{\omega':\omega_e(\tau,\omega')=\omega} E(\tau,\omega')\,,
\end{equation}
 illustrated in Figure \ref{figsqueeze}.9.
Comparing  Figures \ref{figsqueeze}.6, 7 and 8, observe that, although both
 synchrosqueezing and max-pooling decrease the width of the peaks of the energy plot $\omega \rightarrow E(0.6,\omega)$,
 only max-squeezing preserves their heights (as noted in \cite[Sec.~2]{daubechies2011synchrosqueezed} a discretization dependent weighting of $d\omega'$ would have to be introduced to avoid this dependence).

  \begin{figure}[h]
        \begin{center}
                        \includegraphics[width=\textwidth]{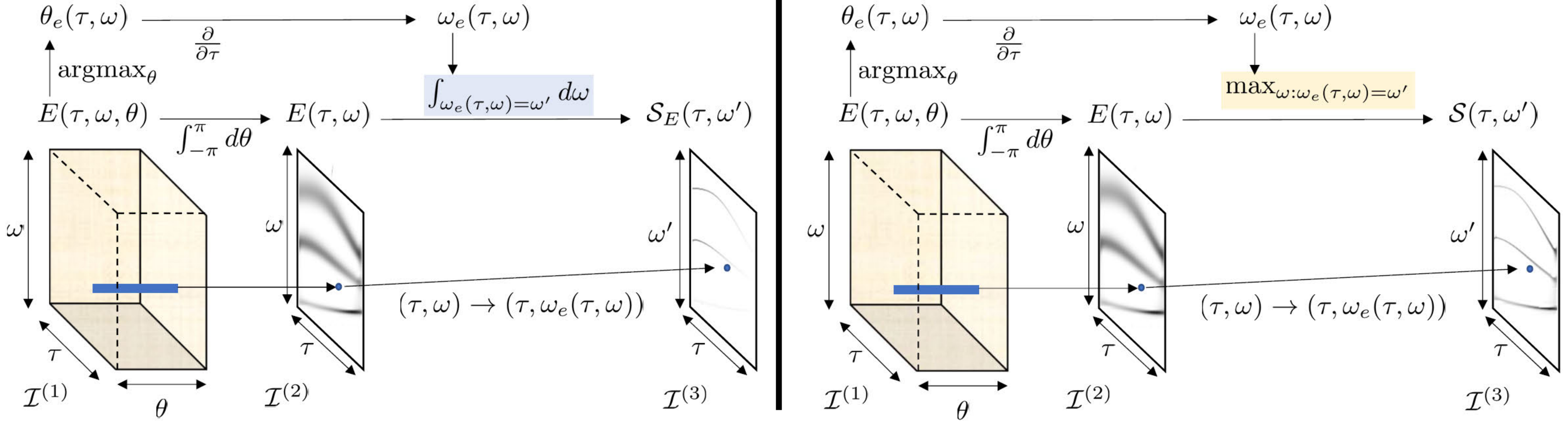}
                \caption{Synchrosqueezed (left) and max-pool (right) energies.}\label{shmd}
        \end{center}
\end{figure}

 Figure \ref{shmd} provides an interpretation of the synchrosqueezed and max-pool energies $\mathcal{S}_E(\tau,\omega)$ and $\mathcal{S}(\tau,\omega)$ in the setting of KMDNet programming, where we note that the left (synchrosqueezed)
and right (max-pool) sub-figures are identical except for the highlighted portions near their top center. In that interpretation $\I^{(1)}$ and $\I^{(2)}$ are, as in Section \ref{sectfdecom} and modulo the noise mode $\sigma$, respectively, the set of time-frequency-phase labels $(\tau,\omega,\theta)\in [0,1]\times
 [\omega_{\min},\omega_{\max}]\times (-\pi, \pi]$ and the set of time-frequency
 labels $(\tau,\omega)\in [0,1]\times [\omega_{\min},\omega_{\max}]$.
 Modulo the noise label $\sigma$, $\I^{(3)}$ is the range of $(\tau,\omega)\rightarrow (\tau,\omega_e(\tau,\omega))$ and the ancestors of $(\tau,\omega')\in \I^{(3)}$ are the  $(\tau,\omega)$ such that $\omega'=\omega_e(\tau,\omega)$.
 Then, in that interpretation, the synchrosqueezed energy is simply the level $3$ energy $E^{(3)}$, whereas $\mathcal{S}(\tau,\omega)$ is the level $3$
  max-pool energy  $\mathcal{S}^{(3)}$.
  Note that the proposed approach naturally generalizes to the case where the periodic waveform $y$ is known and non-trigonometric by simply replacing the cosine function in \eqref{eqkjdkjdjedhjks} by $y$.

\begin{figure}[h!]
        \begin{center}
                        \includegraphics[width=\textwidth]{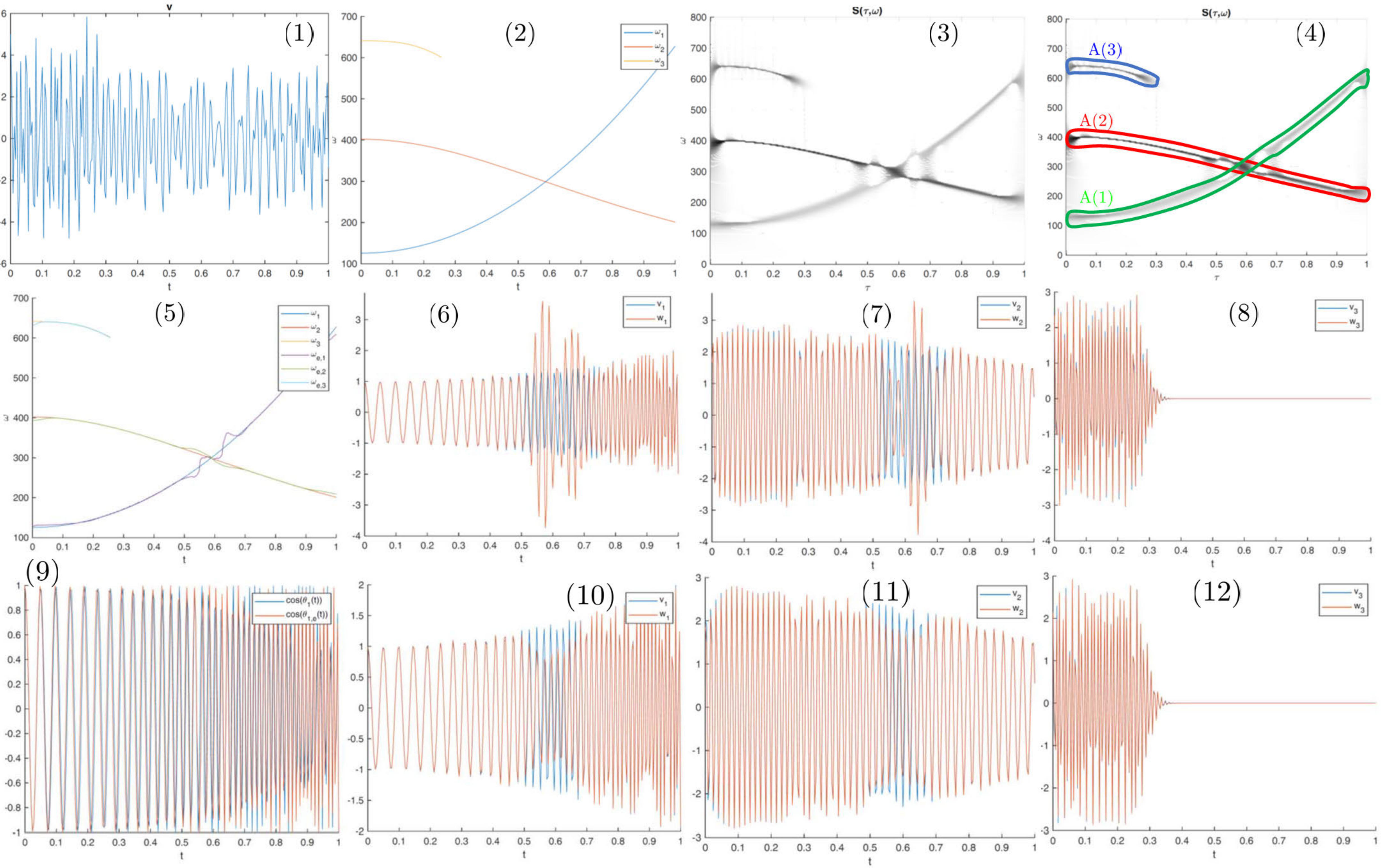}
                \caption{(1) The signal $v=v_1+v_2+v_3+v_\sigma$ where  $v_\sigma\sim \cN(0, \sigma^2 \updelta(s-t))$  and $\sigma=0.01$ (2)  instantaneous frequencies $t\rightarrow \omega_i(t)$ of the modes $i=1,2,3$ (3) $(\tau,\omega)\rightarrow \mathcal{S}(\tau,\omega)$ (4) Sub-domains $A(1), A(2)$ and $A(3)$ of the time-frequency domain (5) approximated instantaneous frequencies $t\rightarrow \omega_{i,e}(t)$ of the modes $i=1,2,3$ (6, 7, 8) $v_1, v_2, v_3$ and their approximations $w_1,w_2,w_3$ obtained from the network shown in Figure \ref{figcross1}  (9) phase $\theta_1$ and its approximation $\theta_{1,e}$ (10, 11, 12) $v_1, v_2, v_3$ and their approximations $w_1,w_2,w_3$ obtained from the network shown in Figure \ref{figcross2}.}\label{hmd3}
        \end{center}
\end{figure}

\subsection{Crossing instantaneous frequencies}
Let us now demonstrate the effectiveness of the max-pooling technique  in its ability to
perform mode recovery when the instantaneous frequencies of the modes cross.
Consider the noisy signal $v$ illustrated in  Figure \ref{hmd3}.1. This signal is composed of $4$ modes, $v=v_1+v_2+v_3+v_\sigma$, where $v_\sigma\sim \cN(0, \sigma^2 \updelta(s-t))$  is a white noise realization with  $\sigma=0.01$.
The modes $v_1,v_2,v_3$ are shown in Figures  \ref{hmd3}.6, 7 and 8, and their instantaneous frequencies $\omega_1, \omega_2, \omega_3$ are shown in Figure \ref{hmd3}.2 (see Footnote \ref{redblue}). Note that $\omega_1$ and $\omega_2$ cross each other around $t\approx 0.6$ and $v_3$ vanishes around $t\approx 0.3$. We  now program two KMDNets  and describe their accuracy in  recovering those modes.

 \begin{figure}[h]
        \begin{center}
                        \includegraphics[width=\textwidth]{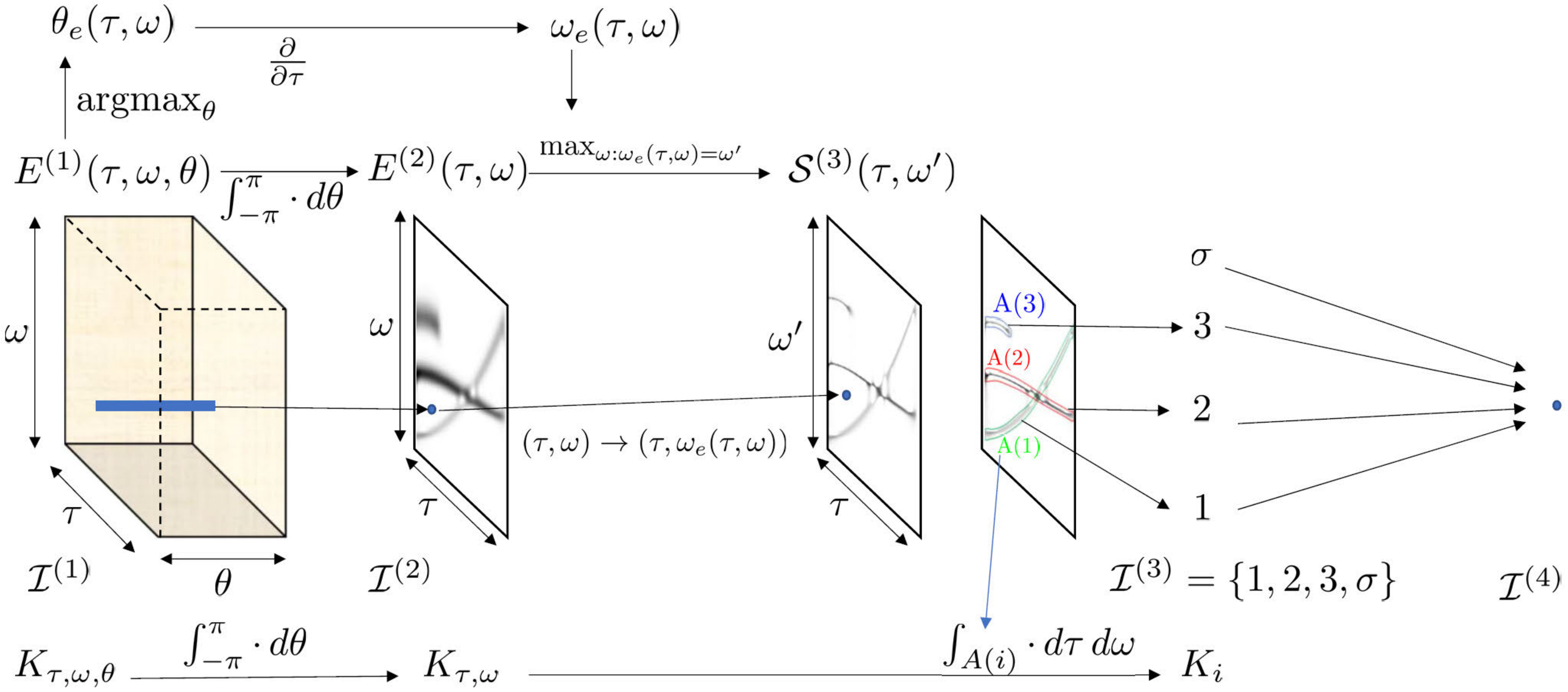}
                \caption{Recovery from domain decomposition. The left-hand side of the figure is
 that of the right-hand side (corresponding to max-pooling)  of Figure \ref{shmd}. The remaining part  is obtained by identifying three subsets $A(1), A(2), A(3)$ of the time-frequency domain $(\tau,\omega)$ and integrating the kernel $K_{\tau,\omega}$ (defined as in \eqref{eqjhedehdius}) over those subsets (as in \eqref{eqjehdjehddsdbis}).} \label{figcross1}
        \end{center}
\end{figure}

The first network, illustrated in Figures \ref{figcross1} and \ref{figblock4} recovers approximations to $v_1,v_2,v_3$ by identifying three subsets $A(1), A(2), A(3)$ of the time-frequency domain $(\tau,\omega)$ and integrating the kernel $K_{\tau,\omega}$ (defined as in \eqref{eqjhedehdius}) over those subsets (as in \eqref{eqjehdjehddsdbis}). For this example, the  subsets $A(1), A(2), A(3)$ are shown in Figure  \ref{hmd3}.4 and  identified as narrow sausages defined by the peaks  of the max-pool energy $\mathcal{S}^{(3)}(\tau,\omega')$ (computed as in \eqref{eqklelhdekjdhud}) shown in  \ref{hmd3}.3). The corresponding approximations  $w_1,w_2,w_3$ (obtained as in \eqref{eqkeldheiudhiud}) of the modes $v_1,v_2,v_3$ are shown in Figures \ref{hmd3}.6, 7 and 8. Note the increased approximation error around $t\approx 0.6$ corresponding to the crossing point between $\omega_1$ and $\omega_2$ and $A(1)$ and $A(2)$.
The estimated instantaneous frequencies $\omega_{i,e}(\tau)=\omega_e\big(\tau,\operatorname{argmax}_{\omega: (\tau,\omega)  \in A(i)}\mathcal{S}^{(3)}(\tau,\omega)\big)$ illustrated in  Figure \ref{hmd3}.5 also show an increased estimation error around that crossing point.

\begin{figure}[h]
        \begin{center}
                        \includegraphics[width=\textwidth]{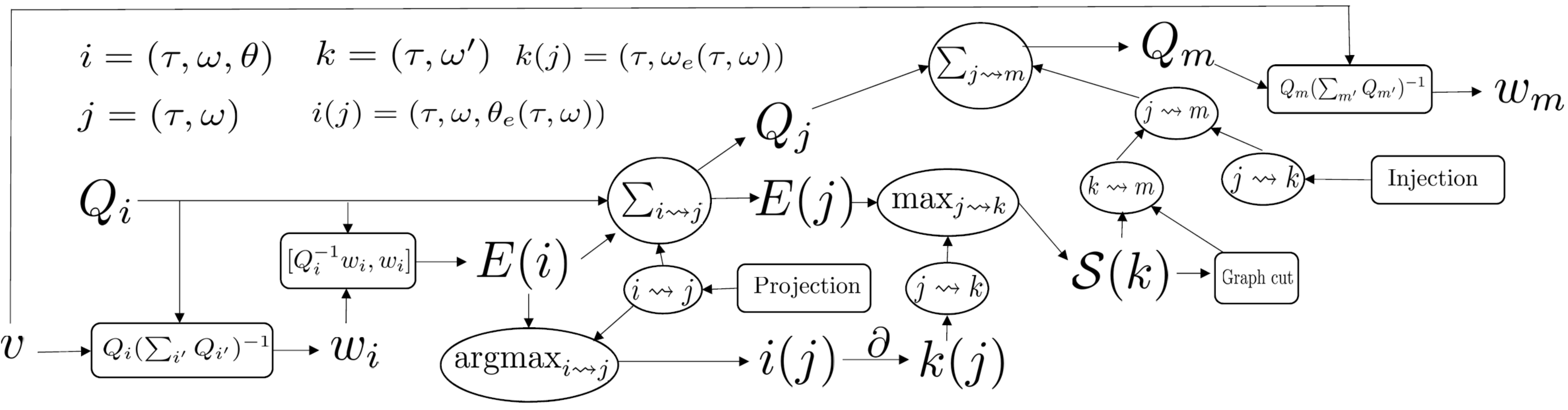}
                \caption{The KMDNet program corresponding to Figure \ref{figcross1}. Upper left  provides the symbolic connections between the indices $i,j,k$ and the time-frequency parameters along with the functional dependencies $i(j)$ and $k(j)$. Beginning with the input $v$ in the lower left, the operators $Q_{i}$ corresponding to the baby kernels $K_{\tau,\omega,\theta}$ are  used to produce optimal recovery estimates $w_{i}$ and the
 corresponding alignment energies $E(i)$. The projection function
$j(i)$ taking $(\tau,\omega, \theta)$ to $(\tau,\omega) $ is the relation $i \leadsto j$ which determines the
integration operation $\int{d\theta}$ indicated as $\sum_{i\leadsto j}$ which then determines summed energies
$E(j):=\sum_{i\leadsto j}{E(i)}$ and covariances $Q_{j}:=\sum_{i\leadsto j}{Q_{i}}$. Moreover, the projection $i \leadsto j$ also determines a max operation
$\arg \max_{\theta }$ which we denote by $\arg \max_{i \leadsto  j}$ and the resulting function
$\theta_{e}(\tau,\omega):=\arg \max_{\theta }{E_{\tau, \omega,\theta}}$, which determines the functional dependency
$i(j)=(\tau,\omega, \theta_{e}(\tau,\omega))$. This function is then differentiated to obtain the functional relation $k(j)=(\tau, \omega_{e}(\tau,\omega))$ where $\omega_{e}(\tau,\omega):=\frac{\partial}
{\partial_{\tau}}\theta_{e}(\tau,\omega)$.  This determines the relation $j \leadsto k$ which determines
the maximization operation $\max_{j \leadsto k}$ that, when applied to the alignment energies $E(j)$,
produces the max-pooled energies $\mathcal{S}(k)$. These energies are then used to determine a graph cut establishing  a relation $k \leadsto m$ where $m$ is a mode index. Combining this relation with the
injection $j \leadsto k$ determines the relation $j \leadsto m$, that then determines the summation
$\sum_{j \leadsto m}$ over the preimages of the relation, thus determining operators
$Q_{m}$ indexed by the mode $m$ by $Q_{m}:=\sum_{j \leadsto m}{Q_{j}}$. Optimal recovery is then applied
to obtain the estimates $w_{m}:=Q_{m}(\sum_{m'}{Q_{m'}})^{-1}$.}
\label{figblock4}
        \end{center}
\end{figure}
\vspace{1in}

 \begin{figure}[h]
        \begin{center}
                        \includegraphics[width=\textwidth]{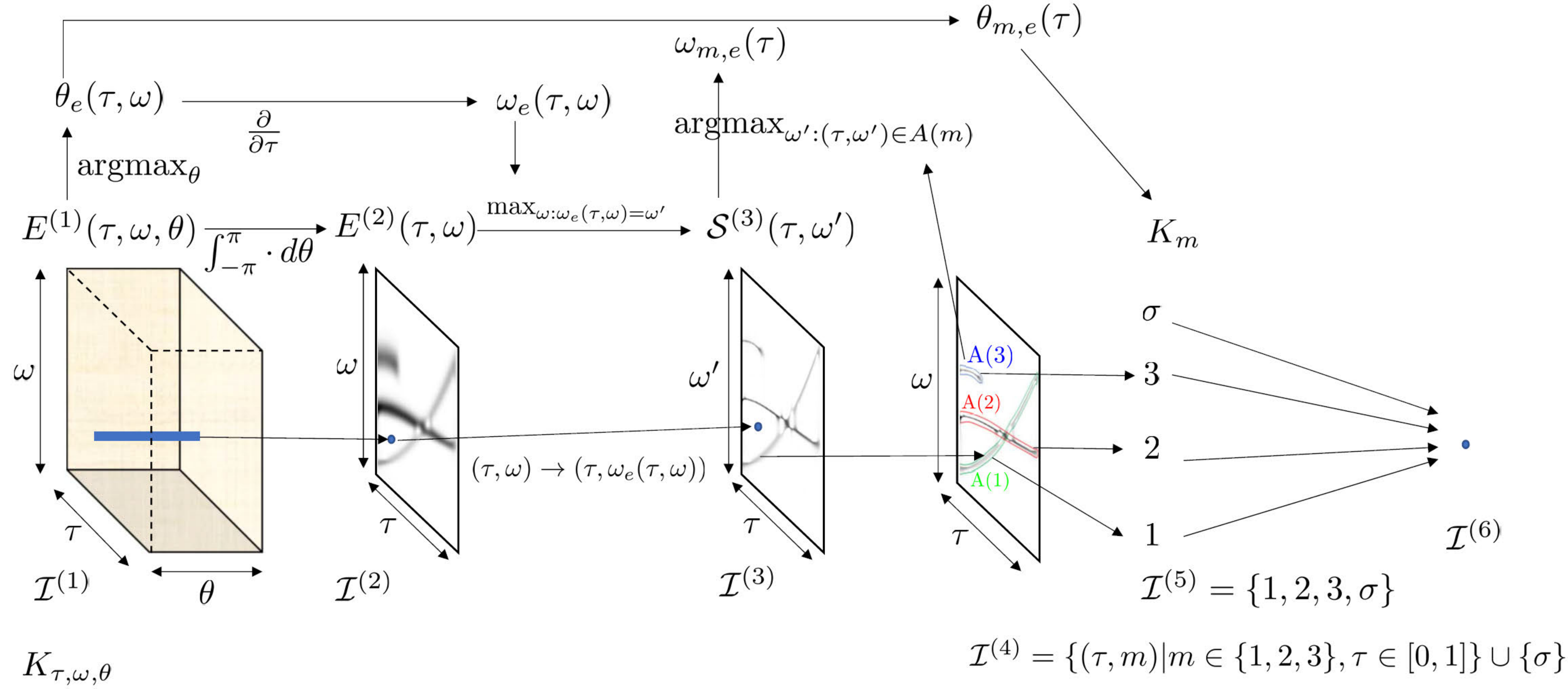}
                \caption{Recovery from instantaneous phases approximations.
The left-hand side of the figure is
 that of the right-hand side (corresponding to max-pooling)  of Figure \ref{shmd} and therefore also that of
Figure \ref{figcross1}, and proceeding to the right as in Figure \ref{figcross1},
 the three subsets $A(1), A(2), A(3)$ of the time-frequency domain $(\tau,\omega)$ and integrating the kernel $K_{\tau,\omega}$ (defined as in \eqref{eqjhedehdius}) over those subsets (as in \eqref{eqjehdjehddsdbis}).
 However, to define the kernels $K_{m}$ for the final optimal recovery, we define $\omega_{m,e}(\tau):=\arg \max_{\omega':(\tau, \omega') \in A(i)}
{\mathcal{S}^{3}(\tau, \omega')}$ to produce the $\theta$ function for each mode $m$ through
$\theta_{m,e}(\tau)=\theta_{e}(\tau,\omega_{m,e}(\tau))$. These functions are inserted into
\eqref{K_sdef} to produce $K_{m}$ and their associated operators $Q_{m}$  which are then used in the finally recovery
$w_{m}=Q_{m}(\sum_{m'}{Q_{m'}})^{-1}v$. }
\label{figcross2}
        \end{center}
\end{figure}
\begin{figure}[h!]
        \begin{center}
                        \includegraphics[width=\textwidth]{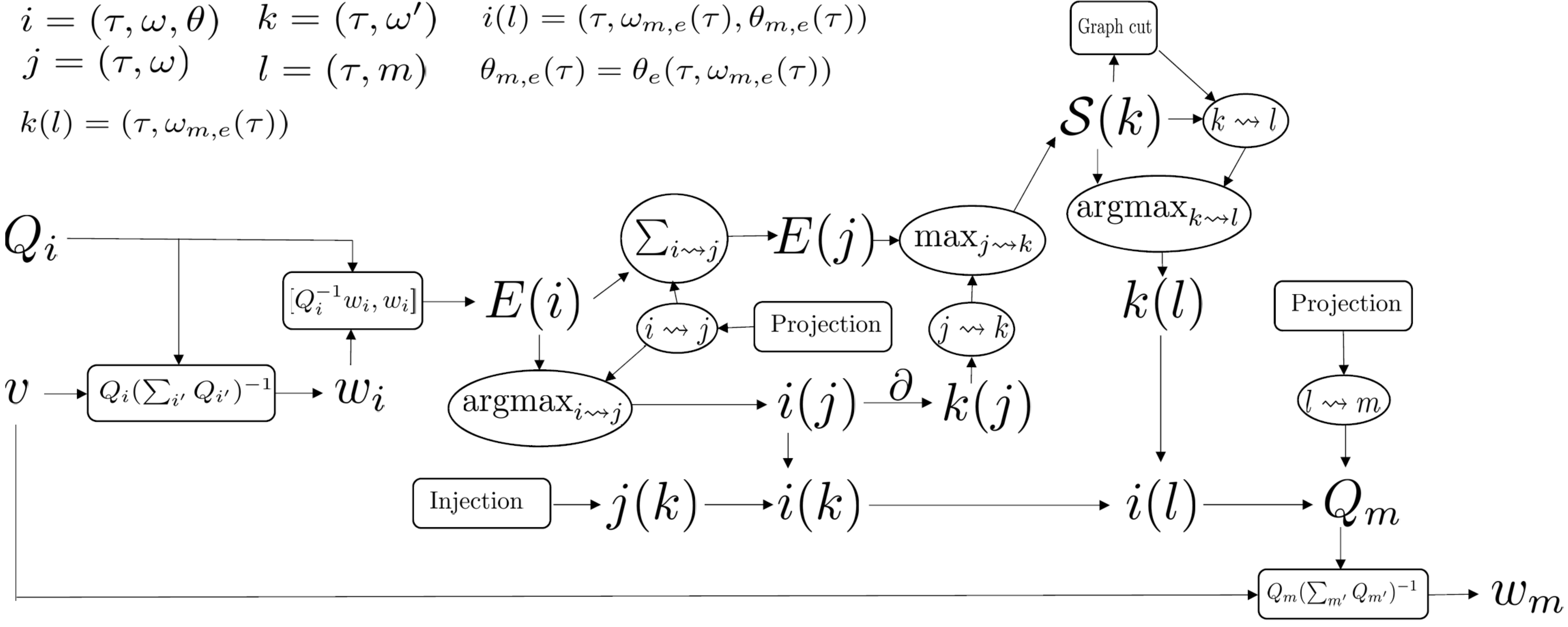}
                \caption{The KMDNet program corresponding to Figure \ref{figcross2}. Upper left  provides the symbolic connections between the indices $i,j,k, l$ and the time-frequency parameters along with the functional dependencies $i(l)$ and $k(l)$ and the definition of $\theta_{m,e}$.
 Beginning with the input $v$ in the lower left, ignoring the bottom two rows for the moment, we begin very much as in Figure \ref{figblock4} moving to the right until the determination of the energies $\mathcal{S}(k)$, the determination of a graph cut and its resulting  $k \leadsto l$, and the resulting $\arg\max$  relation
 $k(l):= \arg \max_{k \leadsto l}{\S(k)}$ which amounts to $k(l)=(\tau, \omega_{m,e}(\tau))$. Returning to the second row from the bottom,
we compose the functional relations of the injection $j(k)$ and the $\arg\max$ function $i(j)$ determined by the relation $i \leadsto j$ and
the energy $E(i)$, to obtain $i(k)$ and then compose this with the argmax function $k(l)$ to produce the functional dependence $i(l)$ defined by
$i(l)=(\tau,\omega_{m,e}(\tau),\theta_{m,e}(\tau))$.
 Using the projection $l \leadsto m$, this determines the function $\theta_{m,e}(\cdot)$ corresponding to the mode label $m$. These functions are inserted into
\eqref{K_sdef} to produce $K_{m}$ and their associated operators $Q_{m}$  which are then used in the finally recovery
$w_{m}=Q_{m}(\sum_{m'}{Q_{m'}})^{-1}v$.
}
\label{figblock5}
        \end{center}
\end{figure}

The second network, illustrated in Figures \ref{figcross2} and \ref{figblock5}, proposes a more robust approach based on the estimates $\theta_{i,e}$ of instantaneous phases $\theta_{i}$ obtained as
\begin{equation}
\label{eieburimg}
\theta_{i,e}(\tau)=\theta_e\big(\tau,\operatorname{argmax}_{\omega: (\tau, \omega)  \in A(i)}\mathcal{S}^{(3)}(\tau,\omega)\big)\,,
\end{equation}
where the $A(i)$ are obtained as in the first network, illustrated in Figure \ref{figcross1}, and
 $\theta_{e}(\tau,\omega)$, used in the definition   \eqref{eieburimg} of
 $\theta_{e,i}(\tau)$, is identified as in \eqref{eqjhdbejhdbs}.
To recover the modes $v_i$, the proposed network proceeds as in Example \ref{eglsdljdlekd} by introducing the kernels
\begin{equation}
\label{K_sdef}
K_i(s,t)=e^{-\frac{(t-s)^2}{\gamma^2}}\big(\cos(\theta_{i,e}(t))\cos(\theta_{i,e}(s))+\sin(\theta_{i,e}(t))\sin(\theta_{i,e}(s))\big)\,,
\end{equation}
with $\gamma =0.2$.
Defining $K_{\sigma}$ as in
\eqref{eqlkiejdheidnde}, the approximations  $w_1,w_2,w_3$  of the modes $v_1,v_2,v_3$, shown
 in  Figures \ref{hmd3}.10, 11 and 12, are obtained as in \eqref{eqkeldheiudhiud} with $f$ defined as the solution of $(K_1+K_2+K_3+K_\sigma) f=v$.
Note that the network illustrated in Figure \ref{figcross2} can be interpreted as the concatenation of $2$ networks. One aimed at estimating the instantaneous phases and the other aimed at recovering the modes based on those phases. This principle of network concatenation is evidently generic.

\clearpage
\section{Alignments calculated in $L^2$}\label{secall2}

The calculation of the energies for our prototypical application was done with respect to the inner product defined by the inverse of the operator associated with $K$ defined in \eqref{K_def}, i.e. the energy of the mode $(\tau,\omega,\theta)$ was defined as $E(\tau,\omega,\theta)=v^T K^{-1} K_{\tau,\omega,\theta} K^{-1}v$  with $K_{\tau,\omega,\theta}$ defined in \eqref{eienutnnhgdghd}. The computational complexity of the method can be accelerated by (1) using the $L^2$ inner product instead of the one defined by $K^{-1}$ (i.e. defining  the energy of the mode $(\tau,\omega,\theta)$ by $E_2(\tau,\omega,\theta)=v^T  K_{\tau,\omega,\theta} v$
 (2) localizing this calculation in a time-window centered around $\tau$ and of width proportional to $1/\omega$.

Our experiments show that this simplification lowers the computational complexity of the proposed approach without impacting its accuracy. Three points justify this observation:
(1) Replacing $E$ by $E_2$ is equivalent to calculating mean-squared alignments with respect to the $L^2$-scalar product instead of the one induced by the inverse of the operator defined by $K$ (2) In the limit where $\sigma \rightarrow \infty$ we have $E\approx \sigma^{-4} E_2$, therefore $E$ and $E_2$ are proportional to each other in the high noise regime (3) If $\omega_{\min}=0$ and $\omega_{\max}=\infty$ then $K_u$ defined by \eqref{eqjhdjbdehdds} is the identity operator on $L^2$.
We will now rigorously show that point (3)
 holds true when we extend the  $\tau$ domain from $[0,1]$ to $\R$ and when the base waveform is trigonometric, and then show in Section \ref{secunivagker} that this results holds true
 independently of the base waveform being used.

Let us recall   the Schwartz class of test functions
\[\mathcal{S}:= \{ f \in C^{\infty}(\R): \sup_{x \in \R}{|x^{m_{1}}D^{m_{2}}f(x)|}< \infty, m_{1}, m_{2} \in
 \mathbf{N} \}\] and the
 confluent hypergeometric function ${}_{1}F_{1}$,
defined by
\[ {}_{1}F_{1}(\alpha, \gamma; z)=1+\frac{\alpha}{\gamma}\frac{z}{1!}+\frac{\alpha(\alpha+1)}{\gamma(\gamma+1)}
\frac{z^{2}}{2!}+\frac{\alpha(\alpha+1)(\alpha+2)}{\gamma(\gamma+1)(\gamma+2)}
\frac{z^{3}}{3!} +  \ldots, \]
see e.g.~see Gradshteyn and Ryzhik \cite[Sec.~9.21]{Gradshteyn2000}.
\begin{Theorem}
\label{thm_Iden}
Consider   extending the definition  \eqref{eqjhdjbdehdds} of the kernel $K_{u}$ so that  the  range of  $\omega$ is extended from $[\omega_{\min},\omega_{\max}]$ to $\R_{+}$ and  that of  $\tau$ is
extended from $[0,1]$ to $\R$, so that
\[
K_{\beta}(s,t)=\int_{-\pi}^{\pi} \int_{\R_{+}} \int_{\R}  K_{\tau,\omega,\theta}(s,t) d\tau \,d\omega\, d\theta,\quad s,t \in \R\, ,
\]
where, as before,
\[
K_{\tau,\omega,\theta}(s,t):=\chi_{\tau,\omega,\theta}(s) \chi_{\tau,\omega,\theta}(t),\quad s,t \in \R,
\]
but where we have introduced a perturbation parameter $0 \leq \beta \leq 1$ defining the Gabor wavelets
 \begin{equation}\label{eqkjdkjdjedhjks2}
\chi_{\tau,\omega,\theta}(t):= \Bigl(\frac{2}{\alpha^{2}\pi^{3}}\Bigr)^\frac{1}{4} \omega^{\frac{1-\beta}{2}}
\cos\bigl(\omega (t-\tau)+\theta\bigr)e^{-\frac{ \omega^2 (t-\tau)^2}{\alpha^2}}, \qquad t \in \R\,,
\end{equation}
defining the elementary kernels.
Defining the scaling constant
\[H(\beta):=2^{\beta-1}\sqrt{\pi}(\sqrt{2}\alpha)^{1-\beta}\Gamma(\frac{\beta}{2}))e^{-\frac{\alpha^{2}}{2}} {}_{1}F_{1}\Bigl(\frac{\beta}{2},\frac{1}{2};\frac{\alpha^{2}}{2}\Bigr),
\]
let
$\mathcal{K}_{\beta}$ denote the integral operator
\[ \bigl(\mathcal{K}_{\beta}f\bigr)(s):=\frac{1}{H(\beta)}\int_{\R}{K_{\beta}(s,t)f(t)dt}\]
associated to the kernel $K_{\beta}$ scaled by $H(\beta)$.
 Then
we have the  semigroup property
 \[\mathcal{K}_{\beta_{1}}\mathcal{K}_{\beta_{2}}f =\mathcal{K}_{\beta_1+\beta_2}f, \quad f \in \mathcal{S}, \quad \beta_1,\beta_2 > 0,  \beta_{1}+\beta_{2} < 1\,,
\]
and
\[
 \lim_{\beta \rightarrow 0}{\bigl(\mathcal{K}_{\beta}f\bigr)(x)}=f(x), \quad x \in \R, \quad f \in \mathcal{S}
\]
where the limit is taken from above.

\end{Theorem}

\section{Universality of  the aggregated kernel}\label{secunivagker}

Let
\[ y(t):=\sum_{-N}^{N}{c_{n}e^{int}}\]
be the Fourier expansion of a general  $2\pi$ periodic complex-valued  waveform, which we will refer to as the {\em base waveform}, and use it to define  wavelets
\[\chi_{\tau,\omega,\theta}(t):=\omega^{\frac{1-\beta}{2}} y\bigl(\omega(t-\tau)+\theta\bigr) e^{-\frac{\omega^{2}}{\alpha^{2}}|t-\tau|^{2}}\]
as in  the $\beta$-parameterized wavelet versions of \eqref{eqkjdkjdjedhjks} in Theorem \ref{thm_Iden},
using the waveform $y$ instead of the cosine. The following lemma evaluates the aggregated kernel
\begin{equation}
\label{K_unknown}
K_{\beta}(s,t):=\Re \int_{-\pi}^{\pi}\int_{\R_{+}}\int_{\R}{\chi_{\tau,\omega,\theta}(s)\chi^{*}_{\tau,\omega,\theta}(t)
 d\tau d\omega d\theta}\, .
\end{equation}
\begin{Lemma}
\label{lem_unknown}
Define the norm
 \begin{equation}
\label{def_norm}
\|y\|^{2}:=\sum_{n=-N}^{N}{e^{-\frac{|n|\alpha^{2}}{2}}|c_{n}|^{2}}
\end{equation}
of the base waveform $y$.
We have
\[ K_{\beta}(s,t)=2\pi|s-t|^{\beta-1}\sum_{n=-N}^{N}{a_{n}(s,t)|c_{n}|^{2}}\]
where
\[a_{n}(s,t)=   \frac{\alpha\sqrt{\pi}}{2\sqrt{2}}(\sqrt{2}\alpha)^{1-\beta}\Gamma(\frac{1-\beta}{2})
e^{-\frac{|n|\alpha^{2}}{2}} {}_{1}F_{1}\Bigl(\frac{\beta}{2};\frac{1}{2};\frac{|n|\alpha^{2}}{2}\Bigr)\,. \]
In particular, at $\beta=0$ we have
\[ K_{0}(s,t)=\alpha^{2}\pi^{2}|s-t|^{-1}\|y\|^{2}\, .
\]
\end{Lemma}

\subsection{Characterizing the norm
$\sum_{n=-N}^{N}{e^{-\frac{|n|\alpha^{2}}{2}}|c_{n}|^{2}}$}
The norm \eqref{def_norm}
of the function
$ y(t):=\sum_{-N}^{N}{c_{n}e^{int}}$  is expressed in terms of its Fourier coefficients
$c_{n}$.  The following lemma evaluates it directly in terms of the function $y$.
\begin{Lemma}
\label{lem_norm}
The norm  \eqref{def_norm} of the function $ y(t):=\sum_{-N}^{N}{c_{n}e^{int}}$
satisfies
\[\|y\|^{2}=\int_{-\pi}^{\pi}\int_{-\pi}^{\pi}{G(t,t')y(t)y^{*}(t')dtdt'}
\]
where
\[G(t,t')= 2\pi
\frac{\sinh(\frac{\alpha^{2}}{2})}{\cosh(\frac{\alpha^{2}}{2})-\cos(t-t')},\quad t,t' \in [-\pi,\pi]\, .\]
\end{Lemma}
\begin{Remark}
The norm \eqref{def_norm} is clearly insensitive to the size of the high frequency (large $n$) components
$c_{n}e^{int}$ of $y$.  On the other hand, the alternative representation
of this norm in Lemma \ref{def_norm} combined with the fact that
 the kernel $G$
 satisfies
\[\frac{\sinh(\frac{\alpha^{2}}{2})}{\cosh(\frac{\alpha^{2}}{2})+1} \leq G(t,t')\leq 2\pi
\frac{\sinh(\frac{\alpha^{2}}{2})}{\cosh(\frac{\alpha^{2}}{2})-1}, \quad t,t' \in [-\pi,\pi]\, ,\]
which, for $\alpha \geq 10$, implies
\[1-10^{-21} \leq G(t,t')\leq 1+10^{-21},
\quad t,t' \in [-\pi,\pi]\, ,\]
implies that
\[ \Bigl|  \|y\|^{2} - \bigl|\int_{-\pi}^{\pi}{y(t)dt}\bigr|^{2}\Bigr| \leq
 10^{-21}\bigl|\int_{-\pi}^{\pi}{|y(t)|dt}\bigr|^{2}
\]
that is, $\|y\|^{2}$ is exponentially close to the square of its integral.
\end{Remark}

\section{Non-trigonometric waveform and iterated  KMD}
\begin{figure}[h!]
	\begin{center}
			\includegraphics[width=\textwidth]{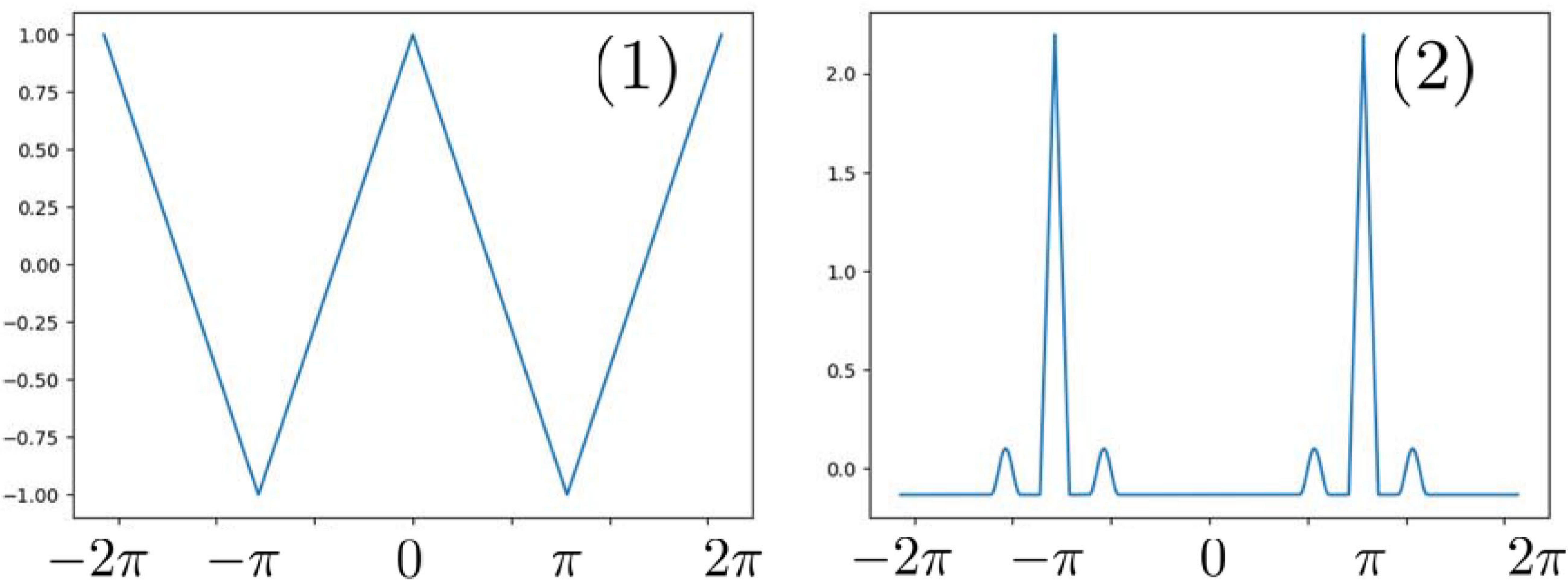}
		\caption{(1) Triangle base waveform (2) EKG base waveform.}\label{waves}
	\end{center}
\end{figure}

\begin{figure}[h]
	\begin{center}
			\includegraphics[width=\textwidth]{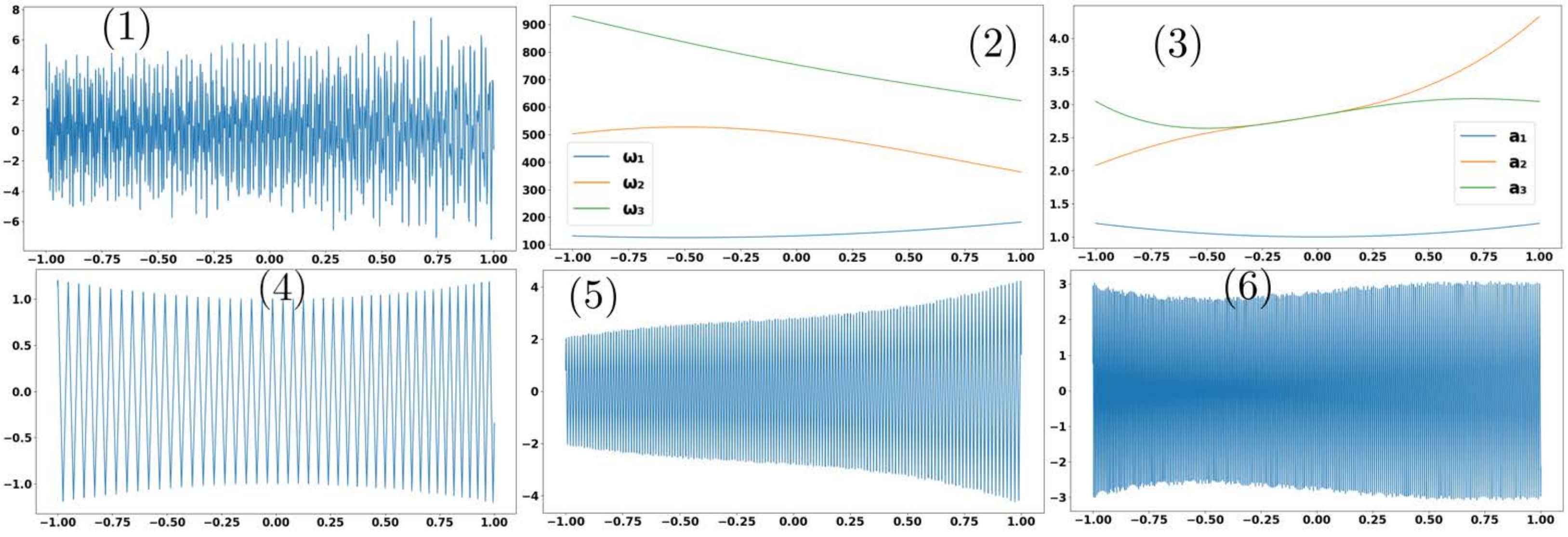}
		\caption{ Triangle base waveform: (1)  Signal $v$ (2) Instantaneous frequencies $\omega_i := \dot{\theta}_i$ (3) Amplitudes $a_i$ (4, 5, 6) Modes $v_1$, $v_2$, $v_3$.}
\label{iter example}
	\end{center}
\end{figure}

\begin{figure}[h]
	\begin{center}
			\includegraphics[width=\textwidth]{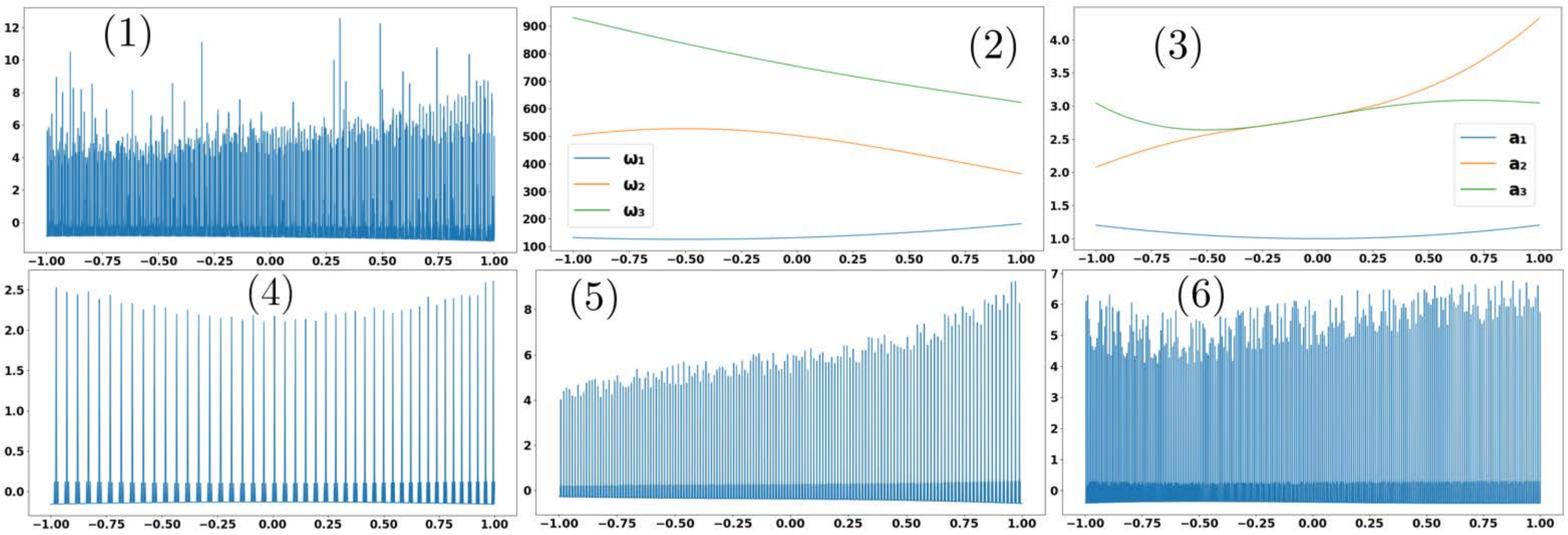}
		\caption{EKG base waveform: (1) Signal $v$ (2)  Instantaneous frequencies $\omega_i := \dot{\theta}_i$ (3) Amplitudes $a_i$ (4, 5, 6) Modes $v_1$, $v_2$, $v_3$.}
\label{iter example bio}
	\end{center}
\end{figure}

We will now consider the mode recovery Problem \ref{pb2} generalized  to the case where the
base waveform of each mode is the same known, possibly non-trigonometric, square-integrable $2\pi$-periodic function $t \rightarrow y(t)$.
The objective of this problem  can be loosely expressed as solving  the following generalization of
Problem \ref{pb2}  towards the resolution of the more general Problem \ref{unk wave pb}.
We now switch  the time domain from $[0,1]$ to $[-1,1]$.
 \begin{Problem}\label{pb3}
For $m\in \mathbb{N}^*$, let $a_1,\ldots,a_m$ be piecewise smooth  functions on $[-1,1]$,  let $\theta_1,\ldots,\theta_m$ be strictly  increasing  functions on $[-1,1]$, and let $y$
 be a square-integrable $2\pi$-periodic function.
Assume that $m$ and the $a_i, \theta_i$ are unknown and the base waveform $y$ is known. 
 We further assume that, for some
$\epsilon>0$, $a_i(t) > \epsilon $ and that $\dot{\theta}_i(t) / \dot{\theta}_j(t) \not\in [1-\epsilon, 1+\epsilon]$ for all $i,
j, t$. Given the observation
$v(t)=\sum_{i=1}^m a_i(t)y\big(\theta_i(t)\big)$ (for $t\in [-1,1]$) recover the modes $v_i:=a_i(t)y\big(\theta_i(t)\big)$.
\end{Problem}

\begin{Example}
Figure \ref{waves} shows two full periods of two  $2\pi$-periodic base waveforms (triangle and EKG) which we will use in our numerical experiments/illustrations.
The EKG (-like) waveform  is  $\big(y_{EKG}(t)-(2\pi)^{-1}\int_0^{2\pi} y_{EKG}(s)\,ds\big)/\|y_{EKG}\|_{L^2([0,2\pi))}$
with $y_{EKG}(t)$ defined on $[0,2\pi)$ as (1) $0.3 - |t - \pi|$ for $|t - \pi| < 0.3 $
 (2) $0.03 \cos^{2}(\frac{\pi}{0.6}(t - \pi + 1)) $ for $|t - \pi + 1| < 0.3$
(3) $0.03 \cos^{2}(\frac{\pi}{0.6}(t - \pi - 1)) $ for $ |t - \pi - 1| < 0.3$ and (4) $0$ otherwise.
\end{Example}

Our approach, summarized in Algorithm \ref{iteralgn} and explained in the following sections,  will be to (1)
 use the max-pool energy $\mathcal{S}$ \eqref{eqklelhdekjdhud}
to obtain, using \eqref{eieburimg}, an estimate of the phase $\theta_\low(t)$ associated with the lowest instantaneous frequency   $\omega_\low=\dot{\theta}_\low$ (as described in Section \ref{subseclowest}) (2) iterate a
{\em  micro-local}  KMD (presented in Section \ref{secmicrolocalkmd})
of the signal $v$ to obtain a highly accurate estimate of the phase/amplitude $\theta_i, a_i$ of their corresponding mode $v_i$ (this iteration can achieve near machine-precision accuracies  when the instantaneous frequencies are separated) (3) Peel off the mode $v_i$ from $v$ (4) iterate to obtain all the modes (5) perform a last micro-local KMD of the signal for higher accuracy.
To illustrate this approach, in the next two sections  we will apply it to the signals $v$
 displayed in Figures \ref{iter example} and \ref{iter example bio},
where the modes of Figure \ref{iter example} are triangular and those of Figure \ref{iter example bio} are EKG.

\subsection{The Micro-local KMD module}\label{secmicrolocalkmd}
We will now describe the micro-local KMD module, which will form the basis for the iterated micro-local KMD algorithm described in Section \ref{sec_microlocalKMDalg}.  It takes a time $\tau$, an estimated phase function of $i$-th mode $\theta_{i,e}$, and a signal $f$, not necessarily equal to $v$.  Suppose the $i$-th mode is of form $v_i(t) = a_i(t) y(\theta_i(t))$ and is indeed a mode within $f$.  The module outputs, (1) an estimate $a(\tau,\theta_{i, e},f)$ of the amplitude $a_i(\tau)$ of the mode $v_i$  and (2) a correction $\delta \theta (\tau,\theta_{i,e},f)$ determining an updated estimate $ \theta_{i,e}(\tau) +\delta \theta (\tau,\theta_{i,e},f)$
 of the estimated mode phase function $\theta_{i,e}$.
  We assume that $a_i$ is strictly positive, that is, 
$a_i(t)\geq a_{0}, t \in [-1,1],$ for some $a_{0}>0$.

Indeed, given $\alpha>0$, $\tau\in [-1,1]$,  differentiable strictly increasing functions
  $ \theta_{0}$ and $ \theta_{e}$ on $[-1,1]$,
 and $n\in \{0,\ldots, d\}$ (we set $d=2$ in applications in this section), let
$ \chi^{\tau, \theta_{e}}_{n, c}$ and  $\chi^{\tau, \theta_{e}}_{n, s}$ be the wavelets defined by
\begin{eqnarray}\label{chi_def}
    \chi^{\tau, \theta_{e}}_{n, c}(t)& := &\cos(\theta_{e}(t))(t - \tau)^n e^{-\big(\frac{\dot{\theta_{0}}(\tau)(t - \tau)}{\alpha}\big)^2} \nonumber \\
    \chi^{\tau, \theta_{e}}_{n, s}(t)& :=& \sin(\theta_{e}(t)) (t - \tau)^n e^{-\big(\frac{\dot{\theta_{0}}(\tau)(t - \tau)}{\alpha}\big)^2}\,,
\end{eqnarray}
and
  let $\xi_{\tau, \theta_{e}}$ be the Gaussian process defined by
\begin{equation}
\label{def_xi}
    \xi_{\tau, \theta_{e}}(t) := \sum_{n = 0}^d \big( X_{n, c} \chi^{\tau, \theta_{e}}_{n, c}(t) + X_{n, s} \chi^{\tau, \theta_{e}}_{n, s}(t) \big) \,,
\end{equation}
where  $X_{n, c}, X_{n, s}$ are independent  $\N(0,1)$ random variables. The function
$\theta_{0}$ will be fixed throughout the iterations whereas the function $\theta_{e}$ will be updated.
Let $f_{\tau}$ be the Gaussian windowed signal defined by
\begin{equation}
\label{def_vwindow}
f_{\tau}(t) = e^{-\big(\frac{\dot{\theta_{0}}(\tau)(t - \tau)}{\alpha}\big)^2} f(t), \quad t \in [-1,1]\,,
\end{equation}
and,
for $(n,j)\in \{0,\ldots,d\}\times \{c,s\}$, let
\begin{equation}
\label{def_Z}
Z_{n,j}(\tau,\theta_{e},f):=\lim_{\sigma \downarrow 0}\E\big[X_{n,j} \big|\xi_{\tau, \theta_{e}} + \xi_\sigma=
f_{\tau}\big]\,,
\end{equation}
where $\xi_\sigma$ is white noise, independent of $\xi_{\tau, \theta_{e}}$, with variance $\sigma^2$.
To compute $Z_{n,j}$, observe that
since both $\xi_{\tau, \theta_{e}}$ and $\xi_{\sigma}$ are Gaussian fields, it follows from
 \eqref{eqkjgdjhed} that
\[\E\big[ \xi_{\tau, \theta_{e}} \big|\xi_{\tau, \theta_{e}} + \xi_\sigma\big] =
 A_{\sigma}(\xi_{\tau, \theta_{e}} + \xi_\sigma)\]
for  the linear mapping
\[ A_{\sigma}= Q_{\tau, \theta_{e}}\bigl(Q_{\tau, \theta_{e}} +\sigma^{2}I\bigr)^{-1}\, ,\]
where $Q_{\tau, \theta_{e}}:L^{2} \rightarrow L^{2}$ is the covariance operator of the Gaussian field
$\xi_{\tau, \theta_{e}}$ and $\sigma^{2}I$ is the covariance operator of $\xi_\sigma$.
Using the characterization of the limit of Tikhonov regularization as the Moore-Penrose inverse,
see e.g.~Barata and Hussein
\cite[Thm.~4.3]{barata2012moore},
along with the orthogonal projections connected with the Moore-Penrose inverse,
we conclude that $\lim_{\sigma \rightarrow 0}{A_{\sigma}}=P_{\chi^{\tau, \theta_{e}}}$,
where $P_{\chi^{\tau, \theta_{e}}}$ is the $L^{2}$-orthogonal projection onto the span
 $\chi^{\tau, \theta_{e}}:=\Span \{\chi^{\tau, \theta_{e}}_{n, c},  \chi^{\tau, \theta_{e}}_{n, s}, n=0,\ldots,d\}$, 
 and therefore
\begin{equation}
\label{Alimit}
 \lim_{\sigma \rightarrow 0}{\E\big[ \xi_{\tau, \theta_{e}} \big|\xi_{\tau, \theta_{e}} + \xi_\sigma\big] }=
P_{\chi^{\tau, \theta_{e}}}(\xi_{\tau, \theta_{e}} + \xi_\sigma)\, .
\end{equation}

Since the definition \eqref{def_xi} can be written
$\xi_{\tau, \theta_{e}}=\sum_{n,j}X_{n,j} \chi^{\tau, \theta_{e}}_{n, j} ,$
summing  \eqref{def_Z}  and using \eqref{Alimit},
  we obtain
\begin{equation}
\label{eieienhugirt}
\sum_{n,j}Z_{n,j}(\tau,\theta_{e},f) \chi^{\tau, \theta_{e}}_{n, j}=P_{\chi^{\tau, \theta_{e}}}f_{\tau}\, .
\end{equation}
Consider the vector function $Z(\tau,\theta_{e},f) \in \R^{2d+2}$ with components $Z_{n,j}(\tau,\theta_{e},f)$,
the $2d+2$ dimensional  Gaussian random vector $X$ with components $X_{n,j}, (n,j)\in \{0,\ldots,d\}\times \{c,s\}$,
and the $(2d+2)\times (2d+2)$ matrix $A^{\tau, \theta_{e}}$ defined by
\begin{equation}
\label{Adef0}
A^{\tau, \theta_{e}}_{(n, j), (n', j')} := \langle \chi^{\tau, \theta_{e}}_{n, j}, \chi^{\tau, \theta_{e}}_{n', j'} \rangle_{L^2[-1,1]}\, .
\end{equation}
 Straightforward linear algebra along with \eqref{eieienhugirt} establish that
the vector  $Z(\tau,\theta_{e},f)$
 can be computed as the solution of the linear system
\begin{equation}
   A^{\tau, \theta_{e}} Z(\tau,\theta_{e},f) =   b^{\tau, \theta_{e}}f,
\end{equation}
where $b^{\tau, \theta_{e}}(f)$ is the $\R^{2d+2}$ vector with components $b^{\tau, \theta_{e}}_{n, j}(f) := \langle \chi^{\tau, \theta_{e}}_{n, j}, f_{\tau} \rangle_{L^2}$.
See  sub-figures (1) and (2) of both the  top  and bottom  of Figure \ref{fig_micloc_window}
  for illustrations of
the windowed signal $f_{\tau}(t)$ and of its projection
$\lim_{\sigma \downarrow 0}\E\big[\xi_{\tau, \theta_{e}} \big|\xi_{\tau, \theta_{e}} + \xi_\sigma=
f_{\tau}\big]$ in \eqref{Alimit} corresponding to the signals
$f$
 displayed in Figures \ref{iter example} and \ref{iter example bio}.

To apply these formulations to construct the module, 
 suppose that $f$  is a single mode \[f(t)=a(t)\cos(\theta(t)),\]
 so that 
\begin{equation}
\label{eqn_appr00}
f_{\tau}(t) =  e^{-\big(\frac{\dot{\theta_{0}}(\tau)(t - \tau)}{\alpha}\big)^2}
 a(t)\cos(\theta(t)) \, ,
\end{equation}
and consider the modified function
\begin{equation}
\label{eqn_appr0}
\bar{f}_{\tau}(t) =  e^{-\big(\frac{\dot{\theta_{0}}(\tau)(t - \tau)}{\alpha}\big)^2}
 \Bigg(\sum_{n=0}^d \frac{a^{(n)}(\tau)}{n!}  (t - \tau)^n\Bigg)\cos(\theta(t)) \,
\end{equation}
obtained by replacing the function $a$ with the first $d+1$ terms of its Taylor series about $\tau$.
In what follows, we will use the expression $\approx $ to articulate an informal approximation analysis.
It is clear that $\bar{f}_{\tau} \in \chi^{\tau, \theta_{e}}$ and,
 since $\frac{\alpha}{\dot{\theta}_{0}(\tau)}$ is small, that
$
\langle \chi^{\tau, \theta_{e}}_{n,j}, f_{\tau}-\bar{f}_{\tau} \rangle_{L^{2}} \approx 0, \forall (n,j)
$ and
therefore
$
 P_{\chi^{\tau, \theta_{e}}} f_{\tau}\approx \bar{f}_{\tau} \, ,
$
and therefore \eqref{eieienhugirt} implies that
\begin{equation}
\label{eieienhugirt5}
\sum_{j'}Z_{0,j'}(\tau,\theta_{e},f) \chi^{\tau, \theta_{e}}_{0, j'}(t)\approx  \bar{f}_{\tau}(t) ,
\quad t \in [-1,1]\, ,
\end{equation}
which by \eqref{eqn_appr0} implies that
\begin{equation}
\label{eieienhugirt2}
\sum_{j'}Z_{0,j'}(\tau,\theta_{e},f) \chi^{\tau, \theta_{e}}_{0, j'}(t)\approx  e^{-\big(\frac{\dot{\theta_{0}}(\tau)(t - \tau)}{\alpha}\big)^2}  a(\tau)\cos(\theta(t)),
\quad t \approx \tau\,,
\end{equation}
which implies that
\begin{equation}
\label{eieienhugirt3}
Z_{0,c}(\tau,\theta_{e},f) \cos(\theta_{e}(t)) +Z_{0,s}(\tau,\theta_{e},f) \sin(\theta_{e}(t))\approx
   a(\tau)\cos(\theta(t)),
\quad t \approx \tau\, .
\end{equation}
Setting $\theta_{\delta}:=\theta-\theta_{e}$ as the approximation error, using the cosine summation
formula,
 we obtain
\[ Z_{0,c}(\tau,\theta_{e},f) \cos(\theta_{e}(t)) +Z_{0,s}(\tau,\theta_{e},f) \sin(\theta_{e}(t))\approx a(\tau)
\bigl(\cos( \theta_{\delta}(t))\cos(\theta_{e}(t)) -\sin( \theta_{\delta}(t))\sin(\theta_{e}(t)\bigr). \]
However,  $t \approx \tau$ implies that $\theta_{\delta}(t) \approx \theta_{\delta}(\tau) $, so that
we obtain
\[ Z_{0,c}(\tau,\theta_{e},f) \cos(\theta_{e}(t)) +Z_{0,s}(\tau,\theta_{e},f) \sin(\theta_{e}(t))\approx a(\tau)
\bigl(\cos( \theta_{\delta}(\tau))\cos(\theta_{e}(t)) -\sin( \theta_{\delta}(\tau))\sin(\theta_{e}(t)\bigr),  \]
which, since $\dot{\theta}_{e}(t)$  positive and bounded away from  $0$, implies that
\begin{eqnarray*}
 Z_{0,c}(\tau,\theta_{e},f)&\approx&a(\tau)\cos( \theta_{\delta}(\tau))\\
Z_{0,s}(\tau,\theta_{e},f)&\approx& -a(\tau)\sin( \theta_{\delta}(\tau))\, .
\end{eqnarray*}
Consequently, writing
\begin{eqnarray}
\label{eienuriig}
a(\tau,\theta_{e},f)&:=&\sqrt{Z^{2}_{0,c}(\tau,\theta_{e},f)+Z^{2}_{0,s}(\tau,\theta_{e},f)}\,\nonumber\\
 \delta\theta(\tau,\theta_{e},f)&: = &\operatorname{atan2}\big(-Z_{0,s}(\tau,\theta_{e},f), Z_{0,c}(\tau,\theta_{e},f)\big)\,,
\end{eqnarray}
we obtain that
$a(\tau,\theta_{e},f)\approx a(\tau)$ and
 $\delta \theta (\tau,\theta_{e},f)\approx
\theta_{\delta}(\tau)$. We will therefore use $a(\tau,\theta_{e},f)$ to estimate the amplitude $a(\tau)$
 of the mode $f$ using the estimate $\theta_{e}$  and
 $\delta \theta(\tau,\theta,f)$ to estimate the mode phase $\theta$ through
$\theta(\tau)=\theta_{e}(\tau)+\theta_{\delta}(\tau)\approx \theta_{e}(\tau)+ \delta \theta(\tau,\theta_{e},f)$.
 Unless otherwise
specified, Equation \eqref{eienuriig} will take $d = 2$.
Experimental evidence indicates that $d=2$ is a sweet spot in the sense that $d=0$ or $d=1$ yields less fitting power, while larger
$d$ entails less stability.   
Iterating this refinement process will allow us to achieve near machine-precision accuracies in our phase/amplitude estimates.
See sub-figures (1) and (2) of  the top and bottom  of Figure \ref{fig_micloc_GPest}  for illustrations of $a(t)$, $a(\tau, \theta_{e}, v)(t)$, $\theta(t) - \theta_{e}(t)$ and  $\delta \theta(\tau, \theta_{e}, v)(t)$
corresponding to the first mode $v_{1}$   of
the  signals $v$
 displayed in Figures \ref{iter example}.4 and \ref{iter example bio}.4.

\subsection{The lowest instantaneous frequency }\label{subseclowest}
\begin{figure}[hbt!]
        \begin{center}
                        \includegraphics[width=\textwidth]{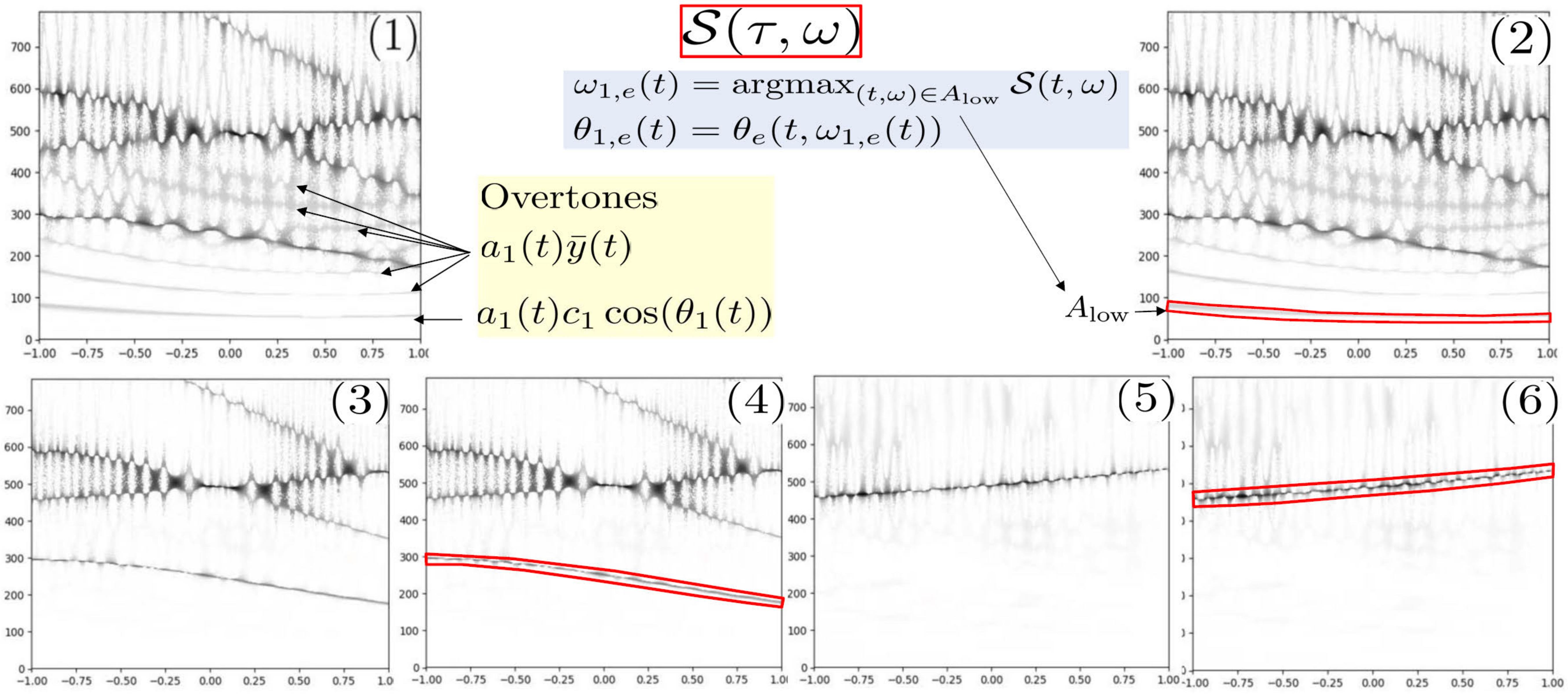}
                \caption{ Max-squeezing with the  EKG base waveform and derivation of the instantaneous phase estimates $\theta_{i,e}$.
(1,2) $(\tau,\omega)\rightarrow \mathcal{S}(\tau,\omega, v)$ and identification of $A_\low$ (3, 4) $(\tau,\omega)\rightarrow \mathcal{S}(\tau,\omega, v - v_{1, e})$ and identification of its $A_\low$  (5,6) $(\tau,\omega)\rightarrow \mathcal{S}(\tau,\omega, v - v_{1, e} - v_{2, e})$ and identification of  its $A_\low$.
}
\label{fig_Alow}
        \end{center}
\end{figure}
We will use  the  max-pool network  illustrated in the right-hand side of Figure \ref{shmd}
and the module of Section \ref{secmicrolocalkmd} to design a module
taking a signal  $v$ as input and producing, as output,   an estimate of the instantaneous phase $\theta_{\low}(v)$  of the mode of $v$ having the lowest instantaneous frequency.  
We restrict our presentation to the situation where the instantaneous frequencies
 $\dot{\theta}_i$ do not cross each other. The main steps of the computation performed by this module are as follows.
Let $\mathcal{S}(\tau,\omega,v)$ be the max-pool energy defined as in \eqref{eqklelhdekjdhud}, where now it is useful to indicate its dependence on $v$.

Let $A_\low$ be a subset of the time-frequency domain $(\tau,\omega)$ identified
 (as in  Figure \ref{fig_Alow}.2) as a narrow sausage around the lowest instantaneous frequency  defined by the local maxima of the  $\mathcal{S}(\tau,\omega,v)$. If no modes can be detected (above a given threshold) in $\mathcal{S}(\tau,\omega,v)$ then we set $\theta_{\low}(v)=\emptyset$.
 Otherwise we 
let \begin{equation}
\label{omegalowe}
\omega_{\low}(\tau):=\omega_e\big(\tau,\operatorname{argmax}_{\omega: (\tau,\omega)  \in A_\low}\mathcal{S}(\tau,\omega)\big)
\end{equation}
  be the estimated instantaneous frequency of the  mode having the lowest instantaneous frequency and, with $ \theta_e$ defined as in \eqref{eqthetate}, let
\begin{equation}\label{thetaie 1}
    \theta_{\low}(\tau) := \theta_e(\tau, \omega_{\low}(\tau))\,
\end{equation}
be the corresponding estimated instantaneous phase (obtained as in \eqref{eieburimg}).

\subsection{The iterated micro-local KMD algorithm. }
\label{sec_microlocalKMDalg}
\begin{figure}[hbt!]
        \begin{center}
                        \includegraphics[width=\textwidth]{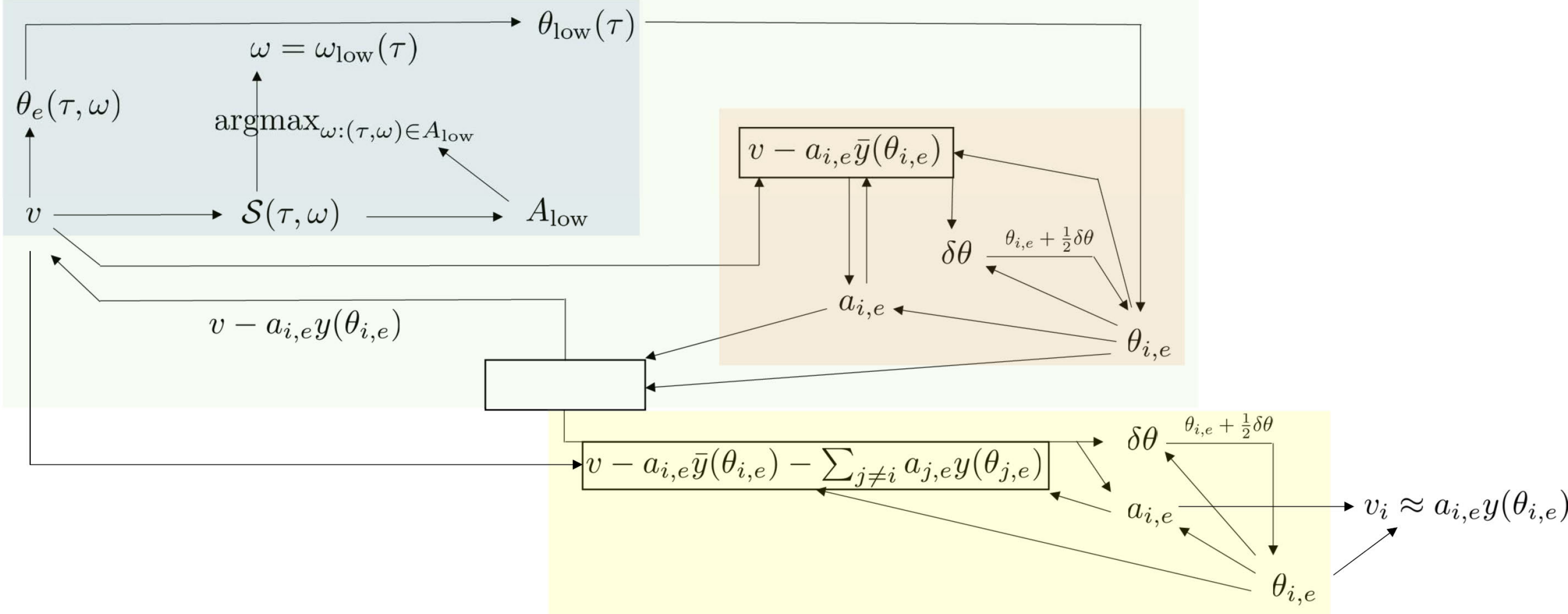}
                \caption{Modular representation of Algorithm \ref{iteralgn},  described in this section. The blue module represents the estimation of the lowest frequency as illustrated in Figure \ref{fig_Alow}.  The brown module represents the iterative estimation of the mode with lowest instantaneous frequency of lines \ref{stit0} through
\ref{stit3} of Algorithm \ref{iteralgn}.  The  yellow module represents the 
 iterative refinement of  all the modes in
  lines \ref{al2stfr1} through
\ref{al2stfr1e}. The brown and yellow modules used to refine phase/amplitude estimates  use the same code.  }
\label{fignetiter}
        \end{center}
\end{figure}

\begin{algorithm}[h!]
\caption{Iterated micro-local KMD.}\label{iteralgn}
\begin{algorithmic}[1]
\STATE\label{al2st1} $i \leftarrow 1$
\STATE  $v^{(1)} \leftarrow v$
\WHILE{true}
\IF {$\theta_\low(v^{(i)})=\emptyset$}
\STATE {break loop}
\ELSE
\STATE {$\theta_{i,e}\leftarrow \theta_\low(v^{(i)})$}
\ENDIF
\STATE   $a_{i, e}(\tau) \leftarrow  0$
\REPEAT  
   \label{stit0}
\FOR{$j$ in $\{1, ..., i\}$}
\STATE\label{stit1}  $v_{j,\mathrm{res}}\leftarrow v - a_{j, e}\bar{y}(\theta_{j, e}) -  \sum_{k \neq j, k \leq i} a_{k, e} y(\theta_{k,
e})$
\STATE\label{stit2}   $a_{j,e}(\tau) \leftarrow a\big(\tau, \theta_{j,e}, v_{j,\mathrm{res}}\big) / c_1$
\STATE\label{stit3}   $\theta_{j,e}(\tau) \leftarrow  \theta_{j,e}(\tau)+\frac{1}{2}\delta \theta \big(\tau, \theta_{j,e},
v_{j,\mathrm{res}}\big)$
\ENDFOR
\UNTIL {$\sup_{i,\tau}\big|\delta \theta \big(\tau, \theta_{i,e},
v_{i,\mathrm{res}}\big)
\big| < \epsilon_1$} \label{stit4}

\STATE\label{stpo}  $v^{(i+1)}  \leftarrow  v - \sum_{j \leq i} a_{j, e} y(\theta_{j, e})$
\STATE   $i \leftarrow i + 1$
\ENDWHILE\label{al2st1e}
\STATE $m\leftarrow i-1$
\REPEAT \label{al2stfr1}
\FOR{$i$ in $\{1, ..., m\}$\footnotemark}
\STATE  $v_{i,\mathrm{res}}\leftarrow v - a_{i, e}\bar{y}(\theta_{i, e}) -  \sum_{j \neq i} a_{j, e} y(\theta_{j, e})$
\STATE   $a_{i,e}(\tau) \leftarrow a\big(\tau, \theta_{i,e}, v_{i,\mathrm{res}}\big)$
\STATE   $\theta_{i,e}(\tau) \leftarrow  \theta_{i,e}(\tau)+\frac{1}{2}\delta \theta\big(\tau, \theta_{i,e}, v_{i,\mathrm{res}}\big)$
\ENDFOR
\UNTIL {$\sup_{j,\tau}\big|\delta \theta \big(\tau, \theta_{j,e},
v_{j,\mathrm{res}}\big)
\big|<\epsilon_2$ } \label{al2stfr1e}
\STATE Return the modes $v_{i, e}(t) \leftarrow a_{i, e}(t) y(\theta_{i, e}(t))$ for $i = 1, ..., m$
\end{algorithmic}
\end{algorithm}
\footnotetext{\label{dummy} This repeat loop,  used to refine the estimates, is optional.  Also,
 all statements in Algorithms with dummy variable $\tau$   imply a loop over all values of $\tau$ in the mesh $\T$. }

The method of  estimating  the lowest instantaneous frequency, described in
Section \ref{subseclowest}, provides a foundation for the iterated micro-local KMD algorithm, Algorithm \ref{iteralgn}.
We now describe  Algorithm \ref{iteralgn}, presented in
 its modular representation in Figure \ref{fignetiter},  using  Figures \ref{fig_Alow},
  \ref{fig_micloc_window} and  \ref{fig_micloc_GPest}.
To that end, let
\begin{equation}
\label{eoeiiie}
    y(t) = c_1\cos(t)+\sum_{n=2}^{\infty} c_n \cos(nt+d_n)
\end{equation}
be the   Fourier representation of the base waveform $y$ (which, without loss of generality, has been shifted so that the first sine coefficient is zero) and
write
\begin{equation}
    \bar{y}(t) := y(t)- c_1 \cos(t)
\end{equation}
for its overtones.

Let us describe how
lines \ref{al2st1} to \ref{al2st1e} provide refined estimates for the amplitude and the phase of each mode
 $v_{i},i\in \{1,\ldots,m\}$ of the signal $v$. Although the overtones of $y$ prevent us from simultaneously approximating all the instantaneous frequencies $\dot{\theta}_i$   from
the max-pool energy of the signal $v$, since
the  lowest mode
$v_{\low}=a_{\low}y(\theta_{\low})$ can be decomposed into the sum
$v_{\low}= a_{\low}c_{1}\cos(\theta_{\low})+a_{\low}\bar{y}(\theta_{\low})$
of a signal $a_{\low}c_{1}\cos(\theta_{\low})$ with a cosine waveform plus the signal
$a_{\low}\bar{y}(\theta_{\low})$ containing its higher frequency overtones, the method of Section \ref{subseclowest} can be applied
 to obtain  an estimate $\theta_{\low,e}$ of $\theta_{\low}$ and \eqref{eienuriig} can be applied to
obtain an estimate $a_{\low,e}c_{1}$ of $a_{\low}c_{1}$ producing an estimate
   $a_{\low,e}c_{1}\cos(\theta_{\low,e})$
 of  the primary component $a_{\low}c_{1}\cos(\theta_{\low})$ of the first mode.  Since
$c_1$ is known, this estimate produces the estimate
$a_{\low,e}\bar{y}(\theta_{\low,e})$ for the overtones of the lowest mode.  Recall that we calculate all quantities over the interval $[-1, 1]$ in this setting.  Estimates near the borders, $-1$ and $1$, will be less precise, but will be refined in the following loops.
To improve the accuracy of this estimate,
in lines \ref{stit2} and \ref{stit3} the micro local KMD of Section \ref{secmicrolocalkmd} is iteratively applied to the residual signal of every previously identified mode $v_{j, \mathrm{res}}
\leftarrow v - a_{j, e}\bar{y}(\theta_{j, e}) -  \sum_{k \neq j, k \leq i} a_{k, e} y(\theta_{k, e})$,
 consisting of  the signal $v$ with  the 
estimated modes $k \neq j$ as well as the overtones of estimated mode $j$ removed.  This residual is the sum of the estimation of the isolated base frequency component of $v_j$ and $\sum_{j > i} v_j$.
The rate parameter $1/2$  in line \ref{stit3} is to avoid overcorrecting the phase estimates, while  the parameters $\epsilon_1$ and $\epsilon_2$ in lines \ref{stit4} and \ref{al2stfr1e} are pre-specified accuracy thresholds.
The  resulting estimated lower modes are then removed from the signal to determine
the  residual $v^{(i+1)} :=  v - \sum_{j \leq i} a_{j, e} y(\theta_{j, e})$  in line  \ref{stpo}.

\begin{figure}[hbt!]
        \begin{center}
                        \includegraphics[width=\textwidth]{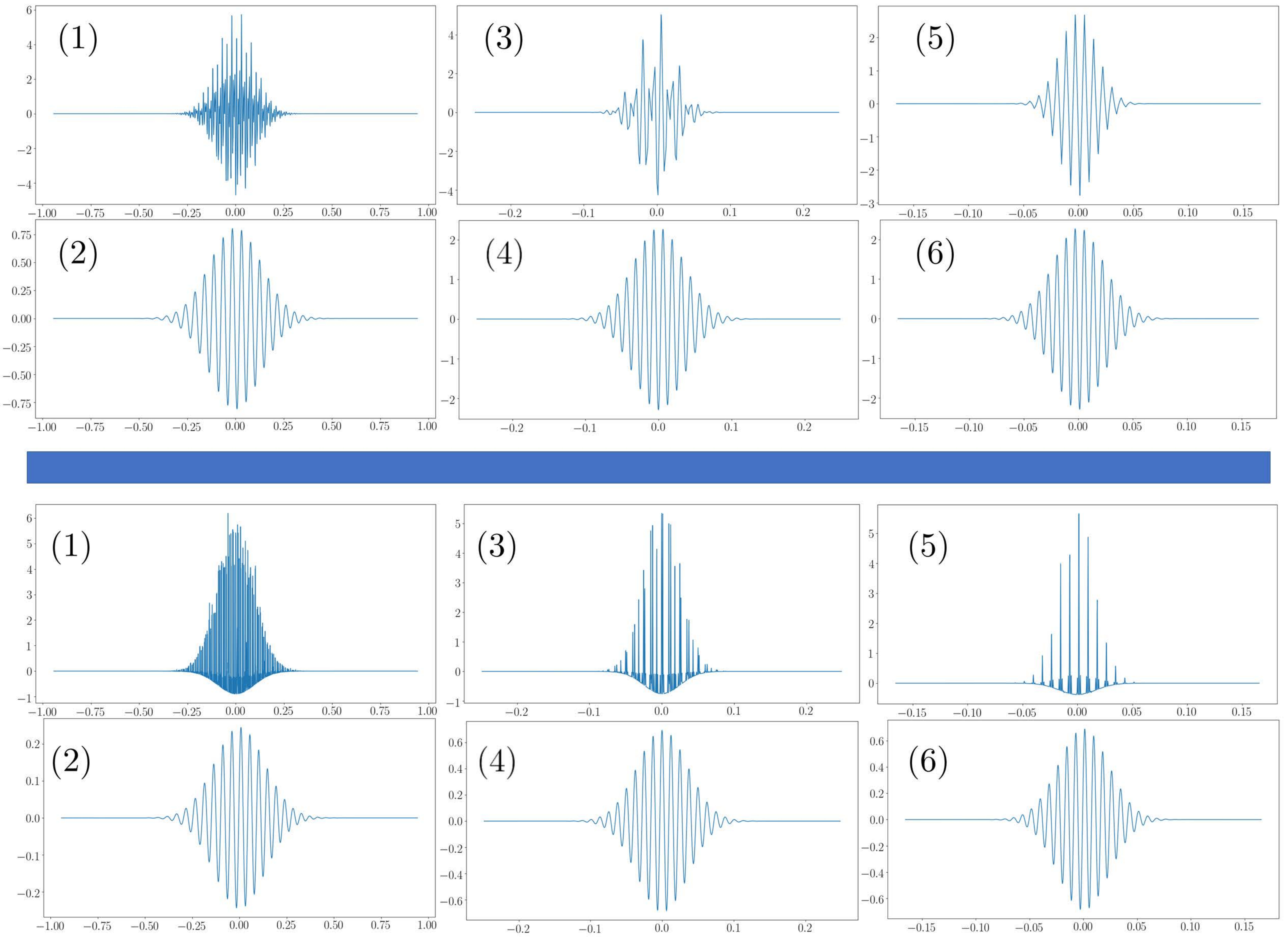}
                \caption{Top: $v$ is as in Figure \ref{iter example} (the base waveform is triangular). Bottom:
 $v$ is as in Figure \ref{iter example bio} (the base  waveform is EKG).  Both top and bottom:  $d=2$, (1) The windowed signal $v_{\tau}$ (2)  $\lim_{\sigma \downarrow 0}\E\big[\xi_{\tau, \theta_{1, e}} \big|\xi_{\tau, \theta_{1, e}} + \xi_\sigma=v_{\tau}\big]$
                (3) $(v-v_{1,e})_{\tau}$ (4) $\lim_{\sigma \downarrow 0}\E\big[\xi_{\tau, \theta_{2, e}} \big|\xi_{\tau, \theta_{2, e}} + \xi_\sigma=(v-v_{1,e})_{\tau}\big]$
                (5) $(v-v_{1,e}-v_{2,e})_{\tau}$ (6) $\lim_{\sigma \downarrow 0}\E\big[\xi_{\tau, \theta_{3, e}} \big|\xi_{\tau, \theta_{3, e}} + \xi_\sigma=(v-v_{1,e}-v_{2,e})_{\tau}\big]$.
                }\label{fig_micloc_window}
        \end{center}
\end{figure}

\begin{figure}[hbt!]
        \begin{center}
                        \includegraphics[width=\textwidth]{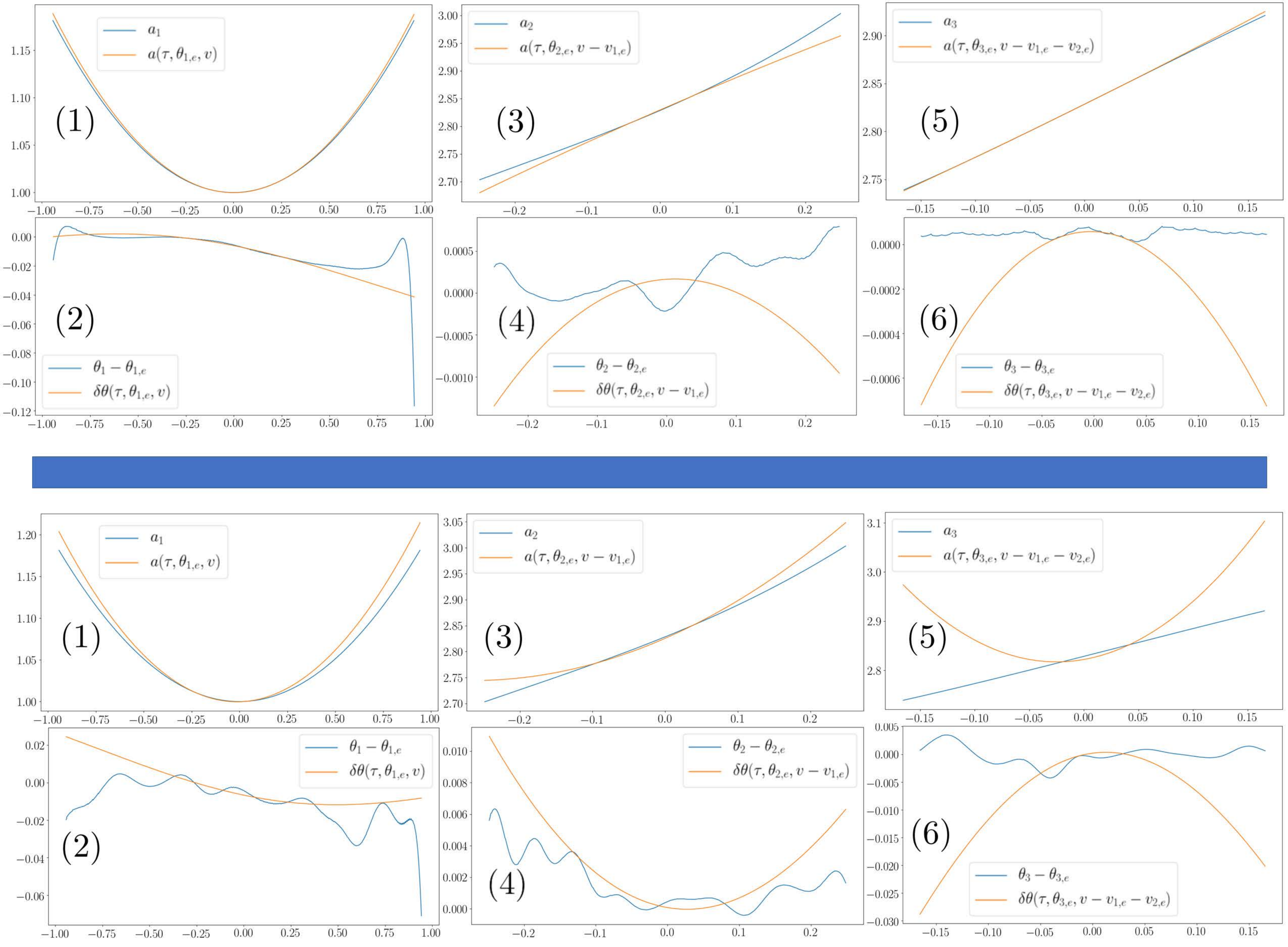}
                \caption{ Top: $v$ is as in Figure \ref{iter example} (the base waveform is triangular). Bottom: $v$ is as in Figure \ref{iter example bio} (the base  waveform is EKG).
   Both top and bottom:  $\tau=0$.         (1) the amplitude of the first mode $a_1(t)$ and its local Gaussian regression estimation $a(\tau, \theta_{1,e}, v)(t)$ (2) the error in estimated phase of the first mode $\theta_1(t) - \theta_{1, e}(t)$ and its local Gaussian regression $\delta \theta (\tau, \theta_{1,e}, v)(t)$ (3, 4) are as (1,2) with $v$ and $\theta_{1,e}$ replaced by $v-v_{1,e}$ and $\theta_{2,e}$
        (5,6)  are as (1,2) with $v$ and $\theta_{1,e}$ replaced by $v-v_{1,e}-v_{2,e}$ and $\theta_{3,e}$.
}
\label{fig_micloc_GPest}
        \end{center}
\end{figure}

Iterating this process, we  peel off
 an estimate   $a_{i, e} y(\theta_{i, e})$ of the mode corresponding to the lowest instantaneous frequency
 of the residual $v^{(i)}:=v-\sum_{j\leq i-1} a_{j,e}y(\theta_{j,e})$ of the signal $v$ obtained   in line  \ref{stpo}, removing the   interference of the first $i-1$ modes, including their  overtones, in our estimate of the instantaneous frequency and phase of the  $i$-th mode.
See Figure \ref{fig_Alow} for the evolution of the $A_{low}$ sausage as these modes are peeled off.  See sub-figures  (3) and (5) of the top and bottom of Figure
\ref{fig_micloc_window} for the results of peeling off the first two estimated modes
of the signal $v$  corresponding to both Figures \ref{iter example} and  \ref{iter example bio} and
sub-figures (4) and (6) for the results of the corresponding projections in \eqref{Alimit}.
 See sub-figures  (3) and (4) of the top and bottom of Figure
\ref{fig_micloc_GPest} for   amplitude and its estimate of  the results of peeling off the first
  estimated mode and sub-figures  (5) and (6) corresponding to peeling off the first two estimated modes
of the signal $v$  corresponding to both Figures \ref{iter example} and  \ref{iter example bio}.

After the amplitude/phase estimates $a_{i, e}, \theta_{i, e}, i \in \{1,\ldots, m\}$,
 have been obtained in lines \ref{al2st1} to \ref{al2st1e}, we have the option to further improve our estimates in a final optimization loop in lines \ref{al2stfr1} to \ref{al2stfr1e}.  This option enables us to achieve even higher accuracies by iterating the micro local KMD of Section \ref{secmicrolocalkmd} on the residual signals $v_{i,\mathrm{res}}\leftarrow v - a_{i, e}\bar{y}(\theta_{i, e}) -  \sum_{j \neq i} a_{j, e} y(\theta_{j, e})$, consisting of the signal $v$ with  all the estimated modes $j\not=i$  and
estimated  overtones of the mode $i$ removed. 

The proposed algorithm can be further improved by (1) applying a Savitsky-Golay filter to locally smooth (de-noise) the curves corresponding to each estimate   $\theta_{i,e}$ (which corresponds to refining our phase estimates through GPR filtering) (2)
starting with a larger  $\alpha$ (to decrease interference from other modes/overtones) and slowly reducing its value in the optional final refinement loop (to further localize our estimates after other components, and hence interference, have been mostly eliminated).

\subsection{Numerical experiments}\label{iter num examp}

Here we present results for both the triangle and EKG base waveform examples. As discussed in the previous section,
 these results are visually displayed in Figures
\ref{fig_micloc_window} and
\ref{fig_micloc_GPest}.

\subsubsection{Triangle wave example}\label{tri wave examp}
The base waveform is the triangle wave displayed in Figure \ref{waves}.
We observe the signal $v$ on a mesh spanning $[-1, 1]$ spaced at intervals of $\frac{1}{5000}$ and aim to recover each mode $v_i$ over this time mesh. We take $\alpha = 25$ within the first refinement loop corresponding to lines  \ref{al2st1} to \ref{al2st1e} and slowly decreased it to $6$ in the final loop corresponding to lines \ref{al2stfr1} to \ref{al2stfr1e}. 
The amplitudes and frequencies of each of the modes are shown in Figure \ref{iter example}.
 The recovery  errors of each mode as well as their amplitude and phase functions over the whole interval $[-1, 1]$ and the interior third $[-\frac{1}{3}, \frac{1}{3}]$ are displayed in Table \ref{iter results} and \ref{iter results0.33} respectively.  In the interior third of the interval, errors were found to be on the order of $10^{-9}$ for the first signal component and approximately $10^{-7}$ for the higher two.  However, over the full interval, the corresponding figures are in the $10^{-4}$ and $10^{-3}$ ranges due to recovery errors near the boundaries, $-1$ and $1$, of the interval.  Still, a plot superimposing $v_i$ and $v_{i, e}$ would visually appear to be one curve over $[-1, 1]$ due to the negligible recovery errors.

\begin{table*}[h]
\centering
\begin{tabular}{ |p{3.25cm}||p{1.85cm}|p{1.85cm}|p{1.85cm}|p{1.85cm}|  }
\hline
 Mode & $\frac{\|v_{i, e} - v_i\|_{L^2}}{\|v_i\|_{L^2}}$ & $\frac{\|v_{i, e} - v_i\|_{L^\infty}}{\|v_i\|_{L^\infty}}$ & $\frac{\|a_{i, e} - a_i\|_{L^2}}{\|a_i\|_{L^2}}$ & $\|\theta_{i, e} - \theta_i\|_{L^2}$ \\
 \hline
 $i=1$ & $5.47 \times 10^{-4}$ & $3.85 \times 10^{-3}$ & $2.80 \times 10^{-4}$ & $4.14 \times 10^{-5}$ \\
  $i=2$ & $6.42 \times 10^{-4}$ & $2.58 \times 10^{-3}$ & $3.80 \times 10^{-5}$ & $1.85 \times 10^{-4}$\\
  $i=3$ & $5.83 \times 10^{-4}$ & $6.29 \times 10^{-3}$ & $2.19 \times 10^{-4}$ & $6.30 \times 10^{-5}$ \\
 \hline
\end{tabular}
\caption{Signal component recovery errors in the triangle base waveform example over $[-1, 1]$.}
\label{iter results}
\end{table*}

\begin{table*}[h]
\centering
\begin{tabular}{ |p{3.25cm}||p{1.85cm}|p{1.85cm}|p{1.85cm}|p{1.85cm}|  }
\hline
 Mode & $\frac{\|v_{i, e} - v_i\|_{L^2}}{\|v_i\|_{L^2}}$ & $\frac{\|v_{i, e} - v_i\|_{L^\infty}}{\|v_i\|_{L^\infty}}$ & $\frac{\|a_{i, e} - a_i\|_{L^2}}{\|a_i\|_{L^2}}$ & $\|\theta_{i, e} - \theta_i\|_{L^2}$ \\
 \hline
 $i=1$ & $1.00 \times 10^{-8}$ & $2.40 \times 10^{-8}$ & $7.08 \times 10^{-9}$ & $6.52 \times 10^{-9}$ \\
  $i=2$ & $2.74 \times 10^{-7}$ & $2.55 \times 10^{-7}$ & $1.87 \times 10^{-8}$ & $2.43 \times 10^{-7}$\\
  $i=3$ & $2.37 \times 10^{-7}$ & $3.67 \times 10^{-7}$ & $1.48 \times 10^{-7}$ & $1.48 \times 10^{-7}$ \\
 \hline
\end{tabular}
\caption{Signal component recovery errors in the triangle base waveform example over $[-\frac{1}{3}, \frac{1}{3}]$.}
\label{iter results0.33}
\end{table*}

\subsubsection{EKG wave example}
The base waveform is the EKG wave displayed in Figure \ref{waves}. We use the same discrete mesh as in the triangle case.  Here, we took $\alpha = 25$ in the loop corresponding to lines  \ref{al2st1} to \ref{al2st1e} and slowly decreased it to $15$ in the final loop corresponding to lines \ref{al2stfr1} to \ref{al2stfr1e}.
  The amplitudes and frequencies of each of the modes are shown in Figure \ref{iter example bio}, while  the recovery error of each mode as well as their amplitude and phase functions are shown both over the whole interval $[-1, 1]$ and the interior third $[-\frac{1}{3}, \frac{1}{3}]$ in Tables \ref{iter results bio} and \ref{iter results bio 0.33} respectively.  Within the interior third of the interval, amplitude and phase relative errors are found to be on the order of $10^{-4}$ to $10^{-5}$ in this setting.  However, over $[-1, 1]$, the mean errors are more substantial, with amplitude and phase estimates in the $10^{-1}$ to $10^{-3}$ range.  Note the high error rates in $L^{\infty}$ stemming from errors in placement of the tallest peak (the region around which is known as the R wave in the EKG community). 
   In the center third of the interval, $v_{i, e}$ and $v_{i}$ are visually indistinguishable due to the small recovery errors.

\begin{table*}[h]
\centering
\begin{tabular}{ |p{3.25cm}||p{1.85cm}|p{1.85cm}|p{1.85cm}|p{1.85cm}|  }
\hline
 Mode & $\frac{\|v_{i, e} - v_i\|_{L^2}}{\|v_i\|_{L^2}}$ & $\frac{\|v_{i, e} - v_i\|_{L^\infty}}{\|v_i\|_{L^\infty}}$ & $\frac{\|a_{i, e} - a_i\|_{L^2}}{\|a_i\|_{L^2}}$ & $\|\theta_{i, e} - \theta_i\|_{L^2}$ \\
 \hline
 $i=1$ & $5.66 \times 10^{-2}$ & $1.45 \times 10^{-1}$ & $4.96 \times 10^{-3}$ & $8.43 \times 10^{-3}$ \\
 $i=2$ & $4.61 \times 10^{-2}$ & $2.39 \times 10^{-1}$ & $2.35 \times 10^{-2}$ & $1.15 \times 10^{-2}$\\
 $i=3$ & $1.34 \times 10^{-1}$ & $9.39 \times 10^{-1}$ & $9.31 \times 10^{-3}$ & $2.69 \times 10^{-2}$ \\
 \hline
\end{tabular}
\caption{Signal component recovery errors on $[-1, 1]$ in the EKG base waveform example.}
 \label{iter results bio}
\end{table*}

\begin{table*}[h]
\centering
\begin{tabular}{ |p{3.25cm}||p{1.85cm}|p{1.85cm}|p{1.85cm}|p{1.85cm}|  }
\hline
 Mode & $\frac{\|v_{i, e} - v_i\|_{L^2}}{\|v_i\|_{L^2}}$ & $\frac{\|v_{i, e} - v_i\|_{L^\infty}}{\|v_i\|_{L^\infty}}$ & $\frac{\|a_{i, e} - a_i\|_{L^2}}{\|a_i\|_{L^2}}$ & $\|\theta_{i, e} - \theta_i\|_{L^2}$ \\
 \hline
 $i=1$ & $1.80 \times 10^{-4}$ & $3.32 \times 10^{-4}$ & $3.52 \times 10^{-5}$ & $2.85 \times 10^{-5}$ \\
 $i=2$ & $4.35 \times 10^{-4}$ & $5.09 \times 10^{-4}$ & $3.35 \times 10^{-5}$ & $7.18 \times 10^{-5}$\\
 $i=3$ & $3.63 \times 10^{-4}$ & $1.08 \times 10^{-3}$ & $7.23 \times 10^{-5}$ & $6.26 \times 10^{-5}$ \\
 \hline
\end{tabular}
\caption{Signal component recovery errors on $[-\frac{1}{3}, \frac{1}{3}]$ in the EKG base waveform example.}
 \label{iter results bio 0.33}
\end{table*}

\section{Unknown base waveforms}\label{secunwav}

Here we consider the extension, Problem \ref{unk wave pb},
 of the mode recovery problem, Problem \ref{pb2}, to the case where the periodic base
 waveform of each mode is unknown and  may be different across modes.  That is, given the observation
\begin{equation}
\label{eienriromgi}
v(t)=\sum_{i=1}^m a_i(t)y_i\big(\theta_i(t)\big), \quad t\in [-1,1],
\end{equation}
  recover the modes
 $v_i:=a_i(t)y_i\big(\theta_i(t)\big)$.
To avoid ambiguities caused by overtones when the waveforms $y_i$ are not only non-trigonometric but also unknown,  we will assume that  the corresponding functions $(k \dot{\theta}_i)_{t\in [-1,1]}$ and $(k' \dot{\theta}_{i'})_{t\in [-1,1]}$  are distinct for $i\not=i'$ and $k,k'\in \mathbb{N}^*$, that is,
  they may be equal for some $t$ but not for all $t$.
We represent the $i$-th base waveform $y_{i}$ through its Fourier series
\begin{equation}
\label{wavefourier}
    y_i(t) = \cos(t)+\sum_{k = 2}^{k_{\max}} \bigl(c_{i, (k,c)} \cos(kt) + c_{i, (k,s)} \sin(kt)\bigr),
\end{equation}
 that, without loss of generality has been scaled and translated. Moreover, 
since we operate in a discrete setting, without loss of generality we can also truncate the series at a finite level $k_{\max}$, which is naturally bounded by the inverse of the resolution of the discretization in time.
\begin{figure}[hbt!]
        \begin{center}
                        \includegraphics[width=\textwidth]{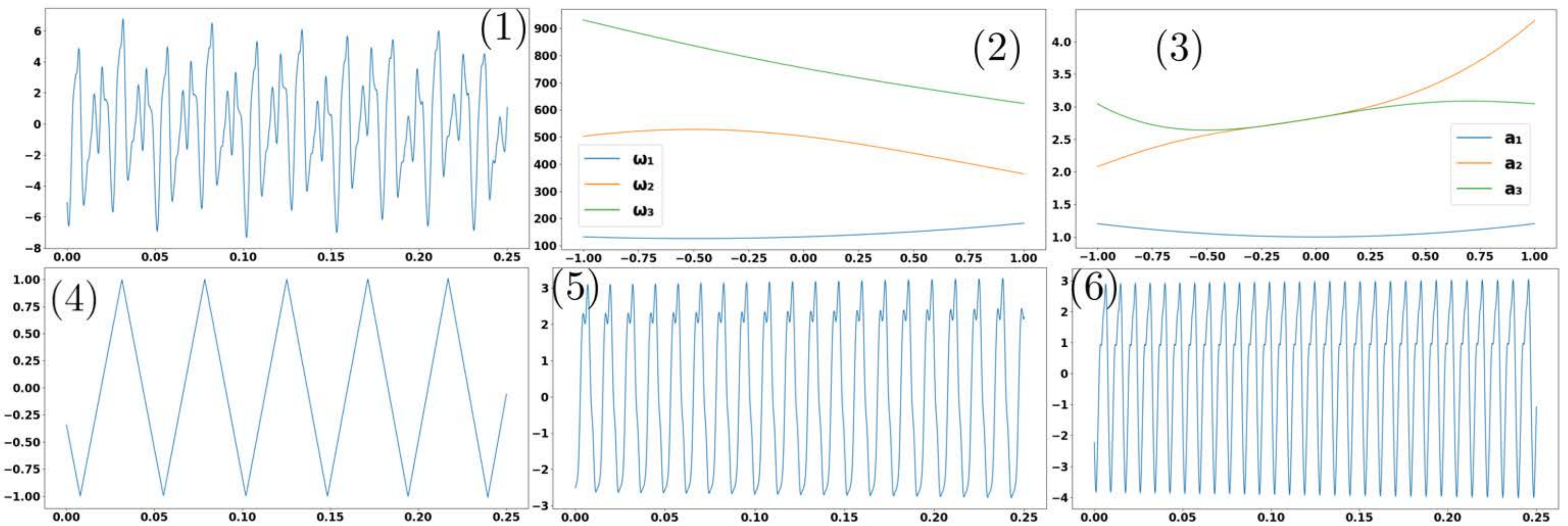}
                \caption{(1) Signal $v$  (the signal is defined over $[-1,1]$ but displayed over $[0, 0.4]$ for visibility) (2)  Instantaneous frequencies $\omega_i := \dot{\theta}_i$ (3) Amplitudes $a_i$ (4, 5, 6) Modes $v_1$, $v_2$, $v_3$ over $[0, 0.4]$ (mode plots have also been zoomed in for visibility).}\label{wavelearn plots}
        \end{center}
\end{figure}
\begin{figure}[hbt!]
        \begin{center}
                        \includegraphics[width=\textwidth]{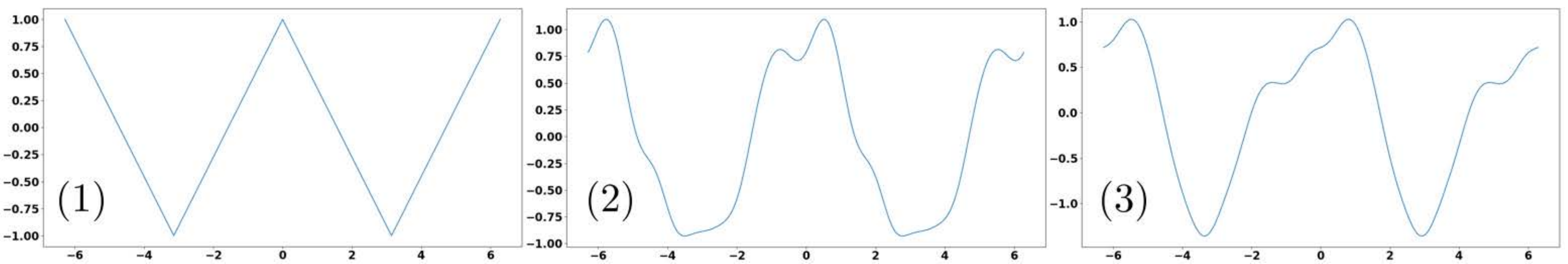}
                \caption{(1) $y_1$ (2) $y_2$ (3) $y_3$}\label{y123}
        \end{center}
\end{figure}
To illustrate our approach, we consider the signal $v=v_{1}+v_{1}+v_{3}$
 and its corresponding modes
$v_{i}:=a_i(t)y_i\big(\theta_i(t)\big)$
  displayed in Figure \ref{wavelearn plots}, where the  corresponding base waveforms $y_1, y_2$ and $y_3$  are shown in Figure \ref{y123} and described in Section \ref{secnumunwav}.

\subsection{Micro-local waveform KMD}\label{waveform kmd sec}
\begin{figure}[hbt!]
        \begin{center}
                        \includegraphics[width=\textwidth]{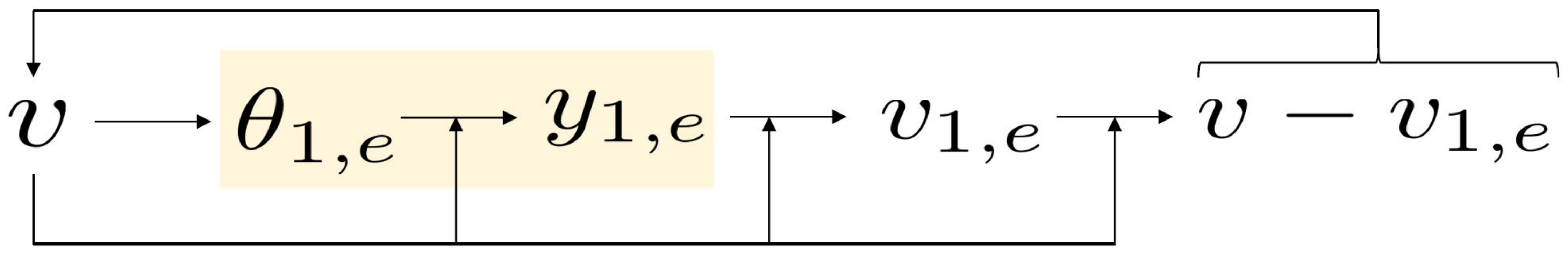}
                \caption{High level structure of  Algorithm \ref{unkiteralgn} for the case when the waveforms are unknown. }
\label{figunknownsame}
        \end{center}
\end{figure}

We now describe the micro-local {\em waveform} KMD, Algorithm \ref{unkiteralgn}, which takes as inputs a time $\tau$,  estimated instantaneous amplitude  and phase functions $t \rightarrow a(t), \theta(t)$, and a signal $v$, and outputs an estimate of the waveform $y(t)$ associated with the phase function $\theta$.  The proposed approach is a direct extension of the one presented in Section \ref{secmicrolocalkmd} and  the shaded part of Figure \ref{figunknownsame} shows the new block which will be added  to Algorithm \ref{iteralgn},  the algorithm designed for the case when waveforms are non-trigonometric and known. As described below this new block produces an estimator $y_{i,e}$ of the waveform $y_i$ from an estimate $\theta_{i,e}$ of the phase $\theta_i$.

Given $\alpha > 0$, $\tau \in [-1,1]$, and differentiable function $t \rightarrow \theta(t)$, define the  Gaussian process
\begin{equation}\label{unk xi def}
    \xi_{\tau, \theta}^{y}(t) = e^{-\big(\frac{\dot{\theta}_{0}(\tau)(t - \tau)}{\alpha}\big)^2}  \Big(X^y_{1, c} \cos\big(\theta(t)\big) + \sum_{k=2}^{k_{\max}} \big(X^y_{k, c}\cos\big(k\theta(t)\big) + X^y_{k, s} \sin\big(k\theta(t)\big)\big)\Big),
\end{equation}
where  $X^y_{1, c}, X^y_{k, c}$, and $X^y_{k, s}$ are independent $\mathcal{N}(0,1)$ random variables.  Let
\begin{equation}
    v_{\tau}(t) := e^{-\big(\frac{\dot{\theta}_{0}(\tau)(t - \tau)}{\alpha}\big)^2} v(t), \quad \tau\in[-1, 1],\,
\end{equation}
 be the windowed signal, and
define
\begin{equation}
    Z^y_{k, j}(\tau, \theta, v) := \lim_{\sigma \downarrow 0}\E\big[X^y_{k,j} \big|\xi^y_{\tau, \theta} + \xi_\sigma=v_{\tau}\big],
\end{equation}
and, for $k \in \{2, \ldots, k_{\max}\}$, $j\in \{c, s\}$, let
\begin{equation}
    c_{k, j}(\tau, \theta, v) := \frac{Z^y_{k, j}(\tau, \theta, v)}{Z^y_{1, c}(\tau, \theta, v)} \, .
\end{equation}

When the assumed phase function $\theta:=\theta_{i,e}$ is close to the phase function $\theta_{i}$ of  the $i$-th
mode of the signal $v$ in the expansion
\eqref{eienriromgi},
$c_{k, j}(\tau,\theta_{i,e} , v)$
 yields an estimate of the Fourier coefficient $c_{i,(k, j)}$ \eqref{wavefourier} of the $i$-th base waveform
$y_{i}$ at time $t = \tau$.   This  waveform recovery is susceptible to error when
 there is interference in the overtone frequencies (that is for the values of $\tau$ at which  $j_1\dot{\theta}_{i_1} \approx j_2 \dot{\theta}_{i_2}$ for $i_1 < i_2$). However,
 since the coefficient $c_{i,(k, j)}$ is independent of time,
we can  overcome this by  computing $c_{k, j}(\tau, \theta_{i,e}, v)$ at each time $\tau$ and take the most common approximate value over all $\tau$ as follows. Let $T \subset [-1,1]$ be the finite set of values of $\tau$ used in the numerical discretization of the time axis with $N:=|T|$ elements.  For an interval $I\subset \R$, let
  \begin{equation}
T_I:=\{\tau\in T | c_{k, j}(\tau, \theta_{i,e}, v) \in I\}\,,
 \end{equation}
and let  $N_I:=|T_{I}|$ denote  the number of elements of $T_I$.
Let $I_{\text{max}}$ be a maximizer of the function $ I \rightarrow N_I$ over intervals of fixed width $L$, and
      define the estimate
\begin{equation}
    c_{k, j}(\theta_{i,e}, v) :=
    \begin{cases}
        \frac{1}{N_{I_\text{max}}} \sum_{\tau \in T_{I_\text{max}}} c_{k, j}(\tau,\theta_{i,e} , v) &,\quad \frac{N_{I_\text{max}}}{N} \geq 0.05 \\
        0 &, \quad  \frac{N_{I_\text{max}}}{N} < 0.05
    \end{cases} \, ,
\end{equation}
 of  the Fourier coefficient $c_{i,(k, j)}$ to be the average of the values of
      $c_{k, j}(\tau, \theta_{i,e}, v)$ over $\tau\in  T_{I_{\text{max}}}$.
The interpretation of the selection of the cutoff  $0.05$ is as follows: if $\frac{N_{I_\text{max}}}{N}$ is small then there is interference in the overtones at all time $[-1, 1]$ and no information may be obtained about the corresponding Fourier coefficient. When the assumed phase function is near that of the lowest frequency mode $v_{1}$, which we write $\theta:=\theta_{1,e}$,
Figures \ref{wavelearn hist}.2 and 4 shows  zoomed-in histograms
 of the functions $\tau \rightarrow c_{(3, c)}(\tau, \theta_{1,e}, v)$ and $\tau \rightarrow  c_{(3, s)}(\tau, \theta_{1,e}, v)$ displayed in Figures  \ref{wavelearn hist}.1 and 3.

\begin{figure}[h]
	\begin{center}
			\includegraphics[width=\textwidth]{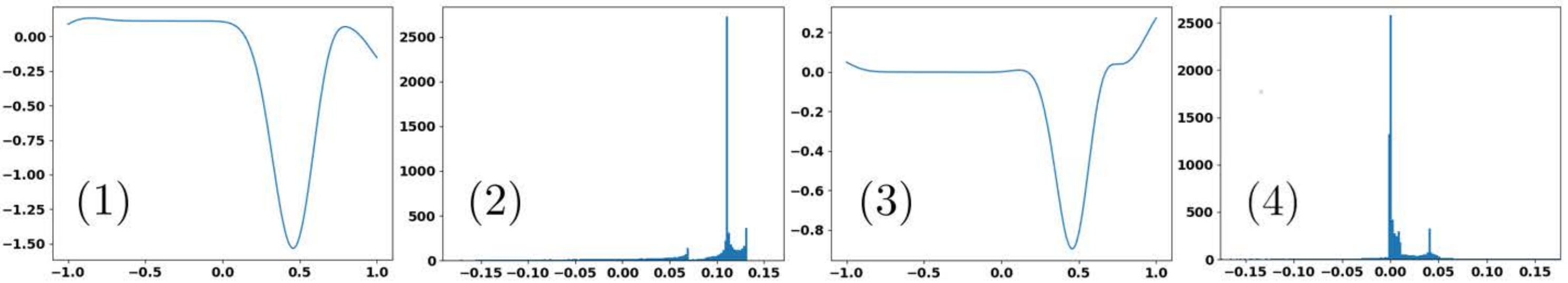}
		\caption{(1) A plot of the function $\tau \rightarrow c_{(3, c)}(\tau, \theta_{1,e}, v)$ (2)
 A histogram (cropping outliers) with bin width $0.002$ of $c_{(3, c)}(\tau, \theta_{1,e}, v)$ values.
 The true value  $c_{1, (3, c)}$ is $1/9$ since $y_1$ is a triangle wave.
 (3) A plot of the function $\tau \rightarrow c_{(3, s)}(\tau, \theta_{1,e}, v)$ (2)
 A histogram (cropping outliers) with bin width $0.002$ of $c_{(3, s)}(\tau, \theta_{1,e}, v)$ values.
  The true value $c_{1, (3, s)}$
 of this overtone  is $0$. }
\label{wavelearn hist}
	\end{center}
\end{figure}

\paragraph{On the interval width $L$.}  In our numerical experiments,
 the recovered modes and waveforms  show little sensitivity to
 the choice of  $L$.  In particular,  we
 set $L$ to be $0.002$, whereas widths between  $0.001$ and $0.01$ yield similar results.
The rationale for the rough selection of the
 value of $L$ is as follows.  Suppose $v = \cos(\omega t)$ and $v' = v + \cos(1.5\omega t)$.  Define the  quantity
\begin{equation}
    \max_\tau \bigl(c_{2, c}(\tau, \theta, v') - c_{2, c}(\tau, \theta, v)\bigr)\,,
\end{equation}
with the intuition of approximating the maximum corruption by the $\cos(1.5\omega t)$ term in the estimated first overtone.  This quantity provides a good choice for $L$ and is mainly dependent on the selection of $\alpha$ and marginally on $\omega$.  For our selection of $\alpha = 10$, we numerically found its value to be approximately $0.002$.

\subsection{Iterated micro-local KMD with unknown waveforms algorithm}
\begin{algorithm}[hbt!]
\caption{Iterated micro-local KMD with unknown waveforms.}\label{unkiteralgn}
\begin{algorithmic}[1]
\STATE\label{al2st1bb} $i \leftarrow 1$ and $v^{(1)} \leftarrow v$
\WHILE{true}
\IF {$\theta_\low(v^{(i)})=\emptyset$}
\STATE {break loop}
\ELSE
\STATE {$\theta_{i,e}\leftarrow \theta_\low(v^{(i)})$} \label{theta est 1}
\STATE {$y_{i,e}\leftarrow \cos(t)$} \label{y est 1}
\ENDIF
\STATE   $a_{i, e}(\tau) \leftarrow  0$
\REPEAT\label{refloop1bb}
\FOR{$l$ in $\{1, ..., i\}$}
\STATE\label{stit1bb}  $v_{l,\mathrm{res}}\leftarrow v - a_{l, e}\bar{y_{l,e}}(\theta_{l, e}) -  \sum_{k \neq l, k \leq i} a_{k, e} y_{l,e}(\theta_{k,
e})$
\STATE\label{stit2bb}   $a_{l,e}(\tau) \leftarrow a\big(\tau, \theta_{l,e}, v_{l,\mathrm{res}}\big) / c_1$
\STATE\label{stit3bb}   $\theta_{l,e}(\tau) \leftarrow  \theta_{l,e}(\tau)+\frac{1}{2}\delta \theta \big(\tau, \theta_{l,e},
v_{l,\mathrm{res}}\big)$
\STATE{$c_{l, (k, j), e} \leftarrow c_{k, j}\big(\theta_{l, e}, v_{l, \mathrm{res}}\big)$} \label{c est 2}
\STATE{$y_{l, e}(\cdot) \leftarrow \cos(\cdot) + \sum_{k=2}^{k_{\max}}
 \bigl(c_{l, (k, c), e} \cos(k\cdot) + c_{l, (k, s), e} \sin(k\cdot)\bigr)$}\label{y est 2}

\ENDFOR
\UNTIL\label{refloop1endbb}{$\sup_{l, \tau}\bigl|\delta \theta \big(\tau, \theta_{l,e},
v_{l, \mathrm{res}}\big)      
\bigr| <  \epsilon_1$}

\STATE\label{stpobb}  $v^{(i+1)}  \leftarrow  v - \sum_{j \leq i} a_{j, e} y_{i,e}(\theta_{j, e})$
\STATE   $i \leftarrow i + 1$
\ENDWHILE\label{al2st1ebb}
\STATE $m\leftarrow i-1$
\REPEAT
\FOR{$i$ in $\{1, \dots, m\}\footnotemark$}
\STATE  $v_{i,\mathrm{res}}\leftarrow v - a_{i, e}\bar{y}_{i, e}(\theta_{i, e}) -  \sum_{j \neq i} a_{j, e} y_{j, e}(\theta_{j, e})$
\STATE   $a_{i,e}(\tau) \leftarrow a\big(\tau, \theta_{i,e}, v_{i, \text{res}}\big)$
\STATE   $\theta_{i,e}(\tau) \leftarrow  \theta_{i,e}(\tau)+\frac{1}{2}\delta \theta\big(\tau, \theta_{i,e}, v_{i,\mathrm{res}}\big)$
\STATE {$c_{i, (k, j), e} \leftarrow c_{k, j}\big(\theta_{i, e}, v - \sum_{j \neq i} a_{j, e} y_{j, e}(\theta_{j, e})\big)$}\label{c est 3}
\STATE {$y_{i, e}(\cdot) \leftarrow \cos(\cdot) + \sum_{k=2}^{k_{\max}}
\bigl(c_{i, (k, c), e} \cos(k\cdot) + c_{i, (k, s), e} \sin(k\cdot)\bigr)$}\label{y est 3}
\ENDFOR
\UNTIL {$\sup_{i, \tau}\bigl|\delta \theta \big(\tau, \theta_{i,e},
v_{i,\mathrm{res}}\big)
\bigr| <  \epsilon_2$}
\label{al2stfr1ebb}
\STATE Return the modes $v_{i, e}(t) \leftarrow a_{i, e}(t) y(\theta_{i, e}(t))$ for $i = 1, ..., m$
\end{algorithmic}
\end{algorithm}
\footnotetext{\label{dummy2} This repeat loop,  used to refine the estimates, is optional.  Also,
 all statements in Algorithms with dummy variable $\tau$   imply a loop over all values of $\tau$ in the mesh $\T$. }

Except for the steps discussed in Section \ref{waveform kmd sec},
Algorithm \ref{unkiteralgn} is identical to Algorithm \ref{iteralgn}.  As illustrated in
Figure \ref{figunknownsame}, we first identify the lowest frequency of the cosine component of each mode  (lines \ref{theta est 1} and \ref{y est 1} in Algorithm \ref{unkiteralgn}).  Next, from lines \ref{refloop1bb} to \ref{refloop1endbb}, we execute a similar refinement loop as in Algorithm \ref{iteralgn} with the addition of an application of micro-local waveform KMD on lines \ref{c est 2} and \ref{y est 2} to estimate base waveforms.   Finally, once each mode has been identified, we again apply waveform estimation in lines \ref{c est 3}-\ref{y est 3} (after nearly eliminating other modes and reducing interference in overtones for higher accuracies).

\subsection{Numerical experiments}\label{secnumunwav}

To illustrate this learning of the base waveform of each mode, we take
 $v(t) = \sum_{i = 1}^3 a_i(t) y_i(\theta_i(t))$,  where the lowest frequency mode $a_1(t) y_1(\theta_1(t))$ has the (unknown)
 triangle waveform $y_{1}$
  of   Figure \ref{waves}.
  We determine the waveforms $y_{i}, i = 2, 3$, randomly by setting $c_{i, (k, j)}$
 to be zero with probability $1/2$ or to be a random  sample from $\N(0,  1/k^4)$ with probability $1/2$,
  for $k \in \{2, \ldots, 7\}$ and $j \in \{c, s\}$.
  The waveforms $y_1, y_2, y_3$ thus obtained are illustrated in Figure \ref{y123}.  The modes $v_1, v_2, v_3$, their amplitudes and instantaneous frequencies are shown in Figure \ref{wavelearn plots}.

\begin{table*}[h]
\centering
\begin{tabular}{ |p{3.00cm}||p{1.85cm}|p{1.85cm}|p{1.85cm}|p{1.85cm}|p{1.85cm}|  }
\hline
 Mode & $\frac{\|v_{i, e} - v_i\|_{L^2}}{\|v_i\|_{L^2}}$ & $\frac{\|v_{i, e} - v_i\|_{L^\infty}}{\|v_i\|_{L^\infty}}$ & $\frac{\|a_{i, e} - a_i\|_{L^2}}{\|a_i\|_{L^2}}$ & $\|\theta_{i, e} - \theta_i\|_{L^2}$ &  $\frac{\|y_{i, e} - y_i\|_{L^2}}{\|y_i\|_{L^2}}$\\
 \hline
 $i=1$ & $6.31 \times 10^{-3}$ & $2.39 \times 10^{-2}$ & $9.69 \times 10^{-5}$ & $1.41 \times 10^{-5}$ & $6.32 \times 10^{-3}$ \\
 $i=2$ & $3.83 \times 10^{-4}$ & $1.08 \times 10^{-3}$ & $5.75 \times 10^{-5}$ & $1.16 \times 10^{-4}$ & $3.76 \times 10^{-4}$\\
 $i=3$ & $3.94 \times 10^{-4}$ & $1.46 \times 10^{-3}$ & $9.53 \times 10^{-5}$ & $6.77 \times 10^{-5}$ & $3.80 \times 10^{-4}$ \\
 \hline
\end{tabular}
\caption{Signal component recovery errors over $[-1, 1]$ when the base waveforms are unknown}\label{tableunwav}
\end{table*}

We use the same mesh and the same value of  $\alpha$ values as  in Section \ref{tri wave examp}.  The main source of error for the recovery of the first mode's base waveform stems from the fact that a triangle wave has an infinite number of overtones, while in our implementation, we estimate only the first 15 overtones. Indeed,
  the   $L^{2}$ recovery error of approximating the first $16$ tones of the  triangle wave is $3.57 \times 10^{-4}$,
while the full recovery errors  are presented in Table \ref{tableunwav}.
  We omitted the plots of the $y_{i, e}$ as they are visually indistinguishable from those of the $y_i$.  Note that errors are only slightly improved away from the borders as the majority of it is accounted for by the waveform recovery error.

\section{Crossing frequencies, vanishing modes, and noise}

The algorithm introduced in this section addresses the following generalization of the mode recovery Problem \ref{pb3}, allowing for
crossing frequencies,
 vanishing modes and noise. The purpose of the $\delta,\epsilon$-condition in Problem \ref{recov general} is to prevent a long
overlap of the instantaneous frequencies of distinct modes.

 \begin{Problem}\label{recov general}
For $m\in \mathbb{N}^*$, let $a_1,\ldots,a_m$ be piecewise smooth  functions on $[-1,1]$, and  let $\theta_1,\ldots,\theta_m$ be
strictly increasing  functions on $[-1,1]$ such that, for  $\epsilon>0$
 and $\delta \in [0,1)$, the length of $t$
with $\dot{\theta_i}(t)/\dot{\theta}_j(t) \in [1-\epsilon, 1+\epsilon]$ is less than $\delta$. 
Assume that $m$ and the $a_i, \theta_i$ are unknown, and the square-integrable $2\pi$-periodic
 base waveform $y$ is known.  
 Given the observation
$v(t)=\sum_{i=1}^m a_i(t)y\big(\theta_i(t)\big) + v_\sigma(t)$ (for $t\in [-1,1]$),
where $v_\sigma$ is a realization of  white noise
 with variance $\sigma^2$, recover the modes
$v_i(t):=a_i(t)y\big(\theta_i(t)\big)$.
\end{Problem}
\begin{figure}[h]
	\begin{center}
			\includegraphics[width=\textwidth]{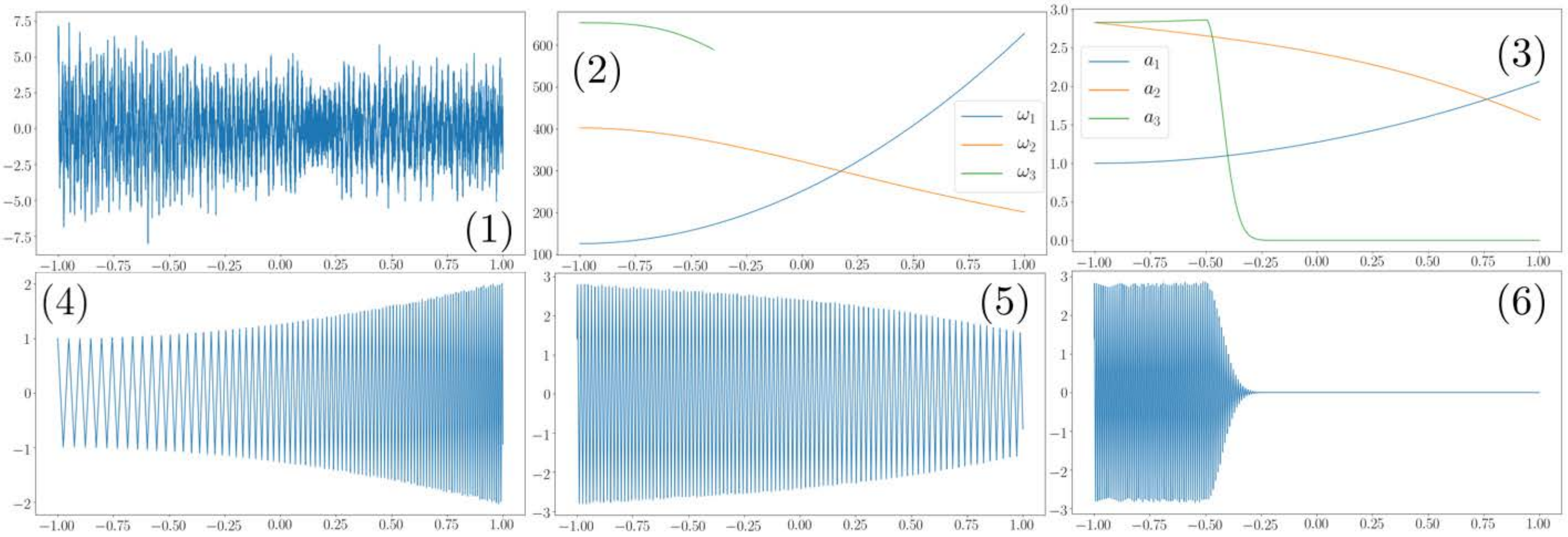}
		\caption{(1) Signal $v$  (2)  Instantaneous frequencies
 $\omega_i := \dot{\theta}_i$ (3) Amplitudes $a_i$  (4, 5, 6) Modes
$v_1$, $v_2$, $v_3$.}
\label{seg fig}
	\end{center}
\end{figure}
We will use the following two examples to illustrate our algorithm, in particular
 the identification of the lowest frequency $\omega_\low(\tau)$,
 at each time $\tau$,  and  the process of obtaining estimates of  modes.
\begin{Example} \label{seg ex}
Consider the problem of recovering the modes of the signal $v = v_1 + v_2 + v_3 + v_\sigma$ shown in
 Figure
\ref{seg fig}.  Each mode has a triangular base waveform.  In this example $v_3$ has the highest frequency and its amplitude vanishes over  $t > -0.25$. The frequencies of $v_1$
and $v_2$, cross  around $t = 0.25$.   $v_\sigma\sim \cN(0,\sigma^2 \delta(t-s))$ is  white noise with standard deviation $\sigma =
0.5$.  While the signal-to-noise ratio is $\Var(v_1 + v_2 + v_3) / \Var(v_\sigma)= 13.1$,
 the SNR ratio against each of the  modes $\Var(v_i) / \Var(v_\sigma)$, 
$i=1,2,3,$ is $2.7$, $7.7$, and $10.7$ respectively.
\end{Example}
\begin{figure}[h]
	\begin{center}
			\includegraphics[width=\textwidth]{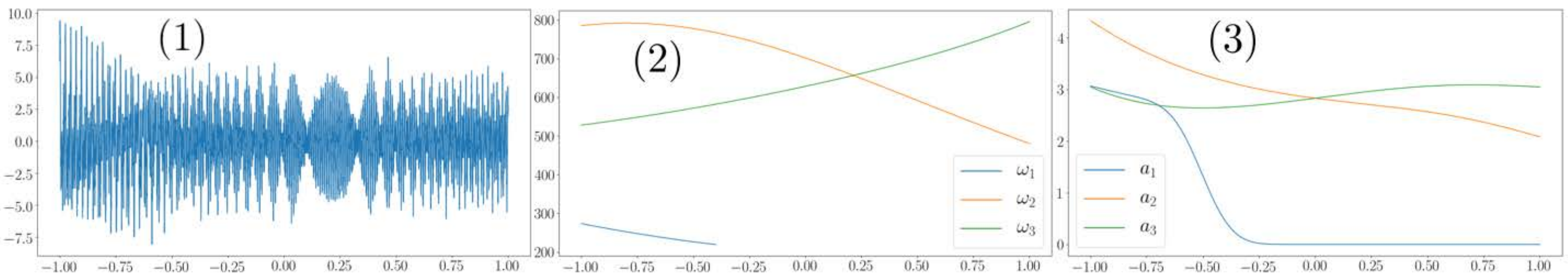}
		\caption{(1) Signal $v$ 
 (2)  Instantaneous frequencies $\omega_i := \dot{\theta}_i$ (3) Amplitudes $a_i$.}
\label{seg fig2}
	\end{center}
\end{figure}
\begin{Example} \label{seg ex2}
Consider the signal $v = v_1 + v_2 + v_3 + v_\sigma$ shown in
 Figure
\ref{seg fig2}.  Each mode has a triangular base waveform. In this example, the vanishing mode, $v_1$, has the
 lowest  frequency over $ t \lesssim  -0.25$ but then  its amplitude vanishes over  $t \gtrsim -0.25$. The frequencies of $v_2$
and $v_3$, cross  around $t = 0.25$.   $v_\sigma\sim \cN(0,\sigma^2 \delta(t-s))$ is  white noise with standard deviation $\sigma =
0.5$.
\end{Example}

 Examples \ref{seg ex} and \ref{seg ex2} of Problem \ref{recov general} cannot directly be solved with
Algorithm \ref{iteralgn}  
  (where the mode with the  lowest frequency is iteratively identified and peeled off) because the lowest observed
instantaneous frequency may no longer be associated with the same mode at different times in $[-1, 1]$ (due to vanishing amplitudes
and crossing frequencies).  Indeed, as can be seen in Figure \ref{seg fig}.2, the mode $v_1$ will have lowest instantaneous frequency at times prior to the intersection, i.e. over $t \lesssim 0.25$, while the lowest frequency is associated with $v_2$ over $t \gtrsim 0.25$.  Further, in Example \ref{seg ex2} which has modes with frequencies illustrated in Figure \ref{seg fig2}.2, Figure \ref{seg fig2}.3  shows that
 the amplitude of the  mode $v_{1}$  vanishes for $t \gtrsim -0.5$ and therefore will not contribute to a lowest frequency estimation in that interval.  Figure \ref{seg fig2}.2 implies  that $v_{1}$ will appear to have the lowest instantaneous frequency for $t \lesssim -0.5$, $v_{2}$ will appear to for  $t \gtrsim 0.25$, and $v_3$ otherwise.

The algorithms introduced in this section will address these challenges by first estimating the lowest frequency mode at each point of time in $[-1,1]$ and dividing the domain into intervals with continuous instantaneous frequency and $\dot{\theta}_\text{low} \approx \omega_\low$ in Algorithm \ref{mode ident alg}.  Divisions to $[-1, 1]$ can be caused by either a mode vanishing or a frequency intersection.  The portions of modes corresponding to these resulting intervals with identified instantaneous frequencies are called \textit{mode fragments}.  Next, Algorithm \ref{mode ext} extends the domain of these fragments to the maximal domain such that the instantaneous frequency is continuous and $\dot{\theta}_\text{low} \approx \omega_\low$, thus determining  what are called \textit{mode segments}.  The difference between fragments and segments is elaborated in the discussion of Figure \ref{low freq}.  Furthermore, in Algorithm \ref{mode process}, the segments that are judged to be an artifact of noise or a mode intersection are removed. After segments are grouped by the judgment of the user of the algorithm into 
 which belong to the same mode, they
 are then joined via interpolation to create estimates of full modes.  Finally, in Algorithm \ref{segalgn}, mode estimates are refined as in the final refinement loop in Algorithm \ref{iteralgn}. 
 
\subsection{Identifying modes and segments}\label{subsecseg}
Algorithm \ref{mode ident alg}, which follows,  presents the main module $m_\text{mode}(v, \mathcal{V}, \mathcal{V}_\text{seg})$  composing Algorithm
\ref{segalgn}. The input of this module
 is the original signal $v$, a set of (estimated) \textit{modes} 
$\mathcal{V} := \{v_{i, e}:[-1,1]\rightarrow \R\}$, and a set $\mathcal{V}_\text{seg}:=\{v^{i, e}:\Td_{i,e} \rightarrow \R\}$
 of  (estimated) \textit{segments} $v^{i, e}$, where
each  mode is defined in terms of its amplitude $a_{i,e}$ and  phase $\theta_{i,e}$
 as $v_{i, e}(t):=a_{i,e}(t) y(\theta_{i,e}(t))$, and each   segment 
is defined in terms of its  amplitude $a^{i,e}$ and  phase $\theta^{i,e}$ 
 as the function  $v^{i, e}(t):=a^{i,e}(t) y(\theta^{i,e}(t))$ on its closed interval domain
$\Td_{i,e}$. In Algorithm \ref{mode ident alg} we  consider a uniform mesh $\T \subset [-1,1]$ with mesh spacing
$\delta t$ and define a \textit{mesh interval} $[a,b]:=\{t \in \T: a \leq t \leq b\}$, using the same notation
 for
 a mesh interval as a regular closed interval. In particular, both the modes and  segments $v_{i,e}, v^{i,e}$
contain, as data, their amplitudes $a_{i,e}, a^{i,e}$ and phase functions $\theta_{i,e}, \theta^{i,e}$,
  while the  segments additionally contain as data their domain $\Td_{i,e}$. Moreover, their
 frequencies
 $\omega_{i,e}, \omega^{i,e}$ can also be directly extracted 
since they are a function of their phase functions. 
 The output of this module is an updated set of modes $\mathcal{V}^\text{out}$ and 
 segments $\mathcal{V}^\text{out}_\text{seg}$.  The first step of this module (lines \ref{peel start} to \ref{peel end} of Algorithm
 \ref{mode ident alg}) is to compute, for  each time $\tau \in [-1,1]$, 
the residual
\begin{equation}\label{eqkhhhhukhbff0}
    v_\tau := v - \sum_{v_{i, e} \in \mathcal{V}} v_{i, e} - \sum_{v^{i, e} \in \mathcal{V}_\text{seg}:\tau \in  \Td_{i,e}} v_\tau^{i, e}\,
\end{equation}
  of the original signal after peeling  off the  modes and  \textit{localized  segments}, where
the localized  segment 
\begin{equation}\label{sub mode peel}
    v^{i, e}_\tau(t) := a^{i, e}(\tau) e^{-\big(\frac{\omega^{i, e}(\tau)(t - \tau)}{\alpha}\big)^2}
 y\bigl((t - \tau) \omega^{i,
    e}(\tau) + \theta^{i, e}(\tau)\bigr), \quad  t \in [-1,1], \, \tau \in  \Td_{i,e},
\end{equation}
 defined from the amplitude, phase and frequency of segment  $v^{i, e}$, is well-defined
on the whole domain $[-1,1]$  when $\tau\in \Td_{i,e}$. Extending $v^{i, e}_\tau$ so that it is defined as the zero function for $ \tau \notin  \Td_{i,e}$, \eqref{eqkhhhhukhbff0} appears more simply as
\begin{equation}\label{eqkhhhhukhbff}
    v_\tau := v - \sum_{\mathcal{V}} v_{i, e} - \sum_{\mathcal{V}_\text{seg}} v_\tau^{i, e}\,.
\end{equation}

  Note that unlike previous sections where the  function
 $\theta_{0}$, common throughout many iterations,  would be determining the width parameter $\dot{\theta}_{0}(\tau)$ in the exponential in \eqref{sub mode peel},  here the latest frequency estimate $\omega^{i,e}$ is used. The peeling \eqref{eqkhhhhukhbff}
  of the modes, as well as the
 segments,
  off of  the signal $v$ is
 to identify other segments with higher instantaneous frequencies.

Next, in line \ref{omegaebeflow} of Algorithm
 \ref{mode ident alg}, we compute the lowest instantaneous frequency $\omega_\low(\tau, v_\tau)$ of
  $v_\tau$ as in  \eqref{omegalowe}, where $A_{\text{low}}$ is determined either by the user or a set of rules,
  e.g.~we identify $\omega_\low(\tau, v_\tau)$  as the lowest
 frequency local maxima of the energy $\mathcal{S}(\tau, \cdot,
v_\tau)$ that is greater than a set threshold $\epsilon_0$ (in our implementations, we set this
threshold as a fixed fraction of $\max_{\tau, \omega} \mathcal{S}(\tau, \omega, v)$).
If no energies are detected above this given threshold in $\mathcal{S}(\tau,\cdot,v_\tau)$  we set $\omega_{\low}(\tau,
v_\tau)=\emptyset$. We  use the abbreviation  $\omega_{\low}(\tau)
$ for $\omega_{\low}(\tau,
v_\tau)$.  Figure \ref{low freq}.2 shows $\omega_\low(\tau)$ derived from $\S$ (Figure \ref{low freq}.1) in Example \ref{seg ex2}.  
\begin{algorithm}[h!]
\caption{Lowest frequency  segment identification}\label{mode ident alg}
\begin{algorithmic}[1]
\STATE \textbf{function} $m_\text{mode}(v, \mathcal{V}, \mathcal{V}_\text{seg})$
\STATE \quad \textbf{for} $v^{i,e}$ in $\mathcal{V}_\text{seg}$ \textbf{do} \label{peel start}
\STATE \quad \quad $v^{i, e}_\tau(t) \leftarrow a^{i, e}(\tau) e^{-\big(\frac{\omega^{i, e}(\tau)(t - \tau)}{\alpha}\big)^2} y((t
- \tau) \omega^{i, e}(\tau) + \theta^{i, e}(\tau))$
\STATE \quad \textbf{end for}
\STATE \quad $v_\tau \leftarrow v - \sum_{ \mathcal{V}} v_{i, e} - \sum_{ \mathcal{V}_\text{seg}}
v_\tau^{i, e}$ \label{peel end}
\STATE \quad Get $\omega_\low(\tau, v_{\tau})$ as in \eqref{omegalowe} and abbreviate it as
$\omega_\low(\tau)$  \label{omegaebeflow}
\STATE \quad  \textbf{if} $\omega_\low(\tau)\not=\emptyset$
\STATE \quad \quad $a_\low(\tau) \leftarrow a(\tau, (\cdot-\tau)\omega_\low(\tau), v_\tau)$ \label{alow}
\STATE \quad \quad $\theta_\low(\tau) \leftarrow \delta\theta(\tau, (\cdot-\tau) \omega_\low(\tau), v_\tau)$ \label{thetalow}
\STATE \quad  \textbf{end if}
\STATE \quad Set $\mathcal{T}$ to be the regular time mesh with spacing $\delta t$
\STATE \quad $\mathcal{T}  \leftarrow  \mathcal{T}  \cap   \{\tau|\omega_\low(\tau)\not=\emptyset\}$
\STATE \quad \textbf{if} $\mathcal{T}  =\emptyset$ \textbf{then}
\STATE \quad \quad $\mathcal{V}_\text{seg} \leftarrow \emptyset$
\STATE \quad \quad \textbf{return} $\mathcal{V}, \mathcal{V}_\text{seg}$ and goto line \ref{gegbdhdb}
\STATE \quad \textbf{end if}
\STATE \quad $\mathcal{T}_\text{cut} \leftarrow \{[\min(\mathcal{T}), \max(\mathcal{T})]\}$
 (Initialize the set of mesh intervals $\mathcal{T}_\text{cut}$)
 \label{Tcut def}
\STATE \quad \textbf{for} successive $\tau_1, \tau_2$ ($\tau_2 - \tau_1 = \delta t$)  in $\mathcal{T}$  \textbf{do} \label{Tcut loop
start}
\STATE \quad \quad \textbf{if} $\Bigg|\log\bigg(\frac{\omega_\low(\tau_2)}{\omega_\low(\tau_1)}\bigg)\Bigg| > \epsilon_1 \text{ or }
\Bigg|\log\bigg(\frac{(\theta_\low(\tau_2) - \theta_\low(\tau_1))(\tau_2 - \tau_1)^{-1}}{\omega_\low(\tau_1)}\bigg)\Bigg| >
\epsilon_2$ \textbf{then}
\STATE \quad \quad \quad \textbf{if} $[\tau_1, \tau_2] \subset [t_1, t_2] \in \mathcal{T}_\text{cut}$ \textbf{then}
\STATE \quad \quad \quad \quad $\mathcal{T}_\text{cut} \leftarrow (\mathcal{T}_\text{cut} \smallsetminus \{[t_1, t_2]\}) \cup
\{[t_1, \tau_1], [\tau_2, t_2]\}$
\STATE \quad \quad \quad \textbf{end if}
\STATE \quad \quad \textbf{end if}
\STATE \quad \textbf{end for}\label{Tcut loop end}

\STATE \quad $v_\low \leftarrow a_\low y(\theta_\low)$

\STATE \quad \textbf{for} $[t_1, t_2]$ in $\mathcal{T}_\text{cut}$ \textbf{do}\label{Tcutsmall loop start}
\STATE \quad \quad $v_{\text{seg}, [t'_1, t'_2]}, t'_1, t'_2 \leftarrow$ MODE\_EXTEND$(v, v_\low \restrict{[t_1, t_2]},\S(\cdot, \cdot, v_\tau))$
\STATE \quad \quad \textbf{if} $ \int_{t'_{1}}^{t'_{2}}{ \omega_\low(\tau)d\tau}> \epsilon_3$
  \textbf{then}
\STATE \quad \quad \quad $\mathcal{V}_\text{seg} \leftarrow \mathcal{V}_\text{seg} \cup \{v_{\text{seg}, [t'_1, t'_2]}\}$
\STATE \quad \quad \textbf{end if}
\STATE \quad \textbf{end for}\label{Tcut small loop end}

\STATE\label{groupdiscardstep} \quad  $\mathcal{V}^\text{out}, \mathcal{V}^\text{out}_\text{seg} \leftarrow$
MODE\_PROCESS$(\mathcal{V},\mathcal{V}_\text{seg}, \S(\cdot, \cdot, v_\tau))$

\STATE \quad \textbf{return} $\mathcal{V}^\text{out}, \mathcal{V}^\text{out}_\text{seg}$
\STATE \textbf{end function}\label{gegbdhdb}
\end{algorithmic}
\end{algorithm}
Then,  using the micro-local KMD approach of
Section \ref{secmicrolocalkmd} with (the maximum polynomial degree) $d$ set to $0$,
 lines \ref{alow} and \ref{thetalow} of Algorithm \ref{mode ident alg} compute an amplitude
\begin{equation}\label{amp def omega}
    a_\low(\tau) := a(\tau, (\cdot-\tau)\omega_\low(\tau) ,v)
\end{equation}
and  phase
\begin{equation}\label{theta def omega}
    \theta_\low(\tau) := \delta\theta(\tau, (\cdot-\tau)\omega_\low(\tau), v)\, 
\end{equation}
at $t = \tau$,
 using \eqref{eienuriig} applied to the  locally estimated phase function $(\cdot-\tau)\omega_\low(\tau)$ determined by                                  
 the estimated instantaneous frequency
$\omega_\low(\tau)$. 
 The approximation \eqref{theta def omega} is justified since
   this estimated phase function $(\cdot-\tau)\omega_\low(\tau)$ vanishes at $t=\tau$,
   so that the discussion below \eqref{eienuriig} demonstrates that the updated estimated phase  $
 0+\delta\theta(\tau,(\cdot-\tau)\omega_\low(\tau) ,v)=\delta\theta(\tau, (\cdot-\tau)\omega_\low(\tau),v)$
 is an estimate of the
instantaneous phase at $t = \tau$ and  frequency $\omega =\omega_\low(\tau)$. Then
$a_\low(\tau) y(\theta_\low(\tau))$ is an estimate, at $t = \tau$, of the mode having the lowest frequency.
 If
$\omega_\low(\tau) = \emptyset$, we leave $a_\low$ and $\theta_\low$
undefined.

\begin{figure}[h]
	\begin{center}
			\includegraphics[width=\textwidth]{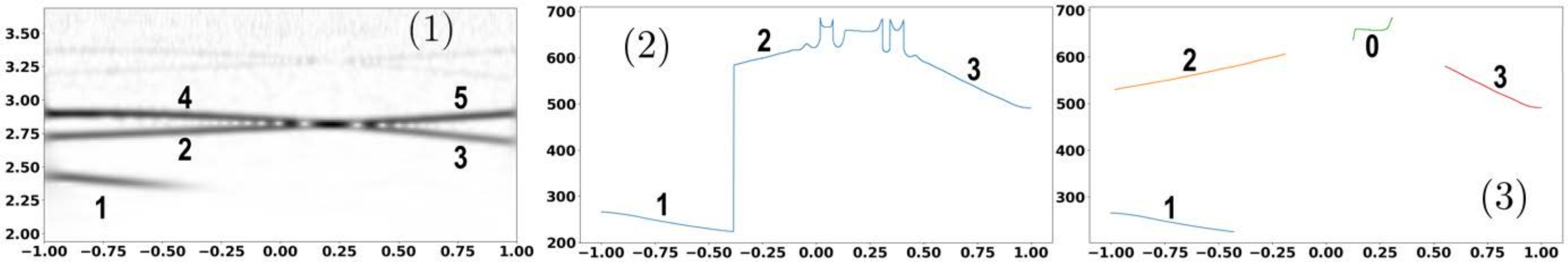}
		\caption{The identification of the first mode segments in Example \ref{seg ex2} is shown.
The scale of the vertical axis is $\log_{10}(\omega)$ in sub-figure (1)
 and $\omega$ in sub-figures (2) and (3)  Segments are labeled in  (1).  (1) Energy $\S(\cdot, \cdot, v)$ (2) the identified lowest frequency  at each time $t$ with consistent segment numbering (3) identified mode segments including an artifact of the intersection, labeled as segment 0.} 
\label{low freq}
	\end{center}
\end{figure}

\begin{algorithm}[h!]
\caption{Mode fragment extension}\label{mode ext}
\begin{algorithmic}[1]
\STATE \textbf{function} $\text{MODE\_EXTEND}(v, v_\text{seg}, \S(\cdot, \cdot, v_\tau))$
\STATE \quad smooth $\leftarrow$ True
\STATE \quad $\tau_1 \leftarrow t_1$
\STATE \quad \textbf{while} smooth is True
\textbf{do} \label{loop ext 1}
\STATE \quad \quad $\theta_1 \leftarrow \theta_\text{seg} (\tau_1)$
\STATE \quad \quad $\omega_1 \leftarrow \dot{\theta}_\text{seg}(\tau_1)$
\STATE \quad \quad $\tau_2 \leftarrow \tau_1 - dt$
\STATE \quad \quad $\omega_2 \leftarrow \operatorname{argmax}_{\omega\in[(1-\varepsilon) \omega_1, (1+\varepsilon) \omega_2]} \S(\tau_2, \omega, v_\tau)$ \label{maxsq modeext1}
\STATE \quad \quad $\theta_2 \leftarrow \delta\theta(\tau_2, (\cdot-\tau_2)\omega_2, v_\tau)$
\STATE \quad \quad \textbf{if} $\bigg|\log\Big(\frac{\omega_2}{\omega_1}\Big)\bigg| > \epsilon_1 \text{ or }
\bigg|\log\Big(\frac{(\theta_2 - \theta_1)(\tau_2 - \tau_1)^{-1}}{\omega_1}\Big)\Bigg| >
\epsilon_2$ \textbf{then}\label{cut modeext1}
\STATE \quad \quad \quad smooth $\leftarrow$ False
\STATE \quad \quad \textbf{else}
\STATE \quad \quad \quad $a_2 \leftarrow a(\tau_2, (\cdot-\tau_2)\omega_2, v_\tau)$
\STATE \quad \quad \quad  $v_\text{seg}(\tau_2) \leftarrow a_2 y(\theta_2)$
\STATE \quad \quad \quad $t_1, \tau_1 \leftarrow \tau_2$
\STATE \quad \quad \textbf{end if}
\STATE \quad \textbf{end while}
\STATE \quad $\tau_1 \leftarrow t_2$
\STATE \quad \textbf{while} smooth is True
\textbf{do} \label{loop ext 2}
\STATE \quad \quad $\theta_1 \leftarrow \theta_\text{seg} (\tau_1)$
\STATE \quad \quad $\omega_1 \leftarrow \dot{\theta}_\text{seg}(\tau_1)$
\STATE \quad \quad $\tau_2 \leftarrow \tau_1 + dt$
\STATE \quad \quad $\omega_2 \leftarrow \operatorname{argmax}_{\omega\in[(1-\varepsilon) \omega_1, (1+\varepsilon) \omega_2]} \S(\tau_2, \omega, v_\tau)$\label{maxsq modeext2}
\STATE \quad \quad $\theta_2 \leftarrow \delta\theta(\tau_2, (\cdot-\tau_2)\omega_2, v_\tau)$
\STATE \quad \quad \textbf{if} $\bigg|\log\Big(\frac{\omega_2}{\omega_1}\Big)\bigg| > \epsilon_1 \text{ or }
\bigg|\log\Big(\frac{(\theta_2 - \theta_1)(\tau_2 - \tau_1)^{-1}}{\omega_1}\Big)\Bigg| >
\epsilon_2$ \textbf{then}\label{cut modeext2}
\STATE \quad \quad \quad smooth $\leftarrow$ False
\STATE \quad \quad \textbf{else}
\STATE \quad \quad \quad $a_2 \leftarrow a(\tau_2, (\cdot-\tau_2)\omega_2, v_\tau)$
\STATE \quad \quad \quad  $v_\text{seg}(\tau_2) \leftarrow a_2 y(\theta_2)$
\STATE \quad \quad \quad $t_2, \tau_1 \leftarrow \tau_2$
\STATE \quad \quad \textbf{end if}
\STATE \quad \textbf{end while}

\STATE \quad \textbf{return} $v_\text{seg}, t_1, t_2$
\STATE \textbf{end function}
\end{algorithmic}
\end{algorithm}

Next, let us describe how we use the values of $(\tau, \omega_\low(\tau))$ to determine the interval domains for   segments.   Writing $\mathcal{T}_\text{cut}$ for the set of
interval domains
 of these segments,  $\mathcal{T}_\text{cut}$ is initially set, in
line \ref{Tcut def}, to contain
  the single element $\mathcal{T}$,  that is, the entire time mesh $\mathcal{T}$.  We split an element of $\mathcal{T}_\text{cut}$ whenever
  $\omega_\low$ is not continuous or
$\dot{\theta}_\low$ and $\omega_\low$ are not approximately equal, as follows. If our identified instantaneous frequency around
$t=\tau$ matches a single mode, we expect neither condition to be satisfied, i.e. we expect both $\omega_\low$ to be continuous and
$\dot{\theta}_\low \approx \omega_\low$.    In our discrete implementation (lines \ref{Tcut loop start} to \ref{Tcut loop end}), we
introduce a cut between two successive points,  $\tau_1$ and $\tau_2$, of the time mesh $\T$,
  if
\begin{equation}\label{disc cut}
    \Bigg|\log\bigg(\frac{\omega_\low(\tau_2)}{\omega_\low(\tau_1)}\bigg)\Bigg|>\epsilon_1\,\, \text{ or }
\,\,\,\,
    \Bigg|\log\bigg(\frac{(\theta_\low(\tau_2) - \theta_\low(\tau_1))(\tau_2 -
    \tau_1)^{-1}}{\omega_\low(\tau_1)}\bigg)\Bigg|>\epsilon_2\, ,
\end{equation}
where $\epsilon_1$ and $\epsilon_2$ are pre-set thresholds.  Each potential mode segment is then identified as $v_\low\restrict{[t_1, t_2]}$ for some $t_{1} < t_{2}, \, t_1, t_2 \in \T$. 

Note that in Figure \ref{low freq}.2, the continuous stretch of $\omega_\low$ labeled by $2$ does not correspond to the  full mode segment labeled by $2$ in Figure \ref{low freq}.1, but a fragment of it. This is because the lowest frequency mode, $v_1$, is identified by $\omega_\low(t)$ over $t \lesssim -0.5$.  We designate this partially identified mode segment as a mode fragment.  Such fragments are extended to fully identified segments (as in $2$ on Figure \ref{low freq}.3) with the MODE\_EXTEND module, with pseudo-code shown in Algorithm \ref{mode ext}.  This  MODE\_EXTEND module iteratively extends the support, $[t_1, t_2]$, by applying, in lines \ref{maxsq modeext1} and \ref{maxsq modeext2}, 
a max-squeezing to identify  instantaneous frequencies
 at neighboring mesh points to the left and right of the  interval $[t_1, t_2]$.  The process is stopped 
if it is detected, in lines \ref{cut modeext1} and \ref{cut modeext2}, 
that the extension is discontinuous in phase according to \eqref{disc cut}.  This sub-module returns (maximally continuous) full mode segments.  Furthermore, to remove segments that may be generated by noise or are mode intersections, in lines \ref{Tcutsmall loop start} to \ref{Tcut small loop end} of Algorithm  \ref{mode ident alg},
 segments such that
\begin{equation}
\label{noisecond}
\int_{t_{1}}^{t_{2}}{ \omega_\low(\tau)d\tau}\leq  \epsilon_3 \,
\end{equation}
where   $\epsilon_3$ is a threshold, are removed.
  In our implementation, we take  $\epsilon_3:=20 \pi$, corresponding to 10 full periods. Note that
  Figure \ref{noise artifacts}.2 shows those segments deemed noise at level $\epsilon_3:=20 \pi$  but which are not deemed noise at level  $3 \pi$, in the step after all three modes have been estimated in Example \ref{seg ex}.
Consequently, it appears that the noise level $\epsilon_3:=20 \pi$ successfully removes most noise artifacts.  Note that the mode segments in Figure \ref{noise artifacts}.2 are short and have quickly varying frequencies compared to those of full modes.

\begin{figure}[h]
	\begin{center}
			\includegraphics[width=\textwidth]{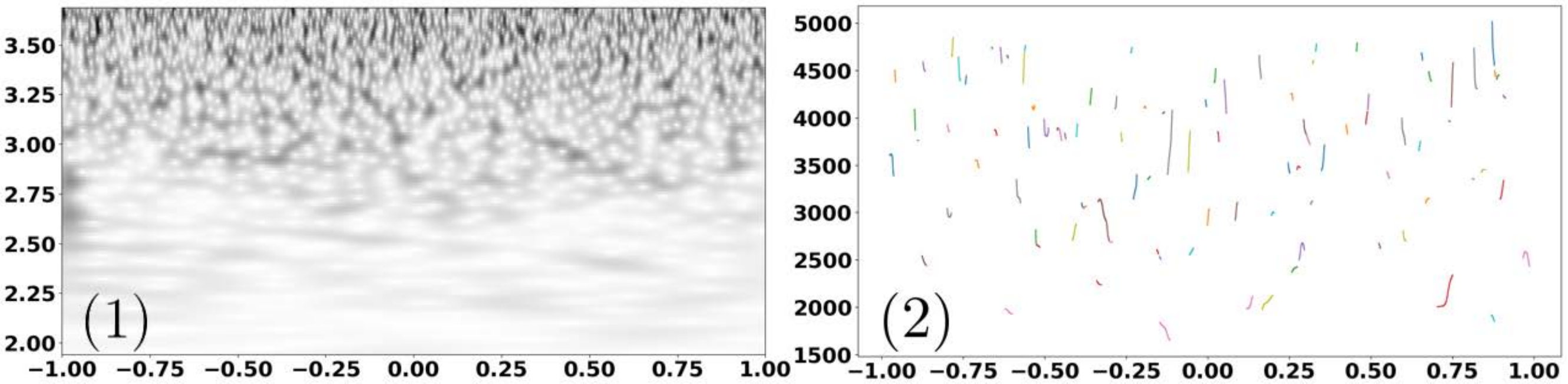}
		\caption{(1) Energy $\S(\cdot, \cdot, v - v_{1,e} - v_{2,e} - v_{3,e})$ (2) identified mode segments ($\mathcal{V}_\text{seg}$ obtained after the loop in Algorithm \ref{mode ident alg} on line \ref{Tcut small loop end}).} 
\label{noise artifacts}
	\end{center}
\end{figure}

\begin{algorithm}[h!]
\caption{Raw  segment processing}\label{mode process}
\begin{algorithmic}[1]
\STATE \textbf{function} $\text{MODE\_PROCESS}(\mathcal{V}, \mathcal{V}_\text{seg}, \S(\cdot, \cdot, v_\tau))$
\STATE \quad $\mathcal{V}_\text{group} \leftarrow \emptyset$ \label{V group}
\STATE \quad \textbf{for} $v^{i,e}$ in $\mathcal{V}_\text{seg}$ \textbf{do} \label{loop group}
\STATE \quad \quad \textbf{if} $v^{i, e}$ corresponds to a mode intersection or noise
\textbf{then}\label{disregard conds}
\STATE \quad \quad \quad $\mathcal{V}_\text{seg} \leftarrow \mathcal{V}_\text{seg} \smallsetminus \{v^{i,e}\}$\label{disregard}
\STATE \quad \quad \textbf{else}\label{group cond}
\STATE \quad \quad \quad \textbf{for} $\mathcal{V}_{\text{group}, j}$ in $(\mathcal{V}_{\text{group}, j'})_{j'}$ \textbf{do} \label{group check
loop}
\STATE \quad \quad \quad \quad \textbf{if} $v^{i,e}$ corresponds to the same mode as $\mathcal{V}_{\text{group}, j}$ \textbf{then}
\STATE \quad \quad \quad \quad \quad $\mathcal{V}_{\text{group}, j} \leftarrow \mathcal{V}_{\text{group}, j} \cup \{v^{i,e}\}$ \label{group enter}
\STATE \quad \quad \quad \quad \quad break for loop
\STATE \quad \quad \quad \quad \textbf{end if}
\STATE \quad \quad \quad \textbf{end for} \label{group check loop end}
\STATE \quad \quad \quad \textbf{if} $v^{i,e}$ not added to any  mode block \textbf{then}
\STATE \quad \quad \quad \quad $(\mathcal{V}_{\text{group}, j'})_{j'}  \leftarrow (\mathcal{V}_{\text{group}, j'})_{j'}
 \cup \{\{v^{i,e}\}\}$\label{add V
group}
\STATE \quad \quad \quad \textbf{end if}\label{group assorted end}
\STATE \quad \quad \textbf{end if}\label{group cond end}
\STATE \quad \textbf{end for}\label{loop group end}

\STATE \quad \textbf{for} $\mathcal{V}_{\text{group}, j}$ in $(\mathcal{V}_{\text{group}, j'})_{j'}$ do \label{full mode check}
\STATE \quad \quad \textbf{if} $\mathcal{V}_{\text{group}, j}$  is complete
\textbf{then} \label{mode join}
\STATE \quad \quad \quad Transform the   segments in  $ \mathcal{V}_{\text{group}, j}$ into a mode
$v_{j,e}$ \label{join group}
\STATE \quad \quad \quad $\mathcal{V} \leftarrow \mathcal{V} \cup \{v_{j,e}\}$ \label{add to modes}
\STATE \quad \quad \quad $\mathcal{V}_\text{seg} \leftarrow \mathcal{V}_\text{seg} \smallsetminus \mathcal{V}_{\text{group},
j}$\label{remove group from seg}
\STATE \quad \quad \textbf{end if}\label{mode join end}
\STATE \quad \textbf{end for}\label{full mode check end}

\STATE \quad \textbf{return} $\mathcal{V}, \mathcal{V}_\text{seg}$
\STATE \textbf{end function}
\end{algorithmic}
\end{algorithm}

Next, line \ref{groupdiscardstep} of Algorithm  \ref{mode ident alg} applies the function MODE\_PROCESS,  Algorithm \ref{mode process},
 to $\mathcal{V}$ and $\mathcal{V}_\text{seg}$, the sets  of modes and  segments,
 as well as the energy $\S(\cdot, \cdot, v_\tau)$, to produce the updated sets 
$\mathcal{V}^{out}$ and $\mathcal{V}^{out}_\text{seg}$. 
This function utilizes a partition
 of a  set
$\mathcal{V}_\text{group}$,  initialized to be empty,  into a set of partition blocks $\bigl(\mathcal{V}_{\text{group}, j}\bigr)_{j}$,
where 
 $\mathcal{V}_{\text{group}, j} \subset \mathcal{V}_\text{group},\, \forall j$. 
 The partition blocks 
  consist of  segments 
that have been identified as corresponding to 
the same mode, indexed locally by $j$.  
 Each   segment in $\mathcal{V}_\text{seg}$ will either be discarded or placed into a partition block.
When a partition block is \textit{complete} it
 will be turned into a  mode in $\mathcal{V}^\text{out}$   
by interpolating instantaneous frequencies and amplitudes in the (small) missing sections of $\T$ and the 
elements of the partition block removed from $\mathcal{V}_\text{group}$ and $\mathcal{V}_{\text{seg}}$.  All partition blocks that are not complete
 will be passed-on to the next iteration.  These selection steps
 depend on the prior information about the modes composing
the signal and may be based on (a) user input  and/or (b) a set of pre-defined rules.  Further details and rationale on the options to
discard,  place  segments into partition blocks, and determine the completeness of a block, will be discussed in the following paragraphs.
The first loop in Algorithm \ref{mode process},  lines \ref{loop group} to \ref{loop group end}, takes each 
 segment $v^{i,e}$ in $\mathcal{V}_\text{seg}$, and either discards it, adds it to a  partition block  in
$\mathcal{V}_\text{group}$, or creates a new partition block with it.  
 On line \ref{disregard conds}, we specify that a  segment is to be discarded (i.e.~removed from the
set of  segments $\mathcal{V}_\text{seg}$) whenever it corresponds to a mode intersection or noise, where
we identify a mode intersection
  whenever two modes' instantaneous frequencies match at any particular time.  This
can be seen in Figure \ref{seg group}.1 where the energies for the higher two frequency modes on $t \gtrsim -.25$
meet in frequency at time $t \approx
0.25$, as well as Figure \ref{low freq}.1, where the lower two frequency modes  on $t \gtrsim -.25$ also meet around $t \approx 0.25$.  Moreover, segment $0$ in Figure \ref{low freq}.3 corresponds an artifact of this mode intersection.  In these two examples, it has been observed selecting $\epsilon_3$ large enough leads to no identified noise artifacts.  However, identified segments with these similar characteristics as those in Figure \ref{noise artifacts}.2, i.e. short with rapidly varying frequency, are discarded, especially if there is a prior knowledge of noise in the signal.  
   
\begin{figure}[h]
	\begin{center}
			\includegraphics[width=\textwidth]{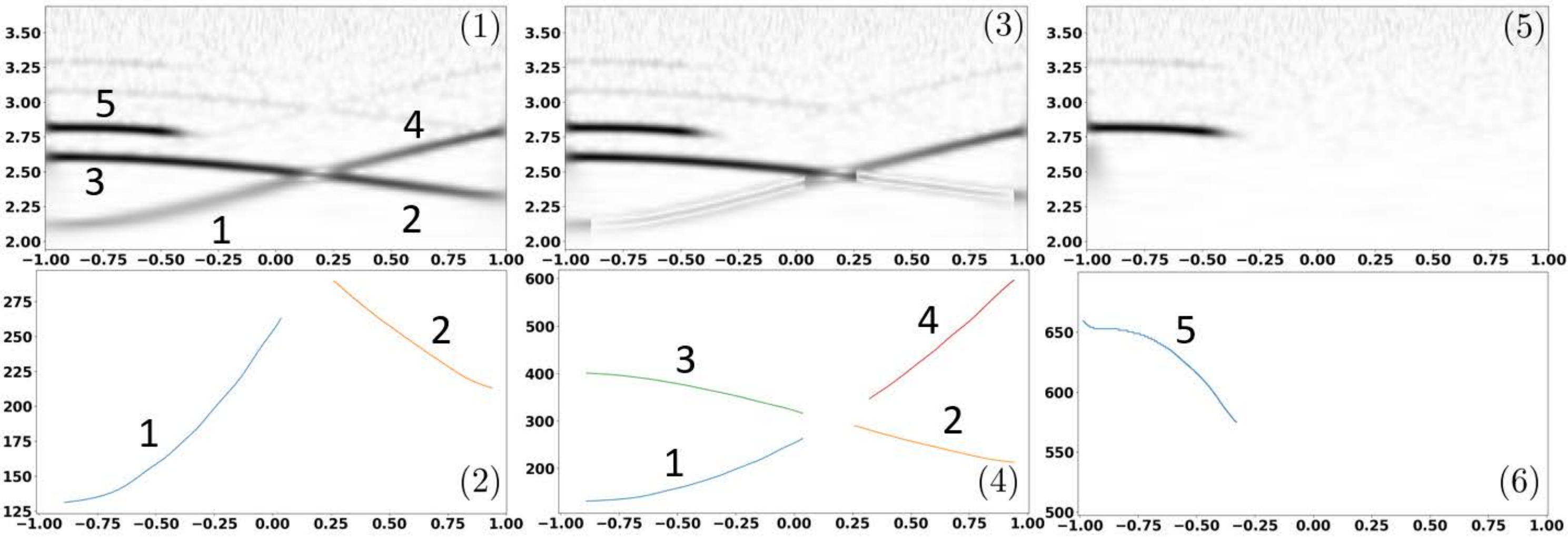}
		\caption{ 
The scale of the vertical axis  is $\log_{10}(\omega)$ in the top row of sub-figures (1,3,5)
 and $\omega$ in  bottom row of sub-figures (2,4,6). 
  Segments are labeled in  (1).  (1, 2) Energy $\S(\cdot, \cdot, v)$ and the identified lowest frequency  segments (3, 4) First updated energy $\S(\cdot, \cdot, v -
v^{1, e} - v^{2, e})$ and its identified lowest frequency  segments (5, 6) Second updated energy $\S(\cdot, \cdot, v - v_{1, e} - v_{2, e})$ and its  identified lowest frequency  segments, where $v_{1,e}$ results from joining mode segments 1 and 4, while $v_{2, e}$ is generated from joining segments 3 and 2.} 
\label{seg group}
	\end{center}
\end{figure}
All segments $v^{i,e}$ that are not discarded are iteratively put into existing partition blocks
 in lines \ref{group check loop}-\ref{group check loop
end} of Algorithm \ref{mode process}, or used to create a  new partition block in  line \ref{add V group}, which we denote by  
 $\{\{
v^{i,e}\}\}$. 
For example, in Figure \ref{seg group}.2, we place segment $1$ into its own partition block on line \ref{add V group} by
default since when $\mathcal{V}_\text{group}$ is empty, the loop from lines \ref{group check loop}-\ref{group check loop end} is not
executed.  Then we do not place segment $2$ in the partition block 
 with segment $1$, but again place it in its own partition block on line \ref{add V group} with the observation they belong to different modes (based on the max-squeezed energy $\S$ in Figure \ref{seg group}.1).  The end result of this iteration is segments $1$ and $2$ placed into separate partition blocks. In the next iteration shown in \ref{seg group}.4, we construct two partition blocks, one consisting of $\{\{v^{1,e}, v^{4, e}\}\}$ and the other $\{\{v^{2,e}, v^{3, e}\}\}$.  In the following iteration, illustrated in Figure \ref{seg group}.6, we again place segment $5$ into its own partition block on line \ref{add V group} by default.  The next iteration is the last since no segments which violate \eqref{noisecond} are observed.
Both blocks are then designated as {\em complete modes}, that is
 correspond to a  mode at all time $[-1, 1]$,  and are used to construct $v_{1, e}$ and $v_{2, e}$.
 This determination can
 be based on (a) user input  and/or (b) a set of pre-defined rules.
  Observing $\S$ at the third stage in Figure \ref{seg group}.5, we designate it as complete.

The final loop of Algorithm \ref{mode process} on lines \ref{full mode check}-\ref{full mode check end} begins by checking whether the block is complete. 
 For a block  deemed complete,  in line \ref{join group}, their segments are combined to create an estimate of their corresponding mode by interpolating the 
 amplitude and phase to fill the gaps and extrapolation by zero to the boundary. 
Then, in line  \ref{add to modes}, this estimated mode  is added to $\mathcal{V}$ and,  in line 
\ref{remove group from seg}, its generating segments removed from
$\mathcal{V}_\text{seg}$.  Finally,  the segments of the incomplete blocks constitute the output
$\mathcal{V}_\text{seg}$  of 
Algorithm \ref{mode process}.

In the implementation corresponding to Figure \ref{seg group}.2, each block consisting of segments $1$ and $2$ respectively are both determined to not be complete, and hence are passed to the next iteration as members of $\mathcal{V}_\text{seg}$ to the next iteration.  In Figure \ref{seg group}.4, the block consisting of  segments $1$ and $4$ and the block consisting of
 segments $2$ and $3$,  are deemed complete since  each block  appears to 
contain different portions of the same  mode (with missing portions corresponding  to the intersection between
the corresponding  modes around $t \approx 0.25$), and consequently  their
 segments are therefore designated to be turned   into 
 modes $v_{1,e}$ from segments $1$ and $4$ and $v_{2,e}$ from $2$ and $3$.  Finally in Figure \ref{seg group}.6, the block consisting of only segment $5$
 is determined to be complete and in line \ref{join group} is extrapolated by zero
 to produce its corresponding mode.  In  Example \ref{seg ex2}, shown in Figure \ref{low freq}.3, we place segments $1$, $2$, and $3$ in separate blocks (and disregard $0$), but only designate the block containing segment $1$ as complete. The output of Algorithm \ref{mode process}, and hence Algorithm \ref{mode ident alg}, are the updated list of
 modes and
segments.

\subsection{The segmented micro-local KMD algorithm}
The segmented iterated micro-local algorithm identifies full modes in the setting of Problem \ref{recov general} and is presented in Algorithm \ref{segalgn}.  Except for the call of the function  $m_\text{mode}$, Algorithm \ref{mode ident alg}, Algorithm \ref{segalgn} is similar to Algorithm \ref{iteralgn}.  It is initialized by $\mathcal{V} =\emptyset$ and  $\mathcal{V}_\text{seg}  =\emptyset$, and the main iteration between lines \ref{al2st1seg} and \ref{al2st1eseg}  identifies the modes or segments with lowest instantaneous frequency and then provides refined estimates for the amplitude and the phase of each mode $v_{i},i\in \{1,\ldots,m\}$ of the signal $v$.  We first apply $m_\text{mode}$ to  identify  segments to be passed-on to the next iteration and mode-segments to be combined into  modes.  This set of recognized  modes $\mathcal{V}^\text{out}$ will be refined in the loop between lines \ref{stit0seg} to \ref{stit4seg} by iteratively applying the micro-local KMD steps of Section \ref{secmicrolocalkmd} on the base frequency of each mode (these steps correspond to the final optimization loop, i.e.~lines \ref{al2stfr1} to \ref{al2stfr1e} in Algorithm \ref{iteralgn}).  The loop is terminated when no additions are made to $\mathcal{V}$ or $\mathcal{V}_\text{seg}$.  

\begin{algorithm}[h!]
\caption{Segmented iterated micro-local KMD.}\label{segalgn}
\begin{algorithmic}[1]
\STATE $\{\mathcal{V}, \mathcal{V}_\text{seg}\} \leftarrow \{\emptyset, \emptyset\}$
\WHILE{true}\label{al2st1seg}
\STATE $\{\mathcal{V}^\text{out}, \mathcal{V}^\text{out}_\text{seg}\} \leftarrow m_\text{mode}(v, \mathcal{V},
\mathcal{V}_\text{seg})$ \label{call mmode}
\IF {$\mathcal{V}^\text{out}_\text{seg}=\emptyset$ and $|\mathcal{V}^\text{out}| = |\mathcal{V}|$
}
\STATE {break loop}
\ENDIF
\IF{$|\mathcal{V}^\text{out}| > |\mathcal{V}|$ } \label{ieieiei}
\REPEAT  \label{stit0seg}
\FOR{$v_{i, e}$ in $\mathcal{V}^\text{out}$}
\STATE\label{stit1seg}  $v_{i,\mathrm{res}}\leftarrow v - a_{i, e}\bar{y}(\theta_{i, e}) -  \sum_{j \neq i} a_{j, e} y(\theta_{j,
e})$
\STATE\label{stit2seg}   $a_{i,e}(\tau) \leftarrow a\big(\tau, \theta_{i,e}, v_{i,\mathrm{res}}\big) / c_1$
\STATE\label{stit3seg}   $\theta_{i,e}(\tau) \leftarrow  \theta_{i,e}(\tau)+\frac{1}{2}\delta \theta \big(\tau, \theta_{i,e},
v_{i,\mathrm{res}}\big)$
\ENDFOR
\UNTIL  {$\sup_{i, \tau} |\delta \theta \big(\tau, \theta_{i,e},
v_{i,\mathrm{res}}\big) |< \epsilon_1$}\label{stit4seg}
\ENDIF
\STATE\label{stposeg}  $\{\mathcal{V}, \mathcal{V}_\text{seg}\} \leftarrow \{\mathcal{V}^\text{out},
\mathcal{V}^\text{out}_\text{seg}\}$
\ENDWHILE\label{al2st1eseg}
\STATE Return the modes $v_{i, e}(t) \leftarrow a_{i, e}(t) y(\theta_{i, e}(t))$ for $i = 1, ..., m$
\end{algorithmic}
\end{algorithm}

\subsection{Numerical experiments}\label{seg examp}

\begin{figure}[h]
	\begin{center}
			\includegraphics[width=\textwidth]{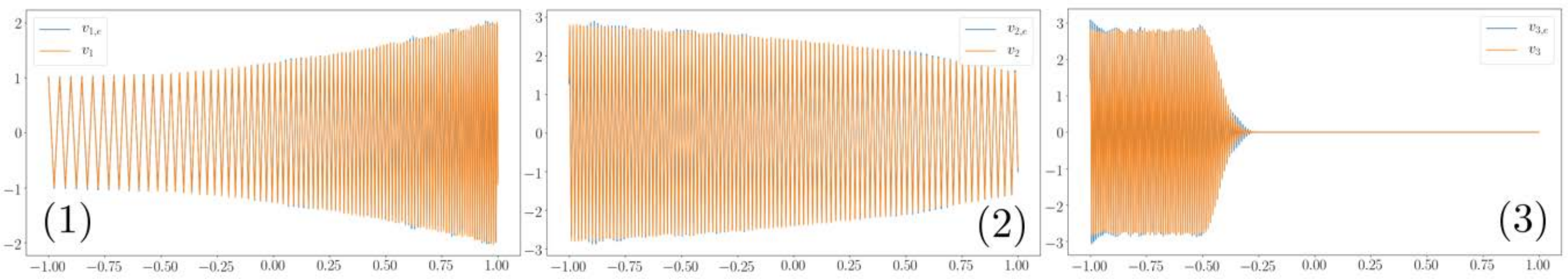}
		\caption{(1) $v_{1,e}$ and $v_1$ (2) $v_{2,e}$ and $v_2$ (3) $v_{3,e}$ and $v_3$. See footnote \ref{redblue}}
\label{seg recov}
	\end{center}
\end{figure}

Figure \ref{seg recov} and Table \ref{seg results} show the accuracy of Algorithm \ref{segalgn} in recovering the modes of the
signal
described in Example \ref{seg ex}, the results for Example \ref{seg ex2} appearing essentially the same, and thereby quantify its robustness to noise, vanishing amplitudes, and crossing frequencies.
We again take the mesh spanning $[-1, 1]$ spaced at intervals of size $\frac{1}{5000}$ and
aim to recover  each  mode $v_i$ on the whole interval $[-1, 1]$.  We kept $\alpha = 25$ constant in our implementation.
The amplitudes and frequencies of the  modes composing $v$ are shown in Figure \ref{seg fig}.    The recovery errors of the modes are
found to be consistently on the order of $10^{-2}$.  Note that in the noise-free setting with identical modes, the recovery error is on the order of $10^{-3}$ implying the noise is mainly responsible for the errors shown in Table \ref{seg results}.

\begin{table*}[h]
\centering
\begin{tabular}{ |p{3.25cm}||p{1.85cm}|p{1.85cm}|p{1.85cm}|p{1.85cm}|  }
\hline
 Mode & $\frac{\|v_{i, e} - v_i\|_{L^2}}{\|v_i\|_{L^2}}$ & $\frac{\|v_{i, e} - v_i\|_{L^\infty}}{\|v_i\|_{L^\infty}}$ &
 $\frac{\|a_{i, e} - a_i\|_{L^2}}{\|a_i\|_{L^2}}$ & $\|\theta_{i, e} - \theta_i\|_{L^2}$ \\
 \hline
 $i=1$ & $3.17 \times 10^{-2}$ & $6.99 \times 10^{-2}$ & $2.24 \times 10^{-2}$ & $1.99 \times 10^{-2}$ \\
  $i=2$ & $2.49 \times 10^{-2}$ & $7.09 \times 10^{-2}$ & $1.64 \times 10^{-2}$ & $1.81 \times 10^{-2}$\\
  $i=3$ & $3.52 \times 10^{-2}$ & $9.52 \times 10^{-2}$ & $3.13 \times 10^{-2}$ & $2.02 \times 10^{-2}$ \\
 \hline
\end{tabular}
\caption{Signal component recovery errors in Example \ref{seg ex}.  Note that the error in phase for mode $i=3$ was calculated over $[-1,-\frac{1}{3}]$ since the phase of a zero signal is undefined.}
\label{seg results}
\end{table*}

\section{Proofs}
\label{sec_Proofs}
\subsection{Proof of Lemma  \ref{lem_basic}}

We first establish that
$ \Psi(v)=\Phi^{+}v,$
where  the  Moore-Penrose inverse $\Phi^{+}$  is defined by
$ \Phi^{+}:=\Phi^{T}\bigl(\Phi \Phi^{T} \bigr)^{-1},  $ where $\Phi^{T}$ is the Hilbert space adjoint of $\Phi$.
To that end, let $w^{*}$ be the solution of \eqref{eqmindsobfirstdeb}.
Since $\Phi:\B \rightarrow V$ is surjective it follows that
$\Phi:\Ker^{\perp}(\Phi) \rightarrow V$ is a bijection and therefore
\[\{w: \Phi w=v\}=w_{0}+\Ker(\Phi)\]
for a unique $w_{0}\in \Ker^{\perp}(\Phi)$. Therefore, setting
$w':=w-w_{0}$ we find that
$(w')^{*}:=w^{*}-w_{0}$ is a solution of
\begin{equation}\label{eqmindsobfirstdeb2}
\begin{cases}
\text{Minimize }\|w'+w_{0}\|_\B\\
\text{Subject to }w'\in \B\text{ and }\Phi w'=0\,,
\end{cases}
\end{equation}
so that by the projection theorem we have
$(w')^{*}=P_{\Ker(\Phi)}(-w_{0})$
where  $P_{\Ker(\Phi)}$ is the orthogonal projection onto
$\Ker(\Phi)$.
Therefore
$ w^{*}=w_{0}+(w')^{*}=w_{0}-P_{\Ker(\Phi)}(w_{0})=P_{\Ker^{\perp}(\Phi)}w_{0}$, so that we obtain
\[  w^{*}=P_{\Ker^{\perp}(\Phi)}w_{0}.\]

Since $\Phi$ is surjective and continuous it follows from the closed range theorem, see e.g.~Yosida
\cite[p.~208]{yosida1995functional} that $\Img(\Phi^{T})=\Ker^{\perp}(\Phi)$ and
$\Ker(\Phi^{T}) =\emptyset$, which implies that
$\Phi \Phi^{T}:V \rightarrow V$ is invertible, so that the Moore-Penrose inverse
$\Phi^{+}:V \rightarrow \B$ of $\Phi$,  is well-defined by
\[ \Phi^{+}:=\Phi^{T}\bigl(\Phi \Phi^{T} \bigr)^{-1}\, . \]
It follows   that
$P_{\Ker^{\perp}(\Phi)}=\Phi^{+}\Phi$
and
$\Phi\Phi^{+}=I_{V}\, $
so  that
\[
w^{*}= P_{\Ker^{\perp}(\Phi)}w_{0}
= \Phi^{+}\Phi w_{0}
= \Phi^{+}v,
\]
that is, we obtain the second assertion
$w^{*}= \Phi^{+}v$.

For the first assertion, suppose that $\Ker{\Phi}=\emptyset$. Since it is surjective, it follows that $\Phi$ is a bijection.
 Then, the unique solution to the minmax problem
is the only feasible one  $w^{*}=\Phi^{-1}v= \Phi^{+}v$.
When $\Ker{\Phi}\neq \emptyset$, observe that
since all $u$ which satisfy
$\Phi u=v$ have the representation  $u=w_{0}+ u'$ for fixed $w_{0} \in \Ker^{\perp}(\Phi)$
 and some $u' \in \Ker(\Phi)$, it follows that
the inner maximum satisfies
\begin{eqnarray*}
 \max_{u\in \B\mid \Phi u=v}\frac{\|u-w\|_\B}{\|u\|_\B}&=&
 \max_{u'\in \Ker(\Phi)}\frac{\|u'+w_{0}-w\|_\B}{\|u'+w_{0}\|_\B}\\
&=&
\max_{u'\in \Ker(\Phi)}
 \max_{ t \in \R}\frac{\|tu'+w_{0}-w\|_\B}{\|tu'+w_{0}\|_\B}\\
&\geq&
 1
\end{eqnarray*}
On the other hand, for
$w:=\Phi^{+}v$, we have
\begin{eqnarray*}
 \max_{u\in \B\mid \Phi u=v}\frac{\|u-w\|_\B}{\|u\|_\B}&=&
\max_{u\in \B\mid \Phi u=v}\frac{\|u-\Phi^{+}v\|_\B}{\|u\|_\B}\\
&=&\max_{u\in \B\mid \Phi u=v}\frac{\|u-\Phi^{+}\Phi u\|_\B}{\|u\|_\B}\\
&=&\max_{u\in \B\mid \Phi u=v}\frac{\|u-P_{\Ker^{\perp}(\Phi)}u\|_\B}{\|u\|_\B}\\
&\leq& 1,
\end{eqnarray*}
which implies that
$w:=\Phi^{+}v$ is a minmax solution.
To see that it is  the unique  optimal solution,  observe that we have just established that
\begin{equation}
\label{eoiejiiirni}
\max_{u\in \B\mid \Phi u=v}\frac{\|u-\Psi(v)\|_\B}{\|u\|_\B}=1
\end{equation}
for any optimal $\Psi:V \rightarrow \B$.
  It then follows that
\[\max_{u\in \B}\frac{\|u-\Psi(\Phi u)\|_\B}{\|u\|_\B}=1\]
which implies that the map
$I-\Psi\circ \Phi:\B \rightarrow  \B$ is a contraction. Moreover,
by selecting  $u \in \Ker(\Phi)$ tending to $0$, it follows
from \eqref{eoiejiiirni} that $\Psi(0)=0$. Since, by definition,
   $\Phi \circ \Psi=I_{V}$, we have
\begin{eqnarray*}
(I-\Psi\circ \Phi)^{2}(u)&=& (I-\Psi\circ \Phi)(u-\Psi\circ \Phi u)\\
&=& u-\Psi\circ \Phi u -\Psi\circ \Phi (u-\Psi\circ \Phi u)\\
&=& u-\Psi\circ \Phi u -\Psi\bigl(\Phi u-\Phi\circ\Psi\circ \Phi u\bigr)\\
&=& u-\Psi\circ \Phi u -\Psi\bigl(\Phi u- \Phi u\bigr)\\
&=& u-\Psi\circ \Phi u -\Psi\bigl(0\bigr)\\
&=& u-\Psi\circ \Phi u
\end{eqnarray*}
so that
the map $I-\Psi\circ \Phi$ is a projection.
Since
$\Phi(u-\Psi\circ \Phi u)=\Phi u-\Phi \circ \Psi\circ \Phi u=0$ it follows that
$\Img(I-\Psi\circ \Phi) \subset \Ker(\Phi)$, but since
 for $b \in  \Ker(\Phi)$, we have
$(I-\Psi\circ \Phi)(b) =b-\Psi\circ \Phi b =b$,  we obtain the equality
$\Img(I-\Psi\circ \Phi) = \Ker(\Phi)$.

To show that a projection of this form is necessarily linear, let us demonstrate that
$\Img(\Psi\circ \Phi) = \Ker^{\perp}(\Phi)$. To that end, use  the decomposition
$\B=\Ker(\Phi)\oplus \Ker^{\perp}(\Phi)$ to write
$u= u'+u''$  with
$u' \in \Ker(\Phi)$ and $u'' \in \Ker^{\perp}(\Phi)$ and write the contractive condition
$\|u-\Psi\circ \Phi u\|^{2} \leq \|u\|^{2}$ as
\[\|u'+u''-\Psi\circ \Phi(u'+u'')\|^{2} \leq \|u'+u''\|^{2},\]
which using the linearity of $\Phi$ and $u' \in \Ker(\Phi)$ we obtain
\[\|u'+u''-\Psi\circ \Phi u''\|^{2} \leq \|u'+u''\|^{2},\]
Suppose that  $\Psi\circ \Phi u''=v'+v''$ with $v'\in \Ker(\Phi)$ nontrivial.
Then, selecting $u'=tv'$, with $t \in \R$, we obtain
\[\|(t-1)v'+u''- v''|^{2} \leq \|tv'+u''\|^{2}\]
which amounts to
\[(t-1)^{2} \|v'\|^{2} +\|u''- v''|^{2} \leq t^{2}\|v'\|^{2} +\|u''\|^{2}\]
and therefore
\[ (1-2t) \|v'\|^{2} +\|u''- v''|^{2} \leq \|u''\|^{2},\]
which  provides a contradiction for $t$ large enough negative.
Consequently, $v'=0$ and  $\Img(\Psi\circ \Phi) \subset \Ker^{\perp}(\Phi)$.  Since
$I=\Psi\circ \Phi +(I-\Psi\circ \Phi)$ with
$\Img(\Psi\circ \Phi) \subset \Ker^{\perp}(\Phi) $ and $\Img(I-\Psi\circ \Phi) \subset \Ker(\Phi)$ it follows that
 $\Img(\Psi\circ \Phi)=\Ker^{\perp}(\Phi)$.  Since $\Psi\circ \Phi$ is a projection it follows that
\[ \Psi\circ \Phi u''=u'',\quad  u'' \in \Ker^{\perp}(\Phi)\,. \]

Consequently, for two elements $u_{1}= u'_{1}+u''_{1}$ and $u_{2}= u'_{2}+u''_{2}$ with
$u'_{i} \in \Ker(\Phi)$ and $u''_{i} \in \Ker^{\perp}(\Phi)$  for $i=1,2$ we have
\begin{eqnarray*}
\bigl(I-\Psi\circ \Phi\bigr)(u_{1}+u_{2})&=& u_{1}+u_{2}-\Psi\circ \Phi(u_{1}+u_{2})\\
&=& u'_{1}+u'_{2}+u''_{1}+u''_{2}-\Psi\circ \Phi(u''_{1}+u''_{2})\\
&=& u'_{1}+u'_{2}\\
&=& u'_{1}+u''_{1}-\Psi\circ \Phi u''_{1}+u'_{2}+u''_{2}-\Psi\circ \Phi u''_{2}\\
&=&  \bigl(I-\Psi\circ \Phi\bigr)(u_{1}) +  \bigl(I-\Psi\circ \Phi\bigr)(u_{2})\, ,
\end{eqnarray*}
and  similarly, for $t \in \R$,
\begin{eqnarray*}
\bigl(I-\Psi\circ \Phi\bigr)(t u_{1})&=t\bigl(I-\Psi\circ \Phi\bigr)( u_{1}),
\end{eqnarray*}
so we conclude that
 $I-\Psi\circ \Phi$ is linear.

Since according to Rao \cite[Rem.~9, p.~51]{rao2011foundations},  a contractive linear projection
 on a
 Hilbert space is an orthogonal projection, it follows that
  the map  $I-\Psi\circ \Phi$ is an orthogonal projection,
 and therefore
$\Psi\circ \Phi=P_{\Ker^{\perp}(\Phi)} $.
 Since $\Phi^{+}$ is the  Moore-Penrose inverse, it follows that $P_{\Ker^{\perp}(\Phi)} =\Phi^{+}\Phi$
so that
 $\Psi\circ \Phi=\Phi^{+}\Phi$, and therefore the assertion
$\Psi=\Phi^{+}$ follows  by right multiplication by
$\Psi$ using the identity $\Phi\circ \Psi=I_{V}$.

\subsection{Proof of Lemma \ref{lem_S}}
\label{sec_lemS}
Let us write
$\Phi:\B \rightarrow V$ as
\[\Phi u=\sum_{i \in \I}{e_{i}u_{i}},\quad u=(u_{i} \in V_{i})_{i\in \I},\]
where we now include the subspace injections
$e_{i}:V_{i} \rightarrow V$ in its description.
Let
$\bar{e}_{i}:V_{i} \rightarrow \B$ denote the component injection
$\bar{e}_{i}v_{i}:=(0,\ldots,0, v_{i},0,\ldots, 0)$ and let
$\bar{e}^{T}_{i}:\B \rightarrow V_{i}$ denote the component projection.
Using this notation, the norm \eqref{eqjehdjhdbdd} on $\B$ becomes
\begin{equation}\label{eqjehdjhdbdd2}
\|u\|_{\B}^2:=\sum_{i\in \I} \|\bar{e}^{T}_{i}u\|_{V_i}^2, \qquad u \in \B\, ,
\end{equation}
with inner product
\[\<u_{1},u_{2}\>_{\B}:=\sum_{i\in \I} \<\bar{e}^{T}_{i}u_{1}, \bar{e}^{T}_{i}u_{2}\>_{V_i}, \qquad u_{1},u_{2}  \in \B\, .\]
Clearly,
$\bar{e}^{T}_{j}\bar{e}_{i}=0, i\neq j$ and $\bar{e}^{T}_{i}\bar{e}_{i}=I_{V_{i}}$, so that
\begin{eqnarray*}
\langle \bar{e}^{T}_{i}u, v_i \rangle_{V_{i}}&=&
\langle \bar{e}^{T}_{i}u, \bar{e}^{T}_{i}\bar{e}_{i}v_i \rangle_{V_{i}}\\
&=&\sum_{j \in \I}{\langle \bar{e}^{T}_{j}u, \bar{e}^{T}_{j}\bar{e}_{i}v_i \rangle_{V_{i}}}\\
&=&\langle u, \bar{e}_{i}v_i \rangle_{\B},
\end{eqnarray*}
implies that $\bar{e}^{T}_{i}$ is indeed the adjoint of $\bar{e}_{i}$.
Consequently we obtain
\[\Phi=\sum_{i \in \I}{e_{i}\bar{e}^{T}_{i}}\]
and therefore its Hilbert space adjoint $\Phi^{T}:V \rightarrow \B$ is
\[\Phi^{T}=\sum_{i \in \I}{\bar{e}_{i}e^{T}_{i}},\]
where
 $e^{T}_{i}:V \rightarrow V_{i}$ is the Hilbert space adjoint
of $e_{i}$. To compute it, use the Riesz isomorphism
\[\iota:V\rightarrow  V^{*}\]
 and the usual duality relationships to obtain
\[ e^{T}_{i}= Q_{i} e^{*}_{i} \iota\, ,\]
where $e^{*}_{i}: V^{*} \rightarrow V^{*}_{i}$ is the dual adjoint projection.
Consequently we obtain
\[\Phi \Phi^{T}= \sum_{j \in \I}{e_{j}\bar{e}^{T}_{j}} \sum_{i \in \I}{\bar{e}_{i}e^{T}_{i}}
= \sum_{i,j \in \I}{e_{j}\bar{e}^{T}_{j}\bar{e}_{i}e^{T}_{i}}
=\sum_{i \in \I}{e_{i}e^{T}_{i}}
=\sum_{i \in \I}{e_{i}Q_{i} e^{*}_{i}}\iota \]
and therefore defining
\[ S:=\sum_{i \in \I}{e_{i}Q_{i} e^{*}_{i}}\]
it follows that
\[\Phi \Phi^{T}=S\iota.\]
Since $\Phi \Phi^{T}$ and $\iota$ are invertible,   $S$ is invertible.
The invertibility of $S$ implies both assertions regarding norms and their duality follows in a straightforward way from the definition of the dual norm. For the Hilbert space version see, e.g.,
\cite[Prop.~11.4]{OwhScobook2018}.

\subsection{Proof of Theorem \ref{thmkjhgjhgyuy}}
We use the notations and results in the proof of Lemma \ref{lem_S}.
The assumption $V=\sum_{i}{V_{i}}$ implies that the information map
$\Phi:\B \rightarrow V$ defined by
\[\Phi u=\sum_{i \in \I}{u_{i}},\quad u=(u_{i} \in V_{i})_{i\in \I},\]
is surjective.
Consequently, Lemma \ref{lem_basic} asserts that the  minimizer of \eqref{eqmindsobfirstdeb} is
$w^{*}=\Psi(v):= \Phi^{+}v,$
where the Moore-Penrose inverse
$ \Phi^{+}:=\Phi^{T}(\Phi \Phi^{T})^{-1}$
of $\Phi$ is well defined, with
$\Phi^{T}:V \rightarrow \B$ being the Hilbert space adjoint to
 $\Phi:\B \rightarrow V$.
The proof  of Lemma \ref{lem_S}   obtained
$ \Phi \Phi^{T} =S\iota$
where $ S:=\sum_{i \in \I}{e_{i}Q_{i} e^{*}_{i}}\, $
 and $\iota:V \rightarrow V^{*}$ is the Riesz isomorphism,
$ e^{T}_{i}= Q_{i} e^{*}_{i}\iota \, ,$
where $e^{T}_{i}:V \rightarrow V_{i}$ is the Hilbert space adjoint
of $e_{i}$ and $e^{*}_{i}:V^{*} \rightarrow V_{i}^{*}$ is its dual space adjoint, and
$\Phi^{T}=\sum_{i \in \I}{\bar{e}_{i}e^{T}_{i}},$
where $\bar{e}_{i}:V_{i} \rightarrow \B$ denotes the component injection
$\bar{e}_{i}v_{i}:=(0,\ldots,0, v_{i},0,\ldots, 0)$.

Therefore, since $(\Phi \Phi^{T})^{-1}=\iota^{-1}S^{-1}$, we obtain
$\Phi^{+}=\sum_{i \in \I}{\bar{e}_{i}Q_{i}e^{*}_{i}\iota}\iota^{-1} S^{-1} \, ,$
which amounts to
\begin{equation}
\label{Phi-}
\Phi^{+}=\sum_{i \in \I}{\bar{e}_{i}Q_{i}e^{*}_{i}} S^{-1} \, ,
\end{equation}
or in coordinates
\[ (\Phi^{+}v)_{i}= Q_{i}e^{*}_{i}S^{-1}v, \quad i \in \I,\]
establishing the first assertion.
The second follows from the general property
$\Phi\Phi^{+}=\Phi\Phi^{T}(\Phi\Phi^{T})^{-1}=I$
of the Moore-Penrose inverse.
The first isometry assertion  follows from
\[\|\Phi^{+}v\|^{2}_{\B}=\sum_{i\in \I}{\|(\Phi^{+}v)_{i}\|^{2}_{V_{i}}}
=\sum_{i\in \I}{\| Q_{i}e^{*}_{i}S^{-1}v\|^{2}_{V_{i}}}
=\sum_{i\in \I}{[Q^{-1}_{i}Q_{i} e^{*}_{i}S^{-1}v,  Q_{i}e^{*}_{i}S^{-1}v]}
=\sum_{i\in \I}{[ e^{*}_{i}S^{-1}v,  Q_{i}e^{*}_{i}S^{-1}v]}\]
\[
=\sum_{i\in \I}{[ S^{-1}v,  e_{i} Q_{i}e^{*}_{i}S^{-1}v]}
=[ S^{-1}v,\sum_{i\in \I}  e_{i} Q_{i}e^{*}_{i}S^{-1}v]
=[ S^{-1}v,SS^{-1}v]
=[ S^{-1}v,v]
=\| v\|^{2}_{S^{-1}}
\]
for  $v \in V$.

For the second, write
$\Phi=\sum_{i\in \I}{e_{i}\bar{e}^{T}_{i}}$ and consider its dual space adjoint
 $\Phi:V^{*} \rightarrow \B^{*}$ defined by
\[\Phi^{*}=\sum_{i\in \I}{\bar{e}^{T,*}_{i}e^{*}_{i}}\,.\]
A straightforward calculation shows that
$\bar{e}^{T,*}_{i}:V_{i}^{*} \rightarrow \B^{*}$ is the component injection into
the product $\B^{*}=\prod_{i\in \I}{V^{*}_{i}}$. Consequently, we obtain
\[ \bar{e}^{T}_{i}Q \bar{e}^{T,*}_{j}=\delta_{i,j}Q_{j}, \quad i, j \in \I,\]
so that
\[\Phi Q \Phi^{*}= \sum_{i\in \I}{e_{i}\bar{e}^{T}_{i}}Q \sum_{j\in \I}{\bar{e}^{T,*}_{j}e^{*}_{j}}
= \sum_{i,j\in \I}{e_{i}\bar{e}^{T}_{i}Q \bar{e}^{T,*}_{j}e^{*}_{j}}
= \sum_{i\in \I}{e_{i}Q_{i} e^{*}_{i}}
= S,\]
and since, for $\phi \in V^{*}$,
\[\|\Phi^{*}\phi \|_{\B}^{2}=\langle \Phi^{*}\phi,\Phi^{*}\phi\rangle_{\B^{*}}=[ \Phi^{*}\phi,Q\Phi^{*}\phi]=[ \phi,\Phi Q\Phi^{*}\phi]=[ \phi,S\phi]
= \| \phi\|^{2}_{S},\]
it follows that $\Phi^{*}$ is an isometry.

\subsection{Proof of Theorem \ref{thm_eieiei}}
Use the Riesz isomorphism between
$V$ and $V^{*}$ to represent the dual space adjoint $\Phi^{*}:V^{*} \rightarrow \B^{*}$ of
$\Phi:\B \rightarrow V$ as  $\Phi^{*}:V \rightarrow \B^{*}$.
It follows from the definition of the Hilbert space adjoint
$\Phi^{T}:V \rightarrow \B$ that
\[ [\Phi^{*}v,b]=\langle v, \Phi b \rangle= \langle\Phi^{T}v,b\rangle_{\B}.\]
Since
 $Q:\B^{*} \rightarrow \B$
\eqref{Qform} defines the $\B$ inner product through
\[ \langle b_{1}, b_{2} \rangle_{\B}=[Q^{-1}b_{1},b_{2}], \quad b_{1}, b_{2} \in \B,\]
it follows that
$[\Phi^{*}v,b] = \langle Q\Phi^{*}v,b \rangle_{\B}$ and therefore
$\langle Q\Phi^{*}v,b \rangle_{\B}=\langle\Phi^{T}v,b\rangle_{\B}, v\in V, b \in \B$, so we conclude that
\[ \Phi^{T}=Q\Phi^{*}\, .\]
Since Theorem \ref{thmkjhgjhgyuy} demonstrated that $\Psi$ is the Moore-Penrose inverse $\Phi^{+}$
 which implies that $\Psi\circ\Phi$ is the orthogonal projection onto $\Img(\Phi^{T})$ it follows that
$\Psi\circ \Phi u \in \Img(\Phi^{T})$. However, the identity
$ \Phi^{T}=Q\Phi^{*}$ implies that $\Img(\Phi^{T})=Q\Img(\Phi^{*})$ so that
we obtain the first part
\[\|u-\Psi(\Phi u)\|_{\B}=\inf_{\phi \in V^*} \|u-Q \Phi^*(\phi) \|_{\B}\]
of the assertion. The second half follows from the definition
\eqref{eqjehdjhdbdd} of $\|\cdot\|_{\B}$.

\subsection{Proof of Proposition  \ref{propeqejdkejdde}}
Restating the assertion using  the injections $e_{i}:V_{i} \rightarrow V$, our objective is
to establish that
\[E(i)=\Var\big([\phi,e_{i}\xi_i]\big)=\Var\big(\<e_{i}\xi_i,v\>_{S^{-1}}\big)\,.\]
Since $[\phi,e_{i}\xi_i]=[e_{i}^{*}\phi,\xi_i]$, it follows that
$[\phi,e_{i}\xi_i] \sim \mathcal{N}(0,[e_{i}^{*}\phi,Q_{i}e_{i}^{*}\phi])$ so that
$\Var\big([\phi,e_{i}\xi_i]\big)=[e_{i}^{*}\phi,Q_{i}e_{i}^{*}\phi]$, which using
$\phi=S^{-1}v$ becomes
\[\Var\big([\phi,e_{i}\xi_i]\big)=[S^{-1}v,e_{i}Q_{i}e_{i}^{*}S^{-1}v]\, .\]
On the other hand, the definitions \eqref{ieieieihhgggg} of $E(i)$,  \eqref{eqkkejddh} of $\|\cdot\|_{V_{i}}$, and Theorem \ref{thmkjhgjhgyuy} imply that
\[ E(i):=\|\Psi_i(v)\|_{V_i}^2
=[Q^{-1}_{i}\Psi_i(v), \Psi_i(v)]\]
\[
=[Q^{-1}_{i}Q_{i}e_{i}^{*}S^{-1}v,Q_{i}e_{i}^{*}S^{-1}v]
=[e_{i}^{*}S^{-1}v,Q_{i}e_{i}^{*}S^{-1}v]
=[S^{-1}v,e_{i}Q_{i}e_{i}^{*}S^{-1}v],
\]
so that we conclude the first part
$E(i)=\Var\big([\phi,e_{i}\xi_i]\big)$ of the assertion.
Since $[\phi,e_{i}\xi_i]=[S^{-1}v, e_{i}\xi_i]=\langle v, e_{i}\xi_i \rangle_{S^{-1}}$ we obtain the second.

\subsection{Proof of Theorem \ref{thmakdjhgjhgyuy}}
Fix $1 \leq k < r \leq q$. To apply Theorem \ref{thmkjhgjhgyuy}, we select
$\B:=\B^{(k)}$ and $V:=\B^{(r)}$ and endow them with the external direct sum vector space structure
of products of vector spaces.
Since the information operator $\Phi^{(r,k)}:\B^{(k)}  \rightarrow \B^{(r)}$ defined in
 \eqref{def_Phi} is diagonal with components
$\Phi^{(r,k)}_{j}:\B^{(k)}_{j}  \rightarrow V^{(r)}_{j}, j \in \I^{(r)}$
 and the norm on
$\B^{(k)}=\prod_{i\in \I^{(r)}}\B^{(k)}_{i}$ is the product norm
$\|u\|^{2}_{\prod_{i\in \I^{(r)}}\B^{(k)}_{i}}=\sum_{i\in \I^{(r)}}{\|u_{i}\|^{2}_{\B^{(k)}_{i}}}, \,
u=(u_{i})_{i \in \I^{(r)}}\,$,
it follows from  the variational characterization of Lemma \ref{lem_basic}, the diagonal nature of the information map $\Phi^{(r,k)}$ and the product metric structure  on $\B^{(k)}$
that the optimal recovery solution $\Psi^{(k,r)}$
 is the diagonal operator with components the optimal solution operators
corresponding to the component information maps $\Phi^{(r,k)}_{j}:\B^{(k)}_{j}  \rightarrow V^{(r)}_{j}, j \in \I^{(r)}$.
Since each
component
\eqref{Phi_j} of the observation operator is
\[\Phi^{(r,k)}_{j}(u):=\sum_{i\in j^{(k)} } u_i, \qquad  u\in \B^{(k)}_j, \]
it follows that
the appropriate subspaces of $V^{(r)}_{j}$ are
  \[V^{(k)}_{i} \subset  V^{(r)}_{j}, \qquad i\in j^{(k)}\, .\]
Moreover, Condition \ref{condpart}   and the semigroup nature of the hierarchy of subspace embeddings  implies that
\[ e_{j,i}^{(k+2,k)}= \sum_{l \in j^{(k+1)}}e_{j,l}^{(k+2,k+1)}  e_{l,i}^{(k+1,k)},\quad
  \, i \in j^{(k)}, \,
\]
where the sum, despite its appearance, is over one term, and by induction  we can establish that
  assumption \eqref{equyuyg76g6}
  implies that
\begin{equation}\label{equyuyg76g6aaa}
Q^{(r)}_j=\sum_{i\in j^{(k)}}  e_{j,i}^{(r,k)} Q_i^{(k)} e_{i,j}^{(k,r)},\quad
j\in \I^{(r)}.
\end{equation}

 Utilizing the  adjoint  $e_{i,j}^{(k,r)}\,:\,V_j^{(r),*} \rightarrow V_i^{(k),*} $
\eqref{ekr}   to the subspace embedding $e_{j,i}^{(r,k)}\,:\,V_i^{(k)}\rightarrow V_j^{(r)}$,
it now follows from Theorem
\ref{thmkjhgjhgyuy} and \eqref{equyuyg76g6aaa}  that these component optimal solution maps
$\Psi^{(k,r)}_{j}:V^{(r)}_{j}\rightarrow \B^{(k)}_{j} $  are those assumed in the theorem in
\eqref{Psi_ij} and \eqref{Psi_j} as
\begin{equation}
\label{eienbfhbvpoioiiu}
\Psi^{(k,r)}_{j}(v_{j}):=\bigl(Q_i^{(k)} e_{i,j}^{(k,r)} Q^{(r),-1}_j v_j\bigr)_{i\in j^{(k)}},
\qquad
   v_j\in V^{(r)}_j\, .
\end{equation}
The first three assertions for each   component $j$   then follow from
Theorem \ref{thmkjhgjhgyuy}, thus establishing the   first three assertions in full.

For the semigroup assertions,  Condition \ref{condpart} implies that,
 for
$k < r < s$ and $l\in \I^{(s)}$, there is a one to one relationship between
$
\{j \in l^{(r)}, i \in j^{(k)}\}$ and $\{i \in l^{(k)}\}$.  Consequently,
 the definition \eqref{def_Phi} of $\Phi^{(r,k)}$
implies
\[
\Phi^{(s,r)}\circ \Phi^{(r,k)}(u)= \Bigl(\sum_{j \in l^{(r)}}{\bigl(\sum_{i \in j^{(k)}}{u_{i}}\bigr))}\Bigr)_{l \in \I^{(s)}}
= \Bigl(\sum_{i \in l^{(k)}}{u_{i}}\Bigr)_{l \in \I^{(s)}}
=\Phi^{(s,k)}(u)\,,
\]
establishing the fourth assertion $\Phi^{(s,k)}=\Phi^{(s,r)}\circ \Phi^{(r,k)}$.

For the fifth,  the definition \eqref{def_Psi} of $\Psi^{(k,r)}$
 implies that
\begin{eqnarray*}
\Psi^{(k,r)}\circ \Psi^{(r,s)}(v)&=&
\big(Q_i^{(k)} e_{i,j}^{(k,r)} Q^{(r),-1}_j \Psi^{(r,s)}_{j}(v)\big)_{i\in j^{(k)}} \\
&=&
\big(Q_i^{(k)} e_{i,j}^{(k,r)} Q^{(r),-1}_jQ_j^{(r)} e_{j,l}^{(r,s)} Q^{(s),-1}_l v_l \big)_{i\in j^{(k)}}\\
&=&
\big(Q_i^{(k)} e_{i,j}^{(k,r)}  e_{j,l}^{(r,s)} Q^{(s),-1}_l v_l \big)_{i\in j^{(k)}}\\
&=&
\big(Q_i^{(k)} e_{i,l}^{(k,s)}   Q^{(s),-1}_l v_l \big)_{i\in l^{(k)}}\\
&=& \Psi^{(k,s)}(v)\, ,
\end{eqnarray*}
establishing
$\Psi^{(k,s)}=\Psi^{(k,r)}\circ \Psi^{(r,s)}$.

The last assertion follows directly from the second and the fifth.

\subsection{Proof of Theorem  \ref{thm_uurirn}}
Since $\xi^{(k)}:\B^{(k),*} \rightarrow \mathbf{H}$ is an isometry to a Gaussian space of real variables
we can abuse notation and write
$\xi^{(k)}(b^{*}) =[b^{*},\xi^{(k)}]$ which emphasizes the interpretation of $\xi^{(k)}$ as
a weak $\B^{(k)}$-valued random variable.
Since, by Theorem \ref{thmakdjhgjhgyuy},
\begin{equation}
\label{Phi_iso}
\Phi^{(k,1),*}: (\B^{(k),*}, \|\cdot\|_{\B^{(k),*}}) \rightarrow  (\B^{(1),*}, \|\cdot\|_{\B^{(1),*}})
\,\, \text{is an isometry}
\end{equation}
and
 $\xi^{(1)}:\B^{(1),*} \rightarrow \mathbf{H}$ is an isometry, it follows that
\[\Phi^{(k,1)}\xi^{(1)}:=\xi^{(1)}\circ \Phi^{(k,1),*}:\B^{(k),*} \rightarrow \mathbf{H}\]
 is an isometry, and therefore a Gaussian field
on $\B^{(k)}$.
Since Gaussian fields transform like Gaussian measures with respect to continuous linear transformations,
we obtain that
$\xi^{(1)} \sim \mathcal{N}(0,Q^{1})$ implies that
\[ \Phi^{(k,1)}\xi^{(1)} \sim \mathcal{N}(0, \Phi^{(k,1)}Q^{1} \Phi^{(k,1),*}),\]
but the isometric nature \eqref{Phi_iso} of
  $ \Phi^{(k,1),*}$ implies that
\[ \Phi^{(k,1)} Q^{(1)}\Phi^{(k,1),*}= Q^{(k)},\]
so we conclude that
\[\Phi^{(k,1)}\xi^{(1)} \sim \mathcal{N}(0, Q^{k})\]
thus establishing the assertion that
$\xi^{(k)}$ is distributed as $\Phi^{(k,1)}\xi^{(1)}$.

The conditional expectation
$\E\bigl[\xi^{(k)}\mid \Phi^{(r,k)}(\xi^{(k)})\big] $ is uniquely characterized by its field  of conditional expectations
$\E\bigl[[b^{*},\xi^{(k)}]\mid \Phi^{(r,k)}(\xi^{(k)})\bigr], b^{*} \in \B^{(k),*}  $,
which, because of the linearity of conditional expectation of Gaussian random variables, appears
as
\[\E\bigl[[b^{*},\xi^{(k)}]\mid \Phi^{(r,k)}(\xi^{(k)})\bigr]=[A_{b^{*}},\Phi^{(r,k)}(\xi^{(k)})]\]
for some $A_{b^{*}} \in V^{*}$. Furthermore, the Gaussian conditioning also implies that the dependence
of $A_{b^{*}}$ on $b^{*}$ is linear so we write $A_{b^{*}}=Ab^{*}$ for some $A:\B^{*} \rightarrow V^{*}$, thereby obtaining
\begin{equation}
\label{eiehuifif}
\E\bigl[[b^{*},\xi^{(k)}]\mid \Phi^{(r,k)}(\xi^{(k)})\bigr]=[Ab^{*},\Phi^{(r,k)}(\xi^{(k)})], \quad b^{*} \in  \B^{(k),*}\, .
\end{equation}
Using the well-known fact, see e.g.~Dudley \cite[Thm.~10.2.9]{Dudley:2002},
 that the conditional expectation of a square integrable random variable
 on a probability space $(\Omega, \Sigma',P)$  with respect to a sub-$\sigma$-algebra $\Sigma' \subset \Sigma$ is the orthogonal projection  onto the closed subspace $L^{2}(\Omega, \Sigma',P) \subset L^{2}(\Omega, \Sigma,P)$, it follows that the conditional expectation satisfies
\[\E\bigl[\bigl([b^{*},\xi^{(k)}]-[Ab^{*},\Phi^{(r,k)}(\xi^{(k)}\bigr)[v^{*},\Phi^{(r,k)}(\xi^{(k)}]\bigr] =0,
 \quad b^{*} \in  \B^{(k),*}, v^{*} \in V^{(k),*}\, .\]
Rewriting this as
\[\E\bigl[\bigl([b^{*},\xi^{(k)}]-[\Phi^{(r,k),*}Ab^{*},\xi^{(k)}]\bigr)[\Phi^{(r,k),*}v^{*},\xi^{(k)}]\bigr] =0,  \quad b^{*} \in  \B^{(k),*}, v^{*} \in V^{(k),*}\, ,\]
we obtain
\begin{eqnarray*}
[b^{*},Q^{(k)}\Phi^{(r,k),*}v^{*}]&=&[\Phi^{(r,k),*}Ab^{*},Q^{(k)}\Phi^{(r,k),*}v^{*}]\\
&=&[b^{*},A^{*}\Phi^{(r,k)}Q^{(k)}\Phi^{(r,k),*}v^{*}]\\
\end{eqnarray*}
for all $b^{*} \in \B^{(k),*}$ and $v^{*} \in V^{(k),*}$, and so conclude that
\[ A^{*}\Phi^{(r,k)}Q^{(k)}\Phi^{(r,k),*}v^{*}=Q^{(k)}\Phi^{(r,k),*}v^{*},\quad b^{*} \in  \B^{(k),*}, v^{*} \in V^{(k),*},\]
which implies that
\begin{equation}
\label{eorngotgmi}
  A^{*}\Phi^{(r,k)}b=b,  \quad b \in \Img(Q^{(k)}\Phi^{(r,k),*})\, .
\end{equation}
Since
\begin{eqnarray*}
\langle \Phi^{(r,k),T}b^{(r)},b^{(k)}\rangle_{\B^{(k)}}&=&
\langle b^{(r)},\Phi^{(r,k)}b^{(k)}\rangle_{\B^{(k)}}\\
&=& [Q^{(k),-1}b^{(r)},\Phi^{(r,k)}b^{(k)} ]\\
&=& [\Phi^{(r,k),*}Q^{(k),-1}b^{(r)},b^{(k)} ]\\
&=& [Q^{(r),-1}Q^{(r)}\Phi^{(r,k),*}Q^{(k),-1}b^{(r)},b^{(k)} ]\\
&=& \langle Q^{(r)}\Phi^{(r,k),*}Q^{(k),-1}b^{(r)},b^{(k)} \rangle_{\B^{(r)}},
\end{eqnarray*}
we conclude that
\[  \Phi^{(r,k),T}=Q^{(r)}\Phi^{(r,k),*}Q^{(k),-1}\, , \]
and therefore
\[ \Img(Q^{(r)}\Phi^{(r,k),*})=\Img(\Phi^{(r,k),T})\, .\]
Consequently, \eqref{eorngotgmi} now reads
\begin{equation}
\label{eorngotgmi2}
  A^{*}\Phi^{(r,k)}b=b,  \quad b \in \Img(\Phi^{(r,k),T})\, .
\end{equation}
Since clearly
\[  A^{*}\Phi^{(r,k)}b=0,  \quad b \in \Ker(\Phi^{(r,k)})\]
it follows that
\[ A^{*}\Phi^{(r,k)}=P_{\Img(\Phi^{(r,k),T})}\]
Since  $P_{\Img(\Phi^{(r,k),T})}=(\Phi^{(r,k)})^{+}\Phi^{(r,k)}$, the identity
$\Phi^{(r,k)}(\Phi^{(r,k)})^{+}=I$ establishes that
\[ A^{*}=(\Phi^{(r,k)})^{+}\]
Since \eqref{eiehuifif} implies that
\[
\E\bigl[[b^{*},\xi^{(k)}]\mid \Phi^{(r,k)}(\xi^{(k)})\bigr]=[b^{*},A^{*}\Phi^{(r,k)}(\xi^{(k)})], \quad b^{*} \in  \B^{(k),*}\, ,
\]
which  in turn implies that
\[
\E\bigl[\xi^{(k)}\mid \Phi^{(r,k)}(\xi^{(k)})\bigr]=A^{*}\Phi^{(r,k)}(\xi^{(k)})\, ,
\]
we obtain
\[
\E\bigl[\xi^{(k)}\mid \Phi^{(r,k)}(\xi^{(k)})\bigr]=(\Phi^{(r,k)})^{+}\Phi^{(r,k)}(\xi^{(k)})\, .
\]
Since Theorem \ref{thmkjhgjhgyuy}  established that the optimal solution operator $\Psi^{(k,r)}$ corresponding to the information map $\Phi^{(r,k)}$ was the Moore-Penrose inverse $\Psi^{(k,r)}=(\Phi^{(r,k)})^{+}$
we obtain
\begin{equation}
\label{eiejeueinnnf}
\E\bigl[\xi^{(k)}\mid \Phi^{(r,k)}(\xi^{(k)})\bigr]=\Psi^{(k,r)}\circ\Phi^{(r,k)}(\xi^{(k)})\,,
\end{equation}
so that
\[
\E\bigl[\xi^{(k)}\mid \Phi^{(r,k)}(\xi^{(k)})=v\bigr]=\Psi^{(k,r)}(v)\,,
\]
thus establishing the final assertion.
To establish the martingale property,  let us define $\hat{\xi}^{(1)}:=\xi^{(1)}$ and
\[ \hat{\xi}^{(k)}:=\E\bigl[\xi^{(1)}\mid \Phi^{(k,1)}(\xi^{(1)})\bigr], \quad k =2,\ldots \, .\]
as a sequence of Gaussian fields  all on the same space $\B^{(1)}$.
\eqref{eiejeueinnnf}  implies that
\begin{equation}
\label{eieinhufgutbh}
\hat{\xi}^{(k)}=\Psi^{(1,k)}\circ\Phi^{(k,1)}(\xi^{(1)}),
\end{equation}
so that the identities    $\Phi^{(r,1)}=\Phi^{(r,k)}\circ \Phi^{(k,1)}$ and
$\Phi^{(k,1)}\circ \Psi^{(1,k)}=I_{\B^{(1)}}$ from Theorem \ref{thmakdjhgjhgyuy} imply that
\begin{eqnarray*}
\E[\hat{\xi}^{(k)}| \Phi^{(r,1)}(\hat{\xi}^{(k)})]&=&\E[\Psi^{(1,k)}\circ\Phi^{(k,1)}(\hat{\xi}^{(1)})|
\Phi^{(r,1)}\circ \Psi^{(1,k)}\circ\Phi^{(k,1)}(\hat{\xi}^{(1)})] \\
&=&\E[\Psi^{(1,k)}\circ\Phi^{(k,1)}(\hat{\xi}^{(1)})|
\Phi^{(r,k)}\circ \Phi^{(k,1)}\circ \Psi^{(1,k)}\circ\Phi^{(k,1)}(\hat{\xi}^{(1)})] \\
&=&\E[\Psi^{(1,k)}\circ\Phi^{(k,1)}(\hat{\xi}^{(1)})|
\Phi^{(r,k)}\circ \Phi^{(k,1)}(\hat{\xi}^{(1)})] \\
&=&\E[\Psi^{(1,k)}\circ\Phi^{(k,1)}(\hat{\xi}^{(1)})|
\Phi^{(r,1)}(\hat{\xi}^{(1)})] \\
&=&\Psi^{(1,k)}\circ\Phi^{(k,1)}\E[\hat{\xi}^{(1)}|
\Phi^{(r,1)}(\hat{\xi}^{(1)})] \\
&=&\Psi^{(1,k)}\circ\Phi^{(k,1)} \hat{\xi}^{(r)}\\
&=&\Psi^{(1,k)}\circ\Phi^{(k,1)}\Psi^{(1,r)}\circ\Phi^{(r,1)}  \hat{\xi}^{(1)}\\
&=&\Psi^{(1,k)}\circ\Phi^{(k,1)}\Psi^{(1,k)}\Psi^{(k,r)}\circ\Phi^{(r,1)}  \hat{\xi}^{(1)}\\
&=&\Psi^{(1,k)}\circ\Psi^{(k,r)}\circ\Phi^{(r,1)}  \hat{\xi}^{(1)}\\
&=&\Psi^{(1,r)}\circ\Phi^{(r,1)}  \hat{\xi}^{(1)}\\
&=&  \hat{\xi}^{(r)},
\end{eqnarray*}
that is
$\hat{\xi}^{(k)}$ is a reverse martingale.

\subsection{ Proof of Theorem \ref{thm_Iden}}
Let us simplify for the moment and define a scaled wavelet
\begin{equation}\label{eqkjdkjdjedhjks3}
\bar{\chi}_{\tau,\omega,\theta}(t):= \omega^{\frac{1-\beta}{2}}
\cos\bigl(\omega (t-\tau)+\theta\bigr)e^{-\frac{ \omega^2 (t-\tau)^2}{\alpha^2}}, \qquad t \in \R\,,
\end{equation}
 so that at $\beta=0$ we have
\begin{equation}
\label{scaling}
\chi_{\tau,\omega,\theta}= \Bigl(\frac{2}{\pi^{3}\alpha^{2}}\Bigr)^\frac{1}{4}\bar{\chi}_{\tau,\omega,\theta}\, .
\end{equation}
Since
\begin{eqnarray*}
K(s,t)&:=&\int_{-\pi}^{\pi} \int_{\R_{+}} \int_{\R} \bar{\chi}_{\tau,\omega,\theta}(s) \bar{\chi}_{\tau,\omega,\theta}(t)  d\tau \,d\omega\, d\theta\\
&=&\int_{-\pi}^{\pi} \int_{\R_{+}} \int_{\R} \cos\bigl(\omega (s-\tau)+\theta\bigr)e^{-\frac{ \omega^2 (s-\tau)^2}{\alpha^2}}\cos\bigl(\omega (t-\tau)+\theta\bigr)e^{-\frac{ \omega^2 (t-\tau)^2}{\alpha^2}}  d\tau \,\omega^{1-\beta}d\omega\, d\theta\\
&=&\int_{-\pi}^{\pi} \int_{\R_{+}} \int_{\R} \cos\bigl(\omega (s-\tau)+\theta\bigr)\cos\bigl(\omega (t-\tau)+\theta\bigr)e^{-\frac{ \omega^2 (s-\tau)^2}{\alpha^2}}e^{-\frac{ \omega^2 (t-\tau)^2}{\alpha^2}}  d\tau \,\omega^{1-\beta}d\omega\, d\theta\, ,
\end{eqnarray*}
 the trigonometric identity
\begin{eqnarray*}
&&\cos\bigl(\omega (s-\tau)+\theta\bigr)\cos\bigl(\omega (t-\tau)+\theta\bigr)\\
&=&\Bigl(\cos\bigl(\omega (s-\tau)\bigr)\cos\theta-\sin\bigl(\omega (s-\tau)\bigr)\sin\theta\Bigr) \Bigl(\cos\bigl(\omega (t-\tau)\bigr)\cos\theta-\sin\bigl(\omega (t-\tau)\bigr)\sin\theta \Bigr)\\
\end{eqnarray*}
and the integral identities
$ \int_{-\pi}^{\pi}\cos^{2}\theta d\theta =\int_{-\pi}^{\pi}\sin^{2}\theta d\theta =\pi$ and
$ \int_{-\pi}^{\pi}\cos\theta \sin \theta d\theta =0$
imply that
\[K(s,t)=\pi \int_{\R_{+}} \int_{\R} \Bigl(\cos\bigl(\omega (s-\tau)\bigr)\cos\bigl(\omega (t-\tau)\bigr)
+\sin\bigl(\omega (s-\tau)\bigr)\sin\bigl(\omega (t-\tau)\Bigr)e^{-\frac{ \omega^2 (s-\tau)^2}{\alpha^2}}e^{-\frac{ \omega^2 (t-\tau)^2}{\alpha^2}}  d\tau \,\omega^{1-\beta}d\omega\,
\]
so  that the cosine subtraction formula
 implies
\[K(s,t)=\pi \int_{\R_{+}} \int_{\R} \cos\bigl(\omega (s-t)\bigr)
e^{-\frac{ \omega^2 (s-\tau)^2}{\alpha^2}}e^{-\frac{ \omega^2 (t-\tau)^2}{\alpha^2}}  d\tau \,\omega^{1-\beta}d\omega,
\]
which amounts to
\begin{equation}
K(s,t)=\pi \Re \int_{\R_{+}} \int_{\R} e^{i\omega (s-t)}
e^{-\frac{ \omega^2 (s-\tau)^2}{\alpha^2}}e^{-\frac{ \omega^2 (t-\tau)^2}{\alpha^2}}  d\tau \,\omega^{1-\beta}d\omega\, .
\end{equation}
Using the identity
\[ e^{-\frac{ \omega^2 |s-\tau|^2}{\alpha^2}} e^{-\frac{ \omega^2 |t-\tau|^2}{\alpha^2}}
= e^{-\frac{ \omega^2}{\alpha^{2} }\bigl(2\tau^{2}-2(s+t)\tau\bigr)} e^{-\frac{ \omega^2}{\alpha^{2} }
\bigl(s^{2}+t^{2}\bigr)}
\]
and the integral identity
\begin{equation}
\label{gaussianint}
\int e^{-a\tau^2-2b\tau}d\tau=\sqrt{\frac{\pi}{a}}e^{\frac{b^2}{a}},  \qquad a > 0,\,  b \in \mathbb{C},
\end{equation}
with the choice $a:=\frac{2\omega^{2}}{\alpha^{2}}$ and
$b:=- \frac{\omega^{2}}{\alpha^{2}}(s+t)$,
so that
$b^{2}/a=  \frac{\omega^{2}}{2\alpha^{2}}(s+t)^{2}$, we can evaluate
  the integral
\[ \int{e^{-\frac{ \omega^2}{\alpha^{2} }\bigl(2\tau^{2}-2(s+t)\tau\bigr)}d\tau}=\frac{\alpha}{\omega}
\sqrt{\frac{\pi}{2}}
e^{ \frac{\omega^{2}}{2\alpha^{2}}(s+t)^{2}}\, .
\]
Consequently,
\begin{eqnarray*}
K(s,t)&=&\pi \Re \int{e^{i\omega (s-t)}e^{-\frac{ \omega^2 |s-\tau|^2}{\alpha^2}} e^{-\frac{ \omega^2 |t-\tau|^2}{\alpha^2}}d\tau\omega^{1-\beta}d\omega} \\
&=&
\pi\Re\int{ e^{i\omega (s-t)}e^{-\frac{ \omega^2}{\alpha^{2} }
\bigl(s^{2}+t^{2}\bigr)}
e^{-\frac{ \omega^2}{\alpha^{2} }\bigl(2\tau^{2}-2(s+t)\tau\bigr)}d\tau\omega^{1-\beta}d\omega} \\
&=&
\pi\Re\int{ e^{i\omega (s-t)}e^{-\frac{ \omega^2}{\alpha^{2} }
\bigl(s^{2}+t^{2}\bigr)}\Bigl(
\int{
e^{-\frac{ \omega^2}{\alpha^{2} }\bigl(2\tau^{2}-2(s+t)\tau \bigr)}d\tau }\Bigr)\omega^{1-\beta}d\omega}
\\
&=&
\alpha
\sqrt{\frac{\pi^{3}}{2}}\Re\int{ e^{i\omega (s-t)}e^{-\frac{ \omega^2}{\alpha^{2} }
\bigl(s^{2}+t^{2}\bigr)} e^{ \frac{\omega^{2}}{2\alpha^{2}}(s+t)^{2}}
\omega^{-\beta} d\omega}
\\
&=&
\alpha
\sqrt{\frac{\pi^{3}}{2}}\Re\int{ e^{i\omega (s-t)}
 e^{- \frac{\omega^{2}}{2\alpha^{2}}(s-t)^{2}}
\omega^{-\beta} d\omega}
\\
&=&
\alpha
\sqrt{\frac{\pi^{3}}{2}}\int{ \cos(\omega (s-t))
 e^{- \frac{\omega^{2}}{2\alpha^{2}}(s-t)^{2}}
\omega^{-\beta} d\omega},
\end{eqnarray*}
that is,
\begin{equation}
K(s,t)=
\alpha
\sqrt{\frac{\pi^{3}}{2}}\int{ \cos(\omega (s-t))
 e^{- \frac{\omega^{2}}{2\alpha^{2}}(s-t)^{2}}
\omega^{-\beta} d\omega}.
\end{equation}

Utilizing the
integral identity
\begin{equation}
\label{id_Grad}
\int_{0}^{\infty}{x^{\mu-1}e^{-p^{2}x^{2}}\cos(ax)dx}=
\frac{1}{2}p^{-\mu}\Gamma(\frac{\mu}{2})e^{-\frac{a^{2}}{4p^{2}}} {}_{1}F_{1}\Bigl(-\frac{\mu}{2}+\frac{1}{2},\frac{1}{2};\frac{a^{2}}{4p^{2}}\Bigr)
, \quad a >0, \mu >0,
\end{equation}
from Gradshteyn and Ryzhik \cite[3.952:8]{Gradshteyn2000},
 with
$ \frac{a^{2}}{4p^{2}}=\frac{\alpha^{2}}{2}$,
$p^{2}=\frac{|s-t|^{2}}{2\alpha^{2}}$, $a:=|s-t|$ and $\mu:=1-\beta$,  we obtain
\begin{eqnarray*}
K(s,t)&=&
\alpha
\sqrt{\frac{\pi^{3}}{2}} \frac{1}{2}(\sqrt{2}\alpha)^{1-\beta}|s-t|^{\beta-1}\Gamma(\frac{1-\beta}{2})e^{-\frac{\alpha^{2}}{2}} {}_{1}F_{1}\Bigl(\frac{\beta}{2},\frac{1}{2};\frac{\alpha^{2}}{2}\Bigr).
\end{eqnarray*}
Consequently, reintroducing the scaling \eqref{scaling}
obtains $K_{u}(s,t)=\Bigl(\frac{2}{\pi^{3}\alpha^{2}}\Bigr)^\frac{1}{2}K(s,t)$ when $\beta=0$. To indicate the dependence on $\beta$, we define
\begin{equation}
\label{eoemiiiir}
K_{\beta}(s,t)=
 \frac{1}{2}(\sqrt{2}\alpha)^{1-\beta}|s-t|^{\beta-1}\Gamma(\frac{1-\beta}{2})e^{-\frac{\alpha^{2}}{2}} {}_{1}F_{1}\Bigl(\frac{\beta}{2},\frac{1}{2};\frac{\alpha^{2}}{2}\Bigr),
\end{equation}
so that $ K_{u}=K_{0}$.
For fixed $\alpha$, at the limit
$\beta=0$,
 we have, recalling that $\Gamma(\frac{1}{2})=\sqrt{\pi}$,
\begin{eqnarray*}
K_{0}(s,t)&=&
\frac{\sqrt{2\pi}}{2}\alpha
 |s-t|^{-1}
e^{-\frac{\alpha^{2}}{2}} {}_{1}F_{1}\Bigl(0,\frac{1}{2};\frac{\alpha^{2}}{2}\Bigr)
\end{eqnarray*}
and since
${}_{1}F_{1}\bigl(0, \frac{1}{2};\frac{\alpha^{2}}{2} \bigr)=1$
we obtain
\begin{eqnarray*}
K_{0}(s,t)&=& \frac{\sqrt{2\pi}}{2}\alpha
 |s-t|^{-1}
e^{-\frac{\alpha^{2}}{2}} \, .
\end{eqnarray*}

The scaling constant $H(\beta)$ defined in the theorem satisfies
\[H(\beta):=\frac{1}{2}(\sqrt{2}\alpha)^{1-\beta}\Gamma(1-\frac{\beta}{2})e^{-\frac{\alpha^{2}}{2}} {}_{1}F_{1}\Bigl(\frac{\beta}{2},\frac{1}{2};\frac{\alpha^{2}}{2}\Bigr) \bar{H}(\beta)
\]
with
\begin{equation}
\label{Hconst}
 \bar{H}(\beta):=2^{\beta}\pi^{\frac{1}{2}}\frac{\Gamma(\frac{\beta}{2})}{\Gamma(1-\frac{\beta}{2})}\, ,
\end{equation}
so that, by \eqref{eoemiiiir} we have
\[ \frac{1}{H(\beta)}K_{\beta}(s,t)= \frac{|s-t|^{\beta-1}}{\bar{H}(\beta)}\, .\]
Therefore, if we let
$\mathcal{K}_{\beta}$ denote the integral operator
\[ \bigl(\mathcal{K}_{\beta}f\bigr)(s):=\frac{1}{H(\beta)}\int_{\R}{K_{\beta}(s,t)f(t)dt}\]
associated to the kernel $K_{\beta}$ scaled by $H(\beta)$,
it follows  that
\[ \bigl(\mathcal{K}_{\beta}f\bigr)(s):\frac{1}{\bar{H}(\beta)}\int_{\R}{|s-t|^{\beta-1}f(t)dt},\]
namely that it is a scaled version of the integral operator
$f \mapsto \int_{\R}{|s-t|^{\beta-1}f(t)dt}$  corresponding to the Riesz potential
$|s-t|^{\beta-1}$.
   Consequently,  according to Helgason
\cite[Lem.~5.4 \& Prop.~5.5]{helgason1999radon}, this scaling of the Riesz potential
by $\bar{H}(\beta)$ implies  the assertions of the theorem.

\subsection{Proof of Lemma \ref{lem_unknown}}
The outer most integral in the definition \eqref{K_unknown} of $K_{\beta}$  is
\begin{eqnarray*}
\int_{-\pi}^{\pi}{ y\bigl(\omega(s-\tau)+\theta\bigr) y^{*}\bigl(\omega(t-\tau)+\theta \bigr)d\theta }&=&
\int_{-\pi}^{\pi}{ \sum_{-N}^{N}{c_{n}e^{i n(\omega(s-\tau)+\theta)}} \sum_{-N}^{N}{c^{*}_{m}e^{-im(\omega(t-\tau)+\theta)}} d\theta}\\
&=& \sum_{n=-N}^{N} \sum_{m=-N}^{N}{e^{in\omega(s-\tau)}e^{-im(\omega(t-\tau)} c_{n} c^{*}_{m}
\int_{-\pi}^{\pi}{ e^{i(n-m)\theta} d\theta }}\\
&=& 2\pi\sum_{n=-N}^{N} {e^{in\omega(s-\tau)}e^{-in\omega(t-\tau)} |c_{n}|^{2}
 }\\
&=& 2\pi\sum_{n=-N}^{N} {e^{in\omega (s-t)} |c_{n}|^{2}
 },
\end{eqnarray*}
so that
\[ K_{\beta}(s,t)=2\pi\sum_{n=-N}^{N}{K_{n}(s,t)|c_{n}|^{2}},\]
where
\begin{eqnarray*}
K_{n}(s,t)&=&\Re \int{e^{in\omega (s-t)}e^{-\frac{ \omega^2 |s-\tau|^2}{\alpha^2}} e^{-\frac{ \omega^2 |t-\tau|^2}{\alpha^2}}d\tau\omega^{1-\beta}d\omega} \\
&=&
\Re\int{ e^{in\omega (s-t)}e^{-\frac{ \omega^2}{\alpha^{2} }
\bigl(s^{2}+t^{2}\bigr)}
e^{-\frac{ \omega^2}{\alpha^{2} }\bigl(2\tau^{2}-2(s+t)\tau\bigr)}d\tau\omega^{1-\beta}d\omega} \\
&=&
\Re\int{ e^{in\omega (s-t)}e^{-\frac{ \omega^2}{\alpha^{2} }
\bigl(s^{2}+t^{2}\bigr)}\Bigl(
\int{
e^{-\frac{ \omega^2}{\alpha^{2} }\bigl(2\tau^{2}-2(s+t)\tau\bigr)}d\tau}\Bigr)\omega^{1-\beta}d\omega}
\\
&=&
\alpha
\sqrt{\frac{\pi}{2}}\Re\int{ e^{in\omega (s-t)}e^{-\frac{ \omega^2}{\alpha^{2} }
\bigl(s^{2}+t^{2}\bigr)} e^{ \frac{\omega^{2}}{2\alpha^{2}}(s+t)^{2}}
\omega^{-\beta} d\omega}
\\
&=&
\alpha
\sqrt{\frac{\pi}{2}}\Re\int{ e^{in\omega (s-t)}
 e^{- \frac{\omega^{2}}{2\alpha^{2}}(s-t)^{2}}
\omega^{-\beta} d\omega}
\\
&=&
\alpha
\sqrt{\frac{\pi}{2}}\int{ \cos(n\omega (s-t))
 e^{- \frac{\omega^{2}}{2\alpha^{2}}(s-t)^{2}}
\omega^{-\beta} d\omega}\, .
\end{eqnarray*}

Consequently, using the integral identity \eqref{id_Grad}
with $a=|n||s-t|,\mu=1-\beta, p^{2}=\frac{|s-t|^{2}}{2\alpha^{2}}$, and therefore
$\frac{a^{2}}{4p^{2}}=\frac{|n|\alpha^{2}}{2}$  and $p=\frac{|s-t|}{\sqrt{2}\alpha}$
we conclude that
\[K_{n}(s,t)=   \frac{\alpha\sqrt{\pi}}{2\sqrt{2}}(\sqrt{2}\alpha)^{1-\beta}|s-t|^{\beta-1}\Gamma(\frac{1-\beta}{2})
e^{-\frac{|n|\alpha^{2}}{2}} {}_{1}F_{1}\Bigl(\frac{\beta}{2};\frac{1}{2};\frac{|n|\alpha^{2}}{2}\Bigr), \]
which does not appear to have a nice dependency on $n$, except for $\beta=0$, where
$  {}_{1}F_{1}\Bigl(0;\frac{1}{2};\frac{|n|\alpha^{2}}{2}\Bigr)=1$
and $\Gamma(\frac{1}{2})=\sqrt{\pi}$,
so that
\[K_{n}(s,t)=   \frac{1}{2}\alpha^{2}\pi
e^{-\frac{|n|\alpha^{2}}{2}}|s-t|^{-1}  \]
and therefore
\[ K_{0}(s,t)=\alpha^{2}\pi^{2}\|y\|^{2}
|s-t|^{-1},
\]
when written in terms of the norm
$\|y\|^{2}:=\sum_{n=-N}^{N}{e^{-\frac{|n|\alpha^{2}}{2}}|c_{n}|^{2}}$.

\subsection{Proof of Lemma \ref{lem_norm}}
For $\gamma >0$,
let us evaluate the function
\begin{equation}
\label{eoeinuiri} \phi(s):=\sum_{n=-\infty}^{\infty}{e^{-|n|\gamma}e^{ins}}
\end{equation}
with Fourier coefficients $\hat{\phi}(n)=e^{-|n|\gamma}$.
Since
\begin{eqnarray*}
 \phi(s)&=&\sum_{n=-\infty}^{\infty}{e^{-|n|\gamma}e^{ins}}\\
&=&\sum_{n=1}^{\infty}{e^{-n\gamma}e^{ins}}+1+\sum_{n=-\infty}^{-1}{e^{n\gamma}e^{ins}}\\
&=&\sum_{n=1}^{\infty}{e^{-n\gamma}e^{ins}}+1+\sum_{n=1}^{\infty}{e^{-n\gamma}e^{-ins}}\\
&=& 1+2\sum_{n=1}^{\infty}{e^{-n\gamma}\cos(ns)},
\end{eqnarray*}
the identity
\begin{equation}
\label{eiemnioojhhuuu}
1+2\sum_{n=1}^{\infty}{e^{-n\gamma}\cos n s}=\frac{\sinh(\gamma)}{\cosh(\gamma)-\cos(s)}
\end{equation}
of Gradshteyn and Ryzhik \cite[1.461:2]{Gradshteyn2000}
implies that
\[ \phi(s)=\frac{\sinh(\gamma)}{\cosh(\gamma)-\cos(s)}\, .\]
Consequently, with the choice $\gamma:=\frac{\alpha^2}{4}$ in \eqref{eoeinuiri}, that is, for
\[ \phi(s):=\sum_{n=-\infty}^{\infty}{e^{-|n|\frac{\alpha^2}{4}}e^{ins}},\]
 we find that
\begin{equation}
\label{irmuufni}
\phi(s)=\frac{\sinh(\frac{\alpha^{2}}{4})}{\cosh(\frac{\alpha^{2}}{4})-\cos(s)}\, .
\end{equation}

We will need two basic facts about the Fourier transform of $2\pi$-periodic functions, see
e.g.~Katznelson \cite[Sec.~I]{katznelson2004introduction}.
If we denote the Fourier transform by
$ \hat{f}(n):=\frac{1}{2\pi}\int_{-\pi}^{\pi}{f(s)e^{- i n s}},  \forall  n,$
the convolution theorem states that for periodic functions  $f,g \in L^{1}[-\pi,\pi]$ that
the convolution
$ \bigl(f\star g\bigr)(s):=\frac{1}{2\pi}\int_{-\pi}^{\pi}{f(s-t)g(t)dt}$
is a well defined periodic function in $L^{1}[-\pi,\pi]$ and that
$  \hat{\bigl(f\star g\bigr)}(n)=\hat{f}(n)\hat{g}(n),  \forall n\, .$
Moreover,  for square integrable $2\pi$-periodic functions in $L^{2}[-\pi,\pi]$,
the Parseval identity is
$\sum_{n=-\infty}^{\infty}{|\hat{f}(n)|^{2}}=\frac{1}{2\pi}\int_{0}^{2\pi}{|f(s)|^{2}}\, .$

Consequently, observing that $c_{n}=0, n < -N, n > N$, the Parseval identity and the convolution formula imply that
\begin{eqnarray*}
\|y\|^{2}&=&\sum_{n=-N}^{N}{e^{-\frac{|n|\alpha^{2}}{2}}|c_{n}|^{2}}\\
&=&\|\bigl(e^{-\frac{|n|\alpha^{2}}{4}}c_{n}\bigr)_{n=-\infty}^{\infty}\|^{2}_{\ell^{2}}\\
&=&\|\bigl(\hat{\phi}\hat{y}\bigr)_{n=-\infty}^{\infty}\|^{2}_{\ell^{2}}\\
&=&\|\bigl(\hat{\phi \star y}\bigr)_{n=-\infty}^{\infty}\|^{2}_{\ell^{2}}\\
&=&\|\phi \star y\|^{2}_{L^{2}[-\pi,\pi]}\\
&=& \int{|\phi \star y|^{2}}\\
&=& \int{\Bigl|\int{\phi(s-t)y(t) dt}\Bigr|^{2}ds}\\
&=& \int{\Bigl(\int{\phi(s-t)y(t)dt}\int{\phi(s-t')y^{*}(t')dt'}\Bigr)ds}\\
&=& \int{\int{\phi(s-t)y(t)\phi(s-t')y^{*}(t')dt dt'}ds}\\
&=& \int{G(t,t')y(t)y^{*}(t')dtdt'}\,,
\end{eqnarray*}
that is,
\[ \|y\|^{2}=\int{G(t,t')y(t)y^{*}(t')dtdt'}\]
where
\begin{equation}
\label{oemogigmntit}
 G(t,t'):=\int{\phi(s-t)\phi(s-t')ds}
\end{equation}
with
\begin{equation}
\label{euenufufnu}
\phi(s)=\frac{\sinh(\frac{\alpha^{2}}{4})}{\cosh(\frac{\alpha^{2}}{4})-\cos(s)}\, .
\end{equation}
We can evaluate $G$
using  the identity \eqref{eiemnioojhhuuu} as follows: Since
\begin{eqnarray*}
G(t,t')&=&\int{\phi(s-t)\phi(s-t')ds}\\
&=&
\int{\Bigl(
1+2\sum_{n=1}^{\infty}{e^{-n\frac{\alpha^{2}}{4}}\cos n (s-t)\Bigr)\Bigl(
1+2\sum_{n'=1}^{\infty}{e^{-n' \frac{\alpha^{2}}{4}}}\cos n' (s-t')\Bigr)}ds},
\end{eqnarray*}
and, for each product, we have
\begin{eqnarray*}
&&\int{
\cos n (s-t)
\cos n' (s-t')ds}\\
&=&\int{
\bigl(\cos n s\cos nt-\sin n s\sin nt\bigr)\bigl(\cos n' s\cos n't'-\sin n' s\sin n't'\bigr)ds}\\
&=&\delta_{n,n'}\int{
\bigl(\cos n s\cos nt-\sin n s\sin nt\bigr)\bigl(\cos n s\cos nt'-\sin n s\sin nt'\bigr)ds}\, .
\end{eqnarray*}
Using the  $L^{2}$-orthogonality of the cosines and the sines and the identities
$\int{\cos^{2} ns}=\pi$ and $\int{\sin^{2} ns}=\pi$, we conclude that
\begin{eqnarray*}
\int{
\bigl(\cos n s\cos nt-\sin n s\sin nt\bigr)\bigl(\cos n s\cos nt'-\sin n s\sin nt'\bigr)ds}
&=& \pi \bigl(\cos nt \cos nt' +\sin nt \sin nt'\bigr)\\
&=& \pi \cos n(t- t')
\end{eqnarray*}
and therefore
\begin{equation}
\label{ienufufimjd}
\int{
\cos n (s-t)
\cos n' (s-t')ds}=\pi\delta_{n,n'} \cos n(t-t') \, .
\end{equation}
Consequently, we obtain
\begin{eqnarray*}
G(t,t')&=&
\int{\Bigl(
1+2\sum_{n=1}^{\infty}{e^{-n\frac{\alpha^{2}}{4}}\cos n (s-t)\Bigr)\Bigl(
1+2\sum_{n'=1}^{\infty}{e^{-n' \frac{\alpha^{2}}{4}}}\cos n' (s-t')\Bigr)}ds}\\
&=&\int{\Bigl(
1+4\sum_{n=1}^{\infty}{e^{-n\frac{\alpha^{2}}{2}}\cos n (s-t)\cos n (s-t')\Bigr)
}ds}\\
&=&
2\pi +4\pi \sum_{n=1}^{\infty}{e^{-n\frac{\alpha^{2}}{2}} \cos n(t-t')}
\end{eqnarray*}
and therefore, using the  identity \eqref{eiemnioojhhuuu} again, we conclude
\[G(t,t')= 2\pi
\frac{\sinh(\frac{\alpha^{2}}{2})}{\cosh(\frac{\alpha^{2}}{2})-\cos(t-t')}\, .
\]


\paragraph{Acknowledgments}
The authors gratefully acknowledge support by  the Air Force Office of Scientific Research under award number FA9550-18-1-0271 (Games for Computation and Learning).

\bibliographystyle{plain}
\bibliography{merged,RPS,extra,kmd}

\end{document}